\documentclass[aps, nofootinbib, showkeys, unsortedaddress, superscriptaddress, notitlepage,longbibliography, onecolumn]{revtex4-2}
\makeatletter

\@addtoreset{equation}{section}
\makeatother

\usepackage{graphicx}
\usepackage{subcaption}
\usepackage{amsmath}
\usepackage{amssymb}
\usepackage{amsthm}
\usepackage{dsfont}
\usepackage{mathrsfs}
\usepackage{geometry}
\usepackage{xcolor}
\usepackage{MnSymbol}
\usepackage{enumitem}
\usepackage{mathtools}
\usepackage{comment}
\mathtoolsset{showonlyrefs}
\usepackage[textsize=tiny]{todonotes}

\usepackage{tikz}
\usepackage{pgfplots}
\pgfplotsset{compat=1.17}
\usepackage{caption}
\usepackage{array}[=2016-10-06]
\definecolor{steelblue}{rgb}{0.27, 0.51, 0.71}
\makeatletter
\pgfmathdeclarefunction{erf}{1}{%
  \begingroup
    \pgfmathparse{#1 > 0 ? 1 : -1}%
    \edef\sign{\pgfmathresult}%
    \pgfmathparse{abs(#1)}%
    \edef\x{\pgfmathresult}%
    \pgfmathparse{1/(1+0.3275911*\x)}%
    \edef\t{\pgfmathresult}%
    \pgfmathparse{%
      1 - (((((1.061405429*\t -1.453152027)*\t) + 1.421413741)*\t 
      -0.284496736)*\t + 0.254829592)*\t*exp(-(\x*\x))}%
    \edef\y{\pgfmathresult}%
    \pgfmathparse{(\sign)*\y}%
    \pgfmath@smuggleone\pgfmathresult%
  \endgroup
}
\makeatother

\newcommand{\EE}{\mathbb{E}}
\newcommand{\RR}{\mathbb{R}}
\newcommand{\CC}{\mathbb{C}}

\newcommand{\NN}{\mathbb{N}}

\newcommand{\eps}{\varepsilon}

\DeclareMathOperator{\Tr}{Tr}
\newcommand\aug{\fboxsep=-\fboxrule\!\!\!\fbox{\strut}\!\!\!}
\DeclareMathOperator{\1}{\mathds{1}}
\DeclareMathOperator*{\argmin}{arg\,min}

\newcommand{\dd}{\mathrm{d}}

\DeclareMathOperator{\trace}{tr}
\newcommand{\dv}[2]{\frac{\dd#1}{\dd#2}}

\newcommand{\abs}[1]{\left\lvert#1\right\rvert}
\newcommand{\norm}[1]{\left\lVert#1\right\rVert}

\DeclarePairedDelimiterX{\infdivx}[2]{(}{)}{%
  #1\;\delimsize\|\;#2%
}

\DeclareMathOperator*{\plim}{plim}
\newcommand{\diag}[1]{\textsc{diag}\left(#1\right)}

\newcommand{\lp}{\left(} 
\newcommand{\rp}{\right)} 

\def\thesection{\arabic{section}}
\def\thesubsection{\arabic{section}.\arabic{subsection}}
\def\thesubsubsection{\arabic{section}.\arabic{subsection}.\arabic{subsubsection}}

\makeatletter
\renewcommand{\p@subsection}{}
\renewcommand{\p@subsubsection}{}
\makeatother

\makeatletter
\renewcommand\@makecaption[2]{%
  \par
  \vskip\abovecaptionskip
  \begingroup
   \small\rmfamily
    \begingroup
     \samepage
     \flushing
     \let\footnote\@footnotemark@gobble
     \@make@capt@title{#1}{#2}\par
    \endgroup
  \endgroup
  \vskip\belowcaptionskip
}
\makeatother

\newtheorem{theorem}{Theorem}[section]
\newtheorem{proposition}{Proposition}[section]

\newtheorem{lemma}{Lemma}[section]

\theoremstyle{definition}
\newtheorem{definition}{Definition}[section]
\theoremstyle{remark}
\newtheorem{remark}{Remark}[section]
\theoremstyle{remark}
\newtheorem{example}{Example}[section]
\theoremstyle{definition}

\usepackage{chngcntr}
\usepackage[linktocpage,colorlinks=true,allcolors=blue]{hyperref}  
\usepackage{orcidlink}

\allowdisplaybreaks

\AtBeginDocument{%
    \newwrite\bibnotes
    \def\bibnotesext{Notes.bib}
    \immediate\openout\bibnotes=\jobname\bibnotesext
    \immediate\write\bibnotes{@CONTROL{REVTEX41Control}}
    \immediate\write\bibnotes{@CONTROL{%
    apsrev41Control,author="08",editor="1",pages="1",title="0",year="1"}}
     \if@filesw
     \immediate\write\@auxout{\string\citation{apsrev41Control}}%
    \fi
}%

\begin{document}

\title{Precise asymptotic analysis of Sobolev training for random feature models}

\author{Katharine Fisher\,\orcidlink{0000-0002-9655-984X}}
\email{kefisher@mit.edu}
\affiliation{Center for Computational Science and Engineering, Massachusetts Institute of Technology, Cambridge, MA 02139, USA}

\author{Matthew T.C.\ Li\,\orcidlink{0009-0002-4049-1943}}
\email{mtcli@umass.edu}
\affiliation{Department of Mathematics and Statistics, University of Massachusetts Amherst, Amherst, MA 01003, USA}

\author{Youssef Marzouk\,\orcidlink{0000-0001-8242-3290}}
\email{ymarz@mit.edu}
\affiliation{Center for Computational Science and Engineering, Massachusetts Institute of Technology, Cambridge, MA 02139, USA}

\author{Timo Schorlepp\,\orcidlink{0000-0002-9143-8854}}
\email{timo.schorlepp@nyu.edu}
\affiliation{Courant Institute of Mathematical Sciences, New York University, New York, NY 10012, USA}

\date{\today}

\begin{abstract}
Gradient information is widely useful and available in applications, and is therefore natural to include in the training of neural networks.
Yet little is known theoretically about the impact of Sobolev training---regression with both function and gradient data---on the generalization error of highly overparameterized predictive models in high dimensions.
In this paper, 
we obtain a precise characterization of this training modality for random feature (RF) models 
in the limit where the number of trainable parameters, input dimensions, and training data tend proportionally to infinity.
Our model for Sobolev training reflects practical implementations
by sketching gradient data onto finite dimensional subspaces.
By combining the replica method from statistical physics with linearizations in operator-valued free probability theory, we derive a closed-form description for the generalization errors of the trained RF models.
For target functions described by single-index models, we demonstrate that supplementing function data with additional gradient data does not universally improve predictive performance. Rather, the degree of overparameterization should inform the choice of training method.
More broadly, our results identify
settings where models perform optimally by interpolating noisy function and gradient data.
\end{abstract}

\maketitle

\tableofcontents


\section{Introduction}
\label{sec:intro}

\begin{figure}
\centering
\includegraphics[width=.8\textwidth]{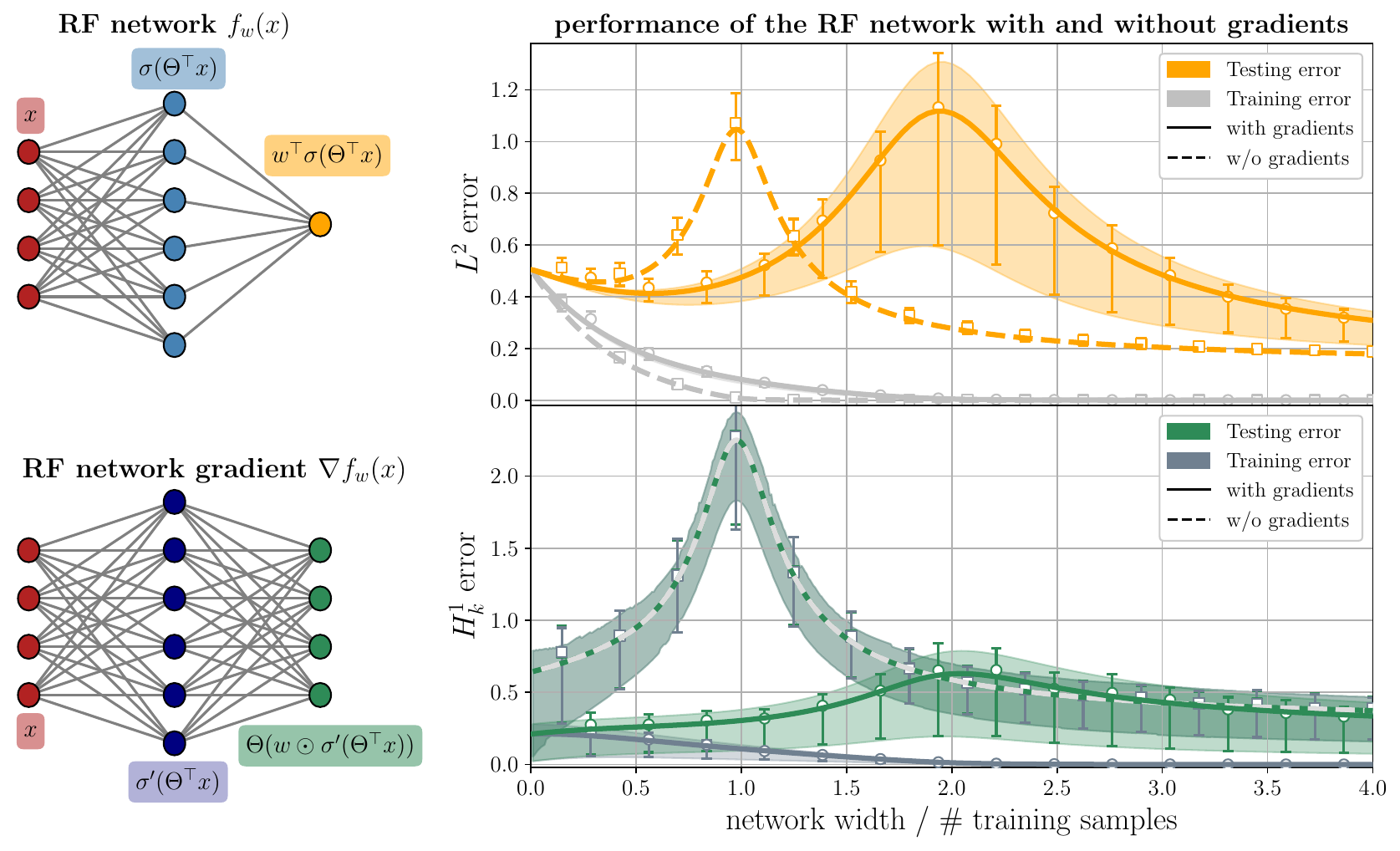}
\caption{Illustration of the single hidden-layer RF model (left column) and its generalization performance in the high-dimensional limit (right column). In the right subfigures, lines correspond to theoretical predictions, while squares and circles show the mean over 1000 Monte Carlo samples in dimension $d=100$ (with error bars at $25\%$ and $75\%$ quantiles of the data). The horizontal axis is $p/n$.
The dashed lines and squares correspond to least squares minimization of the readout weights $w$ using only function data, while solid lines and circles indicate Sobolev training where additional gradient information is used. Shaded regions cover the predicted $25\%$ and $75\%$ quantiles, while thick lines represent the mean. $L^2$ error (top right) refers to the mismatch in predicted function values, while the $H^1_k$ semi-norm error (bottom right) is the gradient mismatch when projected onto the $k$-dimensional subspace used for training. Numerical details (cf.~Section~\ref{sec:setup}): regularization $\lambda = 0.001$, no observational noise, ridge function $\phi(\omega) = \arctan(\omega) + 1 / \cosh(\omega)$, activation function $\sigma = \text{ReLU}$, $k =1$ gradient sketches, $\tau = 1$ gradient term weight, $n/d = 2.345$ number of samples per dimension.}
\label{fig:intro}
\end{figure}

Gradients of a function encode valuable information about its local structure, such as smoothness and sensitivities. An intuitive folklore is that if gradient data are available, they ought to be incorporated into the training of a predictive model.
In line with this reasoning, Sobolev training~\cite{czarnecki-osindero-jaderberg:2017} consists of matching neural network gradients to gradient data in the training loss, in addition to matching the network itself to function data through a standard $L^2$ loss.\footnote{An earlier term for Sobolev training is ``Hermite learning''~\cite{zhouDerivativeReproducingProperties2008,shiHermiteLearningGradient2010}, after Hermite interpolation.} 
This technique has been adopted in many scientific fields where gradients are a target of interest or are accessible either through direct observation, e.g., as in meteorology~\cite{sunUseNWPNowcasting2014} or econometrics~\cite{hallNonparametricEstimationWhen2007}, or through computation~\cite{plessix:2006,Margossian2019,Pulay2014}.
For instance, gradients of energy functions are routinely used to construct machine-learned interatomic potentials, which are crucial in multiscale materials modeling~\cite{behler:2016,zhang-han-wang-etal:2018,deringer-caro-csanyi:2019}.
Derivative informed neural operators (DINOs) find solution maps for high dimensional partial differential equations (PDEs) which empirically outperform  standard neural operators~\cite{o-leary-roseberry-chen-villa-etal:2024,luo-o-leary-roseberry-chen-etal:2025,qiu-bridges-chen:2024,cho-ryu-hwang:2024}.
Additional applications encompass engineering design~\cite{bouhlel-he-martins:2020}, elastoplasticity~\cite{vlassis-sun:2021,vlassis-zhao-ma-etal:2021}, computational finance~\cite{kichler-afghan-naumann:2024}, chaotic dynamical systems~\cite{park-yang-chandramoorthy:2024}, optimal control~\cite{nakamura-zimmerer-gong-kang:2021,onken-nurbekyan-li-etal:2022}, as well as canonical machine learning tasks such as model distillation and transfer learning~\cite{srinivas-fleuret:2018}, and many others~\cite{zagoruyko2016paying,hoffman-roberts-yaida:2019,atzmon-lipman:2020,tsay:2021,rosemberg-garcia-bent-etal:2025}.\\

These reported empirical successes reinforce the belief that gradient data produces better predictions, but theory has yet to delineate which---if any---prediction problems certifiably benefit from Sobolev training.  We address this gap by applying the replica method~\cite{mezard-parisi-virasoro:1987}, an analytical tool originating from the statistical physics of disordered systems, to derive the \emph{first asymptotically exact characterization of Sobolev training in a high-dimensional regime.}\\

As has been much remarked~\cite{zhang2017UnderstandingDL, zhangUnderstandingDeepLearning2021,belkin-hsu-ma-etal:2019,yang-yu-you-etal:2020}, theory has not fully demystified the impressive ability of neural networks to generalize to unseen data even under $L^2$ training. In particular, modern architectures can interpolate their training sets because their parameters vastly outnumber available data ~\cite{zhang2017UnderstandingDL, zhangUnderstandingDeepLearning2021,petersenMathematicalTheoryDeep2024}.  Models with the minimum capacity necessary to ``memorize'' training data fail to generalize, but increasing the size of these models allows them to find solutions with lower test error---thus ``benignly'' overfitting even noisy training sets~\cite{bartlett:2020}. This learning behavior creates a double descent curve, rigorously documented in neural architectures~\cite{belkin-hsu-ma-etal:2019,nakkiranDeepDoubleDescent2021}, kernel methods~\cite{rakhlinConsistencyInterpolationLaplace2019,liangMultipleDescentMinimumNorm2020}, and linear regression~\cite{bartlett:2020,hastie-montanari-rosset-etal:2022}. Crucially, the second descent may plateau to a lower error than that of any comparable underparameterized model. As a natural step towards understanding this behavior of $L^2$ training for nonlinear maps,
the random feature (RF) model~\cite{rahimi-recht:2007}, a two layer network where interior parameters are randomly selected and frozen, serves as a key exemplar for which theoretical results can be obtained~\cite{dhifallah-lu:2020,dascoli-sagun-biroli:2020,gerace-loureiro-krzakala-etal:2021,mei-montanari:2022,goldt-loureiro-reeves-etal:2022,hastie-montanari-rosset-etal:2022}.\\

Incorporating gradients via Sobolev training further challenges our intuition. Since gradients also carry implicit information about function values, it is not obvious whether optimal generalization requires overparameterized models. It is even unclear \textit{a priori} whether benign overfitting can still occur when the network interpolates both the function and gradient data, possibly in the presence of correlated observational noise. Consequently, we aim to elucidate whether the additional information from gradient data supports benign overfitting, and whether underparameterization---or even $L^2$ training---would be preferred over this modality. To investigate such questions, we extend to the Sobolev setting the techniques used to obtain precise asymptotic characterizations of the $L^2$ training and generalization errors of RF models when the input dimension $d$, the number of trainable parameters $p$, and the number of training data points $n$ are taken proportionally to infinity~\cite{dhifallah-lu:2020,dascoli-sagun-biroli:2020,gerace-loureiro-krzakala-etal:2021,mei-montanari:2022,goldt-loureiro-reeves-etal:2022,hastie-montanari-rosset-etal:2022}.\\

Our main contribution is an exact analysis of RF model predictions for various error metrics under Sobolev training. We mimic practical applications by training on $k$-dimensional projections, or sketches, of the target gradient~\cite{czarnecki-osindero-jaderberg:2017, o-leary-roseberry-chen-villa-etal:2024, cocola-hand:2020}, with $k = O(1)$. Under this assumption, we empirically establish a form of Gaussian universality\footnote{Also interchangeably referred to as ``Gaussian equivalence'' in the following.} for RF models, extending previous results that have been rigorously demonstrated for the $L^2$ setting~\cite{goldt-loureiro-reeves-etal:2022, hu-lu:2022}. This allows us to obtain asymptotic predictions by combining non-rigorous tools from statistical physics, namely the replica method~\cite{mezard-parisi-virasoro:1987,gerace-loureiro-krzakala-etal:2021,goldt-loureiro-reeves-etal:2022}, with rigorous tools from free probability theory~\cite{helton-far-speicher:2007,mingo-speicher:2017,belinschi-mai-speicher:2017,helton-mai-speicher:2018,adlam-pennington:2020,moniri-hassani:2024}. Specifically, we present a low-dimensional fixed point system which can be efficiently solved to produce generalization error as a function of network and training set size, input dimension, regularization strength, and activation function.\footnote{See \url{https://github.com/kefisher98/sobolev-random-features} for Python and Julia implementations of the fixed point system.}\\

As an informal illustration of our results, Figure~\ref{fig:intro} compares the generalization error of RF networks for $L^2$ and Sobolev training
for varying values of the ratio $p/n$,
validating our theoretical predictions against numerical training results. The error curves for both function and gradient prediction exhibit double descent, though our theory demonstrates that the location of the interpolation threshold is shifted under Sobolev training. Consequently, benefits from incorporating gradient data depend on the degree to which the model is overparameterized. Moreover, unlike previous results, our generalization errors are intrinsically random, even in the high-dimensional limit, as a consequence of training with random projections of gradients, as shown by the shaded inter-quantile region in the figure. The advantages of Sobolev training are largely limited to the underparameterized regime. Notably, for overparameterized models at the right horizon of the plot, gradient data only slightly improves gradient prediction and actually hurts function prediction. We will specify conditions on the observation model and network activation under which Sobolev training can improve gradient prediction for any network size, but ultimately we find that the performance of $L^2$ training cannot be exceeded for function prediction in the highly overparameterized regime.\\

Our theoretical results are relevant to a wide range of fields in science and engineering~\cite{sunUseNWPNowcasting2014,hallNonparametricEstimationWhen2007,behler:2016,zhang-han-wang-etal:2018,deringer-caro-csanyi:2019,o-leary-roseberry-chen-villa-etal:2024,luo-o-leary-roseberry-chen-etal:2025,qiu-bridges-chen:2024,cho-ryu-hwang:2024,bouhlel-he-martins:2020,vlassis-sun:2021,vlassis-zhao-ma-etal:2021,kichler-afghan-naumann:2024,park-yang-chandramoorthy:2024,srinivas-fleuret:2018,zagoruyko2016paying,hoffman-roberts-yaida:2019,atzmon-lipman:2020,tsay:2021,nakamura-zimmerer-gong-kang:2021,onken-nurbekyan-li-etal:2022,rosemberg-garcia-bent-etal:2025}. While it is by nature difficult to find negative empirical results in the literature, we show here that Sobolev training does not necessarily improve function or gradient prediction, depending on the hyperparameters chosen. We note that in the cited examples, the models considered are nonlinear functions of their trainable parameters, in contrast to the RF model analyzed in this work. Nevertheless, within the lazy training regime, RF models do provide reasonable approximations to deep nonlinear neural networks~\cite{misiakiewicz-montanari:2023}. Several recent papers have taken steps toward theoretically describing feature learning~\cite{ba-erdogdu-suzuki-etal:2022,cui-pesce-dandi-etal:2024,cui-krzakala-zdeborova:2023,pacelli-ariosto-pastore-etal:2023,baglioni-pacelli-aiudi-etal:2024,aiudi-pacelli-baglioni-etal:2025}, which is a higher fidelity model for modern neural networks. We leave extension of these ideas to Sobolev training as future work.

\subsection{Main contributions}

The main contributions of this work are as follows:
\begin{enumerate}
\item To the best of our knowledge,
we propose the first
 mathematical model for Sobolev training of neural networks for which generalization and training errors can be analytically computed. To this end, we augment the training loss of a RF model with a subspace-projected gradient term. We use this model to provide insight into several practically motivated questions:
	\begin{itemize}
		\item Does training with gradients improve generalization?
		\item Can the performance of conventional networks be matched by smaller, Sobolev-trained networks? 
		\item Is explicit regularization necessary when function and gradient observations have (possibly correlated) noise?
		\item What cost-benefit tradeoffs arise 
		if computing each projection of the gradient incurs a given cost?
	\end{itemize}
\item We apply the replica method to produce \emph{precise asymptotics} for the generalization error in the high dimensional limit. Novel technical components include:
\begin{itemize}
	\item a non-standard form of conditional Gaussian universality  to model correlations between the network and its gradients;
	\item conditioning of the replica method on a random variable given by the true gradient-subspace alignment $\varpi$;
	\item use of linear pencil machinery from operator-valued free probability to obtain a fully asymptotic description, i.e., with no need for Monte Carlo simulation in high but finite dimensions.
\end{itemize}
\item The influence of gradients is subtle in the high dimensional regimes of contemporary deep learning: we demonstrate that Sobolev training does not necessarily improve generalization to unseen tests within our model, even when data is noise-free. This result is particularly striking when we consider the prediction of \emph{gradients}.
\item The appendices accompanying this manuscript may be of independent interest to researchers with no prior exposure to either the replica method or free probability. For readers unfamiliar with the replica method, our exposition in Appendix~\ref{app:replica} belabors many technical details which are often left implicit by domain experts.  For readers unfamiliar with free probability,  we provide a condensed and practically oriented summary of the main results of~\cite{mingo-speicher:2017} in Appendix~\ref{app:free-prob-intro}.
\end{enumerate}

We describe our model and its asymptotic analysis in Section~\ref{sec:main-theory-res}. We then evaluate predictions and implications of our theory in Section~\ref{sec:results}.
A discussion of the limitations of our mathematical model and analysis are presented afterwards in Section~\ref{sec:concl}. The mathematical notation used in this paper is summarized in Appendix~\ref{App:notation}, and other appendices will be referenced throughout the main text.

\begin{remark}
We point out that parts of the calculations and results presented in this manuscript are non-rigorous
(as is typical in statistical physics, cf.~\cite{zdeborova:2020}),
but all of them have been extensively validated against numerical simulations. Specifically, beyond the use of the non-rigorous replica method itself (Appendix~\ref{app:replica}), we assume without proof that all overlap parameters and errors concentrate onto their expectations in the proportional asymptotics limit when conditioned on the alignment $\varpi$. The Gaussian universality result used within the replica calculation is partially based on numerical evidence (Appendix~\ref{app:GET}). Similarly, the simplifications of the replica-symmetric fixed-point system, in particular the asymptotic independence of~$\varpi$ and~$\zeta$, and the concentration of random matrix functions (Appendix~\ref{app:simplify-fp}), are based on heuristic arguments and numerics. Given these simplifications, the evaluation of~\eqref{eq:asymptotic_saddle_final_nonhat} based on operator-valued free probability follows the rigorously established methods of~\cite{mingo-speicher:2017} (Appendix~\ref{app:free-prob-intro}).
\end{remark}


\subsection{Related literature}

\subsubsection{Predicting generalization error}

Motivated by understanding the empirical success of overparameterized neural networks, one research direction in recent years has been to study the generalization errors of overparameterized ridge(-less) regression for linear predictors~\cite{bartlett:2020,bartlett:2021,hastie-montanari-rosset-etal:2022} and for kernels~\cite{rakhlinConsistencyInterpolationLaplace2019,liangMultipleDescentMinimumNorm2020}. 
On the other hand, theoretical predictions for finite-size neural networks remain elusive. Instead, existing results focus on the behavior of neural networks in asymptotic regimes. For example, it is known that randomly initialized deep neural networks are equivalent to Gaussian processes in the infinite width limit~\cite{neal2012bayesian,leeDeepNeuralNetworks2018a}. \citet{jacot-gabriel-hongler:2018} demonstrate that the gradient flow of such networks with one hidden layer also corresponds to a deterministic Gaussian process kernel, known as the neural tangent kernel (NTK), for which generalization properties can be analyzed. \citet{adlam-pennington:2020} consider the NTK in the proportional asymptotics limit and demonstrate that the error curves exhibit triple descent as a function of overparameterization. Later work by \citet{canatar-bordelon-pehlevan:2021} provides precise asymptotic characterizations of regression for any kernel.\\

Another approach to deriving generalization errors, which we follow in this work, models the learning problem as analogous to finding the minimum energy configuration of spin glass systems in the thermodynamic limit~\cite{advani-lahiri-ganguli:2013,bahri-kadmon-pennington:2020,decelle:2023,krzakala-zdeborova:2024,cui:2025}. This follows a rich history of leveraging ideas from statistical physics to understand learning theory, pioneered first for the Hopfield model~\cite{hopfield:1982,Amit1985}, and later yielding insights to the learning capacity of perceptrons~\cite{gardner-derrida:1989a,gardner-derrida:1989b}. These approaches enable the study of RF models using the replica method~\cite{mezard-parisi-virasoro:1987} in the proportional asymptotic limit, and precise asymptotic analyses reveal the role of non-linear activation functions in the peaks of the double descent curve, as well as demonstrating that such models have equivalent approximation capacity to linear functions of the inputs~\cite{dhifallah-lu:2020,dascoli-sagun-biroli:2020,gerace-loureiro-krzakala-etal:2021,mei-montanari:2022,goldt-loureiro-reeves-etal:2022,hastie-montanari-rosset-etal:2022}. Moreover, \citet{dascoli-sagun-biroli:2020} show that RF learning can also exhibit triple descent when the ratio of training data to input dimension grows. Other variants of single hidden-layer neural networks have since been studied: for example, \citet{erba-troiani-zdeborova-etal:2025} consider fixed readout weights and quadratic activation functions and equate the learning setup to compressed sensing with nuclear norm regularization \cite{erba-troiani-zdeborova-etal:2025}. We note that tools from statistical physics have also been extended to the study of linearized transformer architectures \cite{boncoraglio2025} and to diffusion model learning dynamics and sampling efficiency \cite{ghio-dandi-krzakala-etal:2024,biroli-bonnaire-de-bortoli-etal:2024,merger-goldt:2025,cui-pehlevan-lu:2025}.\\

In recent years, many works have examined different scaling regimes or nonlinear learning problems, providing a more complete picture of modern machine learning. Characterization of random matrix spectra beyond the linearly proportional asymptotics regime has made it possible to obtain precise asymptotics of RF models with more expressive capacity than linear functions~\cite{misiakiewicz:2022, hu-lu-misiakiewicz:2024, aguirre-lopez-franz-pastore:2025}. Mathematical models of feature learning, training of the hidden layer weights, have also been explored. Deep linear networks~\cite{li-sompolinsky:2021, Hanin2023F, ZavatoneVeth2022} provide a tool for examining models which are nonlinear in their parameters but retain linearity with respect to inputs. Another line of work considers two-stage gradient descent of RF models in the linearly proportional asymptotics regime, where \citet{ba-erdogdu-suzuki-etal:2022,cui-pesce-dandi-etal:2024} demonstrate that applying one sufficiently large gradient descent step to the hidden weights enables RF models to outperform the generalization of linear functions.
We also note that~\citet{cui-pesce-dandi-etal:2024} use the notion of \textit{conditional} Gaussian equivalence and replica calculations (conditional on the random spike of their RF model after one gradient step), which is analogous to the approach used in the present paper for the random subspace alignment $\varpi$ (cf.\ Appendix~\ref{app:GET} for a more detailed discussion of Gaussian universality).
\citet{cui-krzakala-zdeborova:2023,pacelli-ariosto-pastore-etal:2023,baglioni-pacelli-aiudi-etal:2024} consider deep models with nonlinear activation and trainable hidden parameters, in the setting where the widths of each layer tend toward infinity proportionally with the training set size. Building on \cite{pacelli-ariosto-pastore-etal:2023,baglioni-pacelli-aiudi-etal:2024}, \citet{aiudi-pacelli-baglioni-etal:2025} demonstrate that convolutional neural networks can achieve optimal generalization error at finite width (within the proportional asymptotics) in contrast to fully connected neural networks.

\subsubsection{Random matrix theory, free probability, and deep learning}

Much of the prior work surrounding theoretical predictions for neural network generalization involves applications of random matrix theory. The connection between 
random matrix theory and deep learning was first established in the seminal paper of \citet{karoui:2010} for kernel regression, later extended by \citet{peche:2019,pennington-worah:2019} to Gram matrices involving features arising from neural networks. Crucially, these works relate the spectra of random matrices in the proportional asymptotics limit to a fully asymptotic characterization given by their Stieltjes transforms (or, equivalently, Cauchy transforms).\\

In the present work, we derive a fixed point system involving traces of non-commutative random matrices which describes the precise asymptotics of Sobolev training for RF models.
The traces of these random matrices relate to their spectra, which has been studied through the lens of free probability theory, i.e., the study of non-commutative random elements~\cite{far-oraby-bryc-etal:2006,helton-far-speicher:2007,mingo-speicher:2017,belinschi-mai-speicher:2017,helton-mai-speicher:2018}.
Specifically, free probability theory provides an algorithm for linearizing rational functions of random matrices to produce a block matrix for which the operator valued Cauchy transform can be computed. This approach has been also used in the prediction of neural network generalization by~\citet{adlam-pennington:2020,misiakiewicz:2022,moniri-hassani:2024}, and  
these ideas are essential in providing a purely asymptotic characterization---in the sense that evaluating and solving it does not require any sampling in large but finite dimensions---of the fixed-point system that we derive.

\subsubsection{Existing theory for Sobolev training}

We are not aware of any previous work describing the generalization error of neural networks under Sobolev training in the proportional asymptotics limit. However, there are many results in the literature pertaining to Sobolev training in other idealized settings. For inputs with arbitrary distribution~$\mu$, \citet[Theorem 4]{hornik:1991} established that single hidden-layer neural networks with sufficiently large width are dense in the weighted $H^{s,m}(\mu)$ topology, given some additional regularity conditions on the activation functions. \citet[Theorem 4.1]{guring-kutyniok-petersen:2020} make this result quantitative for deep ReLU networks on the unit hypercube by proving upper bounds on the width and depth necessary to achieve arbitrary generalization accuracy in Sobolev norms with $s \leq 1$. For single hidden-layer ReLU networks with fixed readout weights and overparameterized width,
\citet{cocola-hand:2020} show that gradient flow over the hidden weights and biases converges to a global minimum. Furthermore, these minimizers interpolate the function and projected gradient training data. Under a similar setup, \citet{son-etal:2025} prove that Sobolev training improves the conditioning of the Hessian of the population risk over $L^2$ training, thus implicitly accelerating the convergence rate of gradient flow.  For Sobolev training with reproducing kernel Hilbert spaces (RKHSs) on compact metric spaces, \citet{ul-abdeen-jia-kekatos-etal:2023} provide sample complexity bounds for generalization and demonstrate regimes where gradient information improves over standard $L^2$ training.\\

The Sobolev norm also appears in the objective function when using neural networks as PDE solvers~\cite{raissi-perdikaris-karniadakis:2019}. However, derivative \emph{data} are not typically provided here: instead, the derivative term is often related to the function data by applying integration by parts to the PDE operator. In this context, \citet{lu-blanchet-ying:2022} prove statistical rates for solving elliptic inverse problems in an RKHS using Sobolev training, demonstrating implicit acceleration brought on by higher order regularity. \citet{yang-he:2024} also study machine learning PDE solvers with deep ``super ReLU'' networks in the underparameterized setting and prove generalization bounds which relate sample complexity to the width and depth of each network.


\section{Theoretical result: Generalization under subspace Sobolev loss in the proportional asymptotics regime}
\label{sec:main-theory-res}


\subsection{Setup}
\label{sec:setup}

Here, we describe the setup for which we state our theoretical results in Subsection~\ref{sec:error_expressions}. This does \textit{not} encompass the most general setting for which our results can be derived, and we comment on possible extensions---some of which are detailed in Appendix~\ref{app:replica}---below. Throughout, we consider shallow neural networks $f_w \colon \; \RR^d \to \RR$ with input dimension~$d$ and a single hidden layer of width~$p$. For~$p$ given random feature vectors $\Theta = [\theta_1, \dots, \theta_p] \in \RR^{d \times p}$ and trainable readout weights $w \in \RR^p$, we define
\begin{align}
	f_w(x) = w^\top \sigma \left(\Theta^\top x \right) =  \sum_{l = 1}^p w_l \; \sigma \left(\left \langle \theta_l, x \right \rangle\right)\,,
	\label{eq:neural-net-def}
\end{align}
where $\sigma\colon \; \RR \to \RR$ is an (almost everywhere) smooth activation function that is evaluated elementwise whenever applied to vectors or matrices. Let the random feature vectors be independent and identically distributed (iid) Gaussians $\Theta_{ij} \sim {\cal N}(0, 1/d)$, such that $\EE \left[\left \langle \theta_i, \theta_j \right \rangle \right] = \delta_{ij}$. The gradient of the network with respect to input $x$ is the linear combination of the features vectors $\theta_1, \dots, \theta_p \in \RR^d$ given by
\begin{align}
	\nabla f_w(x) = \sum_{l = 1}^p w_l \; \sigma' \left(\left \langle \theta_l, x \right \rangle \right) \theta_{l} = \Theta \; \textsc{diag} \left( \sigma' \left(\Theta^\top x \right) \right) w \,.
	\label{eq:neural-net-grad-def}
\end{align}
Our objective is to study the impact of incorporating derivative information on the generalization capabilities of the network and its gradient in a regression setting. To this end, we assume access to (possibly noisy) training data, consisting of function evaluations $y_i \in \RR$ and gradients $y_i' \in \RR^d$ of an underlying ground truth function at $n$ iid input samples $x_i \sim \mathcal{N}\left(0, I_d \right)$. We also assume---within the typical ``teacher-student'' setting---that there is a random true ``teacher'' feature vector $\theta_0 \sim \mathcal{N} \left(0,I_d / d \right)$ in $\RR^d$, with unit length $\EE \norm{\theta_0}^2 = 1$, such that data are generated according to
\begin{align}
\begin{cases}
	y_i &= \phi\left( \left \langle \theta_0,  x_i \right \rangle \right) +  \eta_i \,,\\
	y_i'  &=  \phi'\left( \left \langle \theta_0,  x_i \right \rangle \right) \theta_0 + \eta_i' \,,
\end{cases}
\label{eq:data-model}
\end{align}
where $\eta_i$ and $\eta_i'$ are potentially correlated noise vectors, and $\phi \colon \; \RR \to \RR$ is a fixed function. The training data hence stem from a ridge function, or single-index model, and the gradients lie parallel to the teacher vector $\theta_0$ for all samples (plus noise).\\

To employ our theoretical analysis, we consider the proportional asymptotics limit, denoted by plim, in which the input dimension $d$, number of samples $n$, and number of features $p$ jointly tend to infinity:
\begin{align}
\plim_{p \to \infty} \quad \Leftrightarrow \quad d,n,p\to \infty\,, \text{ with ratios } \alpha = n/p \text{ and } \gamma = d / p \text{ fixed.}
\label{eq:prop-asymp-def}
\end{align}
The parameters $\alpha, \gamma > 0$ fully characterize the problem in the proportional asymptotics limit with $\alpha^{-1} = p /n$, the ratio of the number of features to samples, denoting the degree of under- or over-parameterization. We shall see these regimes correspond respectively to $p/n<1$ and $p/n > 1$ for standard $L^2$ training, but change when additional gradient information is provided.\\

Instead of training with the full gradient $y_i' \in \RR^d$, we project (or ``sketch'') the gradient data with a known but random matrix $V_k \in \RR^{d \times k}$ into a space with finite and fixed dimension $k$. This projection is necessary for our theoretical framework, but is also inspired by practical considerations elaborated in both the paper on Sobolev training by~\citet{czarnecki-osindero-jaderberg:2017}, as well as DINOs~\cite{o-leary-roseberry-chen-villa-etal:2024}. We model each column vector~$v_1, \dots, v_k$ of $V_k$ to be independent and scaled as $\|v_i\| = O(\sqrt{d})$ as $d \to \infty$. Thus, the column vectors do \textit{not} have unit length, and for concreteness, we consider iid random vectors $v_i \sim {\cal N}(0, I_d)$ here. Roughly, this scaling ensures $v^\top_i \nabla f_w(x) = O(1)$, which balances the contributions from the projected gradients of  both the teacher and the network in the proportional asymptotics regime, even for independent $v_i$ and $\theta_j$ for $j = 0, 1, \dots, p$. This setting corresponds to an \textit{uninformed} choice of the subspace on which the network gradient is trained to match the teacher gradient. Another strategy is to adaptively select this subspace from data~\cite{o-leary-roseberry-chen-villa-etal:2024}, and we comment on this \emph{data-informed} extension in Appendix~\ref{app:data-informed-subspace}.\\

The projection of the gradient data naturally motivates the definition of the alignment parameter $\varpi=V_k^\top \theta_0 \in \RR^k$ (called ``varpi''). Conditioned on a fixed teacher feature, $\varpi$ is a $k$-variate Gaussian random variable. Further defining $\omega_i = \langle \theta_0, x_i \rangle$ and conditioning on $x_i, \theta_0,$ and $V_k$, the training data $\Upsilon_i = (y_i, V_k^\top y_i') \in \RR^{k+1}$ from~\eqref{eq:data-model} consists of samples
$
\Upsilon_i \mid  \omega_i, \varpi \sim P_{\text{data}}\,,
$
where the distribution $P_{\text{data}}$ on~$\RR^{k+1}$ encodes the randomness induced by noise $\eta_i$ and $\eta_i'$. In the current setting, we have
\begin{align}
P_{\text{data}} = {\cal N} \left(\begin{pmatrix}
\phi(\omega)\\ \varpi \phi'(\omega)
\end{pmatrix}, C_\eta \right) .
\label{eq:p-data-additive-gauss}
\end{align}

With this setup, the training problem for the network weights $w \in \RR^p$ consists of minimizing the empirical risk
\begin{align}
\eps_{\text{train}}(w) = \frac{1}{2n} \sum_{i = 1}^n \left[\left(y_i - f_w\left(x_i\right)\right)^2 + \tau \norm{V_k^\top \left( y_i' -  \nabla f_w\left(x_i\right) \right)}^2 \right] + \frac{\lambda}{2 \alpha} \norm{w}^2\,,
\label{eq:training-objective}
\end{align}
where $\tau > 0$ determines the relative weight of the gradient term. Choosing $\lambda>0$ in the Tikhonov regularization term  ensures the existence of the unique minimizer
\begin{align}
w^* &= \argmin_{w \in \RR^p} \eps_{\text{train}}(w) = \left[\alpha^{-1} \lambda I_p + K \right]^{-1} r \,. \label{eq:training-problem} 
\end{align}

The random matrix $K \in \RR^{p \times p}$ and random vector $r \in \RR^p$ are defined as
\begin{align}
\begin{cases}
K &= \frac{1}{n} \left( \sigma \left(\Theta^\top X \right) \left( \sigma \left(\Theta^\top X \right) \right)^\top + \tau (\Theta^\top V_k V_k^\top \Theta) \odot \left( \sigma' \left(\Theta^\top X \right) \left( \sigma' \left(\Theta^\top X \right) \right)^\top \right) \right) \,,\\[4pt]
r &= \frac{1}{n} \left( \sigma \left(\Theta^\top X \right) Y + \tau  \left( \sigma' \left( \Theta^\top X \right) \odot (\Theta^\top V_k V_k^\top Y') \right) \1_n \right)\,.
\end{cases}
\label{eq:K-matrix-def}
\end{align}
Here, we have summarized the training data as $X = [x_1, \dots, x_n] \in \RR^{d \times n}$, $Y = (y_1, \dots, y_n)^\top \in \RR^n$, and $Y' = [y_1', \dots, y_n'] \in \RR^{d \times n}$.  $\1_n = (1, \dots, 1)^\top \in \RR^n$ is the one-vector, and $\odot$ denotes the elementwise (Hadamard) product with respect to the standard basis of $\RR^p$ in which the model has been defined. For $\tau = 0$, the setup reduces to the standard~$L^2$ training previously analyzed in~\cite{gerace-loureiro-krzakala-etal:2021,mei-montanari:2022,goldt-loureiro-reeves-etal:2022}. We present our result for $\tau > 0$ in Subsection~\ref{sec:error_expressions} and comment on the $L^2$ training limit $\tau \downarrow 0$---which is discontinuous for some parameters introduced below---in Remark~\ref{rem:l2-case} afterwards. Here, and throughout, we also assume $\lambda > 0$, though we conjecture that our results remain valid even in the limit $\lambda \downarrow 0$, cf.~\cite{mei-montanari:2022}.\\

\begin{remark}
The factor $\alpha^{-1} = p/n$ in the Tikhonov regularization strength in~\eqref{eq:training-objective}
ensures that the effective regularization strength remains constant
as the width of the network relative to the training data set size changes. Using $\lambda/\alpha$ in~\eqref{eq:training-objective} is consistent with~\cite{gerace-loureiro-krzakala-etal:2021}, while $\lambda/\gamma$ is used in~\cite{mei-montanari:2022} to the same effect. Roughly, the factor of $1/\alpha$ makes all terms in~\eqref{eq:training-objective} have a common $1/n$ prefactor. More concretely, set $\tau = 0$, and suppose $\sigma(\Theta^\top X)$ has iid standard Gaussian components for simplicity. Then the spectral density of $K$ in~\eqref{eq:training-problem} becomes Marchenko--Pastur (MP) with parameter $1/\alpha$, cf.~\eqref{eq:mp-law}. Hence, the choice of $\lambda / \alpha$ in~\eqref{eq:training-objective} makes the regularization move together with the bulk of the spectrum of $K$ as $\alpha$ is varied in~\eqref{eq:training-problem}.
\end{remark}
\vspace{.5cm}
The minimizer~$w^*$ of~\eqref{eq:training-objective} is a random variable that depends on the realization of the training data and other random quantities in the problem. 
We can determine the optimal training error\footnote{Note that the training errors shown in Figure~\ref{fig:intro} are $2 \eps_{\text{train}}^{L^2}$ and $2 \eps_{\text{train}}^{H^1_k}$, in order to make the normalization comparable to the generalization error as defined in~\eqref{eq:gen-error-def}}
\begin{align}
\eps_{\text{train}}\left(w^* \right) = \eps_{\text{train}}^{L^2} + \tau \eps_{\text{train}}^{H^1_k} +  \frac{\lambda}{2 \alpha} \norm{w^*}^2\,,
\end{align}
but our main interest is to compute the generalization error of the trained network for a ``fresh'', independent sample from the data distribution:
\begin{align}
\eps_{\text{gen}}\left(w^* \right) := \EE_{x,y,y'}\left[\left(y - f_{w^*}\left(x\right)\right)^2 + \norm{V_k^\top \left( y' -\nabla f_{w^*}\left(x\right)\right)}^2 \right] = \eps_{\text{gen}}^{L^2} + \eps_{\text{gen}}^{H^1_k}\,.
\label{eq:gen-error-def}
\end{align}
In the proportional asymptotics limit, it is possible to express these errors as a function of only a finite number of low-dimensional summary statistics, i.e., we need not numerically compute the high-dimensional optimal readout weights~$w^*$. These summary statistics correspond
to ``replica-symmetric'' overlap parameters in the language of the replica method. In contrast to other works in the literature though, we find in our setting that the generalization error does \textit{not} concentrate onto its expectation. 
Concretely, the alignment parameter $\varpi$ does not concentrate as $d,n,p \to \infty$, but instead becomes asymptotically distributed as a standard normal $\varpi \sim {\cal N}(0, I_k)$ which is uncorrelated with all other parameters. 
Nevertheless, it remains possible to employ the replica method by conditioning the theoretical predictions on $\varpi$. Since we know the asymptotic law of this random variable, ultimately we obtain a full characterization of the probability distribution of the errors and can, for example, take the expectation over $\varpi$ or report any other summary statistics.\\

Before stating our theoretical results, we define the first two coefficients and the remainder term in the Hermite expansions of the activation function $\sigma$ and its derivative $\sigma'$, as
\begin{align}
\kappa_0 &= \EE \left[\sigma(\xi) \right]\,, \quad \kappa_1 = \EE \left[\xi \sigma(\xi) \right]\,, \quad \kappa_*^2 =  \EE \left[ \sigma(\xi)^2 \right] - \kappa_0^2 - \kappa_1^2\\
\kappa_0' &= \EE \left[\sigma'(\xi) \right] = \kappa_1\,, \quad \kappa_1' = \EE \left[\xi \sigma'(\xi) \right] = \EE \left[\sigma''(\xi) \right]\,, \quad \left(\kappa_*'\right)^2 =  \EE \left[ \left(\sigma'(\xi)\right)^2 \right] - \left(\kappa_0'\right)^2 - \left(\kappa_1'\right)^2\,,
\label{eq:hermite-coeff-def}
\end{align}
where $\xi \sim {\cal N}(0,1)$.
The coefficients~\eqref{eq:hermite-coeff-def} of $\sigma$, and analogous ones for the ridge function $\phi$ in~\eqref{eq:data-model}, fully characterize these functions in the limit~\eqref{eq:prop-asymp-def}. In other words, we can roughly think of them by effectively replacing the nonlinear function $\sigma$ by its linearization in terms of Hermite coefficients via the Gaussian equivalence relations
\begin{align}
\begin{cases}
\sigma\left(\Theta^\top x \right) &\approx \kappa_0 \1_p + \kappa_1 \Theta^\top x + \kappa_* \hat{\eta}\\
\sigma'\left(\Theta^\top x \right) &\approx \kappa_0' \1_p + \kappa_1' \Theta^\top x + \kappa_*' \hat{\eta}'
\end{cases}\,,
\label{eq:get-sigma-replacement}
\end{align}
where all higher-order terms are replaced by the independent Gaussian noises $\hat{\eta},\hat{\eta}' \sim {\cal N}\left(0, I_{p}\right)$, scaled to the same variance as the actual remainder term.\\

We do not assume that $\kappa_0$ vanishes---a typical simplifying assumption in the literature---so we can treat standard activation functions such as the rectified linear unit (ReLU) $\sigma(z) = \max\{0,z\}$ and sigmoid linear unit (SiLU) $\sigma(z) = z/\left(1 + e^{-z}\right)$. By the Gaussian equivalence relations~\eqref{eq:get-sigma-replacement} and~\eqref{eq:get-distribution} with overlap parameters~\eqref{eq:overlap-def} below, if $\kappa_0 = 0$, then the trained network $f_{w^*}$ is incapable of realizing anything other than mean-zero functions of $x$, i.e.,\ necessarily $\EE_{x \sim {\cal N}(0, I_d)} \left[f_{w^*}(x) \right] = 0$. Similarly, if $\kappa_1 = 0$, then $f_{w^*}(x)$ does not actually depend on $x$ in the proportional asymptotics limit.\\

The activation functions  considered in this paper are listed in Table~\ref{tab:activation} in Appendix~\ref{App:notation} along with their Hermite coefficients. An example of a parameter-dependent, non-polynomial activation function with Hermite coefficients that can be adjusted continuously is found in~\cite{pennington-worah:2019}. For a detailed analysis of the role of individual coefficients for the generalization capacities of the RF model under $L^2$-training, we refer to~\cite{dascoli-sagun-biroli:2020}. We remark that in principle, it is sufficient to consider activation functions $\sigma$ (and analogous $\phi$) of the form
\begin{align}
\sigma(z) = \sigma_0
+ \sigma_1 z
+ \frac{\sigma_2}{\sqrt{2!}}\left(z^2 - 1\right)
+ \frac{\sigma_3}{\sqrt{3!}} \left( z^3 - 3 z \right)
+ \frac{\sigma_4}{\sqrt{4!}} \left( z^4 - 6 z^2 + 3 \right) = \sum_{k=0}^4 \frac{\sigma_k}{\sqrt{k!}} \text{He}_k(z)
\label{eq:hermite-expansion-sigma}
\end{align}
with constants $\sigma_0, \dots, \sigma_4 \in \RR$ for the setting studied in this work since these fully exhaust the possible parameter space for the coefficients in~\eqref{eq:hermite-coeff-def} via $\kappa_0 = \sigma_0$, $\kappa_0' = \kappa_1 = \sigma_1$, $\kappa_1' = \sqrt{2} \sigma_2$, $\kappa_* = \sqrt{\sigma_2^2 + \sigma_3^2 + \sigma_4^2}$, $\kappa_*' = \sqrt{3 \sigma_3^2 + 4 \sigma_4^2}$.


\subsection{Asymptotic training and generalization error from fixed-point system}
\label{sec:error_expressions} 

To calculate the training and generalization errors in the proportional asymptotics limit, we require knowledge of the summary statistics listed below. In the following, subscripts $a$ denote scalar quantities, subscripts $b$ denote vectors in $\RR^k$, and subscripts $c$ are used for (symmetric) matrices in $\RR^{k \times k}$. Then we define
\begin{align}
\begin{cases}
s_a &= \kappa_0 \left \langle w^*, \1_p \right \rangle\\
s_b &= \kappa_0' V_k^\top \Theta w^*\\
f_a &= \kappa_1 \left \langle \theta_0, \Theta w^* \right \rangle\\
f_b &= \kappa_1' V_k^\top \Theta \diag{w^*} \Theta^\top \theta_0\\
q_a &= \left \langle w^*, \left[\kappa_*^2 I_p + \kappa_1^2 \Theta^\top \Theta \right] w^* \right \rangle\\
q_b &= \kappa_1 \kappa_1' V_k^\top \Theta \diag{w^*} \Theta^\top \Theta  w^*\\
q_c &= V_k^\top \Theta \diag{w^*} \left[\left(\kappa_*'\right)^2 I_p + \left( \kappa_1' \right)^2 \Theta^\top \Theta \right] \diag{w^*} \Theta^\top V_k\\
\end{cases}\,.
\label{eq:overlap-def}
\end{align}
The central idea is that in the proportional asymptotics regime~\eqref{eq:prop-asymp-def}, the RF model $f_{w^*}$, as defined in~\eqref{eq:neural-net-def}, and its projected gradient $V_k^\top \nabla f_{w^*}$, given by~\eqref{eq:neural-net-grad-def}, behave like noisy linear functions in $x$ for the purpose of calculating the training and generalization error. By comparing~\eqref{eq:get-sigma-replacement} with~\eqref{eq:neural-net-def} and~\eqref{eq:neural-net-grad-def}, this replacement yields Gaussian output of the network and its gradient (conditioned on all other random parameters in the setting) for input $x \sim {\cal N}(0, I_d)$ with mean and covariance determined by the overlap parameters~\eqref{eq:overlap-def}:
\begin{align}
\begin{pmatrix}
\omega = \left \langle \theta_0, x \right \rangle\\
f_{w^*}(x)\\
V_k^\top \nabla f_{w^*}(x)
\end{pmatrix} \sim {\cal N} \left(\begin{pmatrix}
0\\s_a\\s_b
\end{pmatrix}, \begin{pmatrix}
1 & f_a & f_b^\top\\[2pt]
f_a & q_a & q_b^\top\\[2pt]
f_b & q_b & q_c
\end{pmatrix} \right)\,.
\label{eq:get-distribution}
\end{align}
We discuss this linearization further and provide numerical evidence for its validity in Appendix~\ref{app:GET}.\\

The Gaussian equivalence theorem then yields the following deterministic expressions for the $L^2$ and $H^1_k$ seminorm generalization errors in the proportional asymptotics limit, conditioned on the alignment $\varpi \in \RR^k$:
\begin{align}
\label{eq:l2_gen_error}
\plim_{p \to \infty} \eps^{L^2}_{\text{gen}} \mid \varpi &= \EE\left[\left(\phi(\omega) - s_a\right)^2 \right] - 2 \EE[\phi'(\omega)] f_a + (C_{\eta})_{11} + q_a\,,\\
\label{eq:sobo_gen_error}
\plim_{p \to \infty} \eps^{H^1_k}_{\text{gen}} \mid \varpi
&= 
\EE \left[\lVert \varpi \phi'(\omega) - s_b \rVert^2\right] - 2\EE[\phi''(\omega)] \langle \varpi, f_b \rangle + \trace \left[ C_{\eta,2:k+1,2:k+1} \right] + \trace (q_c).
\end{align}
Note that here and in the following equations $\omega \sim {\cal N}(0,1)$, consistent with the marginal distribution in~\eqref{eq:get-distribution}. The network and projected network gradient means are given by $s = (s_a, s_b) \in \RR^{k+1}$ with
\begin{align}
s_a = \begin{cases}
0\,, \quad & \kappa_0 = 0\,,\\
\EE[\phi(\omega)]\,, \quad & \kappa_0 \neq 0\,,\\
\end{cases}
\qquad \quad 
s_b = \begin{cases}
0\,, \quad & \kappa_0' = 0\,,\\
\varpi \EE[\phi'(\omega)]\,, \quad & \kappa_0' \neq 0 \,, \\
\end{cases} \qquad \quad \text{for }\tau>0.
\label{eq:def-sab}
\end{align}
The remaining overlap parameters necessary to evaluate~\eqref{eq:l2_gen_error} and~\eqref{eq:sobo_gen_error} can be found by solving a deterministic system of low-dimensional equations---e.g.,\ numerically via fixed-point iteration---instead of using the definitions from~\eqref{eq:overlap-def} wherein high-dimensional random vectors and matrices must be computed. We obtain this system of equations by applying the saddle point method in the proportional limit within the replica calculation and subsequently taking the low-temperature limit, as detailed in Appendix~\ref{app:replica}. Hence, following the usual recipe of the replica method while conditioning all terms on $\varpi$ produces the solution in a relatively ``mechanical'' way. Since the training problem~\eqref{eq:training-problem} is strictly convex, a unique admissible solution (where the covariance matrix in~\eqref{eq:get-distribution} is positive semidefinite) to the fixed-point system is guaranteed to exist, and this solution corresponds to the replica-symmetric solution of the saddlepoint equations.\\

We collect the overlap parameters as
\begin{align}
	f = \begin{pmatrix}
	f_a \\[2pt] f_b, 
\end{pmatrix}, \quad	 q = \begin{pmatrix}
	q_a & q_b^\top \\[2pt]
	q_b & q_c
	\end{pmatrix}, \quad \Sigma = \begin{pmatrix}
	\Sigma_a & \Sigma_b^\top \\[2pt]
	\Sigma_b & \Sigma_c
	\end{pmatrix}\,,	
\end{align}
and we introduce analogous auxiliary parameters $\hat{f}$, $\hat{q}$, and $\hat{\Sigma}$ via
\begin{align}
\label{eq:update_hat_sobo_h1}
\begin{cases}
\hat{\Sigma} &= \alpha \left(I_{k+1} + D_\tau \Sigma \right)^{-1} D_\tau\\
\hat{f} &= \hat{\Sigma} \begin{pmatrix} \EE[\phi'(\omega)] \\ \varpi \EE[\phi''(\omega)] \end{pmatrix}\\
\hat{q} &= \alpha^{-1}\hat{\Sigma}\Bigg( C_\eta + q + \EE\left[ \left(  \begin{pmatrix} \phi(\omega)  \\ \varpi \phi'(\omega) \end{pmatrix}-s\right)^{\otimes 2} \right] \Bigg) \hat{\Sigma} - \alpha^{-1} \left( \hat{\Sigma} f \otimes \hat{f} + \hat{f} \otimes \left( \hat{\Sigma} f \right) \right)
\end{cases}
\end{align}
where $D_\tau = \diag{1, \tau, \dots, \tau} \in \RR^{(k+1) \times (k+1)}$. 
The hatted overlap parameters map back to $f$, $q$, and $\Sigma$ through
\begin{align}
\label{eq:update_nonhat_sobo_h1}
\begin{cases}
	\Sigma &= \displaystyle \plim_{p\to\infty} \frac{1}{p} \; \Bigg[ \begin{pmatrix} \kappa_1 \1_p^\top \\ \kappa_1' V_k^\top\Theta \end{pmatrix} \left(A^{-1} \odot \Theta^\top\Theta\right) \begin{pmatrix} \kappa_1 \1_p & \kappa_1' \Theta^\top V_k \end{pmatrix} +  \begin{pmatrix} \kappa_*^2 \trace \left[A^{-1} \right] & 0 \\ 0 & (\kappa_*')^2 V_k^\top\Theta \left(A^{-1} \odot I_p \right) \Theta^\top V_k \end{pmatrix} \Bigg]\\[12pt]
	f &= \displaystyle \plim_{p\to\infty} \quad \; \; \;\begin{pmatrix} \kappa_1 \1_p^\top \\ \kappa_1' V_k^\top\Theta \end{pmatrix} \left(A^{-1} \odot \left(\Theta^\top \theta_0 \right)^{\otimes 2} \right) \begin{pmatrix} \kappa_1 \1_p & \kappa_1' \Theta^\top V_k \end{pmatrix} \hat{f} \\[12pt]
	q &= \displaystyle \plim_{p\to\infty} \frac{1}{p} \; \Bigg[ \begin{pmatrix} \kappa_1 \1_p^\top \\ \kappa_1' V_k^\top\Theta \end{pmatrix} \left( \left(A^{-1}\Xi A^{-1} \right) \odot \Theta^\top\Theta \right) \begin{pmatrix} \kappa_1 \1_p & \kappa_1' \Theta^\top V_k \end{pmatrix}
	+ \begin{pmatrix} \kappa_*^2 \trace \left[ A^{-1}\Xi A^{-1} \right] & 0 \\ 0 & (\kappa_*')^2 V_k^\top\Theta  \left( \left(A^{-1}\Xi A^{-1} \right) \odot I_p \right)  \Theta^\top V_k \end{pmatrix}  \Bigg]
\end{cases}
\end{align}
where we have defined the random matrices
\begin{align}
\label{eq:def-random-matrices-for-saddlepoint-eqs}
\begin{cases}
	A = A\left(\hat{\Sigma}\right) &:=  \lambda I_p + \begin{pmatrix} \kappa_1 \1_p & \kappa_1' \Theta^\top V_k \end{pmatrix}   \hat{\Sigma} \begin{pmatrix} \kappa_1 \1_p^\top \\[2pt] \kappa_1' V_k^\top\Theta \end{pmatrix} \odot \Theta^\top \Theta + \begin{pmatrix} \kappa_* \1_p & \kappa_*' \Theta^\top V_k \end{pmatrix}  \begin{pmatrix} \hat{\Sigma}_a & 0 \\ 0 & \hat{\Sigma}_c \end{pmatrix} \begin{pmatrix} \kappa_* \1_p^\top \\[2pt] \kappa_*' V_k^\top\Theta \end{pmatrix} \odot I_p \\[12pt]
	\Xi = \Xi\left(\hat{f}, \hat{q}\right) &:=  \qquad \quad \; \begin{pmatrix} \kappa_1 \1_p & \kappa_1' \Theta^\top V_k \end{pmatrix} \hat{q}  \begin{pmatrix} \kappa_1 \1_p^\top \\[2pt] \kappa_1' V_k^\top\Theta \end{pmatrix} \odot \Theta^\top \Theta  +\begin{pmatrix} \kappa_* \1_p & \; \kappa_*' \Theta^\top V_k \end{pmatrix} \begin{pmatrix} \hat{q}_a & 0 \\ 0 & \hat{q}_c \end{pmatrix} \begin{pmatrix} \kappa_* \1_p^\top \\[2pt] \kappa_*' V_k^\top\Theta \end{pmatrix} \odot I_p \\[10pt]
	& \quad \quad + p \begin{pmatrix} \kappa_1 \1_p & \kappa_1' \Theta^\top V_k \end{pmatrix} \hat{f}^{\otimes 2}  \begin{pmatrix} \kappa_1 \1_p^\top \\[2pt] \kappa_1' V_k^\top\Theta \end{pmatrix} \odot \left( \Theta^\top \theta_0 \right)^{\otimes 2}\,.
\end{cases}
\end{align}
The scalings in the problem setup ensure that all overlap parameters remain $O(1)$ as $d,n,p\to\infty$.
After solving the system given by~\eqref{eq:update_hat_sobo_h1} and~\eqref{eq:update_nonhat_sobo_h1} numerically, in addition to the generalization errors~\eqref{eq:l2_gen_error} and~\eqref{eq:sobo_gen_error}, we obtain the training error at the optimal readout weights $w^*$ via
\begin{align}
\plim_{p \to \infty}  \eps_{\text{train}}^{L^2} \mid \varpi &= \frac{1}{2\alpha} \hat{q}_a\,, \label{eq:training-err-l2}\\
\plim_{p \to \infty}  \eps_{\text{train}}^{H^1_k} \mid \varpi &= \frac{1}{2\alpha} \trace \left[ \hat{q}_c \right]\,, \label{eq:training-err-h1k}\\
\plim_{p \to \infty}  \frac{\lambda}{2 \alpha} \norm{w^*}^2 \mid \varpi &= \frac{\lambda}{2 \alpha} \plim_{p \to \infty} \frac{1}{p} \trace \left[A^{-1} \Xi A^{-1} \right]\,.
\label{eq:training-err-reg}
\end{align}
\begin{remark}
\label{rem:saddlepoint-observations}
We collect a few observations on this result here:
\begin{enumerate}[label={(\alph*)}]
\item Despite their complicated appearance at first glance, the fixed point equations~\eqref{eq:update_hat_sobo_h1} and~\eqref{eq:update_nonhat_sobo_h1} have a relatively simple structure: since the random matrix $A$ in~\eqref{eq:def-random-matrices-for-saddlepoint-eqs} only depends on the parameter matrix $\hat{\Sigma}$, the equations for $\Sigma$ and $\hat{\Sigma}$ form a closed, nonlinear system of equations for the two unknown symmetric $(k+1) \times (k+1)$ matrices. In fact, we will show in Subsection~\ref{sec:eval}, that each $(k+1) \times (k+1)$ matrix depends on only two parameters. Once this system has been solved, the vectors $f$ and $\hat{f}$ are fully determined without further solves. Lastly, the matrices $q$ and $\hat{q}$ can then be found as the solution of a four-dimensional linear system of equations.
\item The remaining difficulties are (i) isolating the dependence of all parameters and results on the alignment $\varpi \sim {\cal N}(0, I_k)$, and (ii) evaluating the high-dimensional limits in~\eqref{eq:update_nonhat_sobo_h1} involving the random feature matrix $\Theta$ and subspace matrix $V_k$. Conceptually, it is crucial to be able to evaluate the high-dimensional limits in~\eqref{eq:update_nonhat_sobo_h1} through analytical or semi-analytical methods that only involve finite-dimensional quantities since only then is the system of equations~\eqref{eq:update_hat_sobo_h1} and~\eqref{eq:update_nonhat_sobo_h1} ``closed'' and actually low-dimensional. We show the resulting system of equations after these simplifications in Section~\ref{sec:eval}.
\item Suppose we consider a more general loss function than~\eqref{eq:training-objective}:
\begin{align}
\eps_{\text{train}}(w) = \frac{1}{n} \sum_{i = 1}^n \ell \left(y_i, \; f_w(x_i), \; V_k^\top y_i', \; V_k^\top \nabla f_w(x_i) \right) + \frac{\lambda}{2 \alpha} \norm{w}^2.
\label{eq:general-loss}
\end{align}
Given convex and differentiable $\ell$, this extension---relevant, e.g.,\ for classification tasks---only modifies the updates~\eqref{eq:update_hat_sobo_h1} for the auxiliary parameters and leaves all other results unchanged. In Appendix~\ref{app:replica}, we derive the general result for the training loss~\eqref{eq:general-loss} and only specify it to~\eqref{eq:training-objective} in the end, incurring no increased technical difficulties. Similarly, we can treat more general noise models $P_{\text{data}}$ than the additive Gaussian case~\eqref{eq:p-data-additive-gauss}, as well as more general random features than $\Theta_{ij} \overset{\text{iid}}{\sim} {\cal N}(0, 1/d)$ provided the random matrix $\Theta^\top \Theta$ has a well-defined spectral density in the proportional asymptotics limit. This flexibility of the replica approach is the main advantage over a direct computation of the high-dimensional limits of the overlap parameters in~\eqref{eq:overlap-def} which demand an explicit expression for the minimizer $w^*$.
\item The values of the activation function mean $\kappa_0$ and its derivative mean $\kappa_0'$ do not explicitly appear in the results, except for discontinuously determining the cases in the definition of~$s$ in~\eqref{eq:def-sab}. These cases correspond to the network being (in)capable of learning the mean of the data and its $\varpi$-conditioned gradient due to the choice of activation function.
\item As anticipated in Section~\ref{sec:setup}, we see from~\eqref{eq:def-sab} and~\eqref{eq:update_hat_sobo_h1} that the overlap parameters and generalization errors only depend on the data-generating ridge function $\phi$ and its derivative $\phi'$ through their low-order Hermite coefficients and remainder term,  analogously to~\eqref{eq:hermite-coeff-def}. Intuitively, only in cases where both $\EE[\omega \phi(\omega)]$ and $\EE[\omega \phi'(\omega)]$ are nonzero do the RF network and gradient actually learn to represent nontrivial (but still linear) functions of $x$. If either of these expectations are zero, the function or gradient data, respectively, effectively corresponds to being generated by a constant function plus noise. This observation will be important when interpreting the predictions described in Section~\ref{sec:results}.
\end{enumerate}
\end{remark}

\begin{remark}
\label{rem:l2-case}
In the limit $\tau \to 0$ the gradient data do not inform training, so we recover the usual $L^2$ training setup. Here, we have $D_\tau \to e_1^{\otimes 2}$ such that $\hat{\Sigma} \propto e_1^{\otimes 2}$, $\hat{f} \propto e_1$ and $\hat{q} \propto e_1^{\otimes 2}$ in~\eqref{eq:update_hat_sobo_h1}. This sparsity leads to a solution of the fixed-point equations with
\begin{align}
f_b = \hat{f_b} = q_b = \hat{q}_b = \hat{q}_c = \Sigma_b = \hat{\Sigma}_b = \hat{\Sigma}_c = 0\,,
\label{eq:zero-overlaps-l2}
\end{align}
recovering the fixed-point system for $\{f_a, \hat{f}_a, \; q_a, \hat{q}_a, \Sigma_a, \hat{\Sigma}_a\}$ from~\cite{gerace-loureiro-krzakala-etal:2021,goldt-loureiro-reeves-etal:2022}. Once obtained, these parameters determine $\Sigma_c$ and $q_c$. For two quantities, the limit $\tau \downarrow 0$ is discontinuous. First, as is apparent from the derivation of~\eqref{eq:update_hat_sobo_h1} in Appendix~\ref{app:replica}, the overlap parameter $s_b = \kappa_0' V_k^\top \Theta w^*$ is no longer determined through the replica-symmetric saddle-point equations when $\tau = 0$ and~\eqref{eq:def-sab} is invalid in this case. As detailed in Appendix~\ref{app:varpi-sb-l2}, we find instead for $\tau = 0$ that
\begin{align}
\begin{pmatrix}
\varpi\\
s_b
\end{pmatrix}
\sim {\cal N} \left(
\begin{pmatrix}
0\\
0
\end{pmatrix}, 
\begin{pmatrix}
1 & f_a\\
f_a & q_a - \kappa_*^2 \norm{w^*}^2
\end{pmatrix} \otimes I_k 
\right)\,.
\label{eq:varpi-sb-distr-l2-train}
\end{align}
Second, the $H^1_k$ training error is not determined by $\hat{q}_c = 0$ in this setting, but we have $\eps_{\text{train}}^{H^1_k} =  \eps_{\text{gen}}^{H^1_k}$ instead. These simplifications reduce the saddle-point equation~\eqref{eq:update_hat_sobo_h1} to
\begin{align}
\begin{cases}
\hat{\Sigma}_a &= \frac{\alpha}{1 + \Sigma_a}\\[2pt]
\hat{f}_a &= \frac{\alpha}{1 + \Sigma_a} \EE\left[ \phi'(\omega) \right]\\[2pt]
\hat{q}_a &= \frac{\alpha}{\left(1 + \Sigma_a \right)^2} \left(C_{\eta, 11} + \EE\left[(\phi(\omega) - s_a)^2 \right] + q_a - 2 f_a \EE \left[\phi'(\omega) \right] \right)\,.
\end{cases}
\label{eq:l2-hat-update}
\end{align}
Since the random matrices in~\eqref{eq:def-random-matrices-for-saddlepoint-eqs} reduce to
\begin{align}
\begin{cases}
A &= \left( \lambda + \kappa_*^2 \hat{\Sigma}_a \right) I_p + \kappa_1^2 \hat{\Sigma}_a \Theta^\top \Theta\,,\\[2pt]
\Xi &= \kappa_*^2 \hat{q}_a I_p + \kappa_1^2 \left(\hat{f}_a^2/\gamma + \hat{q}_a \right) \Theta^\top \Theta\,,
\end{cases}
\label{eq:l2-simplify-random-matrices}
\end{align}
we can easily evaluate the random matrix statistics in~\eqref{eq:update_nonhat_sobo_h1} in terms of the Stieltjes transform $g_\mu(z) := \int_{\RR} \tfrac{\dd \mu(t)}{t-z}$, $z \in \CC \setminus \text{supp}(\mu)$ of the spectral density $\mu$ of~$\Theta \Theta^\top \in \RR^{d \times d}$ in the proportional asymptotics limit:\footnote{We use the Stieltjes transform of $\Theta \Theta^\top$ instead of $\Theta^\top \Theta$ here, so that the result aligns with the convention used in~\cite{gerace-loureiro-krzakala-etal:2021}.}
\begin{align}
\begin{cases}
\Sigma_a &= \frac{\gamma}{\hat{\Sigma}_a} \left[1 - z g_\mu(-z) \right] + \gamma \frac{\kappa_*^2}{\hat{\Sigma}_a \kappa_1^2} \left[ z^{-1} \left(\gamma^{-1}-1\right) + g_\mu(-z) \right]\,,\\[2pt]
f_a &= \frac{\hat{f}_a}{\hat{\Sigma}_a} \left[1 - z g_\mu(-z) \right]\,, \\[2pt]
q_a &= \left(\hat{f}_a^2 + \gamma \hat{q}_a \right) \frac{1}{\hat{\Sigma}_a^2}\left[1 - 2zg_\mu(-z) + z^2 g'_\mu(-z) \right] + \gamma \frac{\kappa_*^4}{\kappa_1^4 \hat{\Sigma}_a^2} \hat{q}_a \left[z^{-2} (\gamma^{-1}-1) + g_\mu'(-z) \right] \\[2pt]
&\quad + \left(2\gamma \hat{q}_a + \hat{f}_a^2 \right) \frac{ \kappa_*^2}{\hat{\Sigma}_a^2 \kappa_1^2} \left[g_\mu(-z) - zg'_\mu(-z) \right]\,,
\end{cases}
\label{eq:l2-nonhat-update}
\end{align}
where $z = \left( \lambda + \kappa_*^2 \hat{\Sigma}_a \right) / \left(\kappa_1^2 \hat{\Sigma}_a \right)$. Notably, the Hadamard products in~\eqref{eq:update_nonhat_sobo_h1} and~\eqref{eq:def-random-matrices-for-saddlepoint-eqs} drop out immediately in this case by~\eqref{eq:zero-overlaps-l2}, and the remaining matrix traces can be expressed via Stieltjes transforms using standard algebraic manipulations as listed in Appendix~\ref{app:l2-stieltjes}. For random features $\Theta_{ij} \overset{\text{iid}}{\sim} {\cal N}(0,1/d)$, the corresponding spectral measure~$\mu$ is the MP law with Stieltjes transform~\cite{bai-silverstein:2010}
\begin{align}
g_\mu(z) = \frac{\tfrac{1}{\gamma}(1 - \gamma) - z + \sqrt{\left(z - \tfrac{1}{\gamma}\left(1 + \gamma\right)\right)^2 - 4 \tfrac{1}{\gamma} }}{2 z }\,.
\end{align}
After solving the low-dimensional system of equations~\eqref{eq:l2-hat-update} and~\eqref{eq:l2-nonhat-update} for the $a$-indexed overlap parameters, the remaining parameter~$q_c$ is determined from the corresponding right-hand side of~\eqref{eq:update_nonhat_sobo_h1} via
\begin{align}
q_c = \underbrace{\plim_{p \to \infty} \tfrac{1}{p} \trace \left[A^{-1} \Xi A^{-1} \right]}_{\overset{\eqref{eq:training-err-reg}}{=}\plim_{p \to \infty} \lVert w^* \rVert^2} \cdot  \underbrace{\plim_{p \to \infty} \tfrac{1}{p} \trace \left[\left(\kappa_1' \right)^2 \Theta^\top \Theta + \left(\kappa_*' \right)^2 I_p \right]}_{=\left(\kappa_1' \right)^2+\left(\kappa_*' \right)^2 } I_k\,,
\label{eq:trace-a-xi-a-l2}
\end{align}
as derived in more detail in Appendix~\ref{app:hada-trace-l2}.
We can then compute the distribution and summary statistics of the $H^1_k$ generalization error according to~\eqref{eq:sobo_gen_error} as
\begin{align}
\EE \left[\eps^{H^1_k}_{\text{gen}} \right] = k \left(\EE \left[(\phi'(\omega))^2 \right]  + q_a + \left(\left(\kappa_1' \right)^2+\left(\kappa_*' \right)^2 - \kappa_*^2 \right) \norm{w^*}^2 - 2 f_a \EE \left[\phi'(\omega) \right]\right) + \trace\left[ C_{\eta,2:k+1,2:k+1} \right]\,.
\label{eq:h1k-gen-l2-train}
\end{align}
Finally, the remaining trace in~\eqref{eq:trace-a-xi-a-l2}, which also appears in the optimal regularization term~\eqref{eq:training-err-reg} and the $H^1_k$ error~\eqref{eq:h1k-gen-l2-train}, can be expressed via Stieltjes transforms, similarly to~\eqref{eq:l2-nonhat-update}, as
\begin{align}
\plim_{p \to \infty}  \norm{w^*}^2 = \plim_{p \to \infty} \tfrac{1}{p} \trace \left[A^{-1} \Xi A^{-1} \right] = \gamma \frac{\kappa_*^2}{\kappa_1^4 \hat{\Sigma}_a^2} \hat{q}_a  \left[z^{-2} (\gamma^{-1}-1) + g_\mu'(-z) \right] +  \left(\hat{f}_a^2 + \gamma \hat{q}_a \right) \frac{1}{\kappa_1^2 \hat{\Sigma}_a^2}\left[g_\mu(-z) - zg'_\mu(-z) \right] \,.
\end{align}
\end{remark}


\subsection{Evaluation of the fixed-point system}
\label{sec:eval}

\subsubsection{Asymptotic simplifications of the fixed point system}
\label{sec:asym_rmf}

As stated in Remark~\ref{rem:saddlepoint-observations} (b), we can further simplify the fixed-point equations~\eqref{eq:update_hat_sobo_h1} and~\eqref{eq:update_nonhat_sobo_h1}. Technical details are deferred to Appendix~\ref{app:simplify-fp}. The result is that (i) the $\varpi$-dependence of the overlap parameters is explicitly given by
\begin{align}
\begin{cases}
s_a = s_a^{(0)}  &  \\
s_{b,i} = s_b^{(1)} \varpi_i,~i \in [k] & \\
\hat{\Sigma}_a = \hat{\Sigma}_a^{(0)} & \Sigma_a = \Sigma_a^{(0)} \\
\hat{\Sigma}_{c,ii} = \hat{\Sigma}_c^{(0)},~i \in [k] & \Sigma_{c,ii} = \Sigma_c^{(0)},~i\in [k]\\
\hat{f}_a = \hat{f}_a^{(0)} & f_a = f_a^{(0)} \\
\hat{f}_{b,i} = \hat{f}_{b}^{(1)} \varpi_i,~i \in [k] & f_{b,i} = f_b^{(1)} \varpi_i,~i \in [k] \\
\hat{q}_a = \hat{q}_a^{(2)} \|\varpi\|^2 + \hat{q}_a^{(0)} & q_a = q_a^{(2)} \|\varpi\|^2 + q_a^{(0)} \\
\trace{\hat{q}_{c}} = \hat{q}_c^{(2)} \|\varpi\|^2 + \hat{q}_c^{(0)}, & \trace{q_{c}} = q_c^{(2)}\|\varpi\|^2 + q_c^{(0)}\,.
\end{cases}
\label{eq:varpi-scaling}
\end{align}
in terms of $\varpi$-independent scalar coefficients, and (ii) the random matrix traces in~\eqref{eq:update_nonhat_sobo_h1} can be reduced and expressed without Hadamard products, resulting in the following system of equations, with $\Tr_p := \plim_{p \to \infty} \tfrac{1}{p}\trace$:
\begin{align}
&\begin{cases}
\Sigma_{a}^{(0)} &= \Tr_p \left[ A^{-1} M_{00} \right]\,,\\
    \Sigma_{c}^{(0)} &= \Tr_p \left[ A^{-1} D_1 M_{11} D_1 \right]\,,\\[5pt]
f_a^{(0)} &= \frac{1}{\gamma} \kappa_1^2 \Tr_p \left[ A^{-1} \Theta^\top \Theta \right] \hat{f}_a^{(0)}\,,\\
f_{b}^{(1)} &= \frac{1}{\gamma} (\kappa_1')^2 \Tr_p \left[ A^{-1} D_1 \Theta^\top \Theta D_1 \right] \hat{f}_{b}^{(1)}\,,\\[5pt]
	q_a^{(0)} &=  \kappa_1^2 \frac{1}{\gamma} \left(\hat{f}_a^{(0)}\right)^2 \Tr_p \left[ A^{-1}\Theta^\top\Theta A^{-1}M_{00} \right] + \hat{q}_a^{(0)} \Tr_p \left[ A^{-1} M_{00} A^{-1} M_{00} \right] + \hat{q}_{c}^{(0)} \Tr_p \left[ A^{-1}D_1 M_{11} D_1 A^{-1} M_{00} \right]\,, \\
	q_a^{(2)} &=  (\kappa_1')^2 \frac{1}{\gamma} \left(\hat{f}_{b}^{(1)} \right)^2 \Tr_p \left[ A^{-1} D_1\Theta^\top \Theta D_1 A^{-1} M_{00}\right] + \hat{q}_a^{(2)} \Tr_p \left[ A^{-1} M_{00} A^{-1} M_{00} \right] + \hat{q}_{c}^{(2)} \Tr_p \left[ A^{-1}D_1 M_{11} D_1 A^{-1} M_{00} \right]\,,  \\
	q_{c}^{(0)} &= k \cdot \kappa_1^2 \frac{1}{\gamma} \left(\hat{f}_a^{(0)} \right)^2 \Tr_p \left[ A^{-1}\Theta^\top\Theta A^{-1} D_1M_{11}D_1 \right] +  k \cdot \hat{q}_a^{(0)} \Tr_p \left[ A^{-1} M_{00} A^{-1}D_1M_{11}D_1 \right]\\
    &\quad\quad + \hat{q}_{c}^{(0)} \Tr_p \left[  A^{-1}D_1 M_{11} D_1 A^{-1}D_1M_{11}D_1 \right]  + (k-1) \cdot \hat{q}_c^{(0)} \Tr_p \left[  A^{-1}D_1 M_{11} D_1 A^{-1}D_2M_{11}D_2 \right]\,,  \\
	q_{c}^{(2)} &=  (\kappa_1')^2 \frac{1}{\gamma} \left(\hat{f}_{b}^{(1)} \right)^2 \Tr_p \left[A^{-1} D_1\Theta^\top \Theta D_1 A^{-1} D_1M_{11}D_1 \right] \\
    &\quad\quad+  (k-1)\cdot(\kappa_1')^2 \frac{1}{\gamma} \left(\hat{f}_{b}^{(1)}\right)^2 \Tr_p \left[ A^{-1} D_1\Theta^\top \Theta D_1 A^{-1} D_2M_{11}D_2 \right] \\
    &\quad\quad + k \cdot \hat{q}_a^{(2)} \Tr_p \left[ A^{-1} M_{00} A^{-1}D_1M_{11}D_1 \right] + \hat{q}_{c}^{(2)} \Tr_p \left[ A^{-1}D_1 M_{11} D_1 A^{-1} D_1M_{11}D_1 \right]  \\
    &\quad\quad + (k-1)\cdot \hat{q}_{c}^{(2)} \Tr_p \left[ A^{-1}D_2 M_{11} D_2 A^{-1} D_1M_{11}D_1 \right]\,,  
\end{cases} \label{eq:asymptotic_saddle_final_nonhat}\\
&\begin{cases}
\hat{\Sigma}_a^{(0)} &= \frac{\alpha}{1 + \Sigma_a^{(0)}}\,, \\
\hat{\Sigma}_c^{(0)} &= \frac{\alpha \tau}{1 + \Sigma_c^{(0)} \tau}\,, \\[7pt]
\hat{f}_a^{(0)} &= \hat{\Sigma}_a^{(0)}\cdot\EE\left[\phi'(\omega)\right]\,, \\
\hat{f}_b^{(1)} &= \hat{\Sigma}_c^{(0)}\cdot\EE \left[\phi''(\omega) \right]\,, \\[5pt]
\hat{q}_a^{(0)} &= \alpha^{-1} \hat{\Sigma}_a^{(0)} \left(C_{\eta,11} + q_a^{(0)} + \EE\left[\left(\phi(\omega) - s_a^{(0)}\right)^2 \right] \right) \hat{\Sigma}_a^{(0)} - 2 \alpha^{-1} \hat{\Sigma}_a^{(0)} f_a^{(0)} \hat{f}_a^{(0)} \,, \\
\hat{q}_a^{(2)} &= \alpha^{-1} \hat{\Sigma}_a^{(0)} q_a^{(2)} \hat{\Sigma}_a^{(0)}\,, \\
\hat{q}_c^{(0)} &= \alpha^{-1} \hat{\Sigma}_c^{(0)} \left(\trace{C_{\eta,2:k+1,2:k+1}} + q_c^{(0)} \right) \hat{\Sigma}_c^{(0)}\,,  \\
\hat{q}_c^{(2)} &= \alpha^{-1} \hat{\Sigma}_c^{(0)} \left(q_c^{(2)} + \EE\left[\left(\phi'(\omega) - s_b^{(1)}\right)^2\right] \right) \hat{\Sigma}_c^{(0)} - 2 \alpha^{-1} \hat{\Sigma}_c^{(0)} f_b^{(1)} \hat{f}_b^{(1)}\,.
\end{cases}
\label{eq:asymptotic_saddle_final_hat}
\end{align}
Here, $A =  \lambda I_p + \hat{\Sigma}_a^{(0)} M_{00} + \hat{\Sigma}_{c}^{(0)} \sum_{i \in[k]} D_i M_{11} D_i$, $M_{00}  = \kappa_1^2 I_p + \kappa_*^2 \Theta^\top \Theta$, $M_{11} = \left(\kappa_1'\right)^2 I_p + \left( \kappa_*' \right)^2 \Theta^\top \Theta$, and $D_i = \diag{\zeta_i}, i \in [k]$ where $\zeta_i \sim {\cal N}(0,I_p)$ are iid random vectors that are independent of $\Theta^\top \Theta$. Notably, the number of unknowns of the fixed-point system becomes independent of $k$, and it only needs to be solved once for a given set of hyperparameters to characterize the full distribution of the overlap parameters and generalization errors. Explicitly, this recovers from~\eqref{eq:l2_gen_error} and~\eqref{eq:sobo_gen_error} that the generalization errors
\begin{align}
\label{eq:l2_gen_error_varpi}
\plim_{p \to \infty} \eps^{L^2}_{\text{gen}} \mid \varpi &= \left( \EE\left[\left(\phi(\omega) - s_a\right)^2 \right] - 2 \EE[\phi'(\omega)] f_a^{(0)} + (C_{\eta})_{11} + q_a^{(0)} \right) + q_a^{(2)} \|\varpi\|^2 \;,\\
\label{eq:sobo_gen_error_varpi}
\plim_{p \to \infty} \eps^{H^1_k}_{\text{gen}} \mid \varpi
&= \left( \trace \left[ C_{\eta,2:k+1,2:k+1} \right] + q_c^{(0)} \right) + \left( 
\EE \left[\left( \phi'(\omega) - 1_{\kappa_0' \neq 0} \EE[\phi'] \right)^2\right] - 2\EE[\phi''(\omega)] f_b^{(1)} + q_c^{(2)} \right) \| \varpi \|^2 \;,
\end{align}
are shifted-and-scaled $\chi_k^2$ random variables with $k$ degrees of freedom as $\varpi \sim \mathcal{N}(0, I_k)$. Similar expressions hold for the training errors~\eqref{eq:training-err-l2} and~\eqref{eq:training-err-h1k}.

\subsubsection{Evaluating the remaining traces using operator-valued free probability}
\label{sec:linearization-main-text}

Here, we show how the traces of random matrices in the right-hand sides of~\eqref{eq:asymptotic_saddle_final_nonhat} can be evaluated \emph{without} Monte Carlo (MC) sampling of the random matrices $\Theta^\top \Theta \in \RR^{p \times p}$ and $D_i = \diag{\zeta_i} \in \RR^{p \times p}$, with $\Theta_{ij} \overset{\text{iid}}{\sim} {\cal N}(0,1/d)$ and $\zeta_i  \overset{\text{iid}}{\sim} {\cal N}(0, I_p)$, for large but finite $p$. Instead, they can be computed as solutions of another self-consistent fixed point system. Thus, equations~\eqref{eq:asymptotic_saddle_final_nonhat} and~\eqref{eq:asymptotic_saddle_final_hat} present a genuinely low-dimensional system of equations capturing the training and testing errors of the RF model.\\ 

We follow the ``lifting'' strategy of operator-valued free probability developed in~\cite{belinschi-mai-speicher:2017, helton-mai-speicher:2018}. The idea is to convert the rational functions of the elementary building blocks $\Theta^\top \Theta, D_1, \dots, D_k$ in the right-hand sides of~\eqref{eq:asymptotic_saddle_final_nonhat} to linear block-matrix pencils. We then compute the traces via the operator-valued Cauchy transform of each pencil, which requires solving a finite-dimensional fixed-point system for the so-called subordinator function. Our strategy differs from the approach of~\citet{adlam-pennington:2020} in a related analysis, where they linearize their random matrix functions to a Gaussian block matrix with free elements and solve the associated Dyson equation. This procedure is not possible here since the $D_i$'s are not free with respect to each other.\\

To keep our presentation self-contained, we defer to Appendix~\ref{app:free-prob-intro} our introduction to all necessary concepts mentioned above; this primer follows \citet{mingo-speicher:2017} and includes a number of toy examples for illustrative purposes. Instead, in this section we demonstrate our approach for a prototypical trace in the right-hand side of~\eqref{eq:asymptotic_saddle_final_nonhat} corresponding to the overlap parameter~$f_b^{(1)}$, namely 
\begin{align}
\label{eq:cauchy_trafo_example}
&\lim_{p \to \infty} \Tr_p \left[ A^{-1} D_1 \Theta^\top \Theta D_1 \right] \\
= &\lim_{p \to \infty} \frac{1}{p} \trace \left[ \left(\lambda I_p + \hat{\Sigma}_a^{(0)} \left(\kappa_*^2 I_p + \kappa_1^2 \Theta^\top \Theta \right) + \hat{\Sigma}_{c}^{(0)} \sum_{i \in[k]} D_i \left(\left(\kappa_*'\right)^2 I_p + \left( \kappa_1' \right)^2 \Theta^\top \Theta \right) D_i \right)^{-1} D_1 \Theta^\top \Theta D_1 \right]\\
=&  \; \varphi \left(\left(z_0 \cdot 1 + z_1 m + \sum_{i \in [k]} g_i \left(z_2 \cdot 1 + z_3 m \right) g_i \right)^{-1} g_1 m g_1 \right)\\
=& \;  - \varphi \left( \left( 0 \cdot 1 - (g_1 m g_1)^{-1} \left( z_0 \cdot 1 + z_1 m + \sum_{i \in [k]} g_i \left(z_2 \cdot 1 + z_3 m \right) g_i \right) \right)^{-1} \right) =: - \varphi \left( \left( 0 \cdot 1 - r  \right)^{-1} \right) = - G_r(0)\,.
\end{align}
Here, $\varphi$ is the limiting state function as in~\eqref{eq:limit-state-trace}, $m$ is an MP$\left(1/\gamma\right)$ element corresponding to the spectral limit of $\Theta^\top \Theta$, the elements $g_1, \dots, g_k$ correspond to the spectral limits of $D_1, \dots, D_k$, which are all free from~$m$, and $z_0=\lambda + \kappa_*^2 \hat{\Sigma}_a^{(0)}$, $z_1 = \kappa_1^2 \hat{\Sigma}_a^{(0)}$, $z_2 =  (\kappa_*')^2 \hat{\Sigma}_c^{(0)}$, and $z_3 = (\kappa_1')^2 \hat{\Sigma}_c^{(0)}$. Then, $G_r(z) \in \CC$ denotes the Cauchy transform of the rational function $r(m, g_1, \dots, g_k)$ at $z \in \CC$, which we evaluate at $z=0$.\\ 

Using the linearization algorithms in Appendices~\ref{sec:linearize-polynomial} and~\ref{sec:linearize-rational} as developed in~\cite{belinschi-mai-speicher:2017, helton-mai-speicher:2018}---which follow from the Schur complement formula---we now construct a block-matrix $\hat{r}$ that is affine-linear in all random elements and satisfies
\begin{align}
G_r(z) = G_{\hat{r}} \left(\begin{pmatrix}
\diag{z, 0, \dots, 0}
\end{pmatrix} \right)_{11}\,,
\end{align} 
where $G_{\hat{r}}(Z)$ denotes the operator-valued Cauchy transform of the block-matrix. Following the provided algorithms (see Appendix~\ref{app:linearization-examples} for a detailed demonstration for multiple toy examples), we obtain
\begin{align} 
&r = (g_1 m g_1)^{-1}\left(z_0 + z_1 m + \sum_{i\in[k]} g_i (z_2 + z_3 m) g_i\right) \notag \\
\xrightarrow{\text{lin}} 
& \quad \hat{r} = \begin{pmatrix} 
\begin{array}{c|cccccccccc}
0 & 0 & 0 & &  & \hdots &  &  0 & 1 & 0 & 0 \\
\hline
0 & 0 & 1 & 0 & g_1 & \hdots & 0 & g_k & 0 & 0& g_1 \\
0 &  &  &  &  & \hdots &  &  & 0  & -m  & 1 \\
0 &  &  &  &  & \hdots &  &  & g_1  & 1  & 0 \\
0 & z_0 + z_1 m & -1& & & & & \\
1 & -1 & 0 & & & & & \\
0 & & &  z_2 + z_3 m & -1 \\
g_1 & & & -1 & 0 \\
\vdots & & & & & \ddots \\
0 & & & & & & z_2 + z_3 m & -1 \\
g_k & & & & & & -1 & 0 
\end{array}
\end{pmatrix} = A \otimes 1 + B \otimes m + \sum_{i=1}^k C_i \otimes g_i\,,
\end{align}
with deterministic coefficient matrices $A = A(z_0, z_2),\; B  = B(z_1, z_3),\; C_1, \dots, C_k \in \CC^{l \times l}$ and a ``lifting dimension'' of $l = 6 + 2k$ in this particular example. Then, abbreviating $C =  C_1 \otimes g_1 + \dots + C_k \otimes g_k$ and observing that $A \otimes 1 + B \otimes m$ and $C$ are free, we can compute the operator-valued Cauchy transform of their sum $G_{\hat{r}}(Z) \in \CC^{l \times l}$ at any $Z \in \CC^{l \times l}$ from $G_{ A \otimes 1 + B \otimes m }$ and $G_{C}$ via a subordinator function $\mathfrak{s} \colon\; \CC^{l \times l} \to  \CC^{l \times l}$ as
\begin{align}
G_{\hat{r}}(Z) = G_{A \otimes 1 + B \otimes m}(\mathfrak{s}(Z))\,.
\label{eq:subord-main-text}
\end{align}
The subordinator $\mathfrak{s}(Z)$ in~\eqref{eq:subord-main-text} is found by solving the $l\times l$-dimensional fixed-point system
\begin{align}
\mathfrak{s}(Z) = H_{C}(H_{A \otimes 1 + B \otimes m}(\mathfrak{s}(Z)) + Z ) +Z\,,
\label{eq:fixed-point-subord}
\end{align}
cf.~\eqref{eq:operator-subord} and \eqref{eq:operator-subord-fixed-p} in Appendix~\ref{app:operator-free}, where $H(Z) := (G(Z))^{-1} - Z$. The fixed-point equation~\eqref{eq:fixed-point-subord} has a unique solution with $\text{Im}(\mathfrak{s}(Z)) \succ 0$ when $\text{Im}(Z) \succ 0$.\\

To evaluate the right-hand side of~\eqref{eq:fixed-point-subord}, we require access to $G_{ A \otimes 1 + B \otimes m }(Z) = G_{B \otimes m }(Z-A) $, which can be computed via a one-dimensional integral over the MP law
\begin{align}
G_{A \otimes 1 + B \otimes m}\left(Z \right) =  \int_\RR \left(Z - A - \lambda B \right)^{-1}\dd \mu(\lambda)\,,
\end{align}
as in~\eqref{eq:lifted-cauchy-int}. The integral is straightforward and efficient to evaluate via, e.g.,\ Gauss--Legendre quadrature for the compactly supported measure of the MP law $\mu$ in~\eqref{eq:mp-law}. We also require evaluations of
\begin{align}
G_{C}(Z) = \frac{1}{(2 \pi)^{k/2}} \int_{\RR^k} \left(Z- \left(w_1 C_1  + \dots + w_k C_k \right) \right)^{-1} \exp \left\{-\frac{1}{2} \norm{w}^2 \right\} \dd^k w \;,
\end{align}
which involves an expectation with respect to the $k$-dimensional standard normal distribution. The number of MC samples needed to resolve this expectation increases with $k$ though the computation is easily parallelized. Computing $(Z-C)^{-1}$ presents an additional challenge as we must invert an $l \times l$-dimensional matrix, where $l$ scales with $k$.  One can exploit potential structure in matrix factors $C_1$, $C_2$, \ldots, $C_k$ to lower the computational cost. For instance, in the example above we recognize 
$$
w_1 C_1 +\ldots + w_k C_k =  e_2 \otimes \begin{pmatrix} 0_{3 \times 1} \\ 0 \\ w_1 \\ \vdots \\ 0 \\ w_k  \\ 0 \\ 0 \\ w_1 \end{pmatrix} + e_3 \otimes \begin{pmatrix} 0_{(3+2k) \times 1} \\ w_1 \\ 0 \\ 0 \end{pmatrix} + \begin{pmatrix} 0_{6 \times 1} \\ 0 \\ g_1 \\ \vdots \\ 0 \\ g_k \end{pmatrix} \otimes e_1 =: U_3 V_3^\top \;,
$$ 
is rank-three with low-rank factors $U_3, V_3 \in \RR^{l \times 3}$, so 
$
(Z-C)^{-1} = Z^{-1} - Z^{-1} U_3 (I_3 + V_3^\top Z^{-1} U_3)^{-1} V_3^\top Z^{-1}
$
by the Woodbury matrix identity. The advantage to this representation is that $Z^{-1}$ only needs to be computed once across all MC samples of~$w$, and evaluating each sample involves only the inverse of a $3\times 3$ matrix, which can even be analytically computed via the method of cofactors. We validate our theoretical expression in the right-hand side of~\eqref{eq:cauchy_trafo_example} against MC evaluations of the left-hand side in Table~\ref{tab:cauchy_trafo}.

\begin{table}
\centering
\caption{Comparison of MC estimator of the left-hand side of~\eqref{eq:cauchy_trafo_example} for $k=1$ evaluated over $1000$ samples in finite dimensions against the operator-valued Cauchy transform approach for the right-hand side of~\eqref{eq:cauchy_trafo_example}. We evaluate $G_C(Z)$ in the subordinator equation~\eqref{eq:fixed-point-subord} using a degree~$150$ Gauss--Hermite quadrature. Other parameters: $n/d=2.345$, $p/n=0.5$, $\hat{\Sigma}_a^{(0)} = 0.2$, and $\hat{\Sigma}_c^{(0)} = 0.4$. For the Hermite coefficients corresponding to the different choices of $\sigma$, see Table~\ref{tab:activation}.}
\renewcommand{\arraystretch}{1.3}
\begin{tabular}{|c|ccc|} 
\hline
 & Linear pencil method & MC ($d=100$) & MC ($d=1000$) \\ \hline
$\sigma=\mathrm{ReLU}$ & $3.7750$ & $3.7594 \pm 0.0055$ &  $3.7750 \pm 0.0017$\\
$\sigma=\mathrm{erf}$ & $6.7234$ & $6.6792 \pm 0.0120$ &  $6.7209 \pm 0.0035$\\
\hline
\end{tabular}
\label{tab:cauchy_trafo}
\end{table}

\subsubsection{Verification of the theory through comparison with Monte Carlo simulations}
 
In summary, after fixing the Sobolev training hyperparameters ($\gamma$, $\alpha$, $\lambda$, $k$, \ldots), we solve equations~\eqref{eq:asymptotic_saddle_final_nonhat} and~\eqref{eq:asymptotic_saddle_final_hat} to determine the overlap parameters. These overlap parameters then allow us to theoretically predict the distributions and moments of the generalization errors, via~\eqref{eq:l2_gen_error} and~\eqref{eq:sobo_gen_error}, and the training errors, via~\eqref{eq:training-err-l2} through~\eqref{eq:training-err-reg}. \\

Algorithmically, this proceeds as follows:
\begin{enumerate}
\item  Solve the closed system for the four scalar parameters $\Sigma_a^{(0)}, \Sigma_c^{(0)}, \hat{\Sigma}_a^{(0)}, \hat{\Sigma}_c^{(0)}$. 

We solve this system using the root-finding algorithm ``excitingmixing'' in SciPy~\cite{virtanen-etal:2020}, which implements Newton's method with a tuned diagonal Jacobian approximation. We terminate the iterations once the relative tolerance of the residual is less than $10^{-2}$. Within each of these `outer' iterations, we evaluate the random matrix traces in the right-hand sides of~\eqref{eq:asymptotic_saddle_final_nonhat} using the linearization method outlined in Subsection~\ref{sec:linearization-main-text}. This involves solving another fixed point equation~\eqref{eq:fixed-point-subord} for each trace, which we achieve using damped fixed point iterations with damping factor $\gamma = 0.2$. These ``inner'' iterations  are terminated once the Frobenius norm between successive iterations of the subordinator is less than $10^{-8}$.

\item Compute $\hat{f}_a^{(0)}$ and $\hat{f}_b^{(1)}$, then compute $f_a^{(0)}$ and $f_b^{(1)}$. 

\item Solve the linear system of equations for $q_a^{(0)}, q_a^{(2)}, q_c^{(0)}, q_c^{(2)}, \hat{q}_a^{(0)}, \hat{q}_a^{(2)}, \hat{q}_c^{(0)}, \hat{q}_c^{(2)}$. 

We directly invert the $8 \times 8$ linear system to produce these $q$ overlap parameters. Note that since the noise covariances $C_{\eta,11}$ and $\trace C_{\eta,2:k+1,2:k+1}$ appear in the right-hand side of this system, we immediately obtain the overlap parameters for \emph{all} noise levels.

\item Assemble the $\varpi$-dependent overlap parameters~\eqref{eq:varpi-scaling} and compute the generalization errors (\eqref{eq:l2_gen_error} and~\eqref{eq:sobo_gen_error}) and training errors (\eqref{eq:training-err-l2} through~\eqref{eq:training-err-reg}) either via sampling $\varpi \sim {\cal N}(0, I_k)$, or by analytically computing moments of the $\chi^2_k$-distributed errors.
\end{enumerate}

Figure~\ref{fig:intro} (right) validates the theoretically predicted error curves against error curves obtained via MC simulations of~\eqref{eq:training-problem}. 
For a single realization of~$\varpi$, simulating a single error curve via MC over $71$ equispaced samples of $p/n \in [0.01, 4.0]$ on our machine with thirty-two 2.60GHz Intel Xeon CPUs and 300 Gb of RAM requires $\sim \mbox{3:47}$ minutes for $d=200$, $\sim \mbox{7:13}$ minutes for $d=500$, $\sim \mbox{18:03}$ minutes for $d=1000$, and $\sim \mbox{68:14}$ minutes for $d=2000$. The main computational bottleneck stems from inverting the dense $p \times p$ matrix in~\eqref{eq:K-matrix-def} whose size grows with~$d$, though we did not explore any preconditioning strategies with iterative solvers. For the same range of parameters, our theoretical predictions using the algorithm above takes $\sim \mbox{34:04}$ minutes to compute and yields the complete error distributions as a function of~$\varpi$. Evidently, extensive computing resources would be required in order to reproduce the parameter scans in Figure~\ref{fig:2d-errors-arctan-plus-reci-cosh-relu} below using MC simulations, particularly when resolving large $\alpha,\gamma$. \\  

Since the theoretical predictions correspond to the proportional asymptotics limit, the only numerical errors originate from the fixed point solves and the operator-valued Cauchy transform evaluations.  In contrast, the MC simulations exhibit finite-size errors from finite $d,n,p$, and statistical errors from finite realizations of~$\Theta$, $V_k$, $\eta$, and $\varpi$.  
Figure~\ref{fig:varpi-dependence} compares the marginal distributions at~$p/n=0.5$ of the $L^2$ and $H^1_k$ generalization error, as well as the marginal distributions of various observable overlap parameters, obtained from both theory and MC simulations. 
Although the MC simulations exhibit finite size effects we observe clear asymptotic convergence as $d \uparrow \infty$ to our theoretical predictions, thus validating our theoretical calculations.\\

\begin{figure}
\centering
\includegraphics[width = \textwidth]{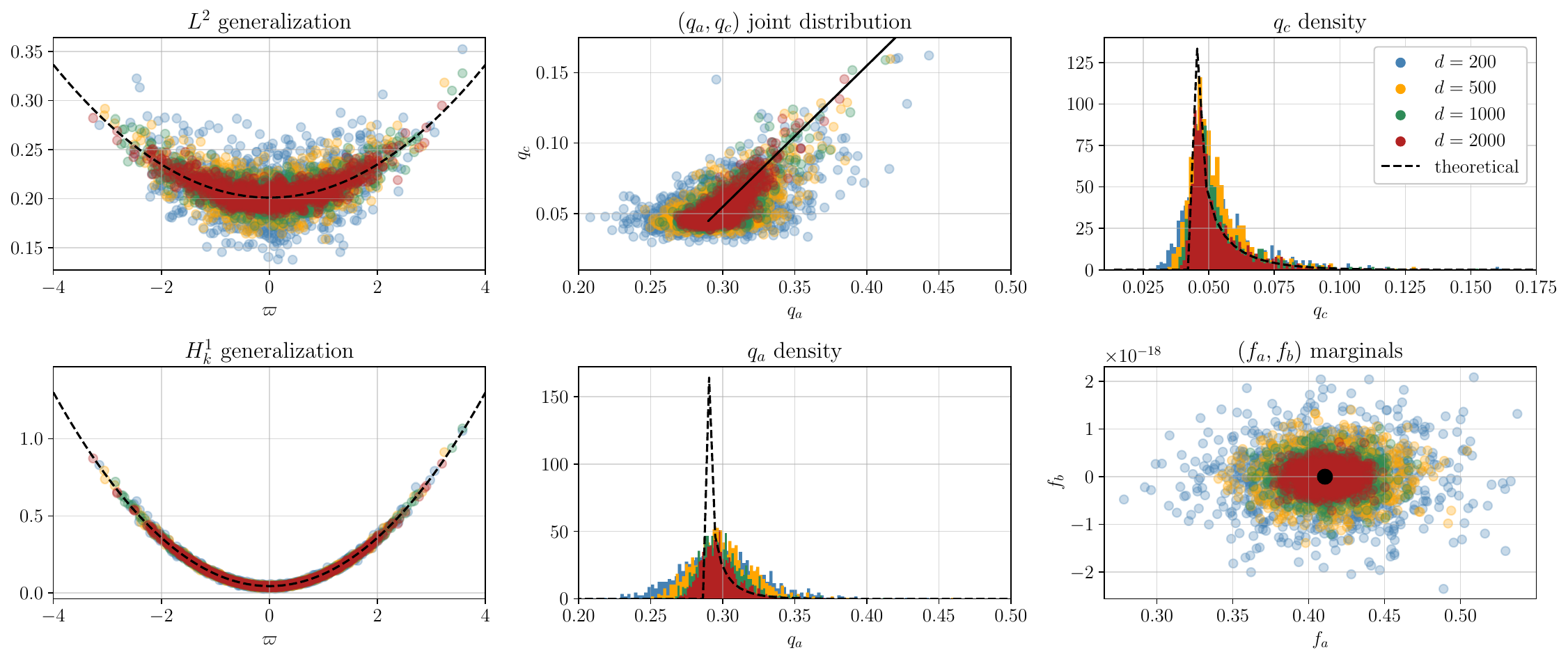}
\caption{Comparison of MC samples of~\eqref{eq:training-problem} to evaluate~\eqref{eq:gen-error-def} and~\eqref{eq:overlap-def} at $p/n=0.5$ and $n/d=2.345$ in finite dimensions $d \in \{200, 500, 1000, 2000\}$, against theoretical predictions~\eqref{eq:l2_gen_error_varpi}, \eqref{eq:sobo_gen_error_varpi}, and~\eqref{eq:asymptotic_saddle_final_hat}. Other parameters: $\sigma=\mathrm{erf}$, $\phi=\arctan$, $k=1$.  Left column: distribution of $L^2$ and $H^1_k$ generalization errors as a function of~$\varpi = V_k^\top \theta_0$. Center and right columns: marginal distributions of the $(f_a, f_b)$ and $(q_a, q_c)$ overlap parameters.}
\label{fig:varpi-dependence}
\end{figure}


\section{Predictions of the theory}
\label{sec:results}


\subsection{Expected generalization error landscapes as a function of $p/d$ and $n/d$\label{sec:landscapes}}

\begin{figure}
\centering
\includegraphics[width = \textwidth]{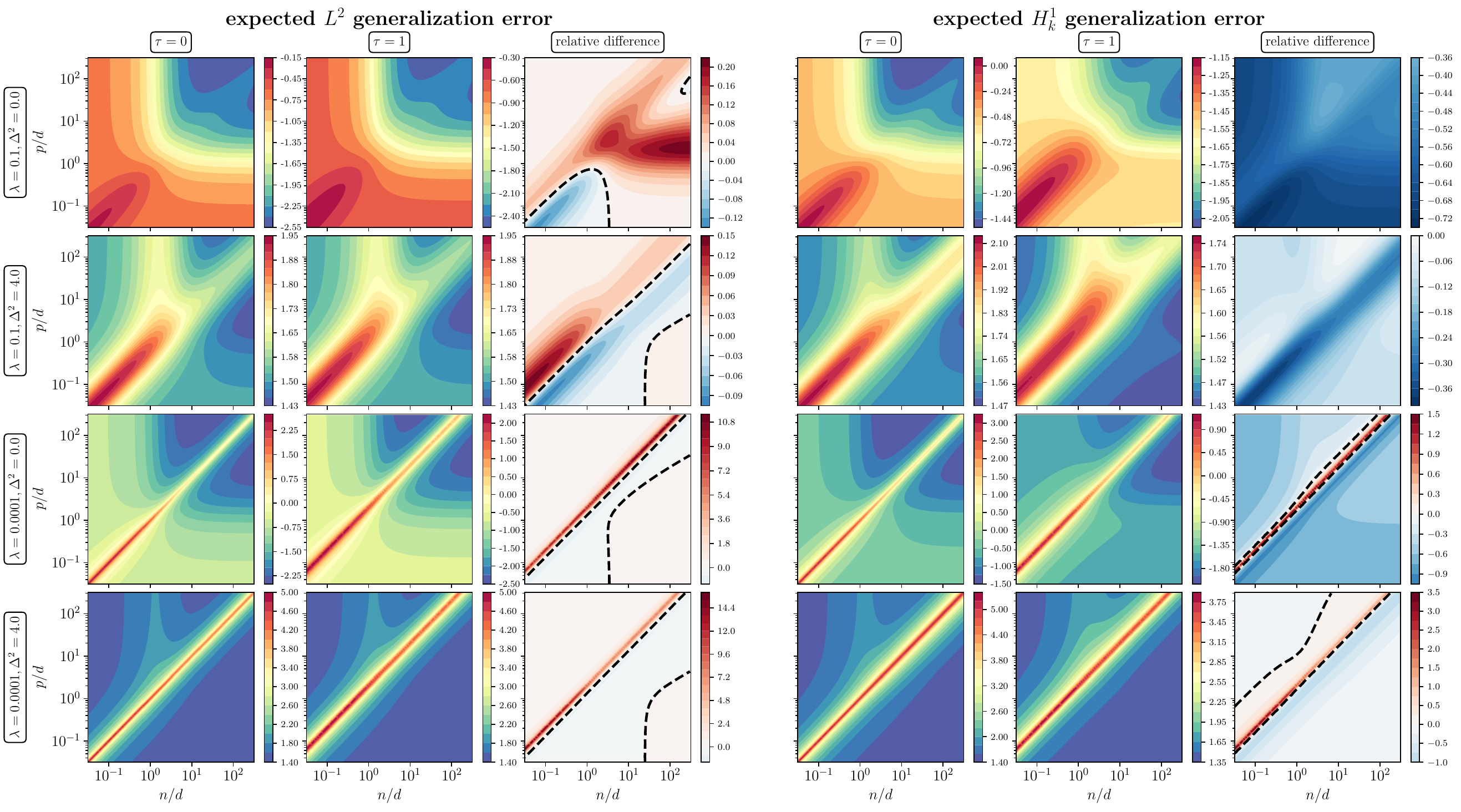}
\caption{Comparison of expected $L^2$ generalization error (left three columns) and $H^1_k$ generalization error (right three columns) of $L^2$ training ($\tau = 0$) and Sobolev training ($\tau = 1$) for $k = 1$ gradient projections as a function of the  number of training samples $n$ and network features $p$, normalized by the dimension $d$. Rows correspond to different regularization strengths $\lambda \in \{10^{-1}, 10^{-4} \}$ and observational noise levels $\Delta^2 \in \{0, 4\}$ for $y_i$ and $V_k^\top y_i'$. Other parameters: $\sigma = \text{ReLU}$, $\phi = \arctan + 1 / \cosh$.
All plots use the expected errors  $\EE [\eps_{\text{gen}}^{L^2} ]$ and $\EE [\eps_{\text{gen}}^{H^1_k} ]$ over the alignment $\varpi \sim {\cal N}(0,1)$ in the limit~\eqref{eq:prop-asymp-def}, as predicted from the theory presented in Section~\ref{sec:main-theory-res}.
The errors themselves are shown on a logarithmic color scale while their relative difference is shown on a linear color scale that is symmetric around zero. Negative relative differences, shown in blue in the third and sixth column, indicate regimes (delimited by the black dashed lines) where Sobolev training outperforms $L^2$ training.}
\label{fig:2d-errors-arctan-plus-reci-cosh-relu}
\end{figure}

To gain a broad overview of Sobolev training, in this section we follow the analysis of~\citet{dascoli-sagun-biroli:2020} and 
investigate two-dimensional ``error landscapes'' as functions of $n/d$ and $p/d$.  The one-dimensional error curves shown in Figure~\ref{fig:intro} in the introduction, and the following subsections below, correspond to vertical slices of such two-dimensional landscapes, i.e., varying $p/n$ at fixed $n/d$, modulo rescaling the axes. The work~\cite{dascoli-sagun-biroli:2020} generates these landscapes for $L^2$ training ($\tau = 0$) of RF models, using the theoretical results of~\cite{gerace-loureiro-krzakala-etal:2021,mei-montanari:2022}, and demonstrates that these capture the same behavior as fully-connected three-layer neural networks trained via stochastic gradient descent. They relate the error landscapes to spectral properties of $K$, i.e., the eigenvalues of~\eqref{eq:K-matrix-def} with $\tau =0$. We reveal similar insights here for Sobolev training $(\tau = 1)$, emphasizing the impact of gradient data on generalization. For the purpose of this comparison, we assume gradient data are obtained ``for free'' and compare $L^2$ and Sobolev training for the same~$n$; we provide comparison which normalizes against different costs for obtaining gradient data in Section~\ref{sec:varying-k}. For simplicity, we focus on $\tau = 0$ versus $\tau = 1$ for $k = 1$. We vary the remaining hyperparameters between large and small regularization $\lambda \in \{10^{-1}, 10^{-4}\}$, large noise vs.\ noiseless training $\Delta^2 \in \{4, 0\}$ with $C_\eta = \Delta^2 \cdot I_{k+1}$ in~\eqref{eq:p-data-additive-gauss}, and different activation functions $\sigma$ and ridge functions $\phi$.\\

Figure~\ref{fig:2d-errors-arctan-plus-reci-cosh-relu} shows the results for the prototypical activation function $\sigma = \text{ReLU}$
and for $\phi = \arctan + 1 / \cosh$. As discussed in Section~\ref{sec:main-theory-res}, the precise form of these functions is not important in the limit~\eqref{eq:prop-asymp-def}, but it does matter which of their low-order Hermite coefficients are nonzero. In this sense, $\sigma = \text{ReLU}$ and $\phi = \arctan + 1 / \cosh$ correspond to the generic case where $\kappa_0,\kappa_1,\kappa_*,\kappa_0',\kappa_1',\kappa_*'$ in~\eqref{eq:hermite-coeff-def}---and the corresponding coefficients for $\phi$---are all nonzero, so that both functions and their derivatives behave as noisy affine-linear functions with nonzero slope and offset. Additional results are provided in Appendix~\ref{app:landscapes} for even or odd $\phi$ and $\sigma$, in which case either the data/network function or gradient has zero slope or offset after linearization.\\

The left three columns of Figure~\ref{fig:2d-errors-arctan-plus-reci-cosh-relu} compare the expected $L^2$ generalization error $\EE [\eps_{\text{gen}}^{L^2} ]$ for different~$\lambda$ and~$\Delta^2$. Broadly speaking, incorporating gradient information into the training loss does not ``topologically'' alter the $L^2$ error landscape. However, a key difference is a shift in the interpolation peak along $p=n$ for $\tau = 0$ to $p = (k+1)n$ for $\tau = 1$. Effectively, gradient observations are treated as additional, independent data. Consequently, as shown by the relative difference plots in the third column in Figure~\ref{fig:2d-errors-arctan-plus-reci-cosh-relu}, the $L^2$ generalization error along the diagonal $p = n$ is larger for $\tau = 0$ than with Sobolev training, whereas the converse is true along the super-diagonal $p = (k+1)n$.
Otherwise, however, the expected $L^2$ generalization error landscapes obtained from Sobolev training demonstrate the same qualitative behavior documented in~\cite{dascoli-sagun-biroli:2020}: a large $\lambda$ regularizes the ``nonlinear peak'' at $p = n$ or $p = (k+1)n$, and there is an additional ``linear'' peak along $n = d$ which is implicitly regularized by the nonlinearity of the activation function. Generically, vertical slices exhibit the phenomenon of double descent~\cite{belkin-hsu-ma-etal:2019} while for certain regularization and signal to noise ratios the horizontal slices can demonstrate triple descent~\cite{dascoli-sagun-biroli:2020}.\\

In general, Figure~\ref{fig:2d-errors-arctan-plus-reci-cosh-relu} demonstrates the surprising result where providing additional gradient data does not uniformly improve (nor uniformly worsen) the $L^2$ generalization performance of RF models.  For the present case of $\sigma = \text{ReLU}$ and $\phi = \arctan + 1 / \cosh$, Sobolev training is advantageous for small networks relative the size of the training data, i.e., for under-parameterized models. This conclusion differs from the numerical results of~\citet[Section 4.1]{czarnecki-osindero-jaderberg:2017} where over-parameterized models trained with gradient data outperform the same models trained using only function data; however, since they only consider low-dimensional $(d=2)$ problems, their setup is far from the asymptotic regime we consider here.\\

\begin{figure}
\centering
\includegraphics[width = \textwidth]{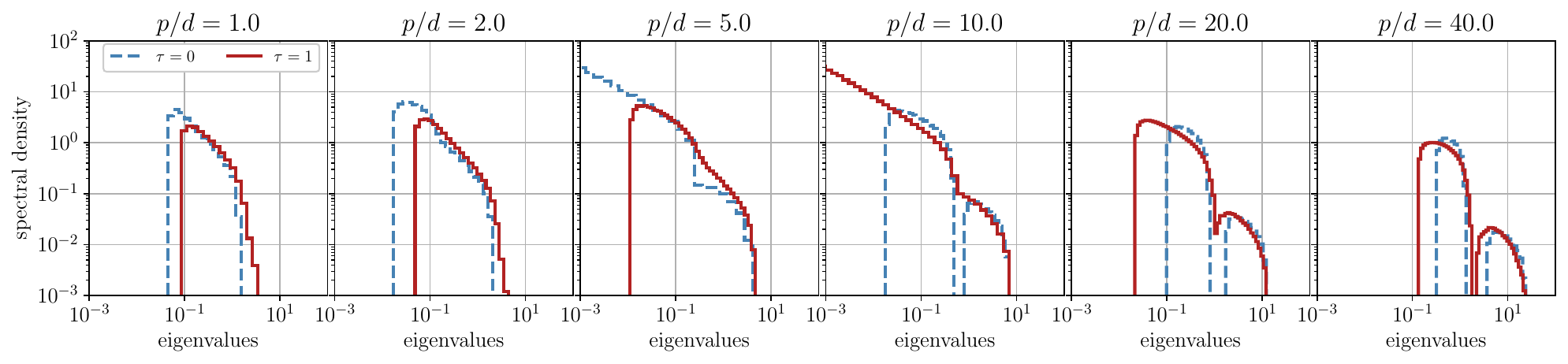}
\caption{Continuous part of the empirical spectral densities for one sample of the feature matrix $K$, defined in~\eqref{eq:K-matrix-def}, at different numbers of features $p/d$. We compare standard $L^2$ training ($\tau = 0$, dashed blue lines) to Sobolev training ($\tau = 1$, $k = 1$, solid red lines). Other parameters are: $n/d = 5$, $d = 1000$, $\sigma = \text{ReLU}$. The spectral gap to $0$ closes at $p = n$ for $L^2$-training and at $p = 2n$ for Sobolev training with $k = 1$.}
\label{fig:spectral-density}
\end{figure}

In contrast, the expected subspace gradient generalization error $\EE [\eps_{\text{gen}}^{H^1_k}]$ depends more strongly on $\tau = 0$ versus $\tau = 1$, as shown in the last three columns of Figure~\ref{fig:2d-errors-arctan-plus-reci-cosh-relu}. When no gradient data are provided ($\tau = 0$), the gradient generalization error is strongly correlated with the $L^2$ generalization error, whereas the two landscapes differ for $\tau = 1$ though less so for large noise levels $\Delta$. Similar to the $L^2$ generalization error, the gradient error also exhibits a peak along the interpolation threshold at $p=n$ for $\tau=0$ and $p=(k+1)n$ for $\tau=1$. In addition to the possibility of ``triple descent'' along horizontal slices as originally documented in~\cite{dascoli-sagun-biroli:2020}, we also observe ``triple descent'' along certain \emph{vertical} slices, i.e., also as a function of network size $p/d$ at fixed training set size $n/d$.\\

Notably, the rightmost columns of Figure~\ref{fig:2d-errors-arctan-plus-reci-cosh-relu} demonstrate an unexpected result: providing gradient data to the training set does \emph{not} uniformly improve the ability of RF models to predict gradients at new inputs. In other words, there are regimes in which one would prefer to disregard the provided gradient training data, rather than assimilating this extra information.
In Figure~\ref{fig:2d-errors-arctan-plus-reci-cosh-relu}, this occurs at small $\lambda$ and in the slightly over-parameterized regime due to the shifted interpolation peak. As the normalized number of features $p/d$ is increased further at fixed normalized sample size~$n/d$, Sobolev training outperforms $L^2$ training at gradient prediction in the massively overparameterized limit in the present case. The intuitive reason for the uniform improvement in gradient prediction of Sobolev training over $L^2$ training at large $\lambda$ is that for $\tau > 0$, the network gradient always correctly represents the gradient mean via the overlap parameter $s_b$, and large $\lambda$ regularizes the double descent peak.\\

As shown in Appendix~\ref{app:landscapes}, for instance for $\sigma = \text{ReLU}$, $\phi = \arctan$ (Figure~\ref{fig:2d-errors-arctan-relu}), independently of the regularization strength, and whether or not the samples are corrupted by additive noise, $L^2$ training in fact outperforms Sobolev training for gradient prediction at massive overparameterization and large $n/d$. This result stands in contrast with the existing literature on Sobolev training~\cite{czarnecki-osindero-jaderberg:2017, o-leary-roseberry-chen-villa-etal:2024} in which massively overparameterized neural networks benefit from incorporating gradient information. One reason for this discrepancy may be that in many scientific applications, observational data is actually sparse, e.g., due to expensive simulations required for each training sample, and hence $n/d \ll 1$. In addition to this, within the RF model considered here and for the odd $\phi$, the projected gradient data effectively behaves like a constant function (in $\langle \theta_0, x \rangle$) plus independent noise in the limit~\eqref{eq:prop-asymp-def}, and so it is not surprising that incorporating this data into the training can hurt generalization performance. Fundamentally, this behavior results from the lack of ``feature learning'' capabilities of RF models in the proportional asymptotics regime, and the corresponding choice of an uninformed subspace for the gradient projections.\\

The authors of \cite{dascoli-sagun-biroli:2020} also connect the two-dimensional generalization error landscape for $L^2$ training to the spectral density of the feature matrix $K$ in~\eqref{eq:K-matrix-def} with $\tau=0$, which can be computed analytically using tools from~\cite{pennington-worah:2019}.  
We perform the same analysis here, although we do not analyze the spectral density of $K$ theoretically, but instead we show the results of sampling $K$ in large but finite dimension $d = 1000$ for different $p/d$ at fixed $n/d = 5$.
Figure~\ref{fig:spectral-density} shows results for $k = 1$, $\tau = 0$ vs.\ $\tau = 1$, and $\sigma = \text{ReLU}$, and the progression from left to right corresponds to a vertical slice along Figure~\ref{fig:2d-errors-arctan-plus-reci-cosh-relu}. The key observation is that peaks in the generalization error landscape correspond to the ill-conditioning of~$K$, i.e., when the spectral gap of the bulk approaches $0$. Figure~\ref{fig:spectral-density} shows that the inclusion of gradient data prevents the spectral gap from closing at $p = n$, but shifts this closure to $p = (k+1)n$ instead.
For large~$p/d$ and other activation functions $\sigma \in \{ \text{SiLU}, \text{erf} \}$ (cf.\ Appendix~\ref{app:landscapes}, Figure~\ref{fig:spectral-density-app}), we observe that the bulk typically splits into three components for Sobolev training with $k =1$, which we attribute to the additional Hadamard product term in~\eqref{eq:K-matrix-def}, as opposed to only two components for standard $L^2$ training at $\tau = 0$. For ReLU specifically in Figure~\ref{fig:spectral-density}, we only see two bulk components: this is presumably due to the degeneracy of its Hermite coefficients, cf.\ Table~\ref{tab:activation}, making two of the bulk components coincide. 
We leave a more detailed spectral analysis of the feature matrix, which we believe is possible using the techniques from Section~\ref{sec:main-theory-res},
as well as more realistic models that include feature learning, to future work.\\

In total, we have shown in this section that we can quickly perform parameter scans using the theoretical predictions from Section~\ref{sec:main-theory-res} without statistical errors or sampling. The results presented here show that effect of gradients is more subtle than naively expected: even if gradients come ``for free'' and there is no observational noise, one should not always include them in the training loss. While the main effect is due to a shift of the interpolation threshold, only the full fixed point solves give a complete, quantitative description for the considered model. Given this overview, the following subsections will now discuss a few specific questions in more detail.\\

\subsection{Impact of observational noise on overfitting}
\label{sec:noise-influence}

\begin{figure}
            \centering
            \begin{tabular}{c c c}
            \includegraphics[width=0.45\linewidth]{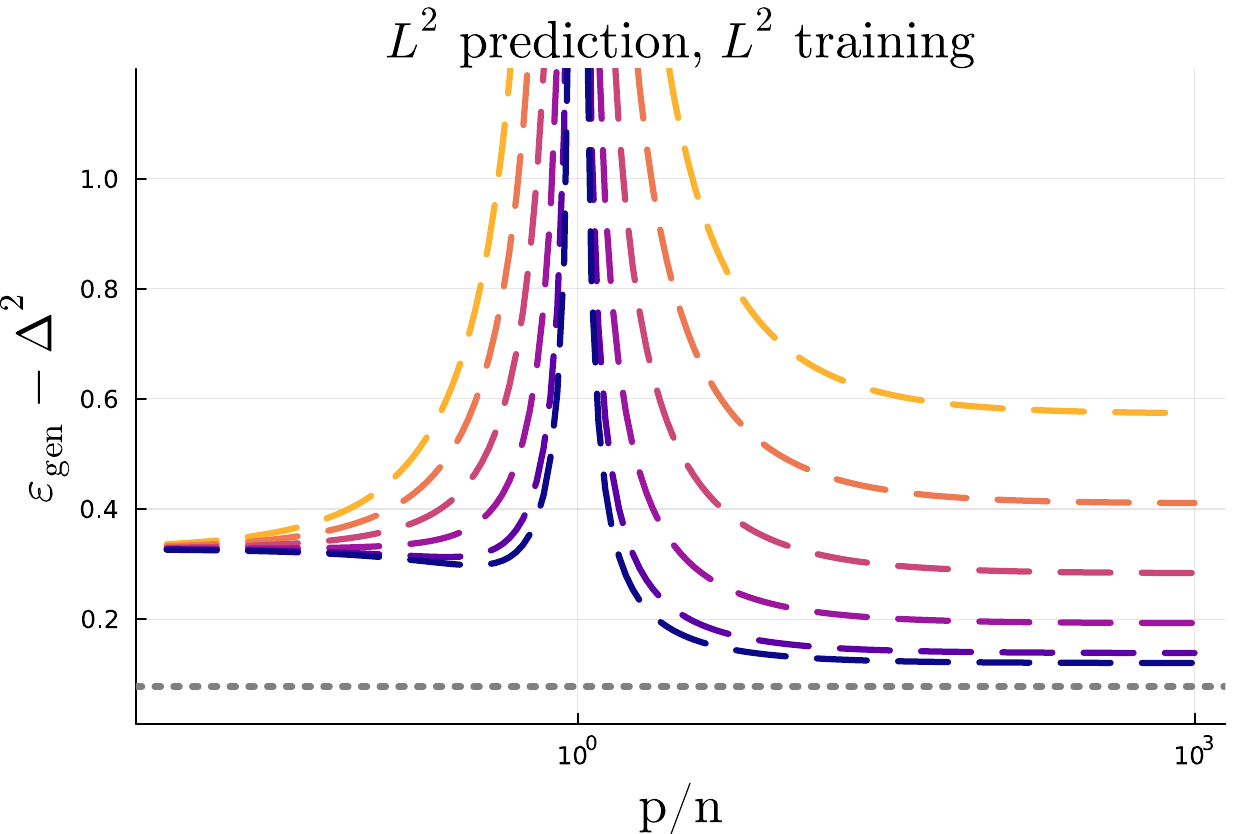}
            & & \includegraphics[width=0.45\linewidth]{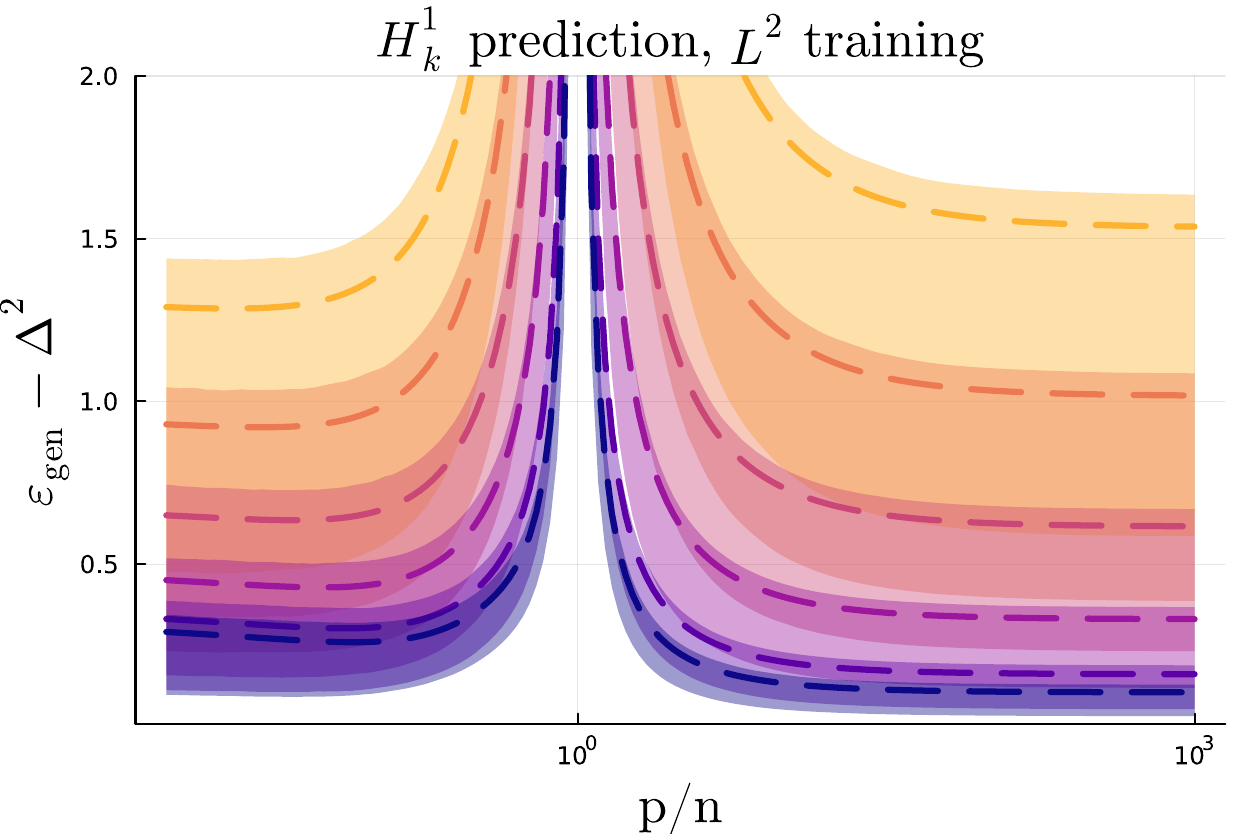} \\
            \includegraphics[width=0.45\linewidth]{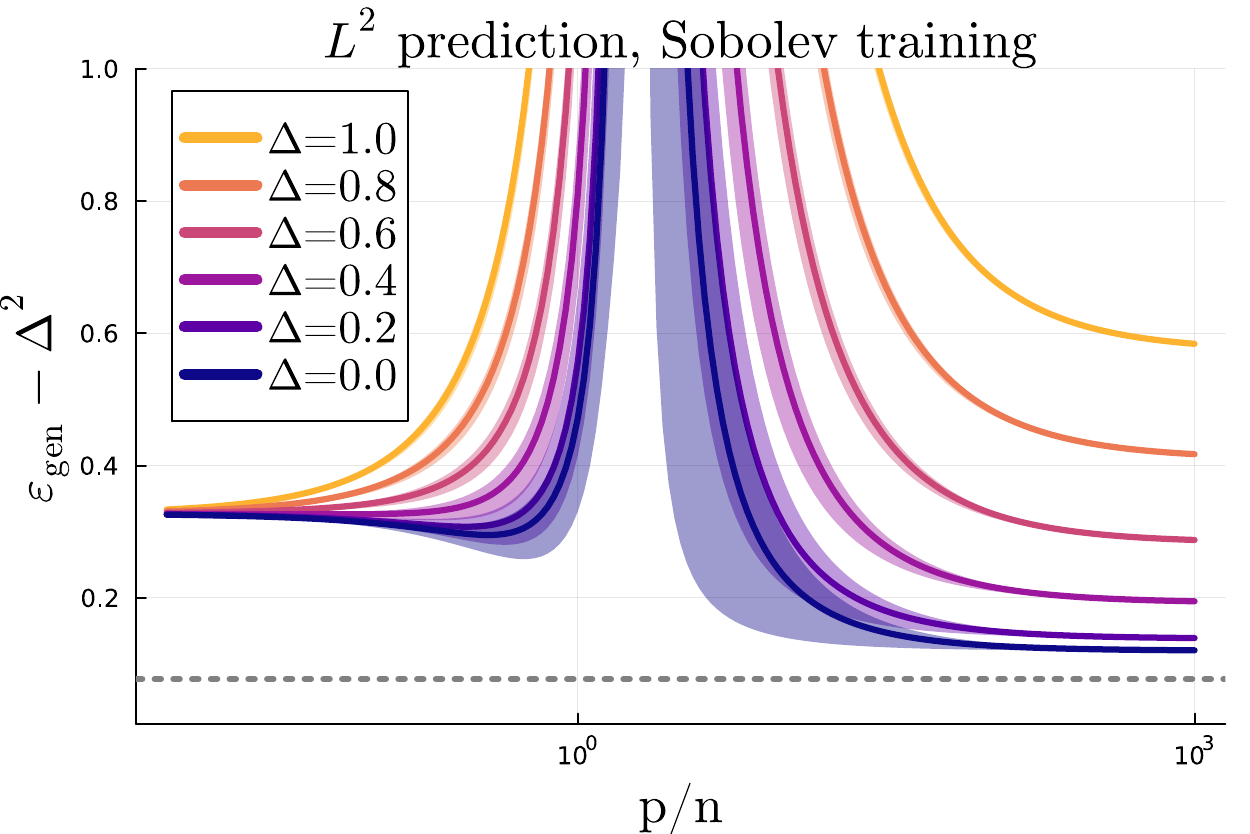}
            & & \includegraphics[width=0.45\linewidth]{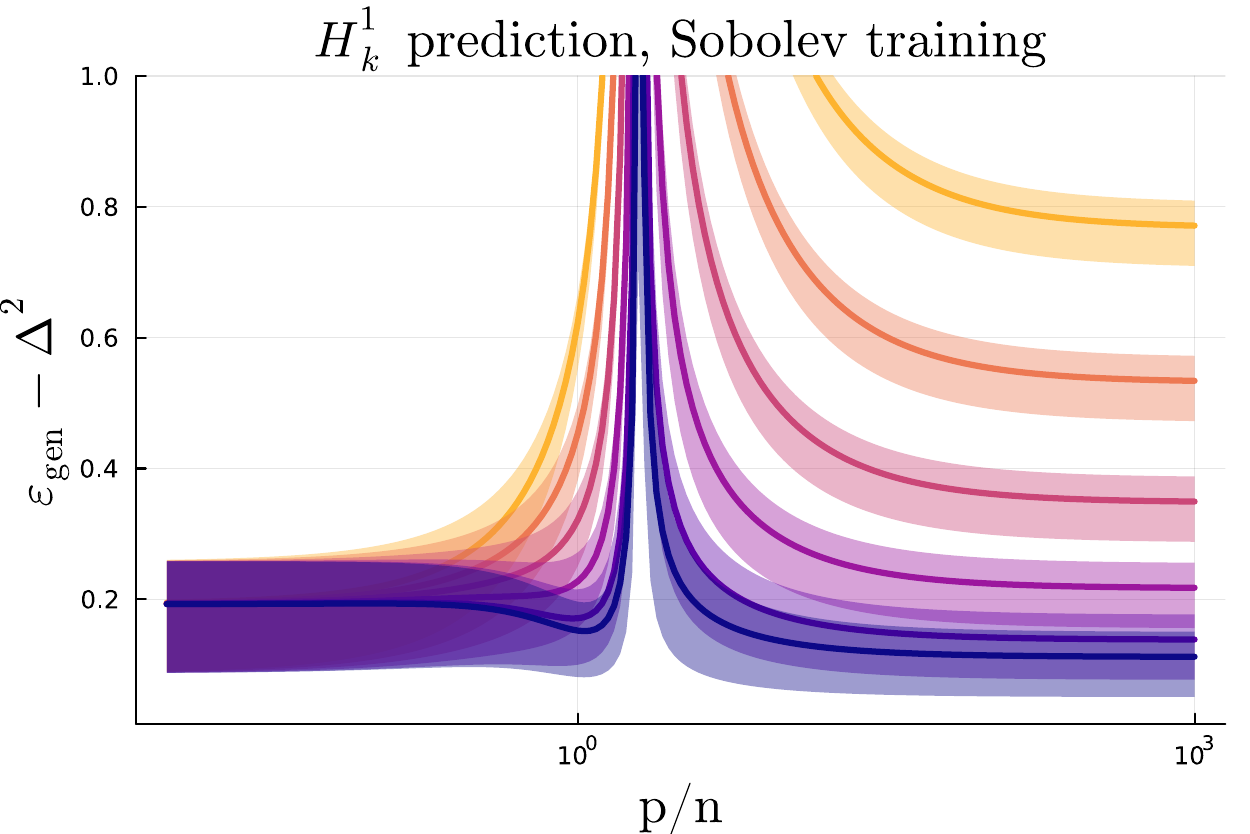} \\
            \end{tabular}
            \caption{Error against ground truth achieved by $L^2$ training (first row) and Sobolev training (second row) on unseen test cases given a range of noise levels in the training data: $\Delta\in\{0.0,0.2,0.4,0.6,0.8,1.0\}$. Left column: the $L^2$ error of network predictions against $\phi(\theta_0^\top x)$, averaged over $x$. A lower bound to accuracy is given by the gray dotted line which marks the magnitude of the nonlinear component of $\phi$. The distributions predicted by Sobolev training are induced by $\varpi=V_k^\top \theta_0$, and ribbons shade between the $20\%$ and $80\%$ quantiles. Right column: the $H_k^1$ error found by averaging the squared difference between the network gradient predictions and $V_k^\top \theta_0 \phi'(\theta_0^\top x)$ over $x$ when $k=1$. The ribbons cover between the $50\%$ and $75\%$ of the $\chi^2$ distribution resulting from the random gradient projection. Parameters: $n/d=2.345$, $\lambda=10^{-6}$, $\phi(\omega) = \omega / 2 - \exp \{-\omega^2/2\}$, and $\sigma=\textrm{SiLU}$.}
       \label{fig:diff_noise} 
       \end{figure}

A puzzling characteristic of deep neural networks is their ability to generalize even when provided with noisy training data~\cite{bartlett:2021} and no explicit regularization. Their success contravenes traditional statistical wisdom as these networks have far more parameters than training samples and consequently achieve near-zero training error since typically no explicit regularization is enforced. In essence, they ``memorize'' the noise in the data. This phenomenon is referred to as \emph{benign overfitting} and has been validated theoretically for simpler models such as linear regression by~\citet{bartlett:2020} and for RF models by~\citet{mei-montanari:2022}. Both works show instances in which overparameterization is necessary to achieve the best possible prediction errors within their respective model classes, even when there is label noise.\\

In this section, we explore whether benign overfitting occurs for Sobolev training by studying~\eqref{eq:training-objective} with~$\lambda \approx 0$. We use $\sigma = \text{SiLU}$, $\phi(\omega) = \omega / 2 - \exp \{-\omega^2/2\}$ here, such that all relevant Hermite coefficients are nonzero (see Appendix~\ref{app:noise} for other $\sigma$ and $\phi$). For simplicity, we again consider Gaussian additive noise $\eta \sim \mathcal{N}(0, \Delta^2)$ applied to $y$ and $\eta' \sim \mathcal{N}(0, \Delta^2 I_k)$ applied to $V_k^\top y'$. Surprisingly, we can assume that $\eta$ and $\eta'$ are independent without loss of generality as correlations between the function and gradient noises do not impact generalization. This insensitivity follows from our theoretical predictions: the expressions for the errors~\eqref{eq:l2_gen_error} and~\eqref{eq:sobo_gen_error}, as well as the overlap parameters $q_a^{(0)}$, $q_c^{(0)}$, $\hat{q}_a^{(0)}$, $\hat{q}_c^{(0)}$, only depend on the marginal noise variances. As a corollary, the overlap parameters which are independent of the noise need to be computed only once for each $\alpha$ and $\gamma$. Then, the generalization errors can be computed for all alignments $\varpi$ and noise strengths $\Delta$ for no additional computational cost, cf. Remark~\ref{rem:saddlepoint-observations}.\\

Figure~\ref{fig:diff_noise} compares the impact of noise levels $\Delta$ on prediction accuracy for $L^2$ and Sobolev training objectives. For a given overparameterization level $p/n$, each curve quantifies the squared error of the network against the noiseless ground truth---i.e.,~\eqref{eq:l2_gen_error} and~\eqref{eq:sobo_gen_error}  less $\Delta^2$ and $k \Delta^2$, respectively. Thus, any error which exceeds the noiseless case can be attributed to the noisy training data rather than an uncertain observation model. As $\lambda \approx 0$ in our setup, each network essentially ``overfits'' the noisy function (as well as the noisy projected gradient data, if available) past the interpolation threshold $p=(k+1)n$.\\

The top-left subfigure in Figure~\ref{fig:diff_noise} shows the $L^2$ generalization error under $L^2$ training. Unsurprisingly, increasing data noise decreases prediction accuracy at a given $p/n$. 
However, as already documented in~\cite{mei-montanari:2022}, we find that RF models exhibit benign overfitting, and the lowest error is achieved by overparameterized models when the noise level $\Delta$ is not too large. Additionally, we demonstrate the same behavior for the $L^2$ generalization under Sobolev training (bottom left), even though the model must additionally memorize the noise in the gradient observations. The dotted gray lines in both subfigures correspond to the approximation error $\EE[\phi(\xi)^2]-\EE[\phi(\xi)]^2 - \EE[\xi\phi(\xi)]^2$ of the best linear approximation to~$\phi$, where $\xi\sim\mathcal{N}(0,1)$ (note the relation to the Hermite coefficients of $\phi$, cf.~\eqref{eq:hermite-coeff-def}). By the Gaussian equivalence theorem, these lines lower bound the achievable accuracy of any RF model in the proportional asymptotics limit~\citep{mei-montanari:2022, ba-erdogdu-suzuki-etal:2022}.\\

$H^1_k$ generalization exhibits a greater difference between $L^2$ and Sobolev training. In the top right subfigure of Figure~\ref{fig:diff_noise}, we observe a similar benign overfitting phenomenon as with $L^2$ generalization under the same setup. However, the entire $H^1_k$ error curves under $L^2$ training ``drift upwards'' as the observational noise increases, which we can attribute to the network failing to learn the mean of the gradient. In contrast, the $H^1_k$ generalization error for Sobolev training (bottom right) does not display this drift. We do see, however, that the critical noise strength at which mean $H^1_k$ error for overparameterized models ceases to improve on the underparameterized regime, is different for the gradient error, which we can interpret as an increased sensitivity to observational noise in the gradients.
	
\subsection{Effect of varying the Tikhonov regularization strength $\lambda$}

\begin{figure}
\centering
\includegraphics[width = .8\textwidth]{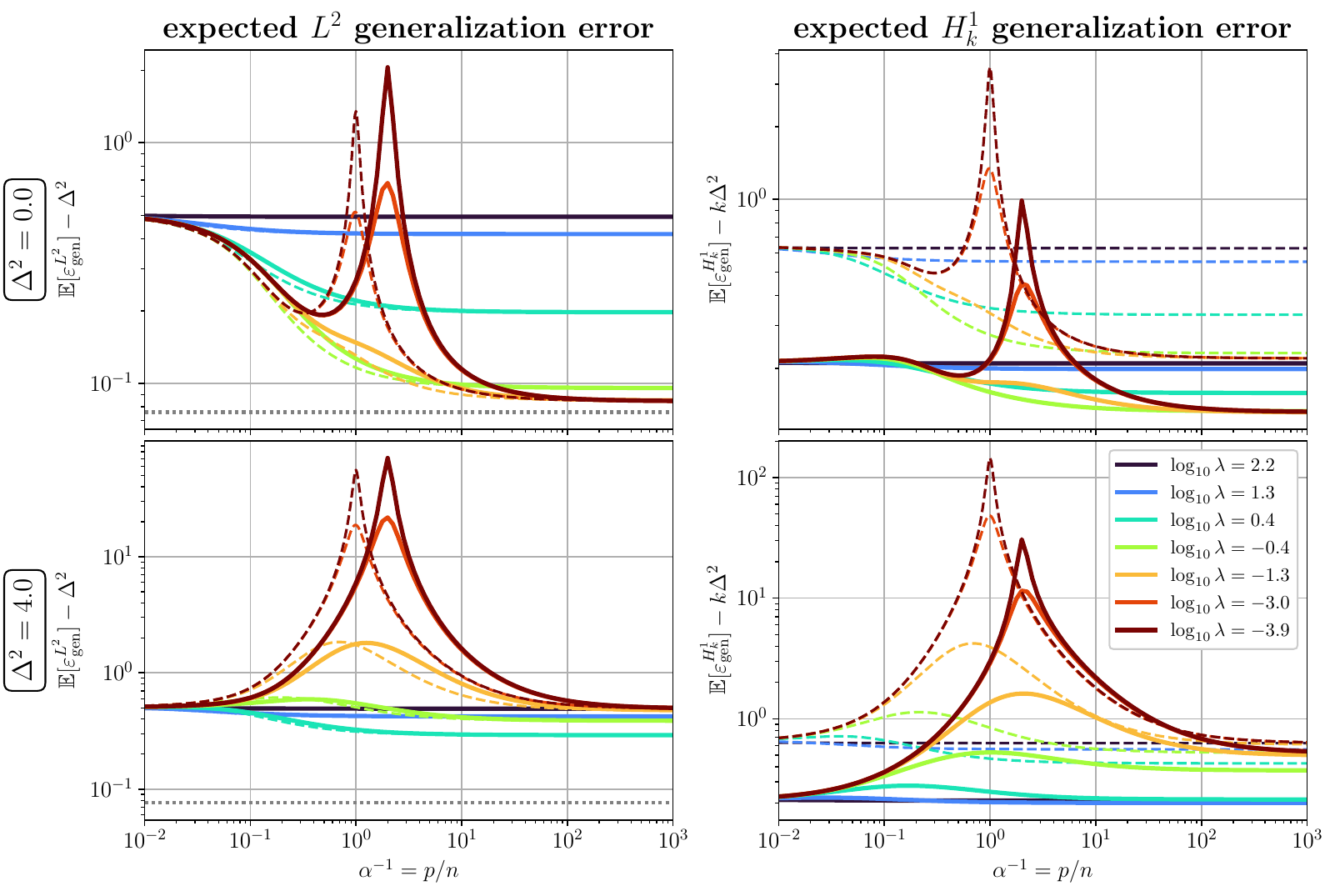}
\caption{Influence of the Tikhonov regularization parameter $\lambda > 0$ in~\eqref{eq:training-objective} on the expected $L^2$ (left column) and $H^1_k$ seminorm (right column) generalization errors of the RF network~\eqref{eq:neural-net-def}. Solid lines show the predictions for Sobolev training ($\tau = 1$), while dashed lines correspond to standard $L^2$ training ($\tau = 0$) without gradients. Top row: Noiseless training data $\Delta^2 = 0$, bottom row: large noise level $\Delta^2 = 4$. Other parameters: $n/d = 10$, $\sigma = \text{ReLU}$, $\phi = \arctan + 1 / \cosh$, $k = 1$. Note the irreducible component~$\Delta^2$ of the generalization errors has been subtracted in these figures for direct comparison of ground truth generalization. The dotted gray lines in the left column show the best achievable error $\EE[\phi(\xi)^2]-\EE[\phi(\xi)]^2 - \EE[\xi\phi(\xi)]^2$.}
\label{fig:lbda}
\end{figure}

In this section, we extend the qualitative analysis of~\citet{mei-montanari:2022} and explore which level of regularization~$\lambda$, if any, leads to optimal generalization errors for RF models. Figure~\ref{fig:lbda} considers this question for noiseless training data (top row) and additive Gaussian noise with variance~$\Delta^2 = 4$ (bottom row). The left column shows the $L^2$ generalization curves for various~$\lambda$, and we note that the curves for Sobolev training ($\tau=1$) are structurally similar to $L^2$ training $(\tau=0)$, modulo the shift in the interpolation threshold. Focusing on the lower envelope over all $L^2$ generalization curves, we observe that in the noiseless setting, the optimal choice of $\lambda$ varies with $p/n$. However, the lowest overall generalization error is attained by overparameterized networks $p/n\uparrow\infty$ with minimum norm regularization $\lambda \downarrow 0$, mirroring what has been empirically observed with deep neural networks.  In contrast, in the low signal-to-noise regime, there is a critical threshold of $\lambda$ which is uniformly optimal for all $p/n$ though once again overparameterization is necessary to achieve the lowest error. Including gradient training data does not result in a significant difference in the best achievable $L^2$ generalization error in either setting.\\

The effect of~$\lambda$ on $H^1_k$ generalization in the right column of Figure~\ref{fig:lbda} is qualitatively similar to the $L^2$ error, up to the following observations, that parallel our discussion in Section~\ref{sec:landscapes}: (i) the $H^1_k$ error curves for Sobolev training, even at small $p/n$ or large $\lambda$, are shifted downward by a constant compared to $L^2$ training, due to the network always learning to represent the gradient mean via $s_b$, (ii) in the present example of $\sigma = \text{ReLU}$, $\phi = \arctan + 1 / \cosh$, Sobolev training always outperforms $L^2$ training at large enough $p/n$, and (iii) an intermediate $\lambda$ and large $p/n$ is still optimal for $H^1_k$ prediction when using Sobolev training at small signal to noise ratio, but the benefit is less pronounced than for $L^2$ error.\\

Regarding the last observation (iii), we intuition that the difference being less pronounced is due to the additional ``noisiness'' of the random subspace projections in the following sense: Suppose we would train only on projected gradient data, with no function data,
and assume $k=1$ and $\Delta = 0$ for simplicity. Conditioned on~$\varpi = V_k^\top \theta_0$, the problem then reduces to the $L^2$ training setup with teacher function $\varpi \phi'$, except now the RF model has randomized activation functions $x \mapsto \left\langle \theta_j, v_k \right \rangle \sigma' \left( \left\langle \theta_j, x \right \rangle \right),~j=1,\ldots,p$, as each $\langle v_k, \theta_j \rangle$ is random. Equivalently, this can be viewed as using randomized Tikhonov regularization strengths~$\lambda / \langle v_k, \theta_j \rangle^2$ for each readout weight of a RF model with fixed activation~$\sigma'$.\\

In Appendix~\ref{app:lbda}, we show and discuss further results of varying $\lambda$ for odd $\phi = \arctan$ (with $\sigma = \text{ReLU}$, in Figure~\ref{fig:lbda-arctan-relu}) and even $\phi = 1 / \cosh$ (with $\sigma = \text{erf}$, in Figure~\ref{fig:lbda-reci-cosh-erf}). In line with our discussion above and in Section~\ref{sec:landscapes}, these results show that large regularization $\lambda \uparrow \infty$ is optimal whenever the linearized true function or gradient has vanishing $\EE \left[\omega \phi(\omega) \right]$, or $\EE[\omega \phi'(\omega)]$, respectively.

\subsection{Impact of gradient computation cost}
\label{sec:varying-k}

\begin{figure}
            \centering
            \begin{tabular}{c c c}
            \includegraphics[width=0.45\linewidth]{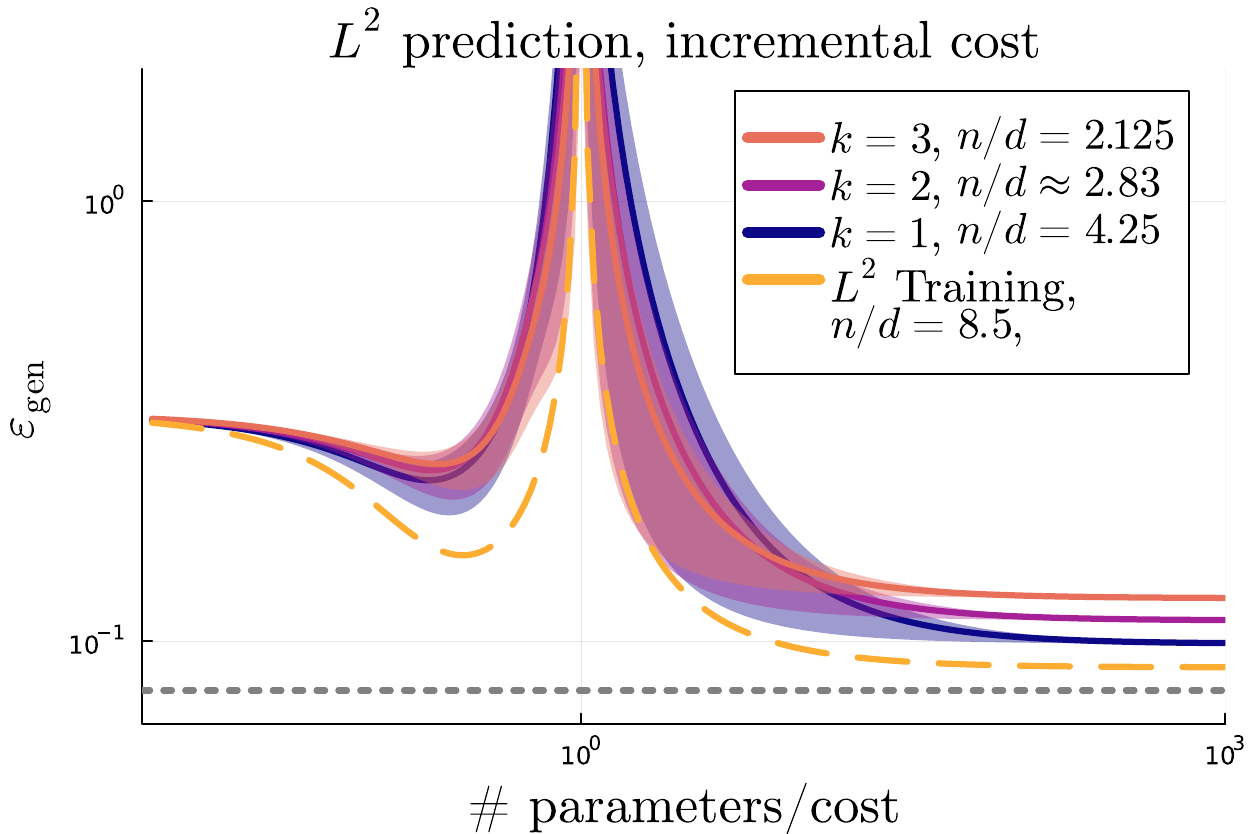}
            & & \includegraphics[width=0.45\linewidth]{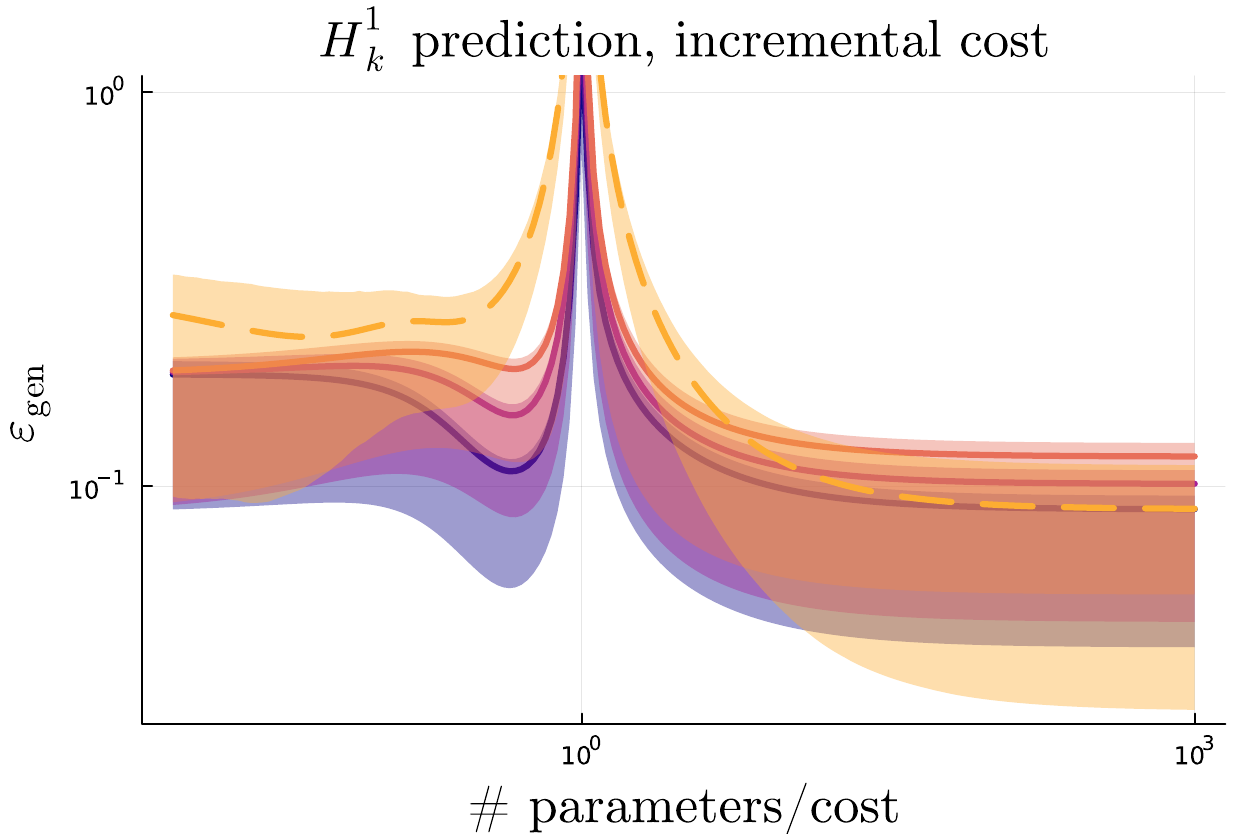} \\
            \end{tabular}
            \caption{Performance of Sobolev training (solid lines) with gradient projection dimension $k\in \{1,2,3\}$ against $L^2$ training baseline (gold, dashed). The computational cost of each projected dimension of the gradient is assumed to be equivalent to the cost of a function evaluation, in contrast to Figures~\ref{fig:intro},~\ref{fig:2d-errors-arctan-plus-reci-cosh-relu},~\ref{fig:diff_noise},~\ref{fig:lbda} where the cost of gradients is assumed negligible. Hence, $n/d=8.5$ for $L^2$ training, and $n/d=8.5/(k+1)$ for Sobolev training here. 
We consider the zero noise case with low regularization $\lambda=10^{-6}$ and activation function $\sigma=\textrm{SiLU}$. Left column: $L^2$ prediction error, with ribbons indicating the $20\%$ and $80\%$ quantiles, and dotted line showing the magnitude of nonlinearity in $\phi(\omega)=\omega/2 - \exp \{-\omega^2/2\}$. Right column: predictive distributions for $H_k^1$ generalization with $50\%$ and $70\%$ quantiles.}
\label{fig:k_effect}
       \end{figure}

Our previous experiments have demonstrated that the advantage of training with gradient data is conditional on the problem settings. Here, we determine whether Sobolev training is worthwhile given the computational \emph{cost} of sampling the training data. Figure~\ref{fig:k_effect} summarizes $L^2$ and $H^1_k$ generalization errors for an ``incremental cost'' model where each component of $V_k^\top y'$ incurs cost comparable to a new function sample, e.g., when using directional derivatives via finite difference stencils along direction~$V_k$. Other cost models and $\sigma,\phi$ are considered in Appendix~\ref{app:cost}. As a baseline, we assume that obtaining a single sample $y_i$ for $L^2$ training incurs a unit cost. Accordingly, the costs associated with the incremental model scale as $(k+1)n$. Thus, along a vertical slice of Figure~\ref{fig:k_effect}, the curves have different $\alpha = n/p$ to ensure a fair comparison.\\ 

More gradient information paradoxically ``harms'' function prediction in Figure~\ref{fig:k_effect} (left). Clearly, asymptotic $L^2$ generalization error is lowest here for models which, at a given cost, allow $n$ to be greatest. On the other hand, for $H_k^1$ prediction (Figure~\ref{fig:k_effect}, right), there is a marked benefit to assimilating $k=1$ gradient sketches for small models relative to the sampling cost. Nevertheless, counter-intuitively, incorporating additional sketches begins to harm gradient predictions for slightly larger models. This remains true in the overparameterized limit $p/n \uparrow\infty$. In this limit and under the cost counting model considered here, the only benefit of Sobolev training is to lower the probability of large~$H^1_k$ errors.


\section{Discussion and outlook}
\label{sec:concl}

We have introduced a simple statistical model for Sobolev training, based on random features and projections of gradient data onto random subspaces of fixed dimension $k$. 
Though this setting is considerably more complicated than the 
$L^2$ training of RF models~\cite{mei-montanari:2022,gerace-loureiro-krzakala-etal:2021,goldt-loureiro-reeves-etal:2022}, we showed that it remains possible to calculate generalization errors analytically, in the proportional asymptotics limit. Our approach involved conditioning on a random overlap parameter before applying the replica method, introducing a non-standard application of the Gaussian equivalence theorem, and using operator-valued free probability to linearize and evaluate traces of rational functions of random matrices.  We validated our theoretical predictions against MC sampling in high dimensions, demonstrating excellent agreement. 
Although portions of our presented calculations are non-rigorous, starting with the replica method itself, since our setting is convex we expect that these arguments could be made mathematically rigorous (cf.~\cite{hu-lu:2022}).\\

We discovered that introducing additional gradient data to the training loss shifts the interpolation threshold to $p = n(k+1)$, as if gradient observations were independent data. Our two-dimensional ``error landscapes'' (cf.~\cite{dascoli-sagun-biroli:2020}) and subsequent analysis (following~\cite{mei-montanari:2022}) showed that gradient data lowers the $L^2$ generalization error in some configurations, but not for all. The dominant effect here is the shift of the interpolation peak.  Counter-intuitively, we demonstrated that incorporating gradient information does not uniformly lower the $H^1_k$ generalization error of the RF network: in particular, Sobolev training with slightly overparameterized models can lead to less accurate predictions of new gradients. Furthermore, we showed that if only \textit{noisy} gradient observations are available (relevant, e.g., in applications where gradients are approximated via finite differences or least squares regression~\cite{czarnecki-osindero-jaderberg:2017,kissel-diepold:2020}, or in applications where gradients are obtained by differentiating pre-trained neural networks~\cite{czarnecki-osindero-jaderberg:2017, srinivas-fleuret:2018}), benign overfitting can still occur within the RF model despite the additional noise.\\

Our analysis highlighted two fundamental limitations that prevent RF models from assimilating gradient information. First, it has been previously documented that RF models in the proportional asymptotics limit are only able to capture the linear component of the data-generating single-index function~\cite{mei-montanari:2022, ba-erdogdu-suzuki-etal:2022}. Our formulation of the Gaussian equivalence theorem demonstrates that this extends to learning only the linear component of the projected gradients as well.  Accordingly, solely optimizing the readout weights and keeping hidden weights fixed precludes any \textit{feature learning} from data.
Second, we showed that the only projections onto vectors with norm~$\|v\| = O(\sqrt{d})$ provide a  compatible scaling for sketching the gradients of RF models. We further demonstrated that only ``un-informed'' subspace projections $V_k$, with columns sampled from~$\mathcal{N}(0,I_d)$ independently of the data, lead to a sensible loss function. RF models are thus unable to fully exploit the ``directional'' information in the gradient data of single-index models.\\ 

In future work, we intend to extend our model of Sobolev training to incorporate feature learning, and towards data-informed choices of subspaces for the gradient projection (as in, e.g.,~\cite{o-leary-roseberry-chen-villa-etal:2024}). Notably, we believe that the ``one large gradient step'' model for the hidden-layer weights in recent work~\cite{ba-erdogdu-suzuki-etal:2022,cui-pesce-dandi-etal:2024} could be extended to capture feature learning with Sobolev training. This strategy leads to the study of spiked random matrices, and we expect that our analysis could be adapted to this setting. It is also of interest to consider Sobolev training for Bayesian neural networks~\cite{pacelli-ariosto-pastore-etal:2023} and to characterize the influence of gradient data on the posterior predictive distribution. Different polynomial scaling regimes, e.g., $p/n=\alpha$ and $n = d^\kappa$, are also known to extend the $L^2$ approximation class of RF models beyond linear functions~\cite{aguirre-lopez-franz-pastore:2025,misiakiewicz:2022,hu-lu-misiakiewicz:2024}, and we anticipate that the same holds given gradient data. More generally, we could study extensions to multi-index data models, and an RF model that is more closely inspired by the task of approximating the solution map of a PDE using function and Jacobian data~\cite{o-leary-roseberry-chen-villa-etal:2024,nelsen-stuart:2021}.  Finally, we have so far only considered adding gradients to the training loss; sketches of higher derivatives, such as Hessian projections, may also be of interest to the machine learning PDE solver community.

\section*{Acknowledgments}

The authors would like to thank Murat A.\ Erdogdu, Michael F.\ Herbst, James Kermode, Bruno Loureiro, Thomas O'Leary-Roseberry, and the MIT UQ group for helpful discussions.
KF expresses gratitude for the support by the National Science Foundation Graduate Research Fellowship (Grant No.\ 1745302). KF and YM acknowledge support from the Department of Energy (DOE), National  Nuclear Security Administration PSAAP-III program (Award Number DE-NA0003965).
MTCL and YM acknowledge support from the DOE Office of Advanced Scientific Computing Research under award number DE-SC0023187.
TS acknowledges the financial support received from the Ruhr University Bochum Research School through a Gateway Fellowship during the initial stages of this work.
The authors thank the MIT Libraries for the resources provided.
The authors acknowledge the MIT Office of Research Computing and Data for providing high performance computing resources that have contributed to the research results reported within this paper.


\appendix

\def\thesection{\Alph{section}}
\def\thesubsection{\Alph{section}.\arabic{subsection}}
\def\thesubsubsection{\Alph{section}.\arabic{subsection}.\arabic{subsubsection}}


\section{Notation and table of mathematical symbols}
\label{App:notation}

We list some general notation for mathematical operations used throughout this paper in Table~\ref{tab:operators}. For a list and explanation of mathematical symbols and variables appearing repeatedly, see Table~\ref{tab:notation}. Lastly, Table~\ref{tab:activation} contains a few possible activation functions~$\sigma$ of the RF model~\eqref{eq:neural-net-def} considered in this work, as well as their Hermite coefficients~\eqref{eq:hermite-coeff-def}.

\begin{table}
\centering
\caption{Mathematical operators.}
\renewcommand{\arraystretch}{1.3}
\begin{tabular}{|c|c|c|c|} 
\hline
$\odot$ & Hadamard product $(A \odot B)_{ij} = A_{ij} B_{ij}$ & $v \otimes w = vw^\top$ & outer product of vectors $v,w$        \\ 
$v^{\otimes 2} = v \otimes v$ & outer product squared of vector $v$ & $A \otimes B$ & Kronecker product of matrices $A,B$\footnote{Note that this overloads the symbol $\otimes$ depending on the objects considered. The Kronecker product of two vectors $v,w$ is given by $\textsc{vec}(v \otimes w)$ in our notation.} \\
$\oplus $ & Kronecker sum $A \oplus B = A \otimes I + I \otimes B$ & $\textsc{rs}(v)$ & reshape vector $v$ into square matrix\\
$\textsc{vec}(M)$ & column-wise flattening of $M$ into vector &  $\textsc{vech}(S)$ &  flattening of upper triangle of symmetric matrix $S$\\
$\diag{v}$ & matrix with vector $v$ on the diagonal & $\Tr_p := \plim_{p \to \infty} \tfrac{1}{p}\trace$ & normalized trace in proportional asymptotic limit  \\
$\langle\cdot,\cdot\rangle_F$ & Frobenius inner product  & $\langle \cdot, \cdot \rangle_{\text{HF}}$ & half Frobenius inner product, cf.~\eqref{eq:hf-def}  \\ 
$G_r$ & Cauchy transform of $r$ & $H_r(z)=(G_r(z))^{-1}-z$ & shifted reciprocal Cauchy transform  \\
$g_\mu$ & Stieltjes transform of density $\mu$ & $\varphi$ & state function of free probability space   \\ 
$\boxplus$ & free additive convolution & $\mathfrak{s}(Z)$ & subordinator, cf.~\eqref{eq:fixed-point-subord} \\
 \hline
\end{tabular}
\label{tab:operators}
\end{table}

\begin{table}
\centering
\caption{Mathematical symbols and variables.}
\renewcommand{\arraystretch}{1.3}
\begin{tabular}{|c|c|c|c|} 
\hline
\multicolumn{1}{|c}{} & \multicolumn{1}{c}{Observation data, network, and training} & \multicolumn{1}{c}{} & \multicolumn{1}{c|}{} \\ \hline
$n$ & number of observation inputs         & $X=[x_1\dots x_n]$ & observation inputs/covariates \\ 
$d$ & dimension of inputs $x_i$            & $Y=(y_1,\dots,y_n)^\top$ & observation outputs for each $x_i$ \\
$Y' = [y_1'\dots y_n']$ & gradients of $y_i$ w.r.t.~$x_i$ &  $V_k$ & random projection of $y'$ to dimension $k$  \\ 
$\Upsilon_i=(y_i,V_k^\top y_i')$ & training data corresponding to $x_i$ & $\theta_0$ & teacher feature \\
 $\phi \colon \RR \to \RR$ & teacher (ridge) nonlinearity & $\phi'$ & derivative of teacher nonlinearity \\
  $\eta_i$ & noise applied to observation $y_i$ &$\eta_i'$ & noise applied to observation $y_i'$ \\ 
  $w$ & learnable network weights & $p$ & dimension of weights $w$ \\ 
  $\Theta=[\theta_1 \dots \theta_p]$ & random network features & $\sigma \colon \RR \to \RR$ & element-wise network activation  \\
  $\alpha$ & finite ratio $n/p$ & $\gamma$  & finite ratio $d/p$ \\
  $\omega_i=\theta_0^\top x_i$ & projection of input onto true feature & $\varpi=V_k^\top \theta_0$ & projection of true feature onto $V_k$ \\
  $\lambda > 0$ & regularization strength & $\tau \geq 0$ & weight of gradient observations in loss \\ 
  $P_{\text{data}}$ & observation distribution on $\RR^{k+1}$ & $C_\eta$ & observation noise covariance \\
$\varepsilon_{\text{train}}$ & Sobolev training error & $\varepsilon_{\text{gen}}$ & Sobolev generalization error \\
$\varepsilon^{L^2}$ & $L^2$ error & $\varepsilon^{H^1_k}$ & $V_k$-projected $H^1$ semi-norm error\\
$[n]$ & index set $\{1, 2, \ldots, n \}$ &  &  \\
  \hline
\multicolumn{1}{|c}{} & \multicolumn{1}{c}{Gaussian universality} & \multicolumn{1}{c}{} & \multicolumn{1}{c|}{} \\ \hline
$\kappa_0$ & constant Hermite coefficient of $\sigma$ & $\kappa_0'$ & constant Hermite coefficient of $\sigma'$ \\
$\kappa_1$ & linear Hermite coefficient of $\sigma$ & $\kappa_1'$ & linear Hermite coefficient of $\sigma'$ \\
$\kappa_*^2$ & magnitude of nonlinear component of $\sigma$ & $(\kappa_*')^2$ &   magnitude of nonlinear component of $\sigma'$ \\
$\hat{\eta}$ & noise of linearization of $\sigma$ & $\hat{\eta}'$ & noise of linearization of $\sigma'$ \\
\hline
\multicolumn{1}{|c}{} & \multicolumn{1}{c}{Overlap parameters and auxiliaries} & \multicolumn{1}{c}{} & \multicolumn{1}{c|}{} \\ \hline
$a$ & subscript of scalar overlaps & $b$ & subscript of $k$-dim.\ vector overlaps \\
$c$ & subscript of $k \times k$ matrix overlaps & $s=(s_a,s_b)^\top$ & network and network gradient mean \\
$f=(f_a,f_b)^\top$ & overlap of network with $\omega$ & $q=\begin{pmatrix} q_a & q_b^\top \\ q_b & q_c \end{pmatrix}$ & network covariance overlap \\ 
$\hat{s}=(\hat{s}_a,\hat{s}_b)^\top$ & auxiliary of $s$ & $\hat{f}=(\hat{f}_a,\hat{f}_b)^\top$ & auxiliary of $f$ \\ 
$\hat{q}=\begin{pmatrix} \hat{q}_a & \hat{q}_b^\top \\ \hat{q}_b & \hat{q}_c \end{pmatrix}$ & auxiliary of $q$ & $D_\tau=\diag{1,\tau,\dots,\tau}$ & weights of each element of $\Upsilon_i$ \\
$A \in \RR^{p \times p}$ & random matrix in fixed point system & $\Xi \in \RR^{p \times p}$ & random matrix in equation for $q$ \\
$\zeta=V_k^\top \Theta$ & projected random features & $D_i=\diag{\zeta_i}$  & diagonalization of $i^{th}$ column of $\zeta$   \\
$(0)$ & superscript of component constant in~$\varpi$ & $(2)$ & superscript of  component quadratic in~$\varpi$ \\
$M_{00}$ & $\kappa_1^2 I_p + \kappa_*^2\Theta^\top \Theta$  & $M_{11}$ & $(\kappa_1')^2I_p + (\kappa_*')^2\Theta^\top \Theta$ \\
\hline
\multicolumn{1}{|c}{} & \multicolumn{1}{c}{Operator-valued free probability} & \multicolumn{1}{c}{} & \multicolumn{1}{c|}{} \\ \hline
$m$ &  $\Theta^\top \Theta$ when $p,n,d\to\infty$ proportionally & $g_i$ & $D_i$ when $p,n,d\to\infty$ proportionally \\
 \hline
\end{tabular}
\label{tab:notation}
\end{table}

\begin{table}
\centering
\caption{Different permissible examples of activation functions $\sigma$ and their Hermite coefficients, as defined in~\eqref{eq:hermite-coeff-def}, for the setting considered in this work. We require $\sigma$ to be weakly differentiable, which excludes e.g.\ $\sigma = \text{sign}$, but does allow for ReLU for example. We do not require $\sigma$ to be an odd function. The coefficients listed with ``$\approx$'' were evaluated numerically, while all others are exact.}
\label{tab:activation}
\vspace{.2cm}
\renewcommand{\arraystretch}{1.5}
\begin{tabular}{|>{\centering\arraybackslash}m{2.5cm}|>{\centering\arraybackslash}m{3cm}|>{\centering\arraybackslash}m{3cm}|>{\centering\arraybackslash}m{1.6cm}|>{\centering\arraybackslash}m{1.6cm}|>{\centering\arraybackslash}m{1.6cm}|>{\centering\arraybackslash}m{1.6cm}|>{\centering\arraybackslash}m{1.6cm}|}
\hline
function & definition & sketch & $\kappa_0$ & $\kappa_1 = \kappa_0'$ & $\kappa_1'$ & $\kappa_*$ & $\kappa_*'$\\
\hline

\vfill Error function (erf) \vfill  &
\vfill $\sigma(z) = \frac{2}{\sqrt{\pi}} \int_0^z e^{-t^2} \dd t$ \vfill &
\begin{tikzpicture}[scale=0.6]
  \begin{axis}[domain=-6:6, samples=100, axis lines=middle, ymin=-1.1, ymax=1.1, xtick=\empty, ytick=\empty, width=5.2cm, height=4cm]
    \addplot[steelblue, very thick] {erf(x)};
  \end{axis}
\end{tikzpicture} &
\vfill $0$ \vfill &
\vfill $\frac{2}{\sqrt{3 \pi}}$ \vfill &
\vfill $0$ \vfill &
\vfill $\approx 0.2004$ \vfill &
\vfill $\frac{2}{\sqrt{\pi}} \sqrt{\frac{3 \sqrt{5} - 5}{15}}$ \vfill\\
\hline

\vfill Sigmoid Linear Unit (SiLU) \vfill &
\vfill $\sigma(z) = \frac{z}{1 + e^{-z}}$\vfill  &
\begin{tikzpicture}[scale=0.5]
  \begin{axis}[domain=-5:5, samples=100, axis lines=middle,  ymin=-0.4, ymax=6, xtick=\empty, ytick=\empty, width=6cm, height=4cm]
    \addplot[steelblue, very thick] {x/(1 + exp(-x))};
  \end{axis}
\end{tikzpicture} &
\vfill $\approx 0.2066$ \vfill &
\vfill $\frac{1}{2}$ \vfill &
\vfill $\approx 0.3508$ \vfill &
\vfill $\approx 0.2512$ \vfill &
\vfill $\approx 0.0799$ \vfill\\
\hline

\vfill Rectified Linear Unit (ReLU) \vfill &
\vfill $\sigma(z) = \max \{0, z \}$\vfill &
\begin{tikzpicture}[scale=0.5]
  \begin{axis}[domain=-5:5, samples=100, axis lines=middle, ymin=-0.4, ymax=6, xtick=\empty, ytick=\empty, width=6cm, height=4cm]
    \addplot[steelblue, very thick] {max(0,x)};
  \end{axis}
\end{tikzpicture} &
\vfill $\frac{1}{\sqrt{2 \pi}}$ \vfill &
\vfill $\frac{1}{2}$ \vfill &
\vfill $\frac{1}{\sqrt{2 \pi}}$ \vfill &
\vfill $\sqrt{\frac{1}{4} - \frac{1}{2 \pi}}$ \vfill &
\vfill $\sqrt{\frac{1}{4} - \frac{1}{2 \pi}}$ \vfill\\
\hline
\end{tabular}
\end{table}


\section{Choice of gradient subspaces}
\label{app:data-informed-subspace}

In equation~\eqref{eq:training-obj-sobo}, we project the gradient data and the network gradient predictions onto a subspace spanned by the columns of a known matrix~$V_k \in \RR^{d \times k}$.  
This setup is inspired by practical considerations since the paper on Sobolev training by~\citet{czarnecki-osindero-jaderberg:2017}, as well as DINOs~\cite{o-leary-roseberry-chen-villa-etal:2024}, advocate for sketching gradients in this manner to lower computational costs.
However, this projection is also necessary for our theory since the replica method can only be applied with a fixed and finite number of overlap parameters. Accordingly, we require $k = O(1)$ in the asymptotic limit.\\

Although \cite{czarnecki-osindero-jaderberg:2017, o-leary-roseberry-chen-villa-etal:2024} recommend projecting the gradients onto columns~$v$ of $V_k$ that have unit norm, this choice does \emph{not} enable RF models to assimilate gradient information in high dimensions. Figure~\ref{fig:unit-Vk-scaling} illustrates the MC simulations of generalization errors of RF models at fixed $n/d=2.345$ and $k=1$ across $d = \{200, 500, 1000, 2000\}$ with $v \sim \mathcal{N}(0, d^{-1} I_d)$ so that $\lim_{d \to\infty} \| v \|=1$ almost surely. As we can observe, the $L_2$ generalization errors approach the theoretical predictions obtained from $L_2$ training as $d\uparrow\infty$, meaning the model behaves equivalently to the setting where gradient data are not provided at all. Phrased differently, under this $V_k$ scaling the projections of the gradient data and the network gradient predictions tend to zero in the proportional asymptotics limit. Thus, in the figure, we observe the $H^1_k$ generalization errors also approach zero though this trend occurs for any model with $O(p^{-1/2})$ readout weight entries.\\

\begin{figure}
    \includegraphics[width=0.85\linewidth]{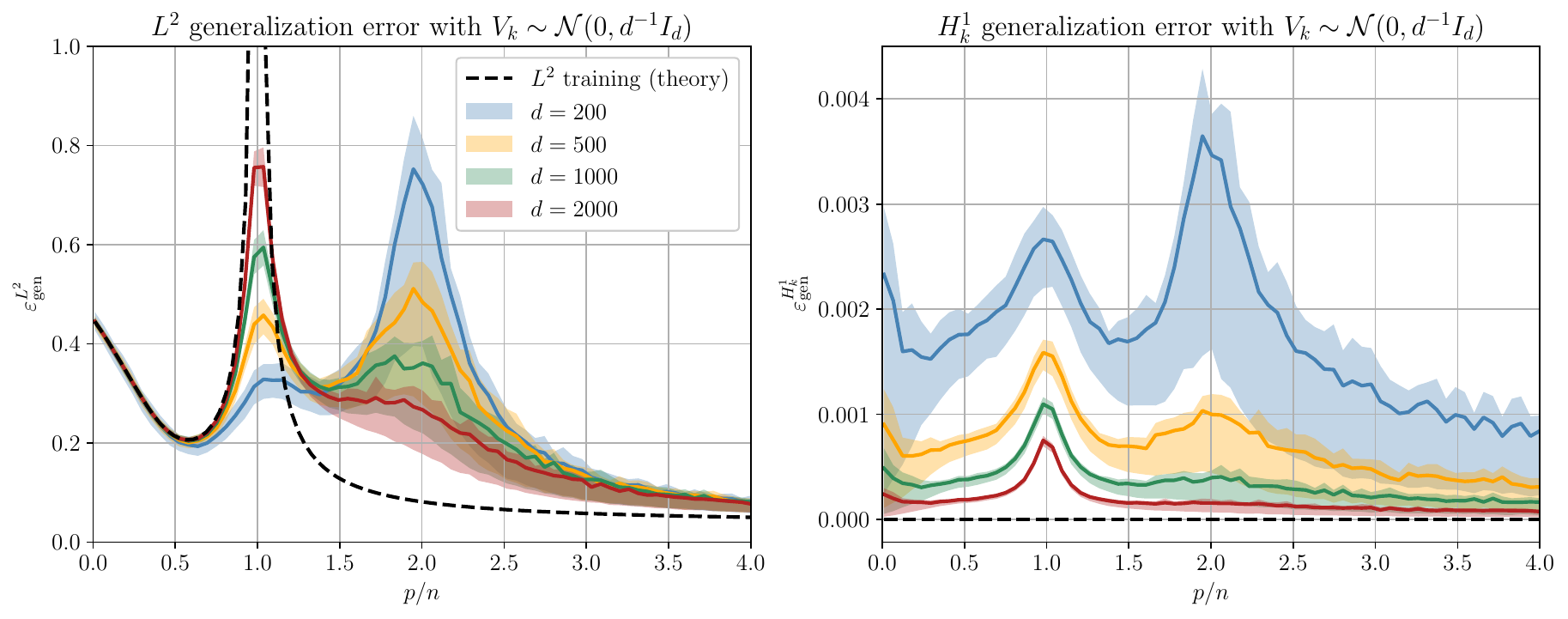}
     \caption{Generalization error curves when $V_k,~k=1$, is sampled from~$\mathcal{N}(0, d^{-1} I_d)$. Other hyperparameters: $n/d = 2.345$, $\lambda = 10^{-6}$, $\sigma=\mathrm{erf}$, and $\phi=\arctan$. The solid lines represent the mean, and the shaded regions represent 25\% and 75\% quantiles of the error distributions for 500 samples. Left: $L^2$ generalization error. Right: $H^1_k$ generalization error. The black dashed lines show the theoretical predictions for the same problem setup under $L^2$ training.}
     \label{fig:unit-Vk-scaling} 
\end{figure}

Instead, it is necessary to have $\| v \| = O(\sqrt{d})$. This constraint ensures $v^\top \nabla_x f_w(x) = O(1)$ so that all terms in~\eqref{eq:training-obj-sobo} have commensurate scaling in the asymptotic limit. In the main text, we choose the columns to be sampled iid from the $d$-dimensional standard Gaussian; we refer to this model as a \emph{data uninformed} subspace. \\

In contrast, \cite{o-leary-roseberry-chen-villa-etal:2024} construct the \emph{data-informed} subspace $V_k$ to span the $k$-leading eigenspace of the matrix $\mathbb{E}[(y')(y')^\top]$, estimated via MC with training data and document improved generalization performance of their neural network model. For gradient data arising from single-index teachers with teacher vector~$\theta_0$, we can model this construction by sampling each column~$v$ of $V_k$ as 
\[
v = \sqrt{d} \varpi \theta_0 + \mathcal{N}(0, (1-\varpi^2) I_d).
\]
We refer to the normalized projection $\frac{1}{\sqrt{d}} v^\top \theta_0 \to \varpi \in [-1, 1]$ as the subspace alignment~(see also \cite{cui-pesce-dandi-etal:2024}), and we interpret the Gaussian noise term as modeling errors incurred from estimating the eigenvectors with finite samples.\\

Although~$\|v\| = O(1)$ and $v^\top \nabla_x f_w = O(1)$ remain as before, unfortunately this data-informed subspace yields $v^\top y' = O(\sqrt{d})$. As a result, the training objective, e.g., the squared Sobolev $H^1_k$ norm, must be adjusted as 
\[
\ell(y_i, f_w(x_i), V_k^\top y_i', V_k^\top \nabla f_w(x_i)) = \frac{1}{2}(y_i - f_w(x_i))^2 + \frac{1}{2}\|\frac{1}{\sqrt{d}} V_k^\top y_i' - V_k^\top \nabla f_w(x_i) \|^2 \;,
\]
cf. equation~\eqref{eq:loss-sobo-standard}, which does not normalize the gradient projection by~$\sqrt{d}$. Evidently, this loss function promotes misspecified gradient models since the idealized outcome ``$y_i' = \nabla f_w(x_i)$'' does not minimize the seminorm component. We believe this to be a fundamental limitation of RF models with proportionally asymptotic scaling $d,n,p\to\infty$ with linear ratios $\alpha = n/p$ and $\gamma = d/p$ fixed. It is an interesting direction for future work to investigate theoretical models which are able to capture the benefit of data informed gradient subspaces.


\section{Gaussian equivalence theorem for gradient observations}
\label{app:GET}

A key step in deriving the fixed point system~\eqref{eq:update_hat_sobo_h1} and~\eqref{eq:update_nonhat_sobo_h1} is to replace the teacher and student networks with asymptotically equivalent expressions (in distribution) that are affine in the pre-activation features $\theta_0^\top x$, $\theta_1^\top x$, \ldots, $\theta_p^\top x$. This is the content of the so-called Gaussian equivalence theorem (GET). For $L^2$ training, the GET adopts the form
\begin{align}
\label{eq:ell2-GET}
	\begin{pmatrix}
		w^\top \sigma\lp \Theta^\top x \rp \\
		\theta_0^\top x
	\end{pmatrix}  \; \xrightarrow[p,d\to\infty]{\; \mathcal{D} \;} \; \begin{pmatrix}
	w^\top \lp \kappa_0 \1_p + \kappa_1 \Theta^\top x + \kappa_* \hat{\eta} \rp \\
		\theta_0^\top x
\end{pmatrix}	 \;,
\end{align}
where $\hat{\eta} \sim \mathcal{N}(0, I_p)$ and the $\kappa$ coefficients are given by~\eqref{eq:hermite-coeff-def}.\\

The convergence of~\eqref{eq:ell2-GET} is rigorously established by \citet{goldt-loureiro-reeves-etal:2022} and \citet{hu-lu:2022}.
Denoting the post-activation features $a_1 = \sigma(\theta_1^\top x)$, \ldots, $a_p = \sigma(\theta_p^\top x)$, both approaches essentially rely on decorrelating $\{a_1,\dots,a_{i-1},a_{i+1},\dots,a_{p}\}$ from $a_i$ to establish a central limit theorem, though the larger structure of their proof techniques differs. In particular, \citet{goldt-loureiro-reeves-etal:2022} focus on low dimensional projections of the features and their Gaussian equivalents. They proceed by bounding the maximum sliced distance between the laws of these objects, and they allow for arbitrary nonlinear activations for each feature, up to a smoothness condition. \citet{hu-lu:2022} use Lindeberg's method to construct an interpolating path between the features and their Gaussian counterparts, and bound differences between points along this path. These authors demonstrate that the training and generalization error produced by the network features converge in probability to the corresponding objects for the Gaussian features, and moreover, the first two moments of these features match. \\

For Sobolev training, the form of GET that we require is
\begin{align}
\label{eq:sobo-GET}
	\begin{pmatrix}
		w^\top \sigma\lp \Theta^\top x \rp \\
		V_k^\top \Theta \;  \diag{  \sigma'\lp \Theta^\top x } \rp w \\
		\theta_0^\top x
	\end{pmatrix}  \; \xrightarrow[p,d\to\infty]{\; \mathcal{D} \;} \; \begin{pmatrix}
	w^\top \lp \kappa_0 \1_p + \kappa_1 \Theta^\top x + \kappa_* \hat{\eta} \rp \\
	V_k^\top \Theta \;  \diag{  \kappa_0' \1_p + \kappa_1' \Theta^\top x + \kappa_*' \hat{\eta}'  }  w \\
		\theta_0^\top x
\end{pmatrix}	 ,
\end{align}
where $\hat{\eta},\hat{\eta}'\sim\mathcal{N}(0,I_p)$ are independent, and the $\kappa$ coefficients are once again given by~\eqref{eq:hermite-coeff-def}.
The shared pre-activation features $\Theta^\top x$ between the network $w^\top \sigma\lp \Theta^\top x \rp$ and its gradient $V_k^\top \Theta \;  \diag{  \sigma'\lp \Theta^\top x } \rp w$ pose a critical obstruction towards rigorously establishing~\eqref{eq:sobo-GET}, though we can obtain partial results in this direction. For instance, by applying Theorem 2 of \citet{goldt-loureiro-reeves-etal:2022}, we can conclude
\begin{align}
\label{eq:h1k-GET}
	\begin{pmatrix}
		V_k^\top \Theta \;  \diag{  \sigma'\lp \Theta^\top x } \rp w \\
		\theta_0^\top x
	\end{pmatrix}  \; \xrightarrow[p,d\to\infty]{\; \mathcal{D} \;} \; \begin{pmatrix}
	V_k^\top \Theta \;  \diag{  \kappa_0' \1_p + \kappa_1' \Theta^\top x + \kappa_*' \hat{\eta}'  }  w \\
		\theta_0^\top x
\end{pmatrix}	 .
\end{align}
Unfortunately, \eqref{eq:h1k-GET} and \eqref{eq:ell2-GET} are not sufficient to imply \eqref{eq:sobo-GET}, and we must also demonstrate
\begin{align}
\label{eq:problem-GET}
	\begin{pmatrix}
		w^\top \sigma\lp \Theta^\top x \rp \\
		V_k^\top \Theta \;  \diag{  \sigma'\lp \Theta^\top x } \rp w 
	\end{pmatrix}  \; \xrightarrow[p,d\to\infty]{\; \mathcal{D} \;} \; \begin{pmatrix}
	w^\top \lp \kappa_0 \1_p + \kappa_1 \Theta^\top x + \kappa_* \hat{\eta} \rp \\
	V_k^\top \Theta \;  \diag{  \kappa_0' \1_p + \kappa_1' \Theta^\top x + \kappa_*' \hat{\eta}'  }  w 
\end{pmatrix} \;.
\end{align}
The correlations between these marginals precludes us from similarly applying Theorem 2 in \citet{goldt-loureiro-reeves-etal:2022}, though we speculate that a modification of the proof technique could sufficiently strengthen it the result to apply to our setting. In the present work, we do not pursue this technical modification, but instead, we provide numerical justification for \eqref{eq:sobo-GET} in Section~\ref{emp_GET}.\\

We note several advancements on the work of \citet{goldt-loureiro-reeves-etal:2022} and \citet{hu-lu:2022} have been put forward in the intervening years. \citet{montanari-saeed:2022} extend the Gaussian equivalence theorem for fixed features to loss functions and regularization that may be non-convex. Leveraging the notion of exponentially concentration vectors, \citet{Seddik2020} demonstrate that successive Lipschitz transformations applied to Gaussian data yield features that have a Gram matrix equivalent to that of a Gaussian mixture model. Both \citet{schroder2023} and \citet{bosch-panahi-hassibi:2023} establish Gaussian equivalence for deep RF models. \citet{cui-krzakala-zdeborova:2023} and \citet{pacelli-ariosto-pastore-etal:2023} conjecture about the next step: deep Gaussian equivalence. In particular, \citet{pacelli-ariosto-pastore-etal:2023} argue that Gaussian equivalence should apply to networks where interior parameters are trainable as a consequence of an extension of the Breuer--Major theorem \citep{Breuer1983} applied by \citet{bardet2013}. Picking up on this thread, \citet{camilli-tieplova-bergamin-etal:2025} prove deep Gaussian equivalence using an interpolation argument. While this body of work has contributed significantly to the understanding of neural network learning, each result requires at most weak correlation in features or training data points. Consequently, to the best of our knowledge, existing work on Gaussian equivalence does not rigorously establish \eqref{eq:sobo-GET}.

\subsection{Empirical support for Sobolev Gaussian equivalence\label{emp_GET}}

Here, we present numerical evidence for the statistical behavior of the RF model and low dimensional projections of its gradients in the proportional asymptotics limit. Figure \ref{fig:sobo-GET} shows a representative result of our experiments to verify~\eqref{eq:sobo-GET}. We consider $d=500$, $\alpha=p/n=2.5$, and $\gamma=d/p=0.1706$. For an error function nonlinearity combined with a fixed set of features, $\theta_0,\theta_1,\dots,\theta_p$, and projection, $V_k$, we solve the ridge regression problem with the Sobolev norm to find the optimal weights, $w_*$. Then, we compute
\begin{align}
	\label{eq:sobo-GET-rf}
	v_{\text{RF}} &= \begin{pmatrix}
	f_{w_*}(x) \\ V_k^\top\nabla f_{w_*}(x) \\ \omega 
\end{pmatrix} = \begin{pmatrix}
	{w_*}^\top \sigma\lp \Theta^\top x \rp \\
		V_k^\top \Theta \;  \diag{  \sigma'\lp \Theta^\top x } \rp {w_*} \\
		\theta_0^\top x
	\end{pmatrix}, \qquad \\ 
	\label{eq:sobo-GET-GET} v_{\text{GET}} &= \begin{pmatrix}
	f_{w_*}^{\text{lin}}(x) \\ V_k^\top\nabla f^{\text{lin}}_{w_*}(x) \\ \omega 
\end{pmatrix} = \begin{pmatrix}
	{w_*}^\top \lp \kappa_0 \1_p + \kappa_1 \Theta^\top x + \kappa_* \hat{\eta} \rp \\
	V_k^\top \Theta \;  \diag{  \kappa_0' \1_p + \kappa_1' \Theta^\top x + \kappa_*' \hat{\eta}'  }  {w_*} \\
		\theta_0^\top x
\end{pmatrix}		, 
\end{align}
for $2000$ samples of $x$ drawn independently from $\mathcal{N}(0,I_p)$. The subplots of Figure \ref{fig:sobo-GET} compares the marginals of \eqref{eq:sobo-GET-rf} and \eqref{eq:sobo-GET-GET} as well as every pairwise point distribution. We see that the limiting Gaussian distribution posited by \eqref{eq:sobo-GET} accurately characterizes the behavior of the RF model for large $n$, $p$, and $d$.\\

	\begin{figure}
		\includegraphics[width=0.7\linewidth]{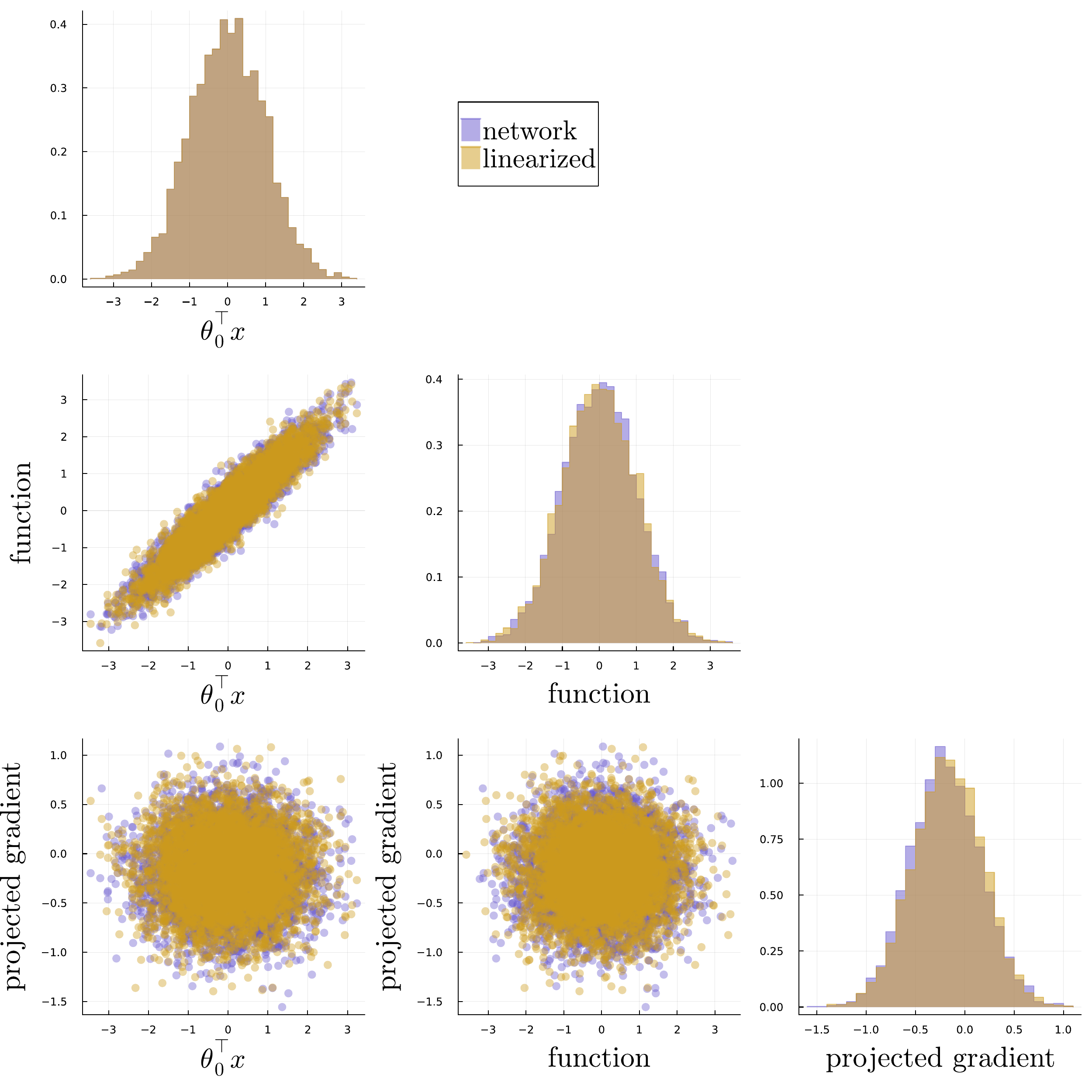}
		 \caption{\label{fig:sobo-GET} Comparison of the RF model (indigo) and the equivalent Gaussian model (gold). The marginals of $\omega=\theta_0^\top x$, $f_w(x)$ and $V_k^\top \nabla f_w(x)$ (and the corresponding linearized models) are shown on the diagonal, and the off diagonal plots show the corresponding pairwise distributions. }
	\end{figure}
	
	We now highlight the dependence of the equivalent Gaussian model on independent noise vectors, $\hat{\eta}$ and $\hat{\eta}'$, which correspond to the models for $y$ and $V_k^\top y'$, respectively. Since the only source of randomness in $v_{\text{RF}}$ is $x$, and this is shared by $f_w(x)$ and $V_k^\top \nabla f_w(x)$, it is not obvious whether $\hat{\eta}$ should be independent from, or equivalent to, $\hat{\eta}'$. We verify that choosing these to be independent produces Gaussian equivalence for two choices of nonlinearity: the error function and the Sigmoid Linear Unit (SiLU). Figure \ref{fig:GET-eta} compares the case where $\hat{\eta}=\hat{\eta}'$ (left column) to the case where the two noise vectors are independent (right column). The difference between the two columns is slight for the error function (first row), but we can distinguish a slight positive correlation in the samples of $v_{\text{RF}}$ that is captured by $v_{\text{GET}}$ when $\hat{\eta}\neq\hat{\eta}'$ and missed when $\hat{\eta}=\hat{\eta}'$. The disparity is much clearer when we consider SiLU (second row). Choosing $\hat{\eta}$ to be independent from $\hat{\eta}'$ thus produces the expected equivalent Gaussian distribution.

	\begin{figure}
            \centering
            \begin{tabular}{c c c}
            \includegraphics[width=0.45\linewidth]{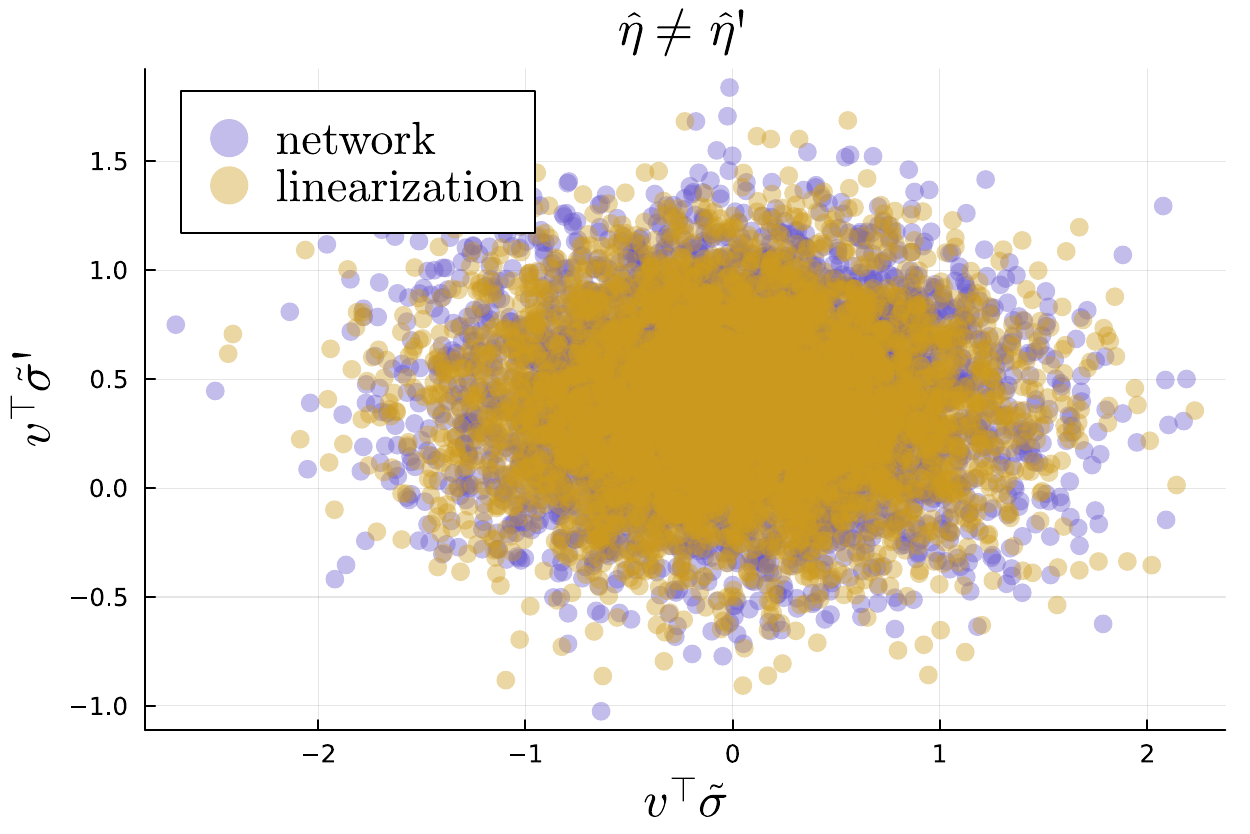}
            & & \includegraphics[width=0.45\linewidth]{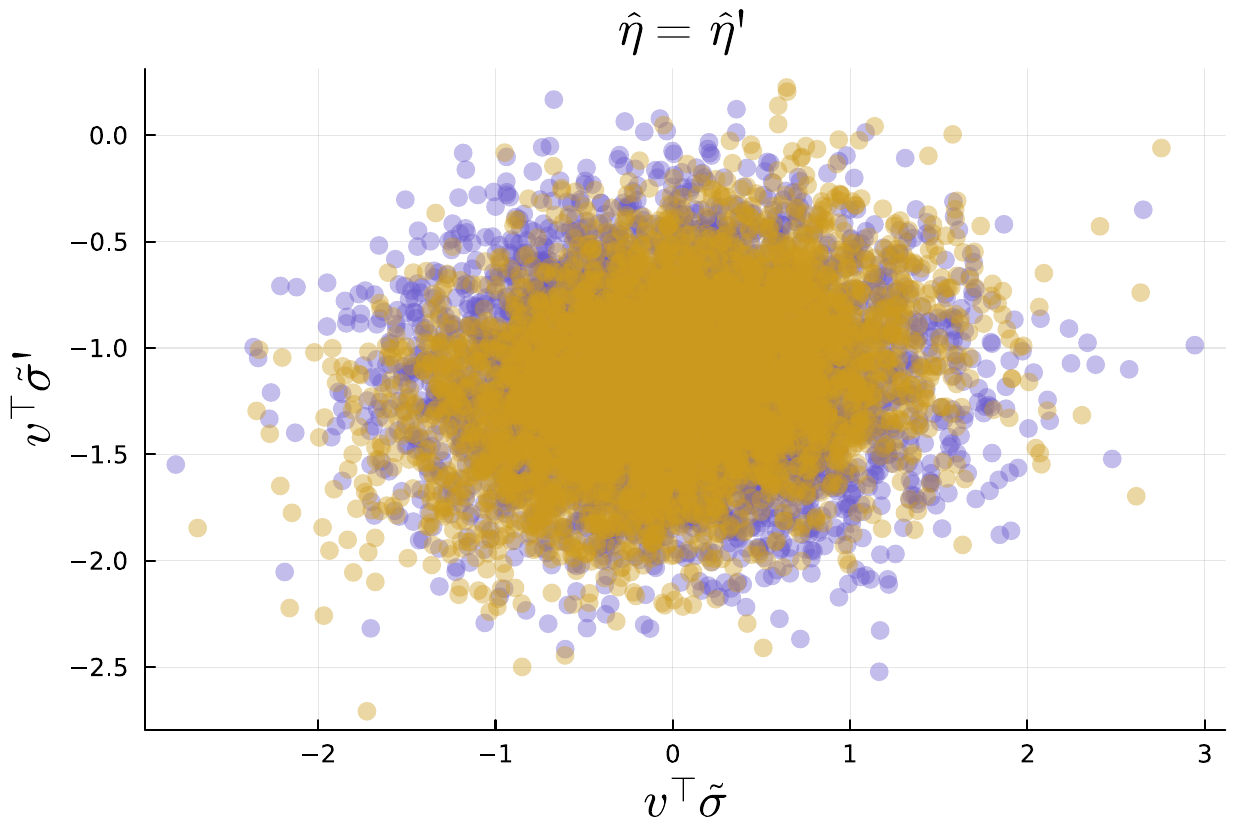} \\
            \includegraphics[width=0.45\linewidth]{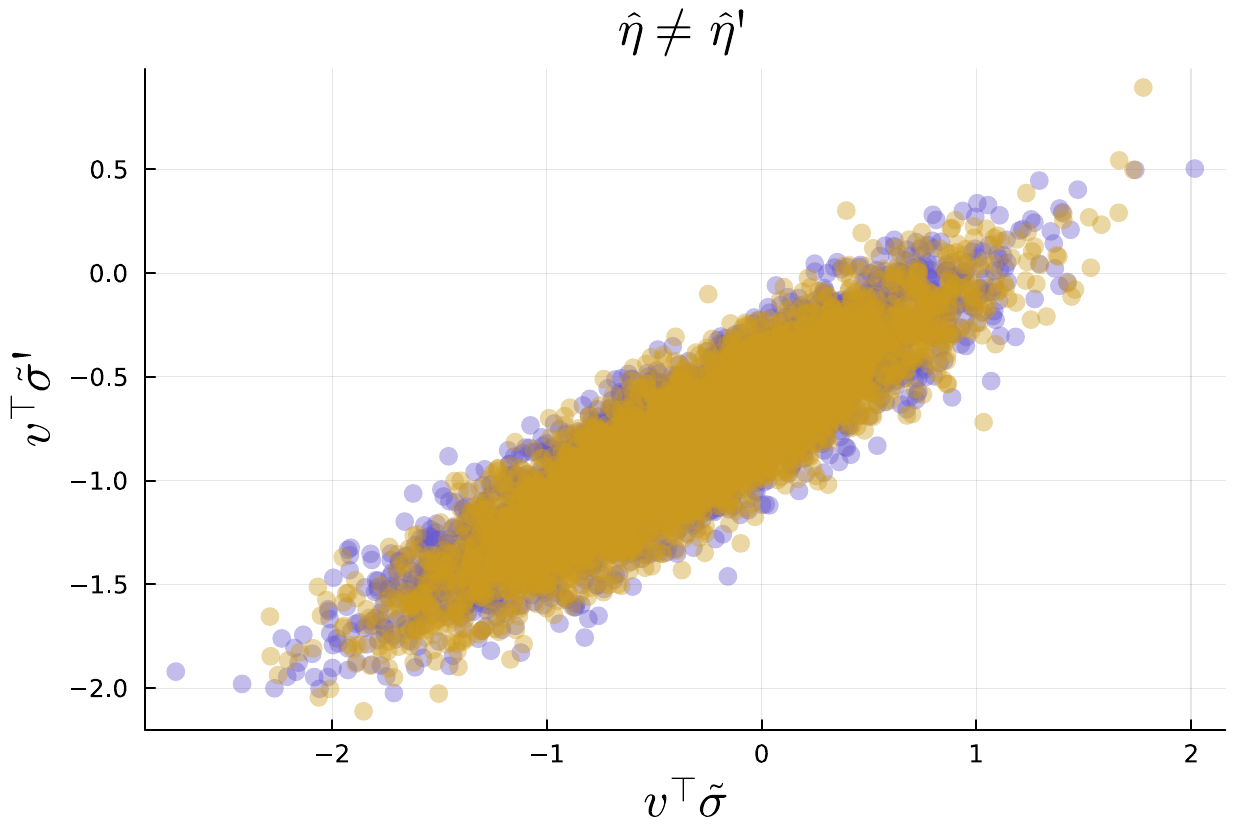}
            & & \includegraphics[width=0.45\linewidth]{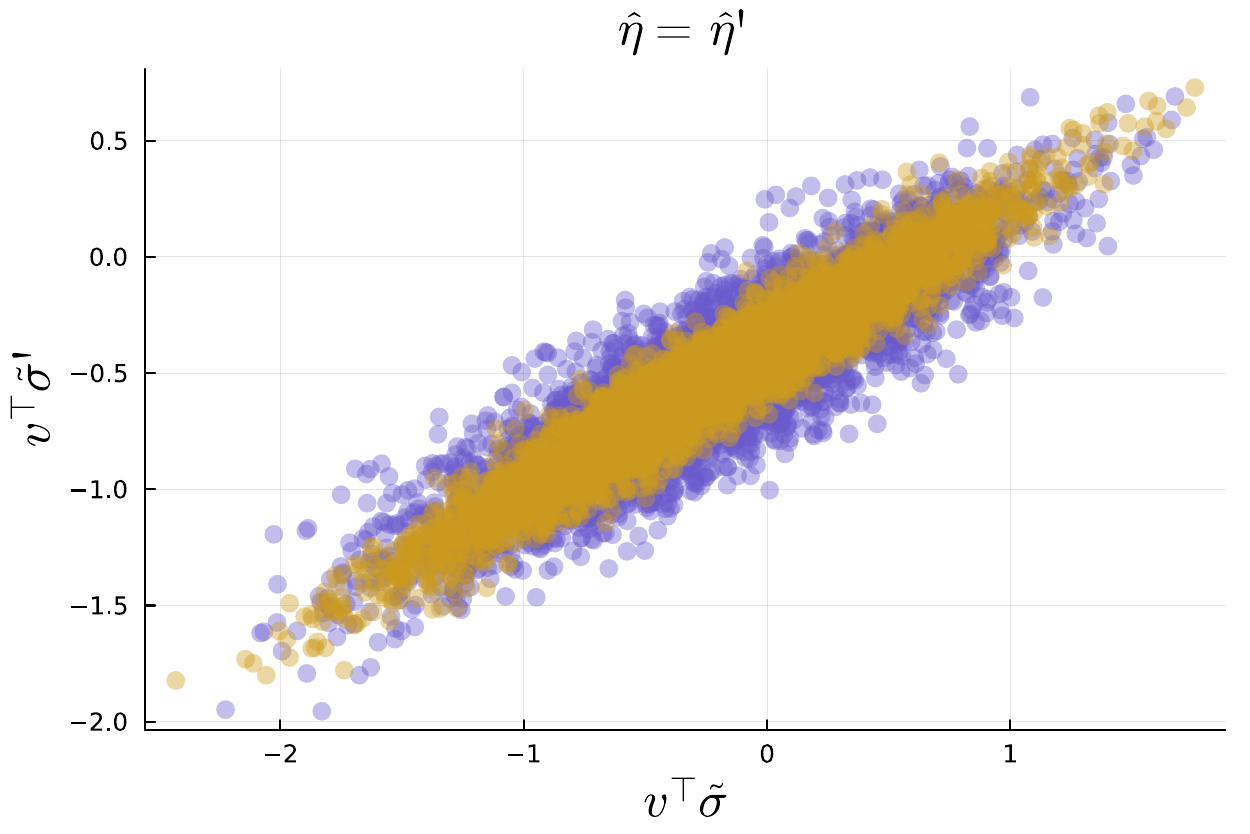} \\
            \end{tabular}
            \caption{\label{fig:GET-eta} Let $\tilde{\sigma}$ be a stand in for the post-activation features of either $f_w(x)$ or the corresponding Gaussian model, equation~\eqref{eq:sobo-GET-GET} (distinguished by the colors indigo and gold, respectively). We compare the joint distribution obtained for the network and its linearization for $v^\top \tilde{\sigma}$ and $v^\top \tilde{\sigma'}$ where $v\sim \mathcal{N}(0,\frac{1}{p}I_p)$. For the left column, we take noise $\hat{\eta}=\hat{\eta}'$ in the Gaussian equivalent model, and in the right column, we consider independent noise. In the first row, we let the error function be the activation in the RF network, and in the second row, we consider the SiLU function.}
    \end{figure}


\section{Replica calculation for subspace Sobolev-type losses}
\label{app:replica}

The replica-based approach~\citep{mezard-parisi-virasoro:1987} we present in this section is standard in the literature, and our presentation closely follows~\citet{gerace-loureiro-krzakala-etal:2021} and \citet{goldt-loureiro-reeves-etal:2022} where similar neural network models are analyzed, though without gradient data. Still, beyond just the necessary calculations, we have included comments and explanations along the way that hopefully make the exposition accessible to a broader audience with no prior exposure to replica techniques.\\ 

We consider the empirical risk minimizer $w^* = \argmin_w \eps_{\text{train}}(w)$ of the regularized loss function
\begin{align}
\eps_{\text{train}}(w) = \frac{1}{n} \sum_{i = 1}^n \left[ \ell \left(y_i,\; f_w(x_i),\; V_k^\top y_i',\; V_k^\top \nabla f_w(x_i) \right) \right] + \frac{\lambda}{2 \alpha} \norm{w}^2\,,
\label{eq:training-obj-sobo}
\end{align}
where $\ell : \RR \times \RR \times \RR^k \times \RR^k \to \RR$ is convex with respect to~$w$, and the $n$ data samples are generated as described in Section~\ref{sec:setup} of the main text. As an example, we have in mind the squared Sobolev $H^1_k$ norm
\begin{align}
\label{eq:loss-sobo-standard}
\ell(y_i, \;f_w(x_i),\; V_k^\top y_i',\;V_k^\top \nabla f_w(x_i)) = \frac{1}{2} (y_i - f_w(x_i))^2 + \frac{1}{2} \left\| V_k^\top y_i' - V_k^\top \nabla f_w(x_i) \right\|^2\,,
\end{align}
though the method applies more generally. Our goal is to calculate, in the proportional asymptotics limit~\eqref{eq:prop-asymp-def}, the expected $H^1_k$ generalization error
\begin{align}
\eps_\text{gen}\left(w^*\right) = \eps_{\text{gen}}^{L^2} \left(w^*\right) + \eps_{\text{gen}}^{H^1_k} \left(w^*\right) =  \EE_{x_0, y_0, y_0'} \left[\left(y_0 - f_{w^*}(x_0)\right)^2 + \norm{V_k^\top \left( y_0' - \nabla f_{w^*}(x_0) \right)}^2 \right]
\label{eq:gen-error-def-appendix}
\end{align}
using a new data sample $(x_0,y_0,y_0')$, independent of the training samples,  while restricting the gradient to the same $k$-dimensional subspace defined via $V_k \in \RR^{d \times k}$ as used for the ``training'' of $w^*$. While the $H^1$ generalization error on the full space would also be interesting, for the reasons discussed in Subsection~\ref{sec:setup}, we restrict ourselves to calculating the subspace error. \\

As a trick that will simplify the calculations later on, we introduce the generalization error on iid copies $\left(x_{0,i}, \, y_{0,i}, \,y_{0,i}' \right)$, $i \in [n]$ as 
\begin{align}
\eps^{n}_{\text{gen}}(w) = \eps^{n, L^2}_{\text{gen}}(w) + \eps^{n, H^1_k}_{\text{gen}}(w)  
= \frac{1}{n} \sum_{i = 1}^n \left( y_{0,i} - f_{w}(x_{0,i}) \right)^2 + \frac{1}{n} \sum_{i = 1}^n \norm{V_k^\top \left( y_{0,i}' - \nabla f_{w}(x_{0,i}) \right)}^2 \,.
\end{align}
Clearly, we have
\begin{equation}
\eps_\text{gen}\left(w^*\right) = \EE_{X_0, Y_0, Y_0'} \left[\eps^{n}_{\text{gen}}\left(w^*\right) \right] \,,
\end{equation}
where $X_0 = [x_{0,1}, \dots , x_{0,n}] \in \RR^{d \times n}$, $Y_0 = (y_{0,1}, \dots, y_{0,n})^\top \in \RR^n$, and $Y_0' = (y_{0,1}', \dots, y_{0,n}') \in \RR^{d \times n}$. The advantage is that the generalization error, when written in this way, becomes structurally similar to the training error.\\

We assume throughout all of the following calculations that---conditional on the alignment $\varpi = V_k^\top \theta_0$---the training and generalization errors, as well as all overlap parameters~\eqref{eq:overlap-def} to be introduced below, concentrate onto their (conditional) expectations in the proportional asymptotics limit~\eqref{eq:prop-asymp-def}.

\subsection{Defining a distribution with inverse temperature $\beta$ for the weights}

The first step consists of mapping the problem to the standard framework of statistical mechanics. For this purpose, we introduce a Gibbs distribution with inverse temperature $\beta > 0$ for the weights $w \in \RR^p$ and consider the corresponding canonical partition function $Z_\beta(h) > 0$ with a ``homogeneous external field'' $h \geq 0$ given by 
\begin{align}
Z_\beta(h) &:= \int_{\RR^p} \exp\left\{-\beta n \left( \eps_{\text{train}}(w) + h \eps_{\text{gen}}^{n}(w)\right)\right\} \dd w \nonumber\\
&= \int_{\RR^p} \exp\left\{ -\beta \sum_{i \in [n]} \ell \left(y_i,\; f_w(x_i),\; V_k^\top y_i',\; V_k^\top \nabla f_w(x_i) \right) -\beta h n \, \epsilon_\text{gen}^{n}(w)\right \}\exp\bigg\{-p \frac{\beta\lambda}{2} \norm{w}^2\bigg\} \dd w \nonumber\\
&= \left(\frac{2 \pi}{\beta \lambda p} \right)^{p/2} \int_{\RR^p} \exp\left\{ -\beta \sum_{i \in [n]} \ell \left(y_i,\; f_w(x_i),\; V_k^\top y_i',\; V_k^\top \nabla f_w(x_i) \right) -\beta h n \epsilon_\text{gen}^{n}(w)\right \} \dd \rho(w)\,,
\label{eq:partition-grad}
\end{align}
where $\rho$ denotes the normalized Gaussian probability measure for $w$ stemming from the Tikhonov regularization term in~\eqref{eq:training-obj-sobo}. Taking the low-temperature limit $\beta \to \infty$ and setting $h=0$, the Gibbs distribution
$Z_\beta^{-1}(0) \exp \left\{-\beta n \eps_{\text{train}}(w) \right\} \dd w$
then concentrates onto the unique minimizer $w^*$ of the training loss, thus recovering the empirical risk minimization setup in this limit. We are then interested in the so-called free energy density $f_\beta$ in the proportional asymptotics limit:
\begin{align}
f_\beta(h) := - \plim_{p\to\infty} \frac{1}{p} \left[ \log \left(\left(\frac{p}{2 \pi} \right)^{p/2} Z_{\beta}(h) \right) \right] = - \plim_{p\to\infty} \frac{1}{p} \left[ \log \hat{Z}_{\beta}(h) \right]\,,
\label{eq:free-ener-def-grad}
\end{align}
where we write $\hat{Z}_{\beta}(h) := \left(p / (2 \pi) \right)^{p/2} Z_{\beta}(h)$ for a conveniently rescaled partition function (the additional prefactor $\left(p / (2 \pi) \right)^{p/2}$ in $\hat{Z}_\beta$ makes the free energy density itself well-defined in the proportional asymptotics limit and can hence be thought of as removing an otherwise logarithmically diverging additive constant;  as such, it does not change any of the derivatives or saddle-point equations we actually need to compute in subsequent sections). As mentioned above, we assume the free energy is \emph{self-averaging} in the proportional asymptotics limit if conditioned on $\varpi$, meaning 
\begin{equation}
f_\beta(h) = \EE_{V_k, \theta_0, \Theta, X,Y,Y',X_0,Y_0,Y_0' \; \mid \; \varpi}[f_\beta(h)],
\end{equation}
where we assume we can freely interchange limits and expectations in the right hand side. We will often abbreviate $\EE_{V_k, \theta_0, \Theta, X,Y,Y',X_0,Y_0,Y_0' \; \mid \; \varpi} = \EE_{\mid \varpi} = \EE$ for brevity whenever it should be clear from context which expectation is taken. The free energy density~\eqref{eq:free-ener-def-grad} is the key quantity to compute since  the training and generalization errors can be formally obtained from $f_\beta$ via differentiation with 
\begin{align}
\plim_{p \to \infty}  \eps_{\text{train}}(w^*) \mid \varpi &= \frac{1}{\alpha} \lim_{\beta \to \infty} \partial_\beta f_\beta (0)
\label{eq:training-err-derivative-grad}\\
\plim_{p \to \infty} \eps_{\text{gen}}(w^*) \mid \varpi &= \frac{1}{\alpha} \lim_{\beta \to \infty} \frac{1}{\beta} \partial_h f_\beta(0)\,.
\label{eq:generalization-err-derivative-grad}
\end{align}
The task is hence to compute the free energy density~\eqref{eq:free-ener-def-grad} in the high-dimensional limit. Once found, differentiating it with respect to the temperature or external field and taking the low-temperature limit yields the training and generalization error that we want to compute. Structurally, we are dealing with a free energy density $f_\beta = \EE \log Z_\beta$ with two nested expectations. This setup is analogous to disordered systems in statistical physics with quenched disorder: the weights $w$ play the role of, e.g.,\ spins with Hamiltonian $\eps_{\text{train}}(w) + h \eps_{\text{gen}}(w)$, and the training data and other model parameters lead to random parameters in the Hamiltonian. The proportional asymptotics limit corresponds to the thermodynamic limit of large system size. Consequently, we can evaluate the free energy density using the well-known \textit{replica trick} as shown below.

\subsection{Replica trick}

A standard approach to calculating the free energy density in this setup is to convert expectations of logarithms into expectations of moments following the evident identity 
\begin{align}
\EE \left[ \log \hat{Z}_{\beta}(h) \right] = \lim_{R \downarrow 0} \frac{1}{R}\log \EE \left[\hat{Z}_{\beta}^R(h)\right] = \lim_{R \downarrow 0} \dv{}{R} \log \EE \left[\hat{Z}_{\beta}^R(h)\right]\,.
\end{align}
The basic idea to calculate~$f_\beta$ is then to exchange the limits $\plim_{p \to \infty}$ and $\lim_{R \downarrow 0}$ and evaluate the expectation $\EE\left[\hat{Z}_{\beta}^R \right]$ at fixed integer~$R$ in the proportional asymptotics limit~\eqref{eq:prop-asymp-def} using the saddlepoint method with $p$ as a large parameter:
\begin{align} 
f_\beta(h) &= - \plim_{p\to\infty} \frac{1}{p} \EE \left[ \log \hat{Z}_{\beta}(h) \right]
= - \plim_{p\to\infty} \frac{1}{p} \lim_{R \downarrow 0} \frac{1}{R}\log \EE \left[\hat{Z}_{\beta}^R(h)\right] = - \lim_{R \downarrow 0} \frac{1}{R} \plim_{p\to\infty} \frac{1}{p}\log \EE \left[\hat{Z}_{\beta}^R(h)\right]\,.
\label{eq:fbeta-grad}
\end{align}
The limit $R \downarrow 0$ is then evaluated afterwards via analytic continuation from integer $R$ to noninteger $R$ for a suitably parameterized ansatz for the saddlepoint evaluation.\\

Momentarily restricting to integer $R \in \NN$ allows us to express the powers of the partition function by introducing \textit{replicas} $w^{(r)}$, $r = 1,\dots, R$, of the weight vector so that (setting $h = 0$ for now for brevity)
\begin{align}
&\EE \left[\hat{Z}_{\beta}^R(0) \right]
 = (\beta \lambda)^{-pR/2} \EE \bigg[  \int_{\RR^p}  \dots \int_{\RR^p} \prod_{r \in [R]} \prod_{i\in[n]} L_{\beta}\left(y_i, \; f_{w^{(r)}}(x_i), \; V_k^\top y_i', \; V_k^\top \nabla f_{w^{(r)}}(x_i) \right) \dd\rho \left(w^{(r)}\right) \bigg] \nonumber\\
 &\overset{\text{training data iid}}{=}  (\beta \lambda)^{-pR/2} \int \prod_{r \in [R]} \dd\rho\left(w^{(r)}\right) \,\left( \EE_{x,y,y' \mid  \theta_0, \Theta, V_k} \bigg[   \prod_{r \in [R]} L_{\beta}\left(y,\; f_{w^{(r)}}(x),\; V_k^\top y', \;V_k^\top \nabla f_{w^{(r)}}(x) \right) \bigg] \right)^n\,.
\label{eq:ezr-simple-grad}
\end{align}
Here, we have introduced the notation
\begin{align}
L_{\beta}\left(y, \; \widetilde{y}, \; V_k^\top y', \; V_k^\top \widetilde{y}\,' \right) &:= \exp \left\{-\beta \,  \ell \left(y,\; \widetilde{y}, \; V_k^\top y', \; V_k^\top \widetilde{y}\,' \right) \right\}
\end{align}
for the exponentiated and temperature-weighted loss function.
A key observation is that the expectation over the training data in~\eqref{eq:ezr-simple-grad} admits a low-dimensional representation that can instead be expressed as an expectation over $O(1)$ many random variables. We can thus rewrite~\eqref{eq:ezr-simple-grad} as
\begin{align}
&\EE \left[\hat{Z}_{\beta}^R(0) \right] = (\beta \lambda)^{-pR/2} \times \nonumber\\
&\qquad \quad \times \EE_{\theta_0, \Theta, V_k, W \mid \varpi} \bigg[ 
    \bigg( \int_{\RR^{k+1}} \dd \Upsilon \; \int_{\RR^{(k+1)R + 1}} \dd \nu \left(\omega, \Upsilon^{(1:R)} \mid \theta_0, \Theta, V_k, W \right) P_\text{data}\left(\Upsilon \mid \varpi, \omega \right)     L^\Pi_{\beta}\left(\Upsilon, \Upsilon^{(1:R)}\right) \bigg)^n \bigg],
    \label{eq:ezr-grad}
\end{align}
with replicated weight matrix $W = \begin{bmatrix} w^{(1)} & \ldots & w^{(R)} \end{bmatrix} \in \RR^{p \times R}$ distributed according to the product measure $\dd\rho^{\otimes R}(W) = \prod_{r \in [R]} \dd\rho \left(w^{(r)} \right)$,
and where we defined
\begin{align}
\begin{cases}
\omega &= \left \langle \theta_0, x \right \rangle \in \RR \,,\\ 
\Upsilon &= \left(y, V_k^\top y' \right)^\top \in \RR^{k + 1} \,,
\end{cases}
\end{align}
so that $P_\text{data}$ describes the conditional distribution of the data $\Upsilon$ given the teacher vector projection $\omega$ of $x$ and the subspace alignment $\varpi = V_k^\top \theta_0$. Additionally, we write
\begin{equation}
\Upsilon^{(r)} = \left(y^{(r)}, \; (V_k^\top y')^{(r)} \right)^\top := \left( f_{w^{(r)}}(x), \; V_k^\top \nabla f_{w^{(r)}}(x) \right)^\top = \left(\sigma \left( \Theta^\top x \right)^\top w^{(r)}, \; V_k^\top \Theta \textsc{diag}\left(\sigma'\left(\Theta^\top x \right) \right) w^{(r)} \right) \in \RR^{k + 1} \,,
\end{equation}
for the corresponding network output and its projected gradient under the replicated weight vector $w^{(r)}$. We further abbreviated all network outputs as the tuple $\Upsilon^{(1:R)} = (\Upsilon^{(1)}, \ldots, \Upsilon^{(R)})$ and the product exponentiated loss function
\begin{equation}
L^\Pi_{\beta}\left(\Upsilon, \; \Upsilon^{(1:R)}\right) := \prod_{r \in [R]} L_{\beta}\left(y, \; y^{(r)}, \; V_k^\top y',\; (V_k^\top y')^{(r)}\right)\,.
\end{equation} 
Lastly, $\nu$ denotes the joint distribution of $\left(\omega,  \Upsilon^{(1:R)}\right)$ for fixed $W$, $\Theta$, $\theta_0$, and $V_k$, so that the \textit{only} randomness comes from marginalizing over $x \sim \mathcal{N}\left(0, I_d \right)$. Finally, reintroducing a nonzero external field $h \neq 0$ for the generalization error term leads to the following, analogous result:
\begin{align}
&\EE \left[\hat{Z}_{\beta}^R(h) \right] = (\beta \lambda)^{-pR/2} \times \nonumber\\
&\qquad \quad \times \EE_{\theta_0, \Theta, V_k, W \mid \varpi} \bigg[ 
    \bigg( \int_{\RR^{k+1}} \dd \Upsilon \; \int_{\RR^{(k+1)R + 1}} \dd \nu \left(\omega, \Upsilon^{(1:R)} \mid \theta_0, \Theta, V_k, W \right) P_\text{data}\left(\Upsilon \mid \varpi, \omega \right)     L^\Pi_{\beta}\left(\Upsilon, \Upsilon^{(1:R)}\right) \bigg)^n \times \nonumber\\
   & \qquad \qquad \qquad \times  \bigg( \int_{\RR^{k+1}} \dd \Upsilon_0 \; \int_{\RR^{(k+1)R + 1}} \dd \nu \left(\omega_0, \Upsilon^{(1:R)}_0 \mid \theta_0, \Theta, V_k, W \right) P_\text{data}\left(\Upsilon_0 \mid \varpi, \omega_0 \right)  \tilde{L}^\Pi_{\beta h}\left(\Upsilon_0, \Upsilon^{(1:R)}_0\right) \bigg)^n \bigg]\,.
   \label{eq:ezr-grad-with-h}
\end{align}
Here, the subscript-$0$ variables in the last line correspond to the independent and iid copies of the data for the generalization error, and the only difference is in the exponentiated loss function
\begin{equation}
\tilde{L}^\Pi_{\beta h}\left(\Upsilon_0, \Upsilon^{(1:R)}_0\right) := \prod_{r \in [R]} \tilde{L}_{\beta h}\left(y_0,\; y^{(r)}_0,\; V_k^\top y'_0,\; (V_k^\top y'_0)^{(r)}\right)
\end{equation} 
with
\begin{align}
\tilde{L}_{\beta h}\left(y_0,\; \widetilde{y}_0,\; V_k^\top y'_0,\; V_k^\top \widetilde{y}_0\,' \right) &:= \exp \left\{-\beta h \left[ \left(y_0 - \widetilde{y}_0\right)^2 + \norm{V_k^\top \left( y_0' - \widetilde{y}_0\,' \right)}^2 \right] \right\}\,.
\end{align}

\subsection{Identifying the necessary overlap parameters under the Gaussian equivalence theorem}

Rewriting the replicated partition function as in~\eqref{eq:ezr-grad-with-h} allows us to apply the Gaussian equivalence theorem to the measure $\nu$ in the proportional asymptotics limit. Conditional on a fixed realization of the feature matrix~$\Theta$, teacher vector~$\theta_0$, and projector~$V_k$, this theorem states that the random variables
\begin{align}
\begin{cases}
z &= \sigma\left(\Theta^\top x\right) \in \RR^p \\
z' &= \sigma'\left(\Theta^\top x\right) \in \RR^p
\end{cases}
\end{align}
can, for the purposes of calculating finite-dimensional summary statistics, be replaced in the proportional asymptotics regime by the Gaussian random variables
\begin{align}
\begin{cases}
\widetilde{z} &= \kappa_0 \1_p + \kappa_1 \Theta^\top x + \kappa_* \hat{\eta}\,,\\
\widetilde{z}\,' &= \kappa_0' \1_p + \kappa_1' \Theta^\top x + \kappa_*' \hat{\eta}'\,,
\end{cases}
\end{align}
with independent Gaussian noises $\hat{\eta}, \hat{\eta}' \sim \mathcal{N}(0, I_p)$ independent of everything else, and Hermite coefficients $\kappa$ and $\kappa'$ as defined in~\eqref{eq:hermite-coeff-def}.
We transform $z$ to the \emph{replicated function data} $y^{(r)} = \widetilde{z}\,^\top w^{(r)}$ and $z'$ to \emph{replicated gradient data} 
$(V_k^\top y')^{(r)} = V_k^\top \Theta \, \textsc{diag}(\widetilde{z}\,')\, w^{(r)}$.
In other words, in the proportional asymptotics limit and conditioned on all random variables other than $(x,\hat{\eta},\hat{\eta}')$, the random variables
$\left( \omega, \;y^{(1:R)}, \;\left( V_k^\top y' \right)^{(1:R)} \right) \in \RR^{(k+1)R + 1}$ are equivalent in law to a Gaussian random variable with mean and covariance
\begin{align}
\mu = \begin{pmatrix}
0 \\[3pt] \kappa_0 W^\top \1_p \\[3pt] \kappa_0' V_k^\top \Theta w^{(1)} \\ \vdots \\[3pt] \kappa_0' V_k^\top \Theta w^{(R)} 
\end{pmatrix}\,, \qquad
\label{eq:overlap-grad}
\Sigma =  \begin{pmatrix} \norm{\theta_0}^2  & \kappa_1  \theta_0^\top \Theta W & \Sigma_{31}^\top \\[6pt]
\kappa_1  W^\top \Theta^\top \theta_0 & W^\top \left(\kappa_1^2 \Theta^\top \Theta + \kappa_*^2 I_p \right) W & \Sigma_{32}^\top \\[6pt]
\Sigma_{31} & \Sigma_{32} & \Sigma_{33} \end{pmatrix} \;,
\end{align}
where 
\begin{equation}
\Sigma_{31} = \begin{pmatrix} \Sigma_{31}^{(1)} \\ \vdots \\ \Sigma_{31}^{(R)} \end{pmatrix} \in \RR^{kR}, \quad \Sigma_{32} = \begin{pmatrix} \Sigma_{32}^{(1)} \\ \vdots \\ \Sigma_{32}^{(R)}\end{pmatrix} \in \RR^{kR \times R}, \quad \Sigma_{33} =  \begin{pmatrix} \Sigma_{33}^{(1,1)} & \ldots & \Sigma_{33}^{(1,R)} \\
\vdots & \ddots & \\
\Sigma_{33}^{(R,1)} & & \Sigma_{33}^{(R,R)}\end{pmatrix} \in \RR^{kR \times kR},
\end{equation}
with
\begin{align}
\begin{cases}
\Sigma_{31}^{(r)} &= \kappa_1' V_k^\top \Theta \textsc{diag}(w^{(r)}) \Theta^\top \theta_0 \in \RR^{k}\\[4pt]
\Sigma_{32}^{(r)} &= \kappa_1 \kappa_1' V_k^\top \Theta \;\textsc{diag}(w^{(r)})  \Theta^\top \Theta W \in \RR^{k \times R}\\[4pt]
\Sigma_{33}^{(r,r')} &= V_k^\top \Theta \;\textsc{diag}(w^{(r)}) \left(  (\kappa_1')^2  \Theta^\top \Theta + (\kappa_*')^2 
I_p \right)\textsc{diag}(w^{(r')}) \Theta^\top V_k \in \RR^{k \times k}\,.
\end{cases}
\end{align}

\subsection{Saddlepoint form of the replicated partition function under the Gaussian equivalence theorem}

Since $\nu \to {\cal N}(\mu, \Sigma)$ in~\eqref{eq:ezr-grad-with-h} becomes Gaussian in the proportional asymptotics limit, the measure will only depend on the finite-dimensional parameters in $\mu$ and $\Sigma$. Consequently, we can factor out the dependency on these parameters as follows: we introduce the matrices
\begin{equation}
\mu' = \begin{pmatrix} 0 \\[4pt] S_a' \\[4pt] S_b' \end{pmatrix}, \quad  \Sigma' = \begin{pmatrix} \rho_a & F_a'^\top & F_b'^\top \\[6pt] F_a' & Q_a' & Q_b'^\top \\[6pt] F_b' & Q_b' & Q_c' \end{pmatrix},
\end{equation}
with overlap parameters $\rho_a \in \RR$, $S_a', F_a' \in \RR^{R}$, $S_b', F_b' \in \RR^{kR}$, $Q_a' = (Q_a')^\top \in \RR^{R \times R}$, $Q_b' \in \RR^{kR \times R}$, and $Q_c' = (Q_c')^\top \in \RR^{kR \times kR}$. These terms will act as integration variables below. We collect them into the tuple of parameters
\begin{equation}
t = ( \rho_a, S_a',\, S_b',\, F_a',\, F_b',\, \textsc{vech}(Q_a'),\, Q_b',\, \textsc{vech}(Q_c')) \in \RR^{d_t}\,,
\end{equation}
with dimension
\begin{align}
d_t = \underbrace{1}_{\rho_a} + \underbrace{R}_{S_a'} + \underbrace{kR}_{S_b'} + \underbrace{R}_{F_a'} + \underbrace{kR}_{F_b'} + \underbrace{\frac{R (R+1)}{2}}_{\textsc{vech}(Q_a')} + \underbrace{kR^2}_{Q_b'} + \underbrace{R \frac{k(k+1)}{2} + \frac{R (R-1)}{2} k^2}_{\textsc{vech}(Q_c')}\,,
\end{align}
where we have taken care to build in the symmetry requirements for $Q_a'$ and block-symmetry of $Q_c'$. We collect the true parameters $\mu, \Sigma$ for given realizations of $\theta_0, \Theta, W$ as
\begin{align}
T\left( \theta_0, \Theta, V_k, W \right) = (\mu(\theta_0, \Theta,  V_k, W),  \,\textsc{vech}(\Sigma(\theta_0, \Theta,  V_k, W)))\,.
\end{align}
Note that from the definition of the mean in~\eqref{eq:overlap-grad}, if $\kappa_0 = 0$, then it will not be necessary to introduce $S_a'$. Similarly, if $\kappa_0' = 0$, the variable $S_b'$ is not needed. All of the following calculations are carried out assuming that $\kappa_0, \kappa_0' \neq 0$; to convert to cases where either is zero, one would simply remove the variables $S_a'$ or $S_b'$, respectively, or their replica-symmetric forms $s_a$ and $s_b$ below.\\

Now, by inserting the Dirac delta identity
\begin{equation}
\label{eq:dirac_overlap-grad}
1 = \int_{\RR^{d_t}} \dd t \; \delta( \mu'(t) - \mu(\theta_0, \Theta,  V_k, W) ) \, \delta( \textsc{vech}(\Sigma'(t)) - \textsc{vech}(\Sigma(\theta_0, \Theta,  V_k, W))) = \int_{\RR^{d_t}} \dd t \; \delta(t - T(\theta_0, \Theta,  V_k, W))
\end{equation}
into~\eqref{eq:ezr-grad-with-h}, we obtain (again with $h=0$ temporarily for brevity)
\begin{align}
&\EE\left[\hat{Z}_{\beta}^R(0)\right]
=  (\beta \lambda)^{-pR/2} \times \nonumber\\
&\times \int_{\RR^{d_t}} \dd t \bigg( \int_{\RR^{k+1}} \dd \Upsilon \int_{\RR^{(k+1)R + 1}} \dd \nu \left(\omega, \Upsilon^{(1:R)} \,|\, t\right) P_\text{data}\left(\Upsilon \,|\, \varpi, \omega\right)     L^\Pi_{\beta}\left(\Upsilon, \Upsilon^{(1:R)}\right) \bigg)^n  \EE_{\mid \varpi}[\delta(t - T(\theta_0, \Theta,  V_k, W))].
\label{eq:ezr-interm-delta-grad}
\end{align}
Finally, we replace the delta function by integration over the dual Fourier variables
\begin{align}
\label{eq:dual_overlap-grad}
\hat{t} = \left(\hat{\rho}_a, \, \hat{S}_a', \;\hat{S}_b', \;\hat{F}_a', \;\hat{F}_b', \; \textsc{vech}(\hat{Q}_a'), \; \hat{Q}'_b, \; \textsc{vech}(\hat{Q}_c') \right) \in i\RR^{d_t} \;.
\end{align}
Assuming the forward Fourier transform convention $\mathcal{F}[f](k) = \int_\RR f(x)\exp\left\{ikx\right\}\dd x$, we have
\begin{align}
\delta(t - T(\theta_0, \Theta, V_k, W)) = (2\pi i)^{-d_t} \int_{i\RR^{d_t}} \exp \left\{ -\langle \hat{t}, t - T(\theta_0, \Theta, V_k, W) \rangle \right\} \dd \hat{t},
\end{align}
with the vectorized (i.e., flattened) inner product $\left \langle \cdot, \cdot \right \rangle$ above.\\

Altogether, these results let us express the replicated partition function~\eqref{eq:ezr-interm-delta-grad} as
\begin{align}
\EE\left[\hat{Z}_{\beta}^R(0)\right]  
&=  (2\pi i)^{-d_t} (\beta \lambda)^{-pR/2} \int_{\RR^{d_t}} \dd t  \int_{i\RR^{d_t}} \dd\hat{t} \; \exp \left\{-\langle t, \hat{t}\rangle\right\} \EE\left[ e^{\langle \hat{t}, T(\theta_0, \Theta, V_k, W) \rangle} \right] \times \nonumber \\
&\qquad\times 
    \bigg( \int_{\RR^{k+1}} \dd \Upsilon \int_{\RR^{(k+1)R + 1}} \dd \rho \left(\omega, \Upsilon^{(1:R)} \mid t\right) P_\text{data}\left(\Upsilon \mid \omega, \varpi \right) L^\Pi_{\beta}\left(\Upsilon, \Upsilon^{(1:R)}\right) \bigg)^n \nonumber\\
&= (2\pi i)^{-d_t}  \int_{\RR^{d_t}} \dd t \int_{i\RR^{d_t}}  \dd \hat{t} \; \exp \left\{ p \Phi^{(R)} (t, \hat{t}) \right\}
\label{eq:zr-laplace-form-grad}
\end{align}
with rate function 
\begin{align}
\Phi^{(R)}\left(t, \hat{t} \right) = \Psi_y^{(R)}(t) + \Psi_w^{(R)}\left(\hat{t} \right) - \frac{1}{p} \left \langle t, \hat{t} \right \rangle
\label{eq:phi-r-grad}
\end{align}
and potentials
\begin{align}
\begin{cases}
\Psi^{(R)}_y(t) &= \alpha \log \lp \int_{\RR^{k+1}} \dd \Upsilon \int_{\RR^{(k+1)R + 1}} \dd \nu \left( \omega, \Upsilon^{(1:R)} \mid t \right) \; L^\Pi_{\beta} \left( \Upsilon, \Upsilon^{(1:R)} \right)  P_\text{data}\left(\Upsilon \mid \omega, \varpi \right)  \rp \,,\\[8pt]
\Psi^{(R)}_w\left(\hat t\right) &= \frac{1}{p} \log \left( ( \beta \lambda)^{-pR/2} \EE_{\mid \varpi} \left[ \exp\{\langle \hat{t}, T(\theta_0, \Theta, V_k, W) \rangle\} \right] \right)\,.
\end{cases}
\label{eq:repl-potentials}
\end{align}
Reinstating $h \neq 0$ adds a third potential
\begin{align}
\Psi^{(R)}_{y_0}(t) = \alpha \log \lp \int_{\RR^{k+1}} \dd \Upsilon_0 \int_{\RR^{(k+1)R + 1}} \dd \nu \left( \omega_0, \Upsilon^{(1:R)}_0 \mid t \right) \; \tilde{L}^\Pi_{\beta h} \left( \Upsilon_0, \Upsilon^{(1:R)}_0 \right)  P_\text{data}\left(\Upsilon_0 \mid \omega_0, \varpi \right)\rp
\end{align}
to the rate function $\Phi^{(R)}\left(t,\hat{t}\right)$. A saddlepoint evaluation of~\eqref{eq:zr-laplace-form-grad} in the asymptotic scaling limit will now yield
\begin{align}
\plim_{p\to\infty} \frac{1}{p} \log \EE\left[\hat{Z}_{\beta}^R\right] =  \text{crit}_{t,\hat{t} \in \CC^{d_t}} \; \Phi^{(R)}\left(t,\hat{t}\right)\,,
\label{eq:saddle-ezr-grad}
\end{align}
where $\text{crit} \; \Phi^{(R)}$ denotes the value of the function $\Phi^{(R)}$ at its critical point.  We proceed with the evaluation of the right-hand side now, by inserting a replica-symmetric ansatz for $t$ and $\hat{t}$, simplifying $\Phi^{(R)}$ in this case, and then taking the limit  
\begin{align}
f_\beta = - \lim_{R \downarrow 0} \frac{1}{R} \; \text{crit}_{t,\hat{t} \in \CC^{d_t}} \; \Phi^{(R)}\left(t,\hat{t}\right)
\label{eq:fbeta-saddle-grad}
\end{align}
for this ansatz. In fact, we will derive the optimality conditions directly for the saddle-point of $\Phi := \lim_{R \downarrow 0}\Phi^{(R)}/R$ within the replica-symmetric ansatz in this limit as detailed below. In the end, taking $\beta \to \infty$ recovers the original training problem setup.
 
\subsection{Replica-symmetric ansatz for the saddlepoint problem}
\label{sec:sobolev_replica_ansatz}

We consider the replica-symmetric ansatz for the overlap parameters 
\begin{equation}
\label{eq:rep-sym-mu}
\mu_\textrm{sym} = \begin{bmatrix} 0 \\ s_a\1_R \\ \textsc{vec}\left(\1_R \otimes\, s_b\right) \end{bmatrix} \in \RR^{1 + R + kR}, \quad
\hat{\mu}_\textrm{sym} = \sqrt{p} \cdot \begin{bmatrix} 0 \\ \hat{s}_a\1_R \\ \textsc{vec}\left(\1_R \otimes \,\hat{s}_b\right) \end{bmatrix} \in \RR^{1 + R + kR}
\end{equation}
and
\begin{align}
\label{eq:rep-sym-sigma}
	\Sigma_\textrm{sym} \ = 
    \ \begin{bmatrix}
	\rho_a      & f_a \1_R^\top    & \textsc{vec}\left(\1_R^\top \otimes f_b^\top\right) \\
	f_a \1_R & \1_R^{\otimes 2} q_{a} + I_R \left(r_a - q_a \right) & \1_R^{\otimes 2} \otimes q_{b}^\top + I_R \otimes \left(r_b - q_b \right)^\top \\
	\textsc{vec}\left(\1_R \otimes f_b\right) & \1_R^{\otimes 2} \otimes q_b + I_R \otimes \left(r_b - q_b \right) & \1_R^{\otimes 2} \otimes q_c + I_R \otimes \left(r_c - q_c \right) \\
	\end{bmatrix} \in \RR^{(1 + R + kR) \times (1 + R + kR)},
\end{align}
as well as
\begin{align}
\label{eq:rep-sym-sigma-hat}
	\hat{\Sigma}_\textrm{sym} \ = 
    p \cdot \begin{bmatrix}
	\gamma \hat{\rho}_a      & \hat{f}_a \1_R^\top    & \textsc{vec}\left(\1_R^\top \otimes \hat{f}_b^\top\right) \\
	\hat{f}_a \1_R & \1_R^{\otimes 2} \hat{q}_a + I_R \left(\hat{r}_a - \hat{q}_a \right) & \1_R^{\otimes 2} \otimes \hat{q}_b^\top + I_R \otimes \left(\hat{r}_b - \hat{q}_b \right)^\top \\
	\textsc{vec}\left(\1_R \otimes \hat{f}_b\right) & \1_R^{\otimes 2} \otimes \hat{q}_b + I_R \otimes \left(\hat{r}_b - \hat{q}_b \right) & \1_R^{\otimes 2} \otimes \hat{q}_c + I_R \otimes \left(\hat{r}_c - \hat{q}_c \right) \\
	\end{bmatrix}  \in \RR^{(1 + R + kR) \times (1 + R + kR)},
\end{align}
where $\rho_a, s_a, f_a, q_a, r_a \in \RR$, and $s_b, f_b, q_b, r_b\in\RR^k$, and $q_c = q_c^\top \in\RR^{k\times k}$, $r_c = r_c^\top \in\RR^{k\times k}$ and same for their hatted counterparts. Note that we have separated diagonal and off-diagonal terms using $r$ and $q$ here. We collect these terms into the replica-symmetric parameter tuples
\begin{align}
\begin{cases}
t_\mathrm{sym} &= \left( \rho_a, s_a, s_b, f_a, f_b, q_a, r_a, q_b, r_b, q_c, r_c\right) \in \RR^{d_\mathrm{sym}} \\ 
\hat{t}_\mathrm{sym} &= \left(\hat{\rho}_a, \hat{s}_a, \hat{s}_b, \hat{f}_a, \hat{f}_b, \hat{q}_a, \hat{r}_a, \hat{q}_b, \hat{r}_b, \hat{q}_c, \hat{r}_c \right) \in \RR^{d_\mathrm{sym}},
\end{cases}
\end{align}
and define the tuples of corresponding mean and covariances $T_\mathrm{sym} = ( \mu_\mathrm{sym}, \Sigma_\mathrm{sym})$ and 
$\hat{T}_\mathrm{sym} = (\hat{\mu}_\mathrm{sym}, \hat{\Sigma}_\mathrm{sym})$. The number of replica-symmetric overlap parameters is $d_\mathrm{sym} = 5 + 5k  + k^2$ for any $R \in \NN$. Note that the particular choice of scaling by $p$ or $d$ (or functions thereof) for the auxiliary parameters in the ansatz is, in principle, arbitrary at this stage and is chosen in such a way that the expressions for the saddlepoint equations are $O(1)$ and simplify later on.\\

We now simplify the potentials $\Psi_y^{(R)}$ and $\Psi_w^{(R)}$ using this ansatz in Sections~\ref{sec:sobo-psi-y-r} and~\ref{sec:sobo-psi-w-r}. In particular, our goal for these next two subsections is to transform them into a form that is suitable for taking $R \downarrow 0$, i.e., where $R$ is just a parameter that can also take non-integer values. Note that the potential $\Psi_{y_0}^{(R)}$ for the generalization error is structurally very similar to $\Psi_y^{(R)}$. We will hence again only present the subsequent calculations for $\Psi_y^{(R)}$ and immediately give the result for $\Psi_{y_0}^{(R)}$ afterwards.\\

Once the potentials have been simplified, we need to make sure that our ansatz for the critical $t$ and $\hat{t}$ is consistent with the known limit $\lim_{R \downarrow 0}\EE \left[\hat{Z}_{\beta}^R\right] = 1$,
meaning that, by~\eqref{eq:zr-laplace-form-grad}, we must guarantee 
\begin{align}
\lim_{R \downarrow 0} \text{crit}_{t_\mathrm{sym}, \hat{t}_\mathrm{sym}} \Phi^{(R)}(t_\mathrm{sym}, \hat{t}_\mathrm{sym}) = \lim_{R \downarrow 0} \Phi^{(R)}\left(t_\mathrm{sym}^*(R), \hat{t}_\mathrm{sym}^*(R) \right) = 0
\label{eq:phi-r-consistent}
\end{align}
for the critical point $\left(t_\mathrm{sym}^*(R), \hat{t}_\mathrm{sym}^*(R)\right)$ (where we emphasize the $R$ dependence) determined through
\begin{align}
\nabla_{t_\mathrm{sym}}\Phi^{(R)}\left(t_\mathrm{sym}^*(R), \hat{t}_\mathrm{sym}^*(R) \right) = \nabla_{\hat{t}_\mathrm{sym}}\Phi^{(R)}\left(t_\mathrm{sym}^*(R), \hat{t}_\mathrm{sym}^*(R) \right) = 0\,.
\label{eq:phi-r-optimality}
\end{align}
If this consistency condition holds (which is true for the ansatz introduced above, as we show below in Section~\ref{sec:sobo-consistency}), we further have for~\eqref{eq:fbeta-saddle-grad} that the limit $R \downarrow 0$ becomes
\begin{align}
f_\beta &= - \lim_{R \downarrow 0} \frac{1}{R} \; \text{crit}_{t_\mathrm{sym},\hat{t}_\mathrm{sym}} \; \Phi^{(R)}\left(t_\mathrm{sym},\hat{t}_\mathrm{sym}\right) \nonumber\\
&= - \left. \frac{\dd}{\dd R} \right|_{R = 0} \left( \Phi^{(R)}\left(t_\mathrm{sym}^*(R),\hat{t}_\mathrm{sym}^*(R)\right) \right)\\ &=   -\left(\left. \frac{\dd}{\dd R}\right|_{R = 0} \Phi^{(R)} \right)\left(t_\mathrm{sym}^*\left(R \downarrow 0\right), \hat{t}_\mathrm{sym}^*\left(R \downarrow 0\right)\right)\,.
\end{align}
The last equality holds because of the optimality conditions:
\begin{align}
\frac{\dd}{\dd R} \left( \Phi^{(R)}\left(t_\mathrm{sym}^*(R), \hat{t}_\mathrm{sym}^*(R) \right) \right) &= \left(\frac{\dd \Phi^{(R)}}{\dd R} \right) \left( t_\mathrm{sym}^*(R), \hat{t}_\mathrm{sym}^*(R) \right) + \left \langle \nabla_{t_\mathrm{sym}} \Phi^{(R)} \left( t_\mathrm{sym}^*(R), \hat{t}_\mathrm{sym}^*(R)\right), \frac{\dd t_\mathrm{sym}^*(R)}{\dd R} \right \rangle \nonumber\\
&\quad+ \left \langle \nabla_{\hat{t}_\mathrm{sym}} \Phi^{(R)} \left( t_\mathrm{sym}^*(R), \hat{t}_\mathrm{sym}^*(R)  \right), \frac{\dd \hat{t}_\mathrm{sym}^*(R)}{\dd R} \right \rangle \\ &= \left(\frac{\dd \Phi^{(R)}}{\dd R} \right) \left( t_\mathrm{sym}^*(R), \hat{t}_\mathrm{sym}^*(R) \right)\,.
\end{align}
We will then make the following standard assumption in this setting:
\begin{align}
f_\beta = -\Phi\left(t_\mathrm{sym}^*\left(R \downarrow 0\right), \hat{t}_\mathrm{sym}^*\left(R \downarrow 0\right)\right) \overset{!}{=} - \text{crit}_{t_\mathrm{sym}, \hat{t}_\mathrm{sym}} \Phi\left(t_\mathrm{sym}, \hat{t}_\mathrm{sym}\right)\,,
\label{eq:limit-exchange-fbet}
\end{align}
where we introduced
\begin{align}
\Phi \left(t_\mathrm{sym}, \hat{t}_\mathrm{sym}\right) := \left( \left. \frac{\dd}{\dd R}\right|_{R = 0} \Phi^{(R)} \right) \left(t_\mathrm{sym}, \hat{t}_\mathrm{sym}\right)\,.
\end{align}
This assumption is convenient because it reduces the calculation of the free energy density to the study of the critical points of~$\Phi$, i.e., directly in the limit $R \downarrow 0$, instead of $\Phi^{(R)}$, thus simplifying subsequent calculations. We calculate the necessary expressions to get $\Phi$ in our problem in Sections~\ref{sec:psi-y-sobo} and~\ref{sec:psi-w-sobo}. Note that for unique critical points, equation~\eqref{eq:limit-exchange-fbet} is equivalent to saying that the limit as $R \downarrow 0$ of the critical point $\left(t^*(R), \hat{t}^*(R) \right)$ of the function $\Phi^{(R)}$ converges to the critical point of the function $\Phi$. This result is not \emph{a priori} obvious since, by differentiating the optimality condition for $\Phi^{(R)}$ in~\eqref{eq:phi-r-optimality} with respect to $R$ and letting $R \downarrow 0$, we find
\begin{align}
&\left( \begin{array}{c}
\nabla_{t_\mathrm{sym}}\Phi^{(R)}\left(t_\mathrm{sym}^*(R), \hat{t}_\mathrm{sym}^*(R) \right)\\
\nabla_{\hat{t}_\mathrm{sym}}\Phi^{(R)}\left(t_\mathrm{sym}^*(R), \hat{t}_\mathrm{sym}^*(R) \right) 
\end{array} \right) = \left(\begin{array}{c}
0\\0
\end{array} \right)\nonumber \\
&\Rightarrow\left( \begin{array}{c}
\nabla_{t_\mathrm{sym}}\Phi\left(t_\mathrm{sym}^*\left(R \downarrow 0\right), \hat{t}_\mathrm{sym}^*\left(R \downarrow 0\right) \right)\\
\nabla_{\hat{t}_\mathrm{sym}}\Phi\left(t_\mathrm{sym}^*\left(R \downarrow 0\right), \hat{t}_\mathrm{sym}^*\left(R \downarrow 0\right) \right)
\end{array} \right) + \nabla^2 \Phi^{\left(R \downarrow 0 \right)}\left(t_\mathrm{sym}^*\left(R \downarrow 0\right), \hat{t}_\mathrm{sym}^*\left(R \downarrow 0\right) \right) \left( \begin{array}{c}
\frac{\dd t_\mathrm{sym}^*(R)}{\dd R} \\
 \frac{\dd \hat{t}_\mathrm{sym}^*(R)}{\dd R}
\end{array} \right)_{R \downarrow 0} = \left(\begin{array}{c}
0\\0
\end{array} \right)\,,
\end{align}
which means that for~\eqref{eq:limit-exchange-fbet} to hold, the second term above, i.e., the Hessian of $\Phi^{(R)}$ applied to the derivative of the critical point with respect to $R$, has to vanish in the limit $R \downarrow 0$. In general, it is easy to construct counter-examples where, for instance,~\eqref{eq:limit-exchange-fbet} does not hold. One such case: suppose
$
\Phi^{(R)}(t_\mathrm{sym}) = \frac{1}{2} (t_\mathrm{sym} - R)^2
$
for $t_\mathrm{sym} \in  \RR$, then $t_\mathrm{sym}^*(R) = R$, hence $\Phi^{(R)}\left(t_\mathrm{sym}^*(R) \right) \equiv 0$ for any $R$, so the consistency condition~\eqref{eq:phi-r-consistent} holds trivially, $\Phi(t_\mathrm{sym}) = -t_\mathrm{sym}$. The left-hand side in~\eqref{eq:limit-exchange-fbet} evaluates to $0$, but the right-hand side is not defined as $\Phi$ does not have any critical points. Nevertheless, we will just assume that~\eqref{eq:limit-exchange-fbet} holds for our particular setting following standard practice.

\subsection{Simplifying $\Psi_y^{(R)}$ for the replica-symmetric ansatz}
\label{sec:sobo-psi-y-r}

Recalling the definition
$
\Upsilon^{(1:R)} = \left( \Upsilon^{(1)}, \ldots, \Upsilon^{(R)} \right) = \left(y^{(1)}, \, (V_k^\top y')^{(1)}, \, \ldots, \, y^{(R)}, \, (V_k^\top y')^{(R)}\right) \in \RR^{(k+1)R}
$
of the replicated network outputs for convenience, here we transform the potential
\begin{align}
\Psi^{(R)}_y(t) &= \alpha \log \lp \int_{\RR^{k+1}} \dd \Upsilon \int_{\RR} \dd \nu(\omega \,|\, t) \int_{\RR^{(k+1)R}} \dd \nu(\Upsilon^{(1:R)} \,|\, \omega, t)  L^\Pi_{\beta} \left(\Upsilon,  \Upsilon^{(1:R)} \right)  P_\text{data}\left(\Upsilon \,|\, \omega, \varpi \right)  \rp 
\end{align}
as defined in~\eqref{eq:repl-potentials} within the replica-symmetric ansatz so that $R$ becomes a parameter that can take on non-integer values. We will simplify the innermost expectation with respect to the Gaussian variables $\Upsilon^{(1:R)} \,|\, \omega, t$ in particular. Regrouping the replica-symmetric overlap parameters into
\begin{align}
t_\mathrm{sym} :=& \lp \rho_a,  \ s= \begin{pmatrix}
s_a \\ s_b
\end{pmatrix}, \ f=\begin{pmatrix}
f_a \\ f_b
\end{pmatrix}, \ q = \begin{pmatrix}
q_a & q_b^\top \\ q_b & q_c
\end{pmatrix}, \  r = \begin{pmatrix}
r_a & r_b^\top \\ r_b & r_c
\end{pmatrix}  \rp\,,
\end{align}
we have 
\begin{align}
\dd\nu\left(\Upsilon^{(1:R)} \,|\, \omega, t_\mathrm{sym}\right) &= (2\pi)^{-\frac{R(k+1)}{2}} \left( \det \Sigma^{(1:R)}_{\omega, t_\mathrm{sym}} \right)^{-1/2} \times \nonumber\\
& \qquad \qquad  \times \exp \left\{-\frac{1}{2}\left(\Upsilon^{(1:R)} - \mu^{(1:R)}_{\omega, t_\mathrm{sym}}\right)^\top \left(\Sigma^{(1:R)}_{\omega, t_\mathrm{sym}}\right)^{-1} \left(\Upsilon^{(1:R)} - \mu^{(1:R)}_{\omega, t_\mathrm{sym}}\right) \right\} \dd \Upsilon^{(1:R)}\,,
\end{align}
with conditional mean and covariance
\begin{align}
\begin{cases}
\mu^{(1:R)}_{\omega, t_\mathrm{sym}} = \textsc{vec}\left(\1_R \otimes \left(s + \frac{\omega}{\rho_a} f \right)\right)\,,\\
\Sigma^{(1:R)}_{\omega, t_\mathrm{sym}} = \1_R^{\otimes 2} \otimes \left(q - \frac{1}{\rho_a} f^{\otimes 2}\right) +  I_R \otimes (r-q)  \,.
\end{cases}
\end{align}
Expanding the quadratic form in the exponent, we have 
\begin{align}
&\biggl\langle \Upsilon^{(1:R)} - \mu^{(1:R)}_{\omega, t_\mathrm{sym}}, \, \left(\Sigma^{(1:R)}_{\omega, t_\mathrm{sym})}\right)^{-1} \left(\Upsilon^{(1:R)} - \mu^{(1:R)}_{\omega, t_\mathrm{sym}}\right) \biggr\rangle \nonumber \\
&= \sum_{r' \in [R]} \sum_{r'' \in [R]} \left \langle \Upsilon^{(r'')} - \mu^{(r'')}_{\omega, t_\mathrm{sym}}, \widetilde{q} \left(\Upsilon^{(r')} - \mu^{(r')}_{\omega, t_\mathrm{sym}}\right) \right \rangle + \sum_{r' \in [R]} \left \langle \Upsilon^{(r')} - \mu^{(r')}_{\omega, t_\mathrm{sym}}, \left( \widetilde{r} - \widetilde{q} \right) \left(\Upsilon^{(r')} - \mu^{(r')}_{\omega, t_\mathrm{sym}}\right) \right \rangle.
\end{align}
where we write
\begin{align}
	\left(\Sigma^{(1:R)}_{\omega, t_\mathrm{sym}}\right)^{-1} &=: \1_R^{\otimes 2} \otimes \, \widetilde{q} + I_R \otimes \, \left( \widetilde{r} - \widetilde{q} \right)\,,
\end{align}
to separate the off-diagonal and diagonal terms of the inverse conditional covariance matrix.
We then replace the double-summation over~$r$ with a single summation using a multi-dimensional Hubbard--Stratonovich transformation: 
\begin{align}
\exp\left\{- \frac{1}{2} \sum_{r' \in [R]} \sum_{r'' \in [R]} \left \langle \Upsilon^{(r'')} - \mu^{(r'')}_{\omega, t_\mathrm{sym}},  \widetilde{q} \left(\Upsilon^{(r')} - \mu^{(r')}_{\omega, t_\mathrm{sym}}\right) \right\rangle\right\} =  \EE_{\xi \sim \mathcal{N}(0, I_{k+1})} \left[\prod_{r' \in [R]} \exp\left\{ \left \langle \Upsilon^{(r')} - \mu^{(r')}_{\omega, t_\mathrm{sym}}, \sqrt{- \widetilde{q}} \; \xi \right \rangle \right\} \right].
\end{align}
Altogether, these results let us express the conditional Gaussian density of $\Upsilon^{(1:R)} \mid \omega,t_\mathrm{sym}$ as
\begin{align}
\dd\nu\left(\Upsilon^{(1:R)} \mid \omega, t_\mathrm{sym}\right) &= (2\pi)^{-\frac{R(k+1)}{2}} \left( \det \Sigma^{(1:R)}_{\omega, t_\mathrm{sym}} \right)^{-1/2} \EE_{\xi \sim \mathcal{N}(0, I_{k+1})}\left[\prod_{r' \in [R]}  \exp\left\{\left \langle \Upsilon^{(r')} - \mu^{(r')}_{\omega, t_\mathrm{sym}}, \sqrt{- \widetilde{q}} \; \xi \right \rangle \right\} \right] \times\nonumber \\
&\quad \times \prod_{r' \in [R]} \exp\left\{-\frac{1}{2} \left \langle \Upsilon^{(r')} - \mu^{(r')}_{\omega, t_\mathrm{sym}}, \left( \widetilde{r} - \widetilde{q} \right) \left(\Upsilon^{(r')} - \mu^{(r')}_{\omega, t_\mathrm{sym}} \right) \right \rangle \right\}\dd \Upsilon^{(1:R)}\,,
\end{align}
resulting in
\begin{align}
 & \ \exp \left\{ \frac{1}{\alpha} \Psi_y^{(R)}\left(t_\mathrm{sym}\right) \right\} \nonumber = (2\pi)^{-\frac{R(k+1)}{2}} \left( \det \Sigma^{(1:R)}_{\omega, t_\mathrm{sym}} \right)^{-1/2} \EE_{\xi \sim \mathcal{N}\left(0,I_{k+1} \right)} \EE_{\omega \sim \mathcal{N}(0,\rho_a)} \Bigg[  \int_{\RR^{k+1}} \dd \Upsilon\; P_\text{data}(\Upsilon \mid \omega, \varpi) \times \nonumber \\
 & \qquad  \qquad\times  \Bigg( \int_{\RR^{k+1}} \dd \tilde{\Upsilon} \; L_\beta(\Upsilon, \tilde{\Upsilon}) \exp\left\{ \left \langle \Tilde{\Upsilon} - \mu_{\omega, t_\mathrm{sym}}, \sqrt{- \widetilde{q}} \; \xi \right \rangle \right\} \exp\left\{-\frac{1}{2} \left \langle \tilde{\Upsilon} - \mu_{\omega, t_\mathrm{sym}}, \left( \widetilde{r} - \widetilde{q} \right)  \left(\tilde{\Upsilon} - \mu_{\omega, t_\mathrm{sym}} \right) \right \rangle \right\} \Bigg)^R \Bigg]\,.
\label{eq:psi-y-pre-limit-r0-grad}
\end{align}
The integration variable $\tilde{\Upsilon}$ denotes any of the $R$ decoupled replicas. Expressing the potential $\Psi_y^{(R)}$ in this way will now allow us to take non-integer $R$ in the subsequent sections. Of course, a similar expression holds for $\Psi^{(R)}_{y_0}$ upon replacing $L_\beta$ with $\tilde{L}_{\beta h}$.

\subsection{Simplifying $\Psi_w^{(R)}$ for the replica-symmetric ansatz}
\label{sec:sobo-psi-w-r}

Next, we simplify the second potential $\Psi_w^{(R)}\left(\hat{t}_\mathrm{sym} \right)$ in~\eqref{eq:phi-r-grad} so that $R$ may take on non-integer values. The potential is defined as 
\begin{align}
\Psi^{(R)}_w\left(\hat t_\mathrm{sym}\right) &= \frac{1}{p} \log \left( \left(\beta \lambda \right)^{-pR/2} \EE_{\mid \varpi}\left[ e^{\langle \hat{T}_\mathrm{sym}, T(\theta_0, \Theta, V_k, W) \rangle} \right] \right)
\end{align}
within the replica-symmetric ansatz, where we understand the inner product notation to mean
\begin{align}
& \left \langle \hat{T}_\mathrm{sym}, T(\theta_0, \Theta, V_k, W) \right \rangle := \left \langle \hat{\mu}_\mathrm{sym}, \,\mu(\theta_0, \Theta, V_k, W) \right \rangle + \left \langle \textsc{vech}(\hat{\Sigma}_\mathrm{sym}), \; \textsc{vech}(\Sigma(\theta_0, \Theta, V_k, W)) \right \rangle \nonumber\\[8pt]
&~= d \hat{\rho}_a\|\theta_0\|^2 + \sqrt{p} \hat{s}_a \sum_{r \in [R]} \kappa_0 \left \langle w^{(r)}, \1_p \right \rangle + \sqrt{p} \sum_{r \in [R]} \left \langle \hat{s}_b, \kappa_0' V_k^\top \Theta w^{(r)} \right \rangle  + p \sum_{r \in [R]} \hat{f}_a \kappa_1 \left \langle \Theta^\top \theta_0, w^{(r)} \right \rangle \nonumber\\
&~+ p \sum_{r \in [R]} \left \langle \hat{f}_b, \kappa_1' V_k^\top \Theta \textsc{diag}\left(w^{(r)}\right) \Theta^\top \theta_0 \right \rangle + p\sum_{r \neq r' \in [R]} \left \langle \hat{q}_b, \kappa_1\kappa_1' V_k^\top \Theta \;\textsc{diag}\left(w^{(r)}\right) \Theta^\top \Theta w^{(r')} \right \rangle \nonumber\\
&~+ p\sum_{r \in [R]}\left \langle \hat{r}_b, \kappa_1\kappa_1'  V_k^\top \Theta \; \textsc{diag}\left(w^{(r)}\right) \Theta^\top \Theta  w^{(r)} \right \rangle \nonumber\\
&~+ p \sum_{r \in [R]} \sum_{\substack{r' \in [R],\\r' > r}} \hat{q}_a \left( \kappa_1^2 \left \langle \Theta w^{(r)}, \Theta w^{(r')} \right \rangle + \kappa_*^2 \left \langle w^{(r)}, w^{(r')} \right \rangle \right) + p\sum_{r \in [R]}
\hat{r}_a \left( \kappa_1^2 \left \langle \Theta w^{(r)}, \Theta w^{(r)} \right\rangle + \kappa_*^2 \left \langle w^{(r)}, w^{(r)} \right \rangle \right) \nonumber\\
&~+ p\sum_{r \in [R]} \sum_{\substack{r' \in [R],\\r' > r}} \left \langle \hat{q}_c, V_k^\top \Theta \;\textsc{diag}\left(w^{(r)}\right) \left((\kappa_1')^2 \Theta^\top \Theta + (\kappa_*')^2 I_p \right)\textsc{diag}\left(w^{(r')}\right)\Theta^\top V_k   \right \rangle_{\text{F}} \nonumber\\
&~+ p \sum_{r \in [R]} \left \langle \hat{r}_c, V_k^\top \Theta \;\textsc{diag}\left(w^{(r)}\right) \left( (\kappa_1')^2 \Theta^\top \Theta + (\kappa_*')^2 I_p \right)\textsc{diag}\left(w^{(r)}\right)  \Theta^\top V_k  \right \rangle_{\text{HF}}  \;.
\end{align}
Note that additional care must be taken not to double count the symmetric entries for the matrices along the block diagonal. To this purpose, above we use the ``half Frobenius inner product'' $\langle \cdot, \cdot \rangle_{\text{HF}} $, which is related to the traditional Frobenius inner product $\langle \cdot, \cdot \rangle_{\text{F}}$ with $\langle A, B \rangle_{\text{F}} = \sum_{i,j = 1}^k A_{ij} B_{ij}$ for $A,B \in \RR^{k\times k}$ via
\begin{equation}
\label{eq:hf-def}
2 \langle P, Q \rangle_{\text{HF}}  =  \langle P, Q \rangle_{\text{F}}  + \langle P, Q \odot I_k \rangle_{\text{F}}\,, \text{ for } P=P^\top,\;Q=Q^\top \in \RR^{k \times k}\,,
\end{equation}
where $\odot$ denotes the Hadamard product with $(A \odot B)_{ij} = A_{ij} B_{ij}$.
Using this definition, as well as the analogous scalar-valued identity 
$\sum_j \sum_{i < j} s_{ij} = \frac{1}{2} \sum_j\sum_{i\neq j} s_{ij}$ for symmetric $s$, we express the above as
\begin{align}
& \left \langle \hat{T}_\mathrm{sym}, T(\theta_0, \Theta, V_k, W) \right \rangle \nonumber\\
&= d\hat{\rho}_a\|\theta_0\|^2 + \sqrt{p}\hat{s}_a \sum_{r \in [R]} \kappa_0 \left \langle w^{(r)}, \1_p \right \rangle + \sqrt{p}\sum_{r \in [R]} \left \langle \hat{s}_b, \kappa_0' V_k^\top \Theta w^{(r)} \right \rangle  + p\sum_{r \in [R]} \hat{f}_a \kappa_1 \left \langle \Theta^\top \theta_0, w^{(r)} \right \rangle \nonumber\\
&~+ p\sum_{r \in [R]} \left \langle \hat{f}_b, \kappa_1' V_k^\top \Theta \textsc{diag}\left(w^{(r)}\right) \Theta^\top \theta_0 \right \rangle + p\sum_{r,r' \in [R]} \left \langle \hat{q}_b, \kappa_1\kappa_1' V_k^\top \Theta \;\textsc{diag}\left(w^{(r)}\right) \Theta^\top \Theta   w^{(r')} \right \rangle \nonumber\\
&~+ p\sum_{r \in [R]} \left \langle \hat{r}_b - \hat{q}_b, \kappa_1\kappa_1' V_k^\top \Theta \;\textsc{diag}\left(w^{(r)}\right)  \Theta^\top \Theta  w^{(r)} \right \rangle \nonumber\\
&~+ \frac{1}{2}p\sum_{r, r' \in [R]} \hat{q}_a \left( \kappa_1^2 \left \langle \Theta w^{(r)}, \Theta w^{(r')} \right \rangle + \kappa_*^2 \left \langle w^{(r)}, w^{(r')} \right \rangle \right) + p\sum_{r \in [R]}
\left(\hat{r}_a - \frac{1}{2}\hat{q}_a\right) \left( \kappa_1^2 \left \langle \Theta w^{(r)}, \Theta w^{(r)} \right \rangle + \kappa_*^2 \left \langle w^{(r)}, w^{(r)} \right \rangle \right) \nonumber\\
&~+ \frac{1}{2}p\sum_{r, r' \in [R]} \left \langle \hat{q}_c, V_k^\top \Theta \;\textsc{diag}\left(w^{(r)}\right)  \left((\kappa_1')^2 \Theta^\top \Theta + (\kappa_*')^2 I_p\right)\textsc{diag}\left(w^{(r')}\right) \Theta^\top V_k  \right \rangle_{\text{F}} \nonumber\\
&~+ \frac{1}{2}p\sum_{r \in [R]} \left \langle \hat{r}_c \odot I_k, V_k^\top \Theta \;\textsc{diag}\left(w^{(r)}\right)  \left( (\kappa_1')^2 \Theta^\top \Theta + (\kappa_*')^2 I_p \right) \textsc{diag}\left(w^{(r)}\right) \Theta^\top V_k  \right \rangle_{\text{F}} \nonumber \\
&~+ \frac{1}{2}p \sum_{r \in [R]} \left \langle \hat{r}_c - \hat{q}_c, V_k^\top \Theta \;\textsc{diag}\left(w^{(r)}\right)  \left ((\kappa_1')^2 \Theta^\top \Theta + (\kappa_*')^2 I_p\right) \textsc{diag}\left(w^{(r)}\right) \Theta^\top V_k \right \rangle_{\text{F}}.
\end{align}
We make the $w$-dependencies clear by re-arranging the terms above to yield
\begin{align}
& \left \langle \hat{T}_\mathrm{sym}, T(\theta_0, \Theta, V_k, W) \right \rangle \nonumber\\
&= d\hat{\rho}_a\|\theta_0\|_2^2 + \sqrt{p}\hat{s}_a \sum_{r \in [R]} \kappa_0 \left \langle \1_p, w^{(r)} \right \rangle + \sqrt{p}\sum_{r \in [R]}\left \langle \kappa_0' \Theta^\top V_k\hat{s}_b, w^{(r)} \right \rangle  + p\sum_{r \in [R]} \hat{f}_a \kappa_1 \left \langle \Theta^\top \theta_0, w^{(r)} \right \rangle \nonumber\\
&~+ p\kappa_1' \sum_{r \in [R]} \left \langle (\Theta^\top V_k\hat{f}_b) \odot (\Theta^\top \theta_0), w^{(r)} \right \rangle + \frac{1}{2}p\sum_{r,r' \in [R]} \left \langle w^{(r)}, 2 \kappa_1\kappa_1'\Theta^\top \Theta  \textsc{diag}\left(\Theta^\top V_k \hat{q}_b\right) w^{(r')} \right \rangle \nonumber\\
&~+ \frac{1}{2}p\sum_{r \in [R]} \left \langle w^{(r)}, 2 \kappa_1\kappa_1' \Theta^\top \Theta \textsc{diag}(\Theta^\top V_k (\hat{r}_b - \hat{q}_b)) w^{(r)} \right \rangle \nonumber\\
&~+ \frac{1}{2}p\sum_{r, r' \in [R]} \hat{q}_{a1} \left \langle w^{(r)}, \left(\kappa_1^2\Theta^\top \Theta + \kappa_*^2 I_p \right) w^{(r')} \right \rangle + \frac{1}{2}p \sum_{r \in [R]}
\left(2 \hat{r}_a - \hat{q}_a \right) \left \langle w^{(r)}, \left(\kappa_1^2 \Theta^\top \Theta +  \kappa_*^2 I_p \right) w^{(r)} \right \rangle\nonumber\\
&~+ \frac{1}{2}p\sum_{r,r' \in [R]} \left \langle w^{(r)}, \left(\Theta^\top V_k \hat{q}_c V_k^\top \Theta \right) \odot \left((\kappa_1')^2 \Theta^\top \Theta + (\kappa_*')^2 I_p\right) w^{(r')} \right \rangle \nonumber \\
&~+ \frac{1}{2}p \sum_{r \in [R]} \left \langle w^{(r)}, \left(\Theta^\top V_k \left(\hat{r}_c \odot I_k+ \hat{r}_c - \hat{q}_c \right) V_k^\top \Theta \right) \odot \left((\kappa_1')^2 \Theta^\top \Theta + (\kappa_*')^2 I_p\right) w^{(r)} \right \rangle  \;.
\end{align}
To shorten the notation, we introduce the auxiliary definitions
\begin{align}
\begin{cases}
J_{\hat{s}_a} &= \kappa_0 \hat{s}_a \1_p \\
J_{\hat{s}_b} &= \kappa_0' \Theta^\top V_k \hat{s}_b \\
J_{\hat{f}_a} &= \sqrt{p} \hat{f}_a \kappa_1 \Theta^\top \theta_0 \\
J_{\hat{f}_b} &= \sqrt{p} \kappa_1' (\Theta^\top V_k\hat{f}_b) \odot (\Theta^\top \theta_0)
\end{cases}
\label{eq:def-j-vectors}
\end{align}
and
\begin{align}
\begin{cases}
A_{\hat{q}_b} &= 2 \kappa_1\kappa_1'\Theta^\top \Theta   \textsc{diag}\left(\Theta^\top V_k \hat{q}_b\right)  \\
A_{\hat{r}_b - \hat{q}_b} &= -2 \kappa_1\kappa_1'\Theta^\top \Theta  \textsc{diag}\left(\Theta^\top V_k (\hat{r}_b - \hat{q}_b) \right)  \\
A_{\hat{q}_a} &= \hat{q}_a \left(\kappa_1^2\Theta^\top \Theta + \kappa_*^2 I_p \right)  \\
A_{2 \hat{r}_a - \hat{q}_a} &= -\left(2 \hat{r}_a - \hat{q}_a\right)  \left(\kappa_1^2\Theta^\top \Theta + \kappa_*^2 I_p \right)  \\
A_{\hat{q}_c} &= \left(\Theta^\top V_k \hat{q}_c V_k^\top \Theta \right) \odot \left((\kappa_1')^2 \Theta^\top \Theta + (\kappa_*')^2 I_p\right)  \\
A_{2 \hat{r}_c - \hat{q}_c} &= -\left(\Theta^\top V_k \left(\hat{r}_c \odot I_k + \hat{r}_c - \hat{q}_c \right) V_k^\top \Theta \right) \odot \left((\kappa_1')^2 \Theta^\top \Theta + (\kappa_*')^2 I_p\right) \\
&= -\left(\Theta^\top V_k \left(\hat{r}_c \odot \left(I_k + \1_k^{\otimes 2} \right) - \hat{q}_c \right) V_k^\top \Theta \right) \odot \left((\kappa_1')^2 \Theta^\top \Theta + (\kappa_*')^2 I_p\right)\,.
\end{cases}
\label{eq:def-a-matrices}
\end{align}
In order to decouple the replicas, we again make use of Hubbard--Stratonovich transformations for all terms that involve double summations. This yields
\begin{align}
&\exp \left\{\left \langle \hat{T}_\mathrm{sym}, T(\theta_0, \Theta, V_k, W) \right \rangle  \right\} = \exp \left\{ d\hat{\rho}_a \norm{\theta_0}^2  \right\} \times \nonumber\\
& \qquad \qquad \times \EE_{\eta_{\hat{q}_b},\eta_{\hat{q}_a},\eta_{\hat{q}_c} \sim \mathcal{N}\left(0,I_p\right)} \bigg[\prod_{r \in [R]} \bigg(\exp \left\{-\frac{1}{2} p \left \langle w^{(r)},\left[A_{\hat{r}_b - \hat{q}_b} + A_{2 \hat{r}_a - \hat{q}_a} + A_{2 \hat{r}_c - \hat{q}_c} \right] w^{(r)} \right \rangle + \sqrt{p} \left \langle w^{(r)}, J \right \rangle\right\}
\bigg)  \bigg]
\label{eq:multidim-hs-trafo-grad}
\end{align}
with ``source'' term $J := J_{\hat{s}_a} + J_{\hat{s}_b} + J_{\hat{f}_a} + J_{\hat{f}_b} + A_{\hat{q}_b}^{1/2} \eta_{\hat{q}_b} + A_{\hat{q}_a}^{1/2} \eta_{\hat{q}_a} + A_{\hat{q}_c}^{1/2} \eta_{\hat{q}_c}$.
All in all, we then obtain
\begin{align}
&\exp \left\{p \Psi_w^{(R)}\left(\hat{t}_\mathrm{sym} \right) \right\} =  \exp \left\{ d \hat{\rho}_a \norm{\theta_0}^2 \right\} \times \nonumber\\
& \; \; \times \EE_{\eta_{\hat{q}_{b1}},\eta_{\hat{q}_{a1}},\eta_{\hat{q}_{c1}} \sim \mathcal{N}\left(0,I_p\right)} \bigg[ \bigg( \EE_w \bigg[ \left( \beta \lambda \right)^{-p/2}  \exp \left\{-\frac{1}{2} p  \left \langle w,\left[ A_{\hat{r}_b - \hat{q}_b} + A_{2 \hat{r}_a - \hat{q}_a} + A_{2 \hat{r}_c - \hat{q}_c} \right] w \right \rangle + \sqrt{p} \left \langle J, w \right \rangle \right\} \bigg] \bigg)^R \bigg]\,,
\label{eq:psi-w-pre-limit-r0-grad}
\end{align}
such that again the $R$-dependence is explicit and allows for taking non-integer values of $R$.

\subsection{Consistency check for the replica-symmetric ansatz and determining $\rho_a, \hat{\rho}_a$}
\label{sec:sobo-consistency}

We conclude from the results~\eqref{eq:psi-y-pre-limit-r0-grad} and~\eqref{eq:psi-w-pre-limit-r0-grad} of the calculations of the previous two sections, as well as
\begin{align}
	\frac{1}{p}\langle t, \hat{t} \rangle  =& \gamma \rho_a \hat{\rho}_a + \frac{R}{\sqrt{p}} s_a \hat{s}_a + \frac{R}{\sqrt{p}} \left \langle s_b, \hat{s}_b \right \rangle  + R f_a \hat{f}_a + R \left \langle f_b, \hat{f}_b \right \rangle + \frac{R(R-1)}{2} q_a\hat{q}_a \nonumber\\
	&+  R r_a \hat{r}_a + R(R-1) \left \langle q_b, \hat{q}_b \right \rangle + R \left \langle  r_b, \hat{r}_b\right \rangle + \frac{R(R-1)}{2}\langle q_c, \hat{q_c} \rangle_{\text{F}} + R \langle r_c, \hat{r}_c \rangle_{\text{HF}} \,,
	\label{eq:overlap-param-product-sobo}
\end{align}
that for this ansatz we have for the rate function~\eqref{eq:phi-r-grad}:
\begin{align}
\lim_{R \downarrow 0} \Phi^{(R)}\left(t_\mathrm{sym}, \hat{t}_\mathrm{sym} \right) = - \gamma \rho_a \hat{\rho}_a + \frac{1}{p} \log \left(\EE_{\theta_0} \left[\exp \left\{ d \hat{\rho}_a \norm{\theta_0}^2 \right\} \right] \right) = -\gamma \left[\rho_a \hat{\rho}_a + \frac{1}{2} \log \left(1 - 2 \hat{\rho}_a  \right) \right]
\label{eq:phi-r-to-zero}
\end{align}
for any parameters $t_\mathrm{sym}, \hat{t}_\mathrm{sym}$. As argued in Section~\ref{sec:sobolev_replica_ansatz}, we need the limit~\eqref{eq:phi-r-to-zero} to be $0$ at the critical $t_\mathrm{sym}^*(R \downarrow 0)$ and $\hat{t}_\mathrm{sym}^*(R \downarrow 0)$ for our ansatz to be consistent. Luckily, by setting the derivative of the right-hand side of~\eqref{eq:phi-r-to-zero} with respect to $\rho_a$ to $0$, we see that
$\hat{\rho}_a^*\left(R \downarrow 0\right) = 0$,
which results in $\lim_{R \downarrow 0} \Phi^{(R)}\left(t_\mathrm{sym}^*(R), \hat{t}_\mathrm{sym}^*(R) \right) = 0$ as desired. Furthermore, by setting the derivative of the right-hand side of~\eqref{eq:phi-r-to-zero} with respect to $\hat{\rho}_a$ to $0$, we also obtain $\rho_a^*\left(R \downarrow 0\right) = 1$,
as we should, since $\rho_a$ corresponds to $\norm{\theta_0}^2$ which concentrates onto its expectation $1$ in the proportional asymptotics limit.
As we assume with~\eqref{eq:limit-exchange-fbet} that the critical points of $\Phi^{(R)}$ converge to the critical point of $\Phi$, we will immediately use $\rho_a = 1$ and $\hat{\rho}_a = 0$ for all of the following calculations and optimality conditions for $\Phi$.

\subsection{Calculating $\Psi_y$}
\label{sec:psi-y-sobo}

As discussed in Section~\ref{sec:sobolev_replica_ansatz}, we now want to obtain the $R$-derivative of the rate function~\eqref{eq:phi-r-grad} at $R = 0$ to subsequently calculate the optimality conditions for the evaluation of $\text{crit}_{t_\mathrm{sym}, \hat{t}_\mathrm{sym}} \left( \lim_{R \downarrow 0}  \frac{1}{R}  \Phi^{(R)} \right)(t_\mathrm{sym},\hat{t}_\mathrm{sym})$
in~\eqref{eq:limit-exchange-fbet}. We start by considering the derivative of~$\Psi_y^{(R)}$, as found in~\eqref{eq:psi-y-pre-limit-r0-grad}, in this section.
Using the identity $\lim_{R \downarrow 0} \frac{1}{R} \log \EE\left[ A(R)^R \right] = \EE\left[\log A(0)\right]$ to interchange logarithms and expectations in~\eqref{eq:psi-y-pre-limit-r0-grad}, we obtain 
\begin{align}
&\frac{1}{\alpha}\Psi_y(t_\mathrm{sym}) := \frac{1}{\alpha} \lim_{R \downarrow 0} \Psi_y^{(R)}(t_\mathrm{sym}) =  -\frac{k+1}{2}\log(2\pi) - \lim_{R \downarrow 0}\frac{1}{2R} \log\det\left(\Sigma^{(1:R)}_{\omega, t_\mathrm{sym}}\right)  +  \lim_{R \downarrow 0} \EE_{\xi \sim \mathcal{N}(0, I_{k+1}), \omega \sim \mathcal{N}(0, \rho_a)} \bigg[ \nonumber\\
&\quad \qquad\int_{\RR^{k+1}} \dd \Upsilon \; P_\text{data}(\Upsilon \mid \omega, \varpi) \log\left(\int_{\RR^{k+1}} \dd \tilde{\Upsilon} \; L_\beta(\Upsilon, \tilde{\Upsilon}) \exp\left\{\left \langle \tilde{\Upsilon} - \mu_{\omega, t_\mathrm{sym}}, \sqrt{-\widetilde{q}} \;\xi - \frac{1}{2}\left( \widetilde{r} - \widetilde{q} \right) \left(\tilde{\Upsilon} - \mu_{\omega, t_\mathrm{sym}} \right) \right \rangle \right\}  \right) \bigg]
\end{align}

We complete the square in the innermost integral according to
\begin{equation}
\exp\left\{ -\frac{1}{2} \left \langle x, A x \right \rangle + \left \langle  J, x \right \rangle \right\} = \exp\left\{ -\frac{1}{2} \left \langle x - A^{-1}J, A \left(x-A^{-1}J\right) \right \rangle\right\} \exp\left\{\frac{1}{2} \left \langle J,  A^{-1} J \right \rangle \right\},
\end{equation}
with $x = \tilde{\Upsilon} - \mu_{\omega, t_\mathrm{sym}}$, $A = \widetilde{r} - \widetilde{q}$, and $J = \sqrt{-\widetilde{q}} \; \xi$. This gives
\begin{align}
\label{eq:sobo-mathemagic-LHS}
\frac{1}{\alpha}\Psi_y(t_\mathrm{sym}) 
&= \EE_{\xi \sim \mathcal{N}(0, I_{k+1})} \EE_{\omega \sim \mathcal{N}(0, \rho_a)} \biggl[ \int_{\RR^{k+1}} \dd \Upsilon  \times \nonumber\\
&\qquad \qquad \times P(\Upsilon \mid \omega, \varpi)  \log\left(\EE_{\tilde{\Upsilon} \sim \mathcal{N}(s + \rho_a^{-1}\omega f + (r-q)\sqrt{(r-q)^{-1}(q - \rho_a^{-1} f^{\otimes 2})(r-q)^{-1} }\xi, r-q)}[L_\beta(\Upsilon, \tilde{\Upsilon})] \right) \biggr].
\end{align}
By a change of variables in $\xi$ and $\omega$ we can re-write this expression as
\begin{align}
\frac{1}{\alpha}\Psi_y(t_\mathrm{sym}) &= \EE_{\xi \sim \mathcal{N}(0, I_{k+1})} \left[ \int_{\RR^{k+1}} \dd \Upsilon\; \EE_{\omega \sim \mathcal{N}\left( \langle f, q^{-1/2} \xi \rangle,\; \rho_a - \langle f, q^{-1} f \rangle \right)} \left[ P_\text{data}(\Upsilon \mid \omega, \varpi)\right] \log \left( \EE_{\tilde{\Upsilon} \sim \mathcal{N}\left(s + q^{1/2} \xi, \; r - q\right)}[L_\beta(\Upsilon, \tilde{\Upsilon})] \right) \right] \nonumber\\
&=: \EE_{\xi \sim \mathcal{N}(0,I_{k+1})} \left[ \int_{\RR^{k+1}}  \dd \Upsilon \; \mathcal{Z}[P_\text{data}](\Upsilon; \bar{m}_1, \sigma_1^2) \log \mathcal{Z}[L_\beta](\Upsilon; \bar{m}_2, \Sigma_2)\right],
\label{eq:sobo-mathemagic-RHS}
\end{align}
with means and covariances
\begin{align}
\begin{cases}
\bar{m}_1 = \left \langle f, q^{-1/2}\xi \right \rangle \in \RR\\ 
\bar{m}_2 = s + q^{1/2} \xi \in \RR^{k+1}
\end{cases}
\quad 
\begin{cases}
\sigma_1^2 = 1 - \left \langle f,  q^{-1} f \right \rangle \in \RR\\
\Sigma_2 = r - q \in \RR^{(k+1)\times (k+1)}
\end{cases}
\end{align}
and the notation
\begin{align}
\begin{cases}
\mathcal{Z}[P_\text{data}]\left(\Upsilon; \bar{m}_1, \sigma_1^2\right) &:= \EE_{\omega \sim \mathcal{N}(\bar{m}_1, \sigma_1^2)}[P_\text{data}(\Upsilon \mid \omega, \varpi)] \\ 
\mathcal{Z}[L_\beta]\left(\Upsilon; \bar{m}_2, \Sigma_2\right) &:= \EE_{\tilde{\Upsilon} \sim \mathcal{N}(\bar{m}_2, \Sigma_2)}[L_\beta(\Upsilon, \tilde{\Upsilon})]\,,
\end{cases}
\end{align}
analogous to~\cite{gerace-loureiro-krzakala-etal:2021}. Altogether, our final result for the $R$-derivative of the potential $\Psi_y^{(R)}$ reads
\begin{align}
\Psi_y(t_\mathrm{sym}) = \alpha \EE_{\xi \sim \mathcal{N}(0, I_{k+1})} \left[ \int_{\RR^{k+1}}  \dd \Upsilon \; \mathcal{Z}[P_\text{data}](\Upsilon; \bar{m}_1, \sigma_1^2) \log \mathcal{Z}[L_\beta](\Upsilon; \bar{m}_2, \Sigma_2)\right]\,.
\label{eq:psi-y-result-sobo}
\end{align}

\subsection{Calculating $\Psi_w$}
\label{sec:psi-w-sobo}

Taking the limit $\lim_{R \downarrow 0} \Psi_w^{(R)}\left(\hat{t}_\mathrm{sym} \right) / R$ is straightforward as the $R$ dependence in~\eqref{eq:psi-w-pre-limit-r0-grad} only manifests in the $R$-fold product over replicas. Using $\hat{\rho}_a = 0$ from the consistency check, we then end up with
\begin{align}
\Psi_w\left(\hat{t}_\mathrm{sym}\right) &:= \lim_{R \downarrow 0} \frac{1}{R} \Psi_w^{(R)}\left(\hat{t}_\mathrm{sym} \right) \nonumber\\
&= \plim_{p \to \infty} \frac{1}{p} \EE_{\eta_{\hat{q}_b},\eta_{\hat{q}_a},\eta_{\hat{q}_c} \sim \mathcal{N}\left(0,I_p\right)} \bigg[\log \bigg( \int_{\RR^p} \lp\frac{p}{2\pi}\rp^{p/2}  \exp \left\{-\frac{1}{2} p \left \langle w, A w \right \rangle + \sqrt{p} \left \langle w, J \right \rangle\right\}  \dd w  \bigg) \bigg]\,.
\label{eq:psi-w-e-log-int-grad}
\end{align}
with 
\begin{align}
A := \beta \lambda I_p + A_{\hat{r}_b - \hat{q}_b} + A_{2 \hat{r}_a - \hat{q}_a} + A_{2 \hat{r}_c - \hat{q}_c}\,.
\label{eq:def-a-init}
\end{align}
The Gaussian integral inside the logarithm evaluates to $\det^{-1/2}(A) \times \exp \left\{ \frac{1}{2} \left \langle  J, A^{-1} J \right \rangle \right\}$ and using identity $\log \det = \trace \log$,
the potential in~\eqref{eq:psi-w-e-log-int-grad} becomes 
\begin{align}
\Psi_w\left(\hat{t}_\mathrm{sym}\right) = \plim_{p \to \infty} \frac{1}{p}\EE_{\eta_{\hat{q}_b},\eta_{\hat{q}_a},\eta_{\hat{q}_c} \sim \mathcal{N}\left(0,I_p\right)} \bigg[ - \frac{1}{2} \trace \log A + \frac{1}{2} \left \langle J, A^{-1} J \right \rangle \bigg]\,.
\label{eq:psi-w-e-grad}
\end{align}
For the quadratic form, we get 
\begin{align}
\EE_{\eta_{\hat{q}_b},\eta_{\hat{q}_a},\eta_{\hat{q}_c} \sim \mathcal{N}\left(0,I_p\right)} \bigg[\tfrac{1}{2} \left \langle J, A^{-1} J \right \rangle \bigg] &= \tfrac{1}{2}\trace \bigg[ \big(J_{\hat{f}_a}J_{\hat{f}_a}^\top + 2J_{\hat{f}_a}J_{\hat{f}_b}^\top + J_{\hat{f}_b}J_{\hat{f}_b}^\top + A_{\hat{q}_a} + A_{\hat{q}_b} + A_{\hat{q}_c}  \big)A^{-1}\bigg]  + \tfrac{1}{2}\left \langle J_{\hat{s}_a} + J_{\hat{s}_b}, A^{-1} (J_{\hat{s}_a} + J_{\hat{s}_b}) \right \rangle\,,
\end{align}
where
\begin{align}
\begin{cases}
J_{\hat{f}_a} J_{\hat{f}_a}^\top &= p (\hat{f}_a)^2 \kappa_1^2 \Theta^\top \theta_0 \theta_0^\top \Theta \\
J_{\hat{f}_a} J_{\hat{f}_b}^\top &= p \hat{f}_a \kappa_1 \kappa_1' \Theta^\top \theta_0 \theta_0^\top \Theta \textsc{diag}(\Theta^\top V_k\hat{f}_b)\\
J_{\hat{f}_b} J_{\hat{f}_b}^\top &= p (\kappa_1')^2 \textsc{diag}(\Theta^\top V_k\hat{f}_b) \Theta^\top\theta_0 \theta_0^\top \Theta \textsc{diag}(\Theta^\top V_k\hat{f}_b) \\
&= p(\kappa_1')^2 (\Theta^\top V_k \hat{f}_b \hat{f}_b^\top V_k^\top \Theta) \odot (\Theta^\top \theta_0 \theta_0^\top \Theta)\,.
\end{cases}
\end{align}
For convenience, we define
\begin{align}
\Xi := J_{\hat{f}_a}J_{\hat{f}_a}^\top + 2J_{\hat{f}_a}J_{\hat{f}_b}^\top + J_{\hat{f}_b}J_{\hat{f}_b}^\top + A_{\hat{q}_a} + A_{\hat{q}_b} + A_{\hat{q}_c},
\label{eq:def-xi-init}
\end{align}
so that
\begin{equation}
\Psi_w\left(\hat{t}_\mathrm{sym}\right) = \plim_{p \to \infty} \frac{1}{2p} \bigg\{- \trace \log A + \trace (\Xi A^{-1}) + \left\langle J_{\hat{s}_a} + J_{\hat{s}_b}, A^{-1} (J_{\hat{s}_a} + J_{\hat{s}_b})\right\rangle \bigg\}.
\label{eq:psi-w-result}
\end{equation}

\subsection{Saddlepoint equations for $\Phi$ at finite $\beta$}
\label{sec:saddle-sobo}

To obtain first order optimality conditions, we will differentiate  the $R$-derivative of \eqref{eq:phi-r-grad} at $R = 0$ with respect to the overlap parameters as outlined in Section~\ref{sec:sobolev_replica_ansatz}. 
Recalling from the consistency check of Section~\ref{sec:sobo-consistency} that we can set $\rho_a=1$ and $\hat{\rho}_a=0$, we have from~\eqref{eq:overlap-param-product-sobo} that 
\begin{align}
\lim_{R \downarrow 0} \frac{1}{R} \plim_{p \to \infty} \frac{1}{p} \langle t, \hat{t} \rangle &= f_a \hat{f}_a + \left \langle f_b, \hat{f}_b \right \rangle  - \frac{1}{2} q_a \hat{q}_a + r_a \hat{r}_a - \left \langle q_b, \hat{q}_b \right \rangle + \left \langle r_b, \hat{r}_b \right \rangle - \frac{1}{2} \left \langle q_c, \hat{q}_c \right \rangle_{\text{F}} + \left \langle r_c, \hat{r}_c \right \rangle_{\text{HF}} \nonumber\\
&= \left \langle f, \hat{f} \right \rangle - \frac{1}{2} \left \langle q, \hat{q} \right \rangle_{\text{F}} + \frac{1}{2} \left \langle r, \hat{r} \right \rangle_{\text{F}} + \frac{1}{2} \left \langle r, \hat{r} \odot I_{k+1} \right \rangle_{\text{F}} \nonumber\\
&= \left \langle f, \hat{f} \right \rangle - \frac{1}{2} \left \langle q, \hat{q} \right \rangle_{\text{F}} + \left \langle r, \hat{r} \right \rangle_{\text{HF}} = \left \langle f, \hat{f} \right \rangle - \frac{1}{2} \trace\left[q \hat{q}\right] + \frac{1}{2} \trace \left[r \hat{r} \right] + \frac{1}{2} \trace \left[r \left(\hat{r} \odot I_{k+1} \right)  \right]\,.
\end{align}
Summarizing what we have obtained so far, we have now found the $R$-derivative of the rate function~\eqref{eq:phi-r-grad} for the replica-symmetric overlap parameter ansatz from Section~\ref{sec:sobolev_replica_ansatz} with
\begin{align}
\Phi \left(t_\mathrm{sym}, \hat{t}_\mathrm{sym}\right) &= \left( \left. \frac{\dd}{\dd R}\right|_{R = 0} \Phi^{(R)} \right) \left(t_\mathrm{sym}, \hat{t}_\mathrm{sym}\right) \nonumber \\
&= \Psi_y\left(t_\mathrm{sym} \right) +\Psi_{y_0}\left(t_\mathrm{sym} \right) + \Psi_w\left(\hat{t}_\mathrm{sym}\right) - \left(\left \langle f, \hat{f} \right \rangle - \frac{1}{2} \left \langle q, \hat{q} \right \rangle_{\text{F}} + \left \langle r, \hat{r} \right \rangle_{\text{HF}} \right)\,,
\label{eq:rf-derivative-summary}
\end{align}
where the potentials are given by
\begin{align}
\begin{cases}
\Psi_y(t_\mathrm{sym}) &= \alpha \EE_{\xi \sim \mathcal{N}(0, I_{k+1})} \left[ \int_{\RR^{k+1}} \mathcal{Z}[P_\text{data}](\Upsilon; \bar{m}_1, \sigma_1^2) \log \mathcal{Z}[L_\beta](\Upsilon; \bar{m}_2, \Sigma_2) \; \dd \Upsilon\right]\\[4pt]
\Psi_{y_0}(t_\mathrm{sym}) &= \alpha \EE_{\xi \sim \mathcal{N}(0, I_{k+1})} \left[ \int_{\RR^{k+1}} \mathcal{Z}[P_\text{data}](\Upsilon; \bar{m}_1, \sigma_1^2) \log \mathcal{Z}[\tilde{L}_{\beta h}](\Upsilon; \bar{m}_2, \Sigma_2) \; \dd \Upsilon\right]\\[4pt]
\Psi_w\left(\hat{t}_\mathrm{sym}\right) &= \plim_{p \to \infty} \frac{1}{2p} \left\{- \trace \log A + \trace (\Xi A^{-1}) + \left\langle J_{\hat{s}_a} + J_{\hat{s}_b}, A^{-1} (J_{\hat{s}_a} + J_{\hat{s}_b})\right\rangle \right\}
\end{cases}
\label{eq:replica-potentials-summary}
\end{align}
and we refer to Sections~\ref{sec:psi-y-sobo} and~\ref{sec:psi-w-sobo} for further definitions and details.\\

As for the optimality conditions for $t_{\text{sym}}$ and $\hat{t}_{\text{sym}}$, we note that because of the symmetry of $q_c, r_c, \hat{q}_c, \hat{r}_c \in \RR^{k \times k}$, the lower and upper triangular elements are not independent. As detailed by~\citet{srinivasam-panda:2023}---and notably in contrast to~\citet{petersen:2008} and others---when differentiating a scalar function $f(S)$ of a \emph{symmetric} matrix $S \in \RR^{d \times d}$, the symmetric gradient $\nabla^{\text{sym}}$ of $f$ at $S$ when varying the function on the manifold of symmetric $d\times d$ matrices corresponds to 
\begin{align}
\nabla^{\text{sym}} f = \text{sym} \left(\nabla f \right)\,,
\label{eq:symmetric-matrix-deriv}
\end{align}
where the gradient on the right-hand side denotes differentiation of $f \colon \RR^{d \times d} \to \RR$ with respect to all $d^2$ components individually, treating them as independent variables, and $\text{sym}(A) = \tfrac{1}{2}(A + A^\top)$.\\

The result for the coupled system of optimality conditions for~\eqref{eq:rf-derivative-summary} at zero external field $h=0$, where $\Psi_{y_0} = 0$, is then given by (using the aforementioned symmetric matrix derivative~\eqref{eq:symmetric-matrix-deriv}):
\begin{align}
\label{eq:optimality1}
\begin{cases}
0 \ \ \ = \nabla_s \Psi_y \,, \qquad \qquad &0  \ \ \ = \nabla_{\hat{s}}\Psi_w\\
\hat{f} \ \ =   \nabla_f \Psi_y\,, \qquad \qquad &f \ \  =  \nabla_{\hat{f}} \Psi_w\\
\hat{q} =  -2 \, \text{sym}\left( \nabla_{q} \Psi_y \right)\,, \qquad \qquad &  q = -2\, \text{sym}\left( \nabla_{\hat{q}} \Psi_w \right)\\
\hat{r} \odot \left(\1_{k+1}^{\otimes 2} + I_{k+1} \right) = 2 \, \text{sym}\left(\nabla_{r} \Psi_y\right)\,, \qquad \qquad & r \odot \left(\1_{k+1}^{\otimes 2} + I_{k+1}\right)  = 2\, \text{sym}\left(\nabla_{\hat{r}} \Psi_w \right)
\end{cases}\,.
\end{align}

Here, we expect $\nabla_{\hat{s}} \Psi_w = 0$ to imply $\hat{s} = 0$, and we also anticipate the implicit system of equations $\nabla_{s} \Psi_y = 0$ to admit a closed-form solution $s^*$ for certain loss functions~$\ell$ in the low-temperature limit $\beta \to \infty$. We proceed with the evaluation of all derivatives now.

\subsubsection{Derivatives of~$\Psi_y$}
We recall that the integrand of the potential $\Psi_y$ in~\eqref{eq:psi-y-result-sobo} is
\begin{equation}
{\cal J}(\Upsilon ; \bar{m}_1, \sigma_1^2, \bar{m}_2, \Sigma_2) = \mathcal{Z}[P_\text{data}](\Upsilon; \bar{m}_1, \sigma_1^2) \cdot \log \mathcal{Z}[L_\beta](\Upsilon; \bar{m}_2, \Sigma_2),
\end{equation}
where
\begin{equation}
\bar{m}_1 =  \left \langle f, q^{-1/2} \xi \right \rangle \in \RR, \quad 
\sigma_1^2 = 1- \left \langle f, q^{-1} f \right \rangle > 0, \quad
\bar{m}_2 = s + q^{1/2} \xi \in \RR^{k+1}, \quad
\Sigma_2 = r-q \in \RR^{(k+1)\times (k+1)}.
\end{equation}
To aid in differentiation, we note that
\begin{equation}
\frac{\dd {\cal J}}{\dd *} = \frac{\partial {\cal J}}{\partial \bar{m}_1} \frac{\partial\bar{m}_1}{\partial *} + 
\frac{\partial {\cal J}}{\partial (\sigma_1^2)} \frac{\partial(\sigma_1^2)}{\partial *} + 
\sum_{i \in [k+1]}\frac{\partial {\cal J}}{\partial \langle e_i, \bar{m}_2 \rangle} \frac{\partial \langle e_i, \bar{m}_2 \rangle}{\partial *} + \sum_{\substack{i,j \in [k+1]\,,\\i \leq j}}
\frac{\dd {\cal J}}{\dd \Sigma_{2, ij}} \frac{\partial\Sigma_{2,ij}}{\partial *}\,,
\end{equation}
where $*$ is a generic stand-in for an overlap parameter. In the above, we separate the partial derivatives the components of~$\bar{m}_2$ as we anticipate differentiating these quantities with respect to a matrix (otherwise, we would need to introduce cumbersome notation for third order tensors). Furthermore, as $\Sigma_2$ is a covariance matrix in ${\cal J}$, it is necessarily symmetric and only has $(k+1)(k+2)/2$ degrees of freedom for which we pick the upper right triangular part of $\Sigma_2$.\\

To find the four derivatives of~${\cal J}$, we apply Stein's identity~\cite{stein:1981} to obtain
\begin{equation}
\begin{cases}
\partial_{\bar{m}_1}{\cal J} = \partial_{\bar{m}_1} \mathcal{Z}[P_\text{data}] \cdot \log \mathcal{Z}[L_\beta] \\[4pt]
\partial_{(\sigma_1^2)}{\cal J}  = \frac{1}{2} \partial_{\bar{m}_1}^2 \mathcal{Z}[P_\text{data}] \cdot \log \mathcal{Z}[L_\beta] \\[4pt]
\partial_{\langle e_i, \bar{m}_2 \rangle}{\cal J}  = \mathcal{Z}[P_\text{data}] \cdot \partial_{\langle e_i, \bar{m}_2 \rangle} \log \mathcal{Z}[L_\beta] \\[4pt]
\nabla_{\Sigma_2} {\cal J} = \frac{1}{2}\mathcal{Z}[P_\text{data}] \cdot \left( \nabla_{\bar{m}_2}^{\otimes 2} \log \mathcal{Z}[L_\beta] + (\nabla_{\bar{m}_2} \log \mathcal{Z}[L_\beta])^{\otimes 2} \right),
\end{cases}
\end{equation}
where for economy we have suppressed the arguments of the function.\\

We also find
\begin{equation}
\begin{cases}
\nabla_s (\bar{m}_1, \sigma_1^2, \langle e_i, \bar{m}_2 \rangle, \Sigma_2) = (0,\; 0,\; e_i,\; 0) \\[4pt]
\nabla_f (\bar{m}_1, \sigma_1^2, \langle e_i, \bar{m}_2 \rangle, \Sigma_2) = (q^{-1/2}\xi,\; -2q^{-1}f,\; 0,\; 0) \\[4pt]
\partial_{r} (\bar{m}_1, \sigma_1^2, \langle e_k, \bar{m}_2 \rangle, \Sigma_{2,ij}) = (0,\; 0,\; 0,\; \frac{1}{2}(e_i \otimes e_j + e_j \otimes e_i)),
\end{cases}
\end{equation}
and
\begin{align}
\label{eq:partial_q_sigma}
\begin{cases}
\partial_{q} \bar{m}_1 &= -\frac{1}{2}\textsc{rs}[(q^{-1/2} \oplus q^{-1/2})^{-1}( q^{-1} \otimes q^{-1}) \textsc{vec}\left(  \xi \otimes  f +   f \otimes \xi \right) ] \\[4pt]
\partial_{q} \sigma_1^2 &= q^{-1} f^{\otimes 2} q^{-1} \\[4pt]
\partial_{q} \langle e_k, \bar{m}_2 \rangle &= \frac{1}{2}\textsc{rs}[(q^{1/2} \oplus q^{1/2})^{-1} \textsc{vec}\left( \xi \otimes e_k + e_k \otimes \xi \right)] \\[4pt]
\partial_{q} \Sigma_{2,ij} &= -\frac{1}{2}(e_i \otimes e_j + e_j \otimes e_i).
\end{cases}
\end{align}
Then, we obtain from the optimality conditions~\eqref{eq:optimality1} that
\begin{align}
\begin{cases}
0 &= \alpha  \EE_{\xi}[\int\dd \Upsilon \; \mathcal{Z}[P_\text{data}] \nabla_{\bar{m}_2} \log\mathcal{Z}[L_\beta]] \\[6pt]
\hat{f}  &= \alpha  \EE_{\xi}[\int\dd \Upsilon \; q^{-1/2} \Big(  
\partial_{\bar{m}_1} \mathcal{Z}[P_\text{data}] \cdot \log \mathcal{Z}[L_\beta] \xi - \partial_{\bar{m}_1}^2 \mathcal{Z}[P_\text{data}] \cdot \log \mathcal{Z}[L_\beta] f
  \Big) ] \\[6pt]
\hat{q} &=  -2\alpha  \EE_{\xi}\biggl[ \int \dd \Upsilon \; \Big( -\frac{1}{2}\textsc{rs}[(q^{-1/2} \oplus q^{-1/2})^{-1}( q^{-1} \otimes q^{-1}) \textsc{vec}\left(  \xi \otimes  f +   f \otimes \xi \right)] \; \partial_{\bar{m}_1}\mathcal{Z}[P_\text{data}] \log\mathcal{Z}[L_\beta] \\[6pt]
&\qquad\qquad +\frac{1}{2} q^{-1}f^{\otimes 2} q^{-1} \; (\partial_{\bar{m}_1}^2 \mathcal{Z}[P_\text{data}]) \log\mathcal{Z}[L_\beta]\\
&\qquad\qquad +\frac{1}{2} \mathcal{Z}[P_\text{data}] \cdot \textsc{rs}[ (q^{1/2}\oplus q^{1/2})^{-1} \textsc{vec} (\xi \otimes  \nabla_{\bar{m}_2} \log\mathcal{Z}[L_\beta]  +   \nabla_{\bar{m}_2} \log\mathcal{Z}[L_\beta] \otimes \xi  ) ] \\
&\qquad\qquad -\frac{1}{2}\mathcal{Z}[P_\text{data}] \cdot \left( \nabla_{\bar{m}_2}^{\otimes 2} \log \mathcal{Z}[L_\beta] + (\nabla_{\bar{m}_2} \log \mathcal{Z}[L_\beta])^{\otimes 2} \right) \Big) \biggr]\\[6pt] 
\hat{r} \odot (\1_{k+1}^{\otimes 2} + I_{k+1}) &= 2\alpha  \EE_{\xi} \biggl[ 
\frac{1}{2}\mathcal{Z}[P_\text{data}] \cdot \left( \nabla_{\bar{m}_2}^{\otimes 2} \log \mathcal{Z}[L_\beta] + (\nabla_{\bar{m}_2} \log \mathcal{Z}[L_\beta])^{\otimes 2} \right) \Big) \biggr]
\end{cases}
\end{align}
The expression for $\hat{q}$ can be simplified through applying the chain rule identities
\begin{equation}
\label{eq:xichainrule}
\nabla_\xi = q^{-1/2} f \partial_{\bar{m}_1} = q^{1/2}\nabla_{\bar{m}_2}.
\end{equation}
By linearity of expectation, the $\xi$~terms in the first components of~$\hat{q}$ and~$\hat{f} $ can be re-written as
\begin{align}
\EE_{\xi}[\xi \partial_{\bar{m}_1} \mathcal{Z}[P_\text{data}] \log \mathcal{Z}[L_\beta]] &=  \EE_{\xi}[\nabla_\xi (\partial_{\bar{m}_1} \mathcal{Z}[P_\text{data}] \log \mathcal{Z}[L_\beta])] \\
&\overset{\eqref{eq:xichainrule}}{=} q^{-1/2}\EE_\xi[f \partial_{\bar{m}_1}^2\mathcal{Z}[P_\text{data}] \cdot\log \mathcal{Z}[L_\beta]] +  q^{1/2}\;\EE_{\xi}[\partial_{\bar{m}_1}\mathcal{Z}[P_\text{data}] \cdot \nabla_{\bar{m}_2} \log\mathcal{Z}[L_\beta]]\,.
\end{align}
Similarly, inspecting the $\xi$~terms in the third component of~$\hat{q}$, we note that after permutation we have quantities of the form
\begin{align}
&\EE_{\xi}[\xi \mathcal{Z}[P_\text{data}] \partial_{\langle e_k, \bar{m}_2\rangle} \log \mathcal{Z}[L_\beta]] =
\EE_{\xi}[\nabla_\xi (\mathcal{Z}[P_\text{data}] \partial_{\langle e_k, \bar{m}_2\rangle} \log \mathcal{Z}[L_\beta])] \nonumber \\
&= \EE_{\xi}\left[ q^{-1/2} f \partial_{\bar{m}_1}\mathcal{Z}[P_\text{data}] \cdot \partial_{\langle e_k, \bar{m}_2 \rangle} \log \mathcal{Z}[L_\beta]  + \mathcal{Z}[P_\text{data}] q^{1/2} \nabla_{\bar{m}_2} (\partial_{\langle e_k, \bar{m}_2 \rangle} \log\mathcal{Z}[L_\beta])\right]\,.
\end{align}
Following this general strategy of converting multiplication by~$\xi$ into differentiation via Stein's identity then applying the chain rule identities, we obtain
\begin{align}
\begin{cases}
0 &=  \EE_{\xi}[\int\dd \Upsilon \; \mathcal{Z}[P_\text{data}] \nabla_{\bar{m}_2} \log\mathcal{Z}[L_\beta]] \\[6pt]
\hat{f}  &= \alpha \EE_{\xi}[\int\dd \Upsilon \; \partial_{\bar{m}_1}\mathcal{Z}[P_\text{data}] \cdot \nabla_{\bar{m}_2} \log\mathcal{Z}[L_\beta] ] \\[6pt]
\hat{q} &=  \alpha  \EE_{\xi}\biggl[ \int \dd \Upsilon \; \Big( -\frac{1}{2}\textsc{rs}[(q^{-1/2} \oplus q^{-1/2})^{-1}( q^{-1} \otimes q^{-1}) (q^{-1/2} \otimes I + I \otimes q^{-1}) \textsc{vec}(f^{\otimes 2})] \; \partial^2_{\bar{m}_1}\mathcal{Z}[P_\text{data}] \log\mathcal{Z}[L_\beta] \\[6pt]
&\qquad\qquad -\frac{1}{2}\textsc{rs}[(q^{-1/2} \oplus q^{-1/2})^{-1}( q^{-1} \otimes q^{-1})\textsc{vec} \left(  q^{1/2}\nabla_{\bar{m}_2} \log\mathcal{Z}[L_\beta] \otimes  f +   f \otimes q^{1/2}\nabla_{\bar{m}_2} \log\mathcal{Z}[L_\beta] \right)] \; \partial_{\bar{m}_1}\mathcal{Z}[P_\text{data}] \\[6pt]
&\qquad\qquad +\frac{1}{2} q^{-1}f^{\otimes 2} q^{-1} \; (\partial_{\bar{m}_1}^2 \mathcal{Z}[P_\text{data}]) \log\mathcal{Z}[L_\beta]\\[6pt]
&\qquad\qquad -\frac{1}{2} \mathcal{Z}[P_\text{data}] \cdot \textsc{rs}[ (q^{1/2}\oplus q^{1/2})^{-1} \textsc{vec}(q^{-1/2}f \otimes  \nabla_{\bar{m}_2} \log\mathcal{Z}[L_\beta]  +   \nabla_{\bar{m}_2} \log\mathcal{Z}[L_\beta] \otimes q^{-1/2}f ) ] \; \partial_{\bar{m}_1}\mathcal{Z}[P_\text{data}] \\[6pt]
&\qquad\qquad -\frac{1}{2} \mathcal{Z}[P_\text{data}] \cdot \textsc{rs}[ (q^{1/2}\oplus q^{1/2})^{-1} \textsc{vec}(q^{1/2} \nabla_{\bar{m}_2} \otimes  \nabla_{\bar{m}_2} \log\mathcal{Z}[L_\beta]  +   \nabla_{\bar{m}_2} \otimes q^{1/2} \nabla_{\bar{m}_2}  \log\mathcal{Z}[L_\beta]  ) ] \\[6pt]
&\qquad\qquad -\frac{1}{2}\mathcal{Z}[P_\text{data}] \cdot \left( \nabla_{\bar{m}_2}^{\otimes 2} \log \mathcal{Z}[L_\beta] + (\nabla_{\bar{m}_2} \log \mathcal{Z}[L_\beta])^{\otimes 2} \right) \Big) \biggr]\\[6pt] 
&\hspace{-.5cm}\hat{r} \odot (\1_{k+1}^{\otimes 2} + I_{k+1}) = \alpha  \EE_{\xi} \biggl[ 
\mathcal{Z}[P_\text{data}] \cdot \left( \nabla_{\bar{m}_2}^{\otimes 2} \log \mathcal{Z}[L_\beta] + (\nabla_{\bar{m}_2} \log \mathcal{Z}[L_\beta])^{\otimes 2} \right) \Big) \biggr].
\end{cases}
\end{align}
The expression for $\hat{q}$ can be dramatically simplified: by making use of the identities
\begin{equation}
\left(q^{-1/2} \oplus q^{-1/2}\right) \left(q^{-1} \otimes q^{-1}\right) \left(q^{-1/2} \oplus q^{1/2}\right) = q^{-1} \otimes q^{-1}
\end{equation}
and
\begin{equation}
\left(q^{1/2} \oplus q^{1/2} \right)^{-1} = \left(q^{-1/2} \otimes q^{-1/2}\right)^{-1} \left(q^{-1/2} \oplus q^{-1/2}\right) =  \left(q^{-1/2} \oplus q^{-1/2}\right)  \left(q^{-1/2} \otimes q^{-1/2} \right)^{-1},
\end{equation}
which can be easily verified by 
applying the eigenvalue decomposition for~$q$, we obtain
\begin{equation}
\hat{q} = \alpha\EE_{\xi} \left[ ( \nabla_{\bar{m}_2} \log\mathcal{Z}[L_\beta] )^{\otimes 2} \right]\,.
\end{equation}
It is also convenient to further define
\begin{equation}
\hat{\Sigma}_2 :=  \hat{q} - \hat{r} \odot \left(\1_{k+1}^{\otimes 2} + I_{k+1} \right) 
\end{equation}
so that the saddlepoint equations~\eqref{eq:optimality1} from the $\Psi_y$-derivatives at finite $\beta$ and for a generic loss function $\ell$ finally reduce to
\begin{align}
\begin{cases}
0 &= \alpha  \EE_{\xi}[\int\dd \Upsilon \; \mathcal{Z}[P_\text{data}] \nabla_{\bar{m}_2} \log\mathcal{Z}[L_\beta]] \\[6pt]
\hat{f}  &= \alpha  \EE_{\xi}\left[\int\dd \Upsilon \; q^{-1/2} \Big(  
\partial_{\bar{m}_1} \mathcal{Z}[P_\text{data}] \cdot \log \mathcal{Z}[L_\beta] \xi - \partial_{\bar{m}_1}^2 \mathcal{Z}[P_\text{data}] \cdot \log \mathcal{Z}[L_\beta] f
  \Big) \right] \\[6pt]
\hat{q} &= \alpha\EE_{\xi} \left[  \int \dd \Upsilon\; \mathcal{Z}[P_\text{data}] \cdot (\nabla_{\bar{m}_2} \log\mathcal{Z}[L_\beta])^{\otimes 2} \right] \\[6pt]
\hat{\Sigma}_{2}  &= -\alpha  \EE_{\xi} \biggl[ 
\int \dd \Upsilon\; \mathcal{Z}[P_\text{data}] \cdot \nabla_{\bar{m}_2}^{\otimes 2} \log \mathcal{Z}[L_\beta] \biggr]\,,
\end{cases}
\label{eq:psi-y-saddle-finite-temp}
\end{align}
which should be compared to the corresponding expression for $L^2$ training in~\cite{gerace-loureiro-krzakala-etal:2021}.

\subsubsection{Derivatives of~$\Psi_w$}

For the~$\hat{s}$ derivatives of the potential $\Psi_w$ as defined in~\eqref{eq:psi-w-result}, we have
\begin{align}
\begin{cases}
\nabla_{\hat{s}_a} \Psi_w &= \plim_{p\to\infty} \frac{1}{p} \left( \kappa_0^2 \1_p^\top A^{-1} \1_p \hat{s}_a + \kappa_0\kappa_0' \1_p^\top A^{-1} \Theta^\top V_k \hat{s}_b \right) \\[4pt]
\nabla_{\hat{s}_b} \Psi_w &= \plim_{p\to\infty} \frac{1}{p} \left( \kappa_0 \kappa_0' V_k^\top \Theta A^{-1} \1_p \hat{s}_a + (\kappa_0')^2 V_k^\top \Theta A^{-1} \Theta^\top V_k \hat{s}_b \right)\,,
\end{cases}
\end{align}
so at optimality by~\eqref{eq:optimality1} we have
\begin{equation}
\label{eq:opt_cond_hat_s_ab}
\plim_{p\to\infty} \frac{1}{p} 
\begin{pmatrix} \kappa_0^2 \1_p^\top A^{-1} \1_p & \kappa_0\kappa_0' \1_p^\top A^{-1} \Theta^\top V_k  \\[6pt]
\kappa_0 \kappa_0' V_k^\top \Theta A^{-1} \1_p & (\kappa_0')^2 V_k^\top \Theta A^{-1} \Theta^\top V_k \end{pmatrix}
\begin{pmatrix} \hat{s}_a \\ \hat{s}_b \end{pmatrix}
= 0.
\end{equation}
We expect this to imply $\hat{s} = 0$ at optimality.\\

For the~$\hat{f}$ derivatives of the potential $\Psi_w$ as defined in~\eqref{eq:psi-w-result}, we have
\begin{align}
\begin{cases}
\nabla_{\hat{f}_a} \Psi_w &= \plim_{p \to \infty} \left( \kappa_1^2 \trace(\Theta^\top \theta_0 \theta_0^\top \Theta A^{-1}) \hat{f}_a + \kappa_1 \kappa_1' \trace(\Theta^\top \theta_0 \theta_0^\top \Theta \textsc{diag}(\Theta^\top V_k \hat{f}_b) A^{-1} ) \right) \\[4pt]
&= \plim_{p \to \infty} \left( \kappa_1^2 \trace(\Theta^\top \theta_0 \theta_0^\top \Theta A^{-1}) \hat{f}_a + \kappa_1 \kappa_1' \1_p^\top ((A^{-1} \Theta^\top \theta_0 \theta_0^\top \Theta) \odot I_p) \Theta^\top V_k \hat{f}_b \right) \\[4pt]
\nabla_{\hat{f}_b} \Psi_w &=  \plim_{p \to \infty} \left (\kappa_1 \kappa_1' V_k^\top \Theta ((A^{-1} \Theta^\top \theta_0 \theta_0^\top \Theta) \odot I_p) \1_p \hat{f}_a + (\kappa_1')^2 V_k^\top \Theta (A^{-1} \odot (\Theta^\top \theta_0 \theta_0^\top \Theta)) \Theta^\top V_k \hat{f}_b  \right)\,.
\end{cases}
\end{align}

We shall also require the following lemma.
\begin{lemma}
For all symmetric $A, B, C \in \RR^{p \times p}$ we have the identity
$
\trace( (A \odot B) C)) = \trace( (B \odot C) A) 
$.
\end{lemma}
\begin{proof}
Consider the singular value decomposition of $B = \sum_k \sigma_k v_k v_k^\top$. Then
\begin{equation}
A \odot B = A \odot \left(\sum_k \sigma_k v_k v_k^\top\right) = \sum_k \sigma_k A \odot (v_kv_k^\top) = \sum_k \sigma_k \textsc{diag}(v_k) A \; \textsc{diag}(v_k).
\end{equation}
Substituting this into the trace, we note that
\begin{equation}
\trace ((A \odot B) C) = \sum_k \sigma_k \trace (\textsc{diag}(v_k) A\; \textsc{diag}(v_k) C) = \sum_k \sigma_k \trace( A\; \textsc{diag}(v_k) C\; \textsc{diag}(v_k) ) = \trace(A (B \odot C)).
\end{equation}
\end{proof}

This lemma shows, for example, the identity
\begin{equation}
\trace (A^{-1} \Theta^\top \Theta) = \trace (A^{-1} (\Theta^\top \Theta \odot \1_p\1_p^\top)) =  \trace ((A^{-1} \odot (\Theta^\top \Theta)) \1_p\1_p^\top) = \1_p^\top (A^{-1} \odot (\Theta^\top \Theta)) \1_p.
\end{equation}
For~$\hat{q}_a, \hat{q}_b, \hat{q}_c$ the derivatives are
\begin{align}
\begin{cases}
\nabla_{\hat{q}_a}\Psi_w &= \plim_{p\to\infty}-\frac{1}{2p} \trace\left(  A^{-1} \Xi A^{-1} (\kappa_1^2 \Theta^\top \Theta + \kappa_*^2 I_p)  \right) \\[4pt]
\nabla_{\hat{q}_b}\Psi_w &= \plim_{p\to\infty}-\frac{1}{p}  V_k^\top \Theta \left( (A^{-1}\Xi A^{-1} \kappa_1\kappa_1' \Theta^\top\Theta)\odot I_p \right) \1_p \\[4pt]
\nabla_{\hat{q}_c}\Psi_w &= \plim_{p\to\infty}-\frac{1}{2p}  V_k^\top \Theta \left(  (A^{-1} \Xi A^{-1}) \odot((\kappa_1')^2\Theta^\top \Theta + (\kappa_*')^2 I_p)\right) \Theta^\top V_k,
\end{cases}
\end{align}
where for $\hat{q}_c$ we made use of the lemma above to compute the symmetric gradient. More succinctly, making use of the identity $(\nabla_{Q}^{\text{sym}}f)_{ij} = (1 - \frac{1}{2}\delta_{ij})\nabla_{q_{ij}}f$, we have
\begin{align}
\nabla_{\hat{q}}^{\text{sym}} \Psi_w = \plim_{p\to\infty} \frac{-1}{2p} \Bigg( &\; \begin{pmatrix} \kappa_1 \1_p^\top \\ \kappa_1' V_k^\top\Theta \end{pmatrix} ((A^{-1}\Xi A^{-1} ) \odot (\Theta^\top\Theta)) \begin{pmatrix} \kappa_1 \1_p & \kappa_1' \Theta^\top V_k \end{pmatrix} \\
&+\begin{pmatrix} \kappa_*^2 \1_p^\top ((A^{-1}\Xi A^{-1} ) \odot I_p) \1_p & 0 \\ 0 & \left( \kappa_*' \right)^2 V_k^\top\Theta ((A^{-1}\Xi A^{-1} ) \odot I_p) \Theta^\top V_k \end{pmatrix}\;  \Bigg).
\end{align}

Similarly, for~$\hat{r}_a, \hat{r}_b, \hat{r}_c$ the derivatives are
\begin{align}
\begin{cases}
\nabla_{\hat{r}_a} \Psi_w &= \plim_{p\to\infty}  \frac{1}{p} \trace \left( (A^{-1} + A^{-1}\Xi A^{-1}) (\kappa_1^2\Theta^\top \Theta + \kappa_*^2 I_p)  \right) \\[4pt]
\nabla_{\hat{r}_b}\Psi_w &= \plim_{p\to\infty}\frac{1}{p}  V_k^\top \Theta \left( (  (A^{-1} + A^{-1}\Xi A^{-1}) \kappa_1\kappa_1' \Theta^\top\Theta  )\odot I_p \right) \1_p  \\[4pt]
\nabla_{\hat{r}_c}\Psi_w &= \plim_{p\to\infty}\frac{1}{2p} (V_k^\top \Theta \left( (A^{-1} + A^{-1}\Xi A^{-1}) \odot ((\kappa_1')^2 \Theta^\top \Theta + (\kappa_*')^2 I_p) \right) \Theta^\top V_k) \odot (I_k + \1_k\1_k^\top),
\end{cases}
\end{align}
or altogether we can write 
\begin{align}
\nabla_{\hat{r}}^{\text{sym}} \Psi_w 
&= \plim_{p\to\infty} \frac{1}{2p} \Bigg( \; \begin{pmatrix} \kappa_1 \1_p^\top \\ \kappa_1' V_k^\top\Theta \end{pmatrix} ((A^{-1} + A^{-1}\Xi A^{-1} ) \odot (\Theta^\top\Theta)) \begin{pmatrix} \kappa_1 \1_p & \kappa_1' \Theta^\top V_k \end{pmatrix} + \ldots \notag \\[6pt]
&+
\begin{pmatrix} (\kappa_*)^2 \1_p^\top  ((A^{-1} + A^{-1}\Xi A^{-1} ) \odot I_p) \1_p & 0 \\ 0 & (\kappa_*')^2 V_k^\top\Theta  ((A^{-1} + A^{-1}\Xi A^{-1} ) \odot I_p)  \Theta^\top V_k \end{pmatrix} \;  \Bigg)
\odot (I_{k+1} + \1_{k+1}\1_{k+1}^\top).
\end{align}

To simplify the update for $f$ further, we use the identity 
\begin{align}
	\1^\top ( AB \odot I ) v  =  \trace( AB \; \textsc{diag}(v)  ) = \trace( I B \; \textsc{diag}(v) A  )  = \1^\top ( A \odot B) v
\end{align}
for symmetric $A$ and generic $B$ and $v$ of appropriate dimensions.\\

All in all, from those optimality conditions in~\eqref{eq:optimality1} that involve $\Psi_w$-derivatives, we obtain the set of equations
\begin{align}
\begin{cases} 
0 &= 
\displaystyle \plim_{p\to\infty} \frac{1}{p} \Bigg( \; \begin{pmatrix} \kappa_0 \1_p^\top \\ \kappa_0' V_k^\top\Theta \end{pmatrix} A^{-1} \begin{pmatrix} \kappa_0 \1_p & \kappa_0' \Theta^\top V_k \end{pmatrix} \Bigg) \, \hat{s} \\[10pt]
f &=  \displaystyle\plim_{p\to\infty} \begin{pmatrix} \kappa_1 \1_p^\top \\ \kappa_1' V_k^\top\Theta \end{pmatrix} (A^{-1} \odot (\Theta^\top \theta_0 \theta_0^\top \Theta)) \begin{pmatrix} \kappa_1 \1_p & \kappa_1' \Theta^\top V_k \end{pmatrix}  \hat{f} \\[10pt]
q &=
\displaystyle \plim_{p\to\infty} \frac{1}{p} \Bigg( \; \begin{pmatrix} \kappa_1 \1_p^\top \\ \kappa_1' V_k^\top\Theta \end{pmatrix} ((A^{-1}\Xi A^{-1} ) \odot (\Theta^\top\Theta)) \begin{pmatrix} \kappa_1 \1_p & \kappa_1' \Theta^\top V_k \end{pmatrix} \\[1pt]
&\qquad \qquad + \begin{pmatrix} (\kappa_*)^2 \1_p^\top((A^{-1}\Xi A^{-1} ) \odot I_p) \1_p & 0 \\ 0 &
 (\kappa_*')^2 V_k^\top\Theta((A^{-1}\Xi A^{-1} ) \odot I_p)  \Theta^\top V_k \end{pmatrix} \;  \Bigg)  \\[4pt]
r &=  
\displaystyle \plim_{p\to\infty} \frac{1}{p} \Bigg( \; \begin{pmatrix} \kappa_1 \1_p^\top \\ \kappa_1' V_k^\top\Theta \end{pmatrix} ((A^{-1} + A^{-1}\Xi A^{-1} ) \odot (\Theta^\top\Theta)) \begin{pmatrix} \kappa_1 \1_p & \kappa_1' \Theta^\top V_k \end{pmatrix} \\[4pt]
&\qquad \qquad 
+\begin{pmatrix} (\kappa_*)^2 \1_p^\top ((A^{-1} + A^{-1}\Xi A^{-1} ) \odot I_p) \1_p & 0 \\ 0 & (\kappa_*')^2 V_k^\top\Theta  ((A^{-1} + A^{-1}\Xi A^{-1} ) \odot I_p)  \Theta^\top V_k \end{pmatrix} \;  \Bigg) \,.
\end{cases} 
\label{eq:psi-w-saddle-finite-beta-1}
\end{align}
Using $\Sigma_2 = r - q$ instead of $r$, the last equation can be simplified to
\begin{align}
\Sigma_2 =  
\plim_{p\to\infty} \frac{1}{p} \Bigg(& \; \begin{pmatrix} \kappa_1 \1_p^\top \\ \kappa_1' V_k^\top\Theta \end{pmatrix} (A^{-1} \odot (\Theta^\top\Theta)) \begin{pmatrix} \kappa_1 \1_p & \kappa_1' \Theta^\top V_k \end{pmatrix} \\
& +\begin{pmatrix} (\kappa_*)^2 \1_p^\top (A^{-1} \odot I_p) \1_p & 0 \\ 0 &  (\kappa_*')^2 V_k^\top\Theta (A^{-1} \odot I_p) \Theta^\top V_k \end{pmatrix} \;  \Bigg)\,.
\label{eq:psi-w-saddle-finite-beta-2}
\end{align}

\subsection{Training error as $\beta \to \infty$, and temperature scalings of the overlap parameters}
\label{sec:trainig-err}

Recall that the training error is given by~\eqref{eq:training-err-derivative-grad} with the free energy density 
\begin{align}
f_\beta(h=0)& = - \mathrm{crit}_{t_\mathrm{sym}, \hat{t}_\mathrm{sym}} \, \left\{ \Psi_y(t_\mathrm{sym}) + \Psi_w \left(\hat{t}_\mathrm{sym} \right) - \left( \left \langle f, \hat{f} \right \rangle - \frac{1}{2} \left \langle q, \hat{q} \right \rangle_{\text{F}} + \left \langle r, \hat{r} \right \rangle_{\text{HF}}  \right) \right\}\,.
\label{eq:training-error-through-derivs}
\end{align}
We hence would like to consider the low-temperature limit $\beta \to \infty$ in the saddle-point equations derived so far. Making the $\beta$-dependence explicit, we can schematically write the free energy density as
\begin{align}
f_\beta(0) = -\Phi_\beta\left(t^*_{\text{sym}}\left(\beta\right), \hat{t}^*_{\text{sym}}\left(\beta \right) \right) \quad \text{with} \quad \nabla_{t_\mathrm{sym}}\Phi_\beta \left(t_\mathrm{sym}^*\left(\beta\right), \hat{t}_\mathrm{sym}^*\left(\beta \right) \right) = \nabla_{\hat{t}_\mathrm{sym}}\Phi_\beta \left(t_\mathrm{sym}^*\left(\beta\right), \hat{t}_\mathrm{sym}^*\left(\beta \right) \right) = 0\,,
\end{align}
with superscript $*$ denoting the critical point. By using the chain rule and optimality conditions, similar to Section~\ref{sec:sobolev_replica_ansatz}, we have
$
\partial_\beta f_\beta(0) = - (\partial_\beta \Phi_\beta )\left(t_\mathrm{sym}^*\left(\beta\right), \hat{t}_\mathrm{sym}^*\left(\beta \right) \right)\,,
$
meaning that we only have to explicitly differentiate $\Phi$ in $\beta$ and only need to insert the solution of the saddlepoint equations, without differentiating through them. Hence
\begin{align}
\partial_\beta f_\beta(0) = - \left(\partial_\beta \Psi_w \right)\left(\hat{t}^*_{\text{sym}} \right) - \left( \partial_\beta \Psi_y \right) \left(t^*_{\text{sym}} \right)\,.
\end{align}
Calculating these derivatives at finite $\beta$, then substituting the optimal overlap parameters and taking the limit $\beta \to \infty$, we find, since the only explicit $\beta$-dependence within $\Psi_w$ as given by~\eqref{eq:psi-w-result} is in $A$, that
\begin{align}
	\label{eq:psiw-deriv-interm-sobo}
	\partial_\beta \Psi_w = \plim_{p\to\infty} \frac{-\lambda}{2p} \bigg(\trace\left[A^{-1}\right] + \trace \left[ A^{-1} \Xi A^{-1} \right] + \left\langle J_{\hat{s}_a} + J_{\hat{s}_b}, A^{-2} (J_{\hat{s}_a} + J_{\hat{s}_b}) \right\rangle \bigg)\,.
\end{align}
For the derivative $\partial_\beta \Psi_y$ from~\eqref{eq:psi-y-result-sobo}, we note that
\begin{align}
    \partial_\beta \mathcal{Z}\left[L_\beta\right]\left(\Upsilon; \Bar{m}_2, \Sigma_2\right) = -\EE_{\tilde{\Upsilon} \sim \mathcal{N}\left(s + q^{1/2}\xi, r - q \right)} \left[  \ell \left(\Upsilon, \tilde{\Upsilon} \right)   \exp \left\{ - \beta \ell \left(\Upsilon, \tilde{\Upsilon} \right) \right\}\right]\,,
\end{align}
such that
\begin{align}
  \partial_\beta \Psi_y = - \alpha \EE_\xi \left[ \int_\RR \dd \Upsilon \; \mathcal{Z}[P_\text{data}] \left(\Upsilon;\Bar{m}_1,\sigma_1^2\right) \frac{   \int_\RR \frac{\dd \tilde{\Upsilon} \; \ell \left(\Upsilon, \tilde{\Upsilon}\right)    \exp \left\{ \frac{-1}{2} \left(\tilde{\Upsilon} -  s+\sqrt{s}\xi \right)^\top (r - q)^{-1}\left(\tilde{\Upsilon} -  s+\sqrt{s}\xi \right) - \beta \ell \left(\Upsilon,  \tilde{\Upsilon}\right)  \right\}     }{(2 \pi)^{(k+1)/2} \det(r - q)^{1/2}}  }{   \int_\RR \frac{\dd \tilde{\Upsilon} \;  \exp \left\{ \frac{-1}{2} \left(\tilde{\Upsilon} -  s+\sqrt{s}\xi \right)^\top (r - q)^{-1}\left(\tilde{\Upsilon} -  s+\sqrt{s}\xi \right) - \beta \ell \left(\Upsilon,  \tilde{\Upsilon}\right)  \right\}     }{(2 \pi)^{(k+1)/2} \det(r - q)^{1/2}} }   \right]\,.
  \label{eq:psiy-deriv-interm-sobo}
\end{align}
The form of~\eqref{eq:psiy-deriv-interm-sobo} is essential in positing an ansatz for the critical overlap parameters  
$t_\mathrm{sym}^*(\beta \to \infty)$ and $\hat{t}_\mathrm{sym}^*(\beta \to \infty)$. We defer these computations to Subsection~\ref{sec:beta-scaling-sobo} below. This ansatz then permits us to obtain semi-analytical simplifications for the training error in the proportional asymptotics limit. Specifically,
using the scaling relations introduced in Subsection~\ref{sec:beta-scaling-sobo} below, in the low-temperature limit we have
$r - q = \Sigma_2  = \Sigma_2^\infty / \beta$
whereas the other parameters in~\eqref{eq:psiy-deriv-interm-sobo} do not scale with $\beta$. Applying Laplace's method for the two $\tilde{\Upsilon}$ integrals in the numerator and denominator of~\eqref{eq:psiy-deriv-interm-sobo} as $\beta \to \infty$ (while dropping the $\infty$ superscripts) then leads to
\begin{align}
\lim_{\beta \to \infty} \partial_\beta \Psi_y = - \alpha \EE_{\xi \sim \mathcal{N}(0,I_{k+1})} \left[\int_\RR \dd \Upsilon \; \mathcal{Z}[P_\text{data}]\left(\Upsilon;\Bar{m}_1,\sigma_1^2\right) \cdot \ell \left(\Upsilon, \tilde{\Upsilon}_\ell^* \left(\Upsilon; \bar{m}_2, \Sigma_2 \right) \right) \right]
\label{eq:dbetapsiw-general}
\end{align}
where we defined the minimizer
\begin{align}
	\label{eq:sobo-minimizer}
    \tilde{\Upsilon}_\ell^* \left(\Upsilon; \bar{m}_2, \Sigma_2 \right) = \argmin_{\tilde{\Upsilon} \in\RR^{k+1}} \left[ \frac{1}{2}\left(\tilde{\Upsilon} - \bar{m}_2\right)^\top \Sigma_2^{-1} \left(\tilde{\Upsilon} - \bar{m}_2\right) +  \ell(\Upsilon, \tilde{\Upsilon}) \right]\,,
\end{align}
and the parameters are (all of which are $O(1)$ in $\beta$)
\begin{align}
\bar{m}_1 = \langle f,q^{-1/2} \xi\rangle\,, \quad \sigma_1^2 =  1- \langle f, q^{-1} f\rangle \,, \quad \bar{m}_2 = s + q^{1/2}\xi \,, \quad \Sigma_2 = r - q\,.
\end{align}

It is possible to further simplify the expression for the limit of $\partial_\beta \Psi_y$ for specific loss functions~$\ell$ and data distributions~$P_\mathrm{data}$. We detail these calculations for the Gaussian observation model and Sobolev training below in Section~\ref{sec:add-gauss-standard-sobo}.

\subsubsection{Optimal overlap parameters as $\beta \to \infty$}
\label{sec:beta-scaling-sobo}

We will posit an ansatz for the optimal overlap parameters in the $\beta \to \infty$ limit here. This reparameterization yields an effective low-temperature system of saddle point equations which only needs to be solved once, instead of for each element of an increasing sequence of $\beta$ realizations. We will also consider the scaling of derived parameters
\begin{align}
\begin{cases}
	\Sigma_2 &= r - q \\
	\hat{\Sigma}_2 &= \hat{q} -  \hat{r} \odot \left(\1_{k+1} \1_{k+1}^\top + I_{k+1}\right)
\end{cases}
\end{align}
To propose this ansatz, we first examine $\partial_\beta \Psi_w$ as given in~\eqref{eq:psiw-deriv-interm-sobo}. The only explicit dependence of this expression on the inverse temperature~$\beta$ is through the matrix~$A = \beta \lambda I_p + \dots$ as defined through~\eqref{eq:def-a-init} and~\eqref{eq:def-a-matrices}. Consequently, we expect that both $A$ and $J$ (the latter was defined in~\eqref{eq:def-j-vectors}) scale linearly with $\beta$. We can then expect 
\begin{align}
\begin{cases}
	\hat{s} = \beta \hat{s}^\infty\,, \qquad \quad &\hat{q} = \beta^2  \hat{q}^\infty\,,\\
	\hat{f} = \beta \hat{f}^\infty\,, \qquad \quad &\hat{\Sigma}_2 = \beta \hat{\Sigma}_2^\infty\,.
\end{cases}
\end{align}
We now consider $\partial_\beta \Psi_y$ in~\eqref{eq:psiy-deriv-interm-sobo}. We will only obtain a nontrivial result, as calculated above using Laplace's method in~\eqref{eq:dbetapsiw-general}, if the integrals with respect to $\tilde{\Upsilon}$ will contract about their value at \eqref{eq:sobo-minimizer}. This behavior will occur only if $r-q$ is of order~$\beta^{-1}$ while the other overlap parameters in \eqref{eq:psiy-deriv-interm-sobo} are constant in $\beta$. Thus, we define
\begin{align}
\begin{cases}
	s =  s^\infty\,, \qquad \quad &q =  q^\infty \\
	f =  f^\infty\,, \qquad \quad &\Sigma_2 = \frac{1}{\beta} \Sigma_2^\infty
\end{cases}
\end{align}
Recalling the definitions
\begin{align}
\begin{cases}
	A &=  \beta\lambda I_p + \begin{pmatrix} \kappa_1 \1_p & \kappa_1' \Theta^\top V_k \end{pmatrix} \hat{\Sigma}_2 \begin{pmatrix} \kappa_1 \1_p^\top \\[2pt] \kappa_1' V_k^\top\Theta \end{pmatrix} \odot \Theta^\top \Theta + \begin{pmatrix} \kappa_* \1_p & \kappa_*' \Theta^\top V_k \end{pmatrix} \begin{pmatrix}
\hat{\Sigma}_{2,a} & 0\\
0 & \hat{\Sigma}_{2,c}
	\end{pmatrix} \begin{pmatrix} \kappa_* \1_p^\top \\[2pt] \kappa_*' V_k^\top\Theta \end{pmatrix} \odot I_p \\
	\Xi &= p \cdot \begin{pmatrix} \kappa_1 \1_p & \kappa_1' \Theta^\top V_k \end{pmatrix} \lp \hat{f}\hat{f}^\top \rp \begin{pmatrix} \kappa_1 \1_p^\top \\[2pt] \kappa_1' V_k^\top\Theta \end{pmatrix} \odot \left( \Theta^\top \theta_0 \theta_0^\top \Theta \right) + \begin{pmatrix} \kappa_1 \1_p & \kappa_1' \Theta^\top V_k \end{pmatrix}  \hat{q}  \begin{pmatrix} \kappa_1 \1_p^\top \\[2pt] \kappa_1' V_k^\top\Theta \end{pmatrix} \odot \Theta^\top \Theta  \nonumber \\
    &\qquad \qquad +\begin{pmatrix} \kappa_* \1_p & \kappa_*' \Theta^\top V_k \end{pmatrix} \begin{pmatrix}
    \hat{q}_a & 0\\
    0 & \hat{q}_c
    \end{pmatrix} \begin{pmatrix} \kappa_* \1_p^\top \\[2pt] \kappa_*' V_k^\top\Theta \end{pmatrix} \odot I_p 
    \end{cases}
\end{align}
from~\eqref{eq:def-a-init}, \eqref{eq:def-xi-init}, \eqref{eq:def-a-matrices} and \eqref{eq:def-j-vectors}, we can further define $A = \beta A^{\infty}$ and $\Xi = \beta^2 \Xi^\infty$.
In particular, we then see from~\eqref{eq:psiw-deriv-interm-sobo} that
$
\lim_{\beta \to \infty }\partial_\beta \Psi_w = \plim_{p\to\infty} \tfrac{-\lambda}{2p} \trace \left[ A^{-1} \Xi A^{-1} \right] 
$
in terms of the zero-temperature parameters as long as $\hat{s} = 0$, which leads to~\eqref{eq:training-err-reg} for the regularization term at optimality in the main text.\\

We now use the zero-temperature parameters to construct a set of corresponding saddle point equations.  All overlap parameters in the following are also in the $\beta \to \infty$ regime but we suppress the $\infty$ superscripts for concision. We can then write from~\eqref{eq:psi-w-saddle-finite-beta-1} and~\eqref{eq:psi-w-saddle-finite-beta-2} that 
\begin{align}
\begin{cases}
\label{eq:lowtempsobosaddle}
	f &= \displaystyle \plim_{p\to\infty} \begin{pmatrix} \kappa_1 \1_p^\top \\ \kappa_1' V_k^\top\Theta \end{pmatrix} (A^{-1} \odot (\Theta^\top \theta_0 \theta_0^\top \Theta)) \begin{pmatrix} \kappa_1 \1_p & \kappa_1' \Theta^\top V_k \end{pmatrix} \hat{f} \\[10pt]
	q &= \displaystyle \plim_{p\to\infty} \frac{1}{p} \Bigg( \ \begin{pmatrix} \kappa_1 \1_p^\top \\ \kappa_1' V_k^\top\Theta \end{pmatrix} ((A^{-1}\Xi A^{-1} ) \odot (\Theta^\top\Theta)) \begin{pmatrix} \kappa_1 \1_p & \kappa_1' \Theta^\top V_k \end{pmatrix} \\[10pt]
	& \qquad  \qquad   \qquad  \qquad +  \begin{pmatrix} (\kappa_*)^2 \1_p^\top ((A^{-1}\Xi A^{-1} ) \odot I_p) \1_p & 0 \\
	0 &  (\kappa_*')^2 V_k^\top\Theta ((A^{-1}\Xi A^{-1} ) \odot I_p) \Theta^\top V_k \end{pmatrix} \  \Bigg) \\[10pt]
	\Sigma_2 &= \displaystyle \plim_{p\to\infty} \frac{1}{p} \Bigg( \ \begin{pmatrix} \kappa_1 \1_p^\top \\ \kappa_1' V_k^\top\Theta \end{pmatrix} (A^{-1} \odot (\Theta^\top\Theta)) \begin{pmatrix} \kappa_1 \1_p & \kappa_1' \Theta^\top V_k \end{pmatrix} \\[10pt]
	& \qquad   \qquad \qquad   \qquad + \begin{pmatrix} (\kappa_*)^2 \1_p^\top (A^{-1} \odot I_p) \1_p & 0 \\ 
	0 & (\kappa_*')^2 V_k^\top\Theta (A^{-1} \odot I_p)  \Theta^\top V_k \end{pmatrix} \ \Bigg).
\end{cases}
\end{align}
With this, we have derived~\eqref{eq:update_nonhat_sobo_h1} in the main text. Notably, this set of equations does not depend on the choice $\ell$. Similarly, we use the fact that ${\cal N}(\bar{m}_1,\Sigma_2)$ concentrates in the low temperature limit to find from~\eqref{eq:psi-y-saddle-finite-temp} that
\begin{align}
\begin{cases}
0 &=  \alpha \EE_{\xi,\omega,\Upsilon}[ \Sigma_2^{-1} (\tilde{\Upsilon}_\ell^* - \bar{m}_2) ] \\[4pt]
\hat{f} &= \alpha \EE_{\xi,\omega,\Upsilon}\left[ \frac{\omega- \bar{m}_1}{\sigma_1^2}  \Sigma_2^{-1} (\tilde{\Upsilon}_\ell^* - \bar{m}_2)  \right] = \alpha \Sigma_2^{-1}\EE_{\xi,\omega}\left[ \partial_{\omega} \EE_{\Upsilon \mid \omega}[\tilde{\Upsilon}_\ell^*] \right]\\[4pt]
\hat{q}  &=  \alpha \; \Sigma_2^{-1}\EE_{\xi,\omega,\Upsilon}[(\tilde{\Upsilon}_\ell^* - \bar{m}_2)^{\otimes 2}] \Sigma_2^{-1}   \\[4pt]
\hat{\Sigma}_2  &=  \alpha  \Sigma_2^{-1} (I_{k+1} - \EE_{\xi,\omega,\Upsilon}\left[ \nabla_{\bar{m}_2} \tilde{\Upsilon}_\ell^{*\top} \right] ).
\end{cases}
\label{eq:psi-y-saddle-inf-beta-gen}
\end{align}
where $\xi\sim {\cal N}(0,I_{k+1})$ and $\omega\sim {\cal N}(\bar{m}_1,\sigma_1^2)$. The second equality for~$\hat{f}$ is obtained by recognizing $\omega$ is conditionally independent of~$\bar{m}_2$ given~$\xi$ and applying Stein's identity. Note that the first condition provides an implicit optimality condition for $s$. Here, it now only remains to specify $\ell$ to the standard subspace Sobolev loss~\eqref{eq:loss-sobo-standard} in order to arrive at~\eqref{eq:update_hat_sobo_h1} from the main text.

\subsection{Calculating the generalization error}

We can evaluate the generalization error~\eqref{eq:gen-error-def-appendix} from~\eqref{eq:generalization-err-derivative-grad} and~\eqref{eq:rf-derivative-summary} via
\begin{align}
\plim_{p \to \infty} \eps_{\text{gen}}(w^*) \mid \varpi = -\frac{1}{\alpha} \lim_{\beta \to \infty} \frac{1}{\beta}\left( \left. \frac{\partial}{\partial h} \right|_{h = 0} \Psi_{y_0} \right)\left(t_\mathrm{sym}^*(h=0), \hat{t}_\mathrm{sym}^*(h=0)\right)\,,
\end{align}
which requires, by the same reasoning as in previous Sections~\ref{sec:sobolev_replica_ansatz} and~\ref{sec:trainig-err}, only the partial derivative of the rate function $\Phi$ in $h$. Starting from the expression for the potential $\Psi_{y_0}$ given in~\eqref{eq:replica-potentials-summary}, we differentiate in $h$ at $h=0$ to obtain
\begin{align}
&\plim_{p \to \infty} \eps_{\text{gen}}(w^*) \mid \varpi = \lim_{\beta \to \infty} \EE_{\xi \sim {\cal N}(0, I_{k+1})} \left[ \int_{\RR^{k+1}} \dd \Upsilon \; \mathcal{Z}[P_\text{data}](\Upsilon; \bar{m}_1, \sigma_1^2) \; \EE_{\tilde{\Upsilon} \sim \mathcal{N}(\bar{m}_2, \Sigma_2/\beta)} \left[\norm{\Upsilon - \tilde{\Upsilon}}^2 \right] \right] \nonumber\\
&= \lim_{\beta \to \infty} \EE_{\xi \sim {\cal N}(0, I_{k+1})} \left[ \EE_{\omega \sim {\cal N}(\langle f, q^{-1/2} f \rangle, 1 - \langle f, q^{-1} f \rangle)} \left[ \EE_{\Upsilon \sim P_{\text{data}}(\cdot \mid \omega, \varpi)} \left[ \EE_{\tilde{\Upsilon} \sim {\cal N}(s + q^{1/2} \xi, \Sigma_2 / \beta)} \left[ \norm{\Upsilon - \tilde{\Upsilon}}^2 \right] \right] \right] \right]
\label{eq:gen-error-from-replica-interm}
\end{align}
where we have already made the $\beta$-scaling of all quantities explicit. By calculating the mean and covariance of the jointly normal random variables $(\xi, \omega, \tilde{\Upsilon})$ in this expression, we find that as $\beta \to \infty$, we can write~\eqref{eq:gen-error-from-replica-interm} as
\begin{align}
\plim_{p \to \infty} \eps_{\text{gen}}(w^*) \mid \varpi = \EE_{(\omega,\tilde{\Upsilon})} \left[ \EE_{\Upsilon \sim P_{\text{data}}(\cdot \mid \omega, \varpi)} \left[ \norm{\Upsilon - \tilde{\Upsilon}}^2 \right] \right]  \quad \text{ with } \quad (\omega,\tilde{\Upsilon})^\top \sim {\cal N} \left(\begin{pmatrix}
0\\s
\end{pmatrix}, \begin{pmatrix}
1 & f^\top\\
f & q
\end{pmatrix} \right)\,.
\label{eq:gen-error-from-replica}
\end{align}
Of course, this corresponds to the definition of the $H^1_k$ generalization error~\eqref{eq:gen-error-def-appendix} where we have effectively simply replaced the network output using the Gaussian equivalence theorem, as discussed around~\eqref{eq:get-distribution} in the main text already. While we could have arrived at this conclusion immediately on an intuitive level, as we did in the main text, the systematic derivation of the generalization error via an external field $h$ in the partition function makes it clear why exactly the overlap parameters as determined from the replica-symmetric saddle-point equations are indeed related to the generalization error. Finally, for the specific case of an additive Gaussian observation model~\eqref{eq:p-data-additive-gauss}, it is then straightforward to see that~\eqref{eq:gen-error-from-replica} implies~\eqref{eq:l2_gen_error} and~\eqref{eq:sobo_gen_error} in the main text.

\subsection{Specifying the setup to standard subspace Sobolev loss and additive Gaussian noise observations}
\label{sec:add-gauss-standard-sobo}

Here, we want to simplify the saddle-point equations~\eqref{eq:psi-y-saddle-inf-beta-gen} and the training error term~\eqref{eq:dbetapsiw-general} as much as possible for the standard loss function~\eqref{eq:loss-sobo-standard} given by $\ell(\Upsilon,\tilde{\Upsilon})= \frac{1}{2} \| \Upsilon -\tilde{\Upsilon}\|^2$. The minimizer in~\eqref{eq:sobo-minimizer} becomes
\begin{align}
	\tilde{\Upsilon}^*_\ell(\Upsilon;\bar{m}_2,\Sigma_2) \ =  \ (\Sigma_2^{-1} + I_{k+1})^{-1} \left(  \Sigma_2^{-1} \bar{m}_2 + \Upsilon \right).
\end{align}
If we further assume the Gaussian observation model~\eqref{eq:p-data-additive-gauss}
we compute the necessary quantities in~\eqref{eq:psi-y-saddle-inf-beta-gen} as
\begin{align}
\begin{cases}
	\EE_{\Upsilon \mid \omega} \left[ \frac{1}{2} \nabla_{\Upsilon}\| \Upsilon-\tilde{\Upsilon}^*_\ell\|^2 \right]  = \left(  \Sigma_2 + I_{k+1} \right)^{-2}  \left( \begin{pmatrix} \phi(\omega) \\ \varpi \phi'(\omega) \end{pmatrix} -  \bar{m}_2 \right) \\[6pt]
	\EE_{\Upsilon \mid \omega}\left[  \Sigma_2^{-1} (\tilde{\Upsilon}_\ell^* - \bar{m}_2) \right]  = (\Sigma_2 + I_{k+1})^{-1}\left(  \begin{pmatrix} \phi(\omega)\\  \varpi \phi'(\omega) \end{pmatrix} - \bar{m}_2  \right) \\[6pt]
	 \Sigma_2^{-1} \EE_{\Upsilon\mid \omega}\left[(\tilde{\Upsilon}_\ell^* - \bar{m}_2)^{\otimes 2} \right]  \Sigma_2^{-1}  =   (\Sigma_2 + I_{k+1})^{-1}\left( C_\eta + \left( \begin{pmatrix} \phi(\omega)\\  \varpi \phi'(\omega) \end{pmatrix} - \bar{m}_2 \right)^{\otimes 2} \right) (\Sigma_2 + I_{k+1})^{-1} 
	 \end{cases}
\end{align}
Putting everything together, we can then simplify~\eqref{eq:psi-y-saddle-inf-beta-gen} to
\begin{align}
\label{eq:update_hat_sobo_h1-appendix}
\begin{cases}
s_a &= \mathbb{E}[\phi(\omega)] \1_{\kappa_0 \neq 0}\\
s_b &= \varpi \mathbb{E}[  \phi'(\omega)] \1_{\kappa_0' \neq 0}\\
\hat{f} &= \alpha (\Sigma_2 + I_{k+1})^{-1} \begin{pmatrix} \EE[\phi'] \\ \varpi \EE[\phi''] \end{pmatrix} \\[10pt]
\hat{q} &= \alpha (\Sigma_2 + I_{k+1})^{-1} \Bigg( C_\eta + \EE\left[ \lp \begin{pmatrix} \phi(\omega)  \\ \varpi \phi'(\omega) \end{pmatrix} - \bar{m}_2 \rp^{\otimes 2}  \right] \Bigg)  (\Sigma_2 + I_{k+1})^{-1} \\
\hat{\Sigma}_2 &= \alpha \left( \Sigma_2^{-1} - \Sigma_2^{-1} (\Sigma_2^{-1} + I_{k+1})^{-1} \Sigma_2^{-1} \right) = \alpha (I_{k+1} + \Sigma_2)^{-1}
\end{cases}
\end{align}
Note that the expectations in~\eqref{eq:update_hat_sobo_h1-appendix} can all be reduced to one-dimensional Gaussian integrals  with respect to $\omega \sim {\cal N}(0,1)$. Indeed, we have
\begin{align}
&\EE_{\xi \sim \mathcal{N}(0,I_{k+1})}\EE_{\omega \sim {\cal N}(\langle f, q^{-1/2} \xi \rangle, 1 - \langle f, q^{-1} f \rangle)}\left[ \lp \begin{pmatrix} \phi(\omega)  \\ \varpi \phi'(\omega) \end{pmatrix} - s - q^{1/2} \xi \rp^{\otimes 2} \right] \nonumber
\\
&= q + \EE_{\omega \sim \mathcal{N}(0,1)} \left[ \left( \begin{pmatrix} \phi(\omega)\\ \varpi \phi'(\omega) \end{pmatrix} -s \right)^{\otimes 2} \right] - f \otimes
\EE_{\omega \sim \mathcal{N}(0,1)} \left[\begin{pmatrix}
\phi'(\omega)\\\varpi \phi''(\omega)
\end{pmatrix} \right] - \EE_{\omega \sim \mathcal{N}(0,1)} \left[\begin{pmatrix}
\phi'(\omega)\\\varpi \phi''(\omega)
\end{pmatrix} \right] \otimes f\,,
\end{align} 
which finally leads us to~\eqref{eq:update_hat_sobo_h1} in the main text (note that we kept a general weight $\tau>0$ instead of $\tau = 1$ for the derivative term of the loss function in the main text, but the corresponding saddlepoint equations for general $\tau$ can be derived from straightforward modifications of the calculations presented in this section). As for the training error~\eqref{eq:training-error-through-derivs}, for the loss function $\ell(\Upsilon,\tilde{\Upsilon})= \frac{1}{2} \| \Upsilon -\tilde{\Upsilon}\|^2$, the expression~\eqref{eq:dbetapsiw-general} becomes
\begin{equation}
\lim_{\beta \to \infty} \partial_\beta \Psi_y = -\frac{\alpha}{2} \trace \Bigg( (\Sigma_2+ I_{k+1})^{-2} \left(C_\eta + \mathbb{E}_{\xi \sim {\cal N}(0, I_{k+1})} \mathbb{E}_{\omega \mid \xi \sim \mathcal{N}(\bar{m}_1, \sigma_1^2)} \bigg[\left( \begin{pmatrix} \phi(\omega) \\ \varpi \phi'(\omega) \end{pmatrix} - \bar{m}_2\right)^{\otimes 2}\bigg] \right) \Bigg) = -\frac{1}{2}\left(\hat{q}_a + \trace[\hat{q}_c] \right)
\end{equation}
using~\eqref{eq:update_hat_sobo_h1-appendix}, which hence leads us to~\eqref{eq:training-err-l2} and~\eqref{eq:training-err-h1k} in the main text for the training error at optimality. To recognize that $\hat{q}_a$ indeed corresponds to the $L^2$ part of the training error and $ \trace[\hat{q}_c]$ to the $H^1_k$ semi-norm part, as claimed in the main text, we could have perturbed $L_\beta$ throughout all derivations of this section as
\begin{equation}
	\label{eq:externalized_Lbeta}
	L_{\beta}(h_1,h_2) := \exp\Bigg\{ -\beta(\Upsilon-\tilde{\Upsilon})^\top \begin{pmatrix}
	(1+h_1) & 0^\top \\
	0 & (1+h_2)I_k
\end{pmatrix} (\Upsilon-\tilde{\Upsilon}) \Bigg\} 
\end{equation}
and differentiate with respect to either $h_1$ or $h_2$ at $0$ to isolate the respective part of the training error. The result is the identification~\eqref{eq:training-err-l2} and~\eqref{eq:training-err-h1k} as expected.


\section{Simplifications in the $L^2$ training setting}
\label{app:l2-simplify}

This appendix contains a number of technical details for the simplifications of the replica-symmetric saddlepoint equations to the case of $L^2$ training without gradients discussed in Remark~\ref{rem:l2-case}. This setting reduces to~\cite{gerace-loureiro-krzakala-etal:2021,goldt-loureiro-reeves-etal:2022}, except that we also compute the $H^1_k$ generalization error produced by training with $L^2$ loss.

\subsection{The distribution of $(\varpi, s_b)$ for $L^2$ training}
\label{app:varpi-sb-l2}

In this subsection, we motivate~\eqref{eq:varpi-sb-distr-l2-train} for the joint distribution of the alignment parameter $\varpi = V_k^\top \theta_0$ and the projected network gradient mean $s_b = \kappa_0' V_k^\top \Theta w^*$ for $L^2$ training. Notably, equation~\eqref{eq:varpi-sb-distr-l2-train} departs from~\eqref{eq:def-sab} where $\varpi$ and $s_b = \varpi \EE \left[\phi'(\omega) \right] \1_{\kappa_0' \neq 0}$ are perfectly correlated for any $\tau > 0$. The key difference between these two situations is that $w^*$ is independent of $V_k$ for $\tau = 0$ only, and there is otherwise some additional randomness in $s_b$ that is independent of $\varpi$. Let us assume based on the numerical evidence in Figure~\ref{fig:varpi-sb} that $(\varpi, s_b)$ for $\tau = 0$ are jointly normally distributed in the proportional asymptotics limit with a non-degenerate covariance matrix. Of course, their mean will be $\EE \left[(\varpi, s_b) \right] = 0$ by independence of $V_k$ from all other random quantities for $L^2$ training. It remains to evaluate their second moments:
\begin{align}
\begin{cases}
&\EE \left[\varpi_i \varpi_j \right] = \EE \left[ \left \langle v_i, \theta_0 \right \rangle \left \langle v_j, \theta_0 \right \rangle   \right] = \EE \left[\norm{\theta_0}^2 \right] \delta_{ij} = \delta_{ij}\,,\\
&\EE \left[\varpi_i s_{b,j} \right] = \kappa_0' \EE \left[\left \langle v_i, \theta_0 \right \rangle \left \langle v_j, \Theta w^* \right \rangle \right] = \kappa_0' \EE \left[ \left \langle \theta_0, \Theta w^* \right \rangle \right] \delta_{ij} \overset{\eqref{eq:overlap-def}}{=} f_a \delta_{ij}\,,\\
&\EE \left[s_{b,i} s_{b,j} \right] = \left(\kappa_0' \right)^2\EE \left[ \left \langle v_i, \Theta w^* \right \rangle \left \langle v_j, \Theta w^*\right \rangle   \right] = \left(\kappa_0' \right)^2 \EE \left[\norm{\Theta w^*}^2 \right] \delta_{ij} \overset{\eqref{eq:overlap-def}}{=} \left(q_a - \kappa_*^2 \norm{w^*}^2 \right) \delta_{ij}\,.
\end{cases}
\end{align}
These results lead us to~\eqref{eq:varpi-sb-distr-l2-train} in the main text. Conditioned on the alignment parameter, the distribution of the overlap parameter $s_b \mid \varpi$ becomes Gaussian with mean $f_a \varpi$ and variance $(q_a - \kappa_*^2 \|w^*\|^2 - f_a^2) I_k$.

\begin{figure}
\centering
\begin{subfigure}{0.49\textwidth}
\includegraphics[width = \textwidth]{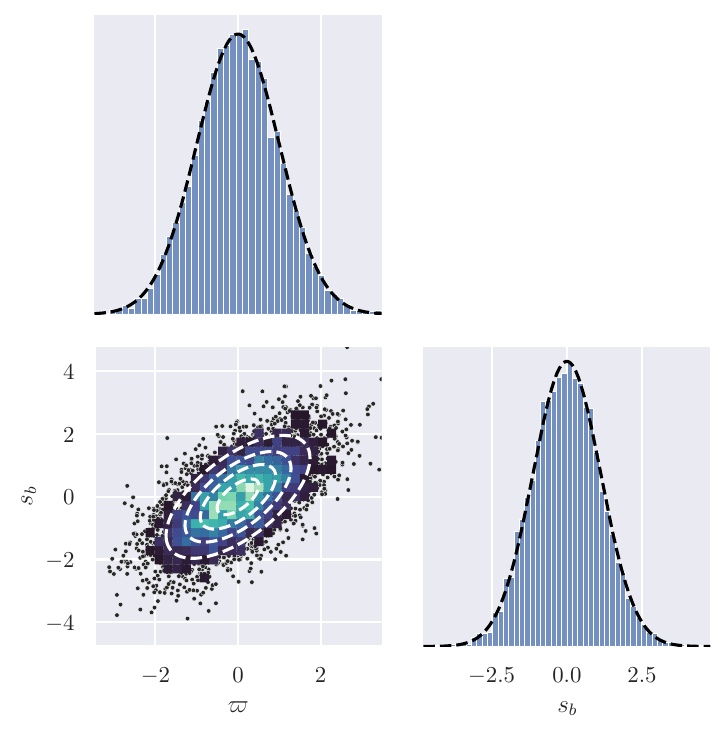}
\caption{$\tau = 0$ ($L^2$ training)}
\label{fig:varpi-sb-tau-0}
\end{subfigure}
\hfil
\begin{subfigure}{0.49\textwidth}
\includegraphics[width = \textwidth]{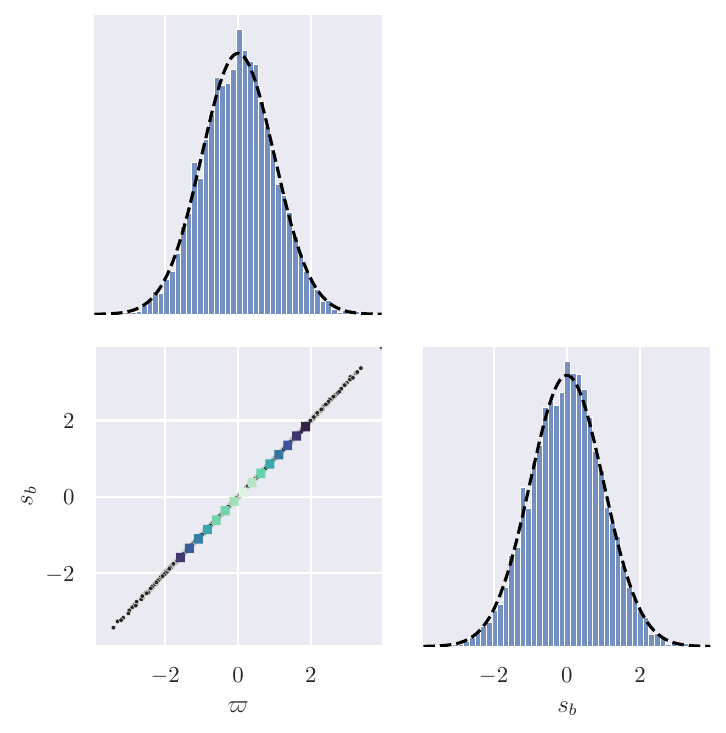}
\caption{$\tau = 1$ (standard Sobolev training)}
\label{fig:varpi-sb-tau-1}
\end{subfigure}
\caption{Joint distribution of $(\varpi, s_b)$ for $k=1$, comparing $L^2$ training ($\tau = 0$, left) and Sobolev training ($\tau = 1$, right) from samples. The histograms are generated from $5000$ samples of $\varpi = \left \langle v, \theta_0 \right \rangle$ and $s_b = \kappa_0' \left \langle v, \Theta w^* \right \rangle$ at $\alpha = 1.25$, $\gamma = 0.75$, $\lambda = 1.25 \cdot 10^{-4}$ in dimension $d = 1600$, with $\sigma = \text{SiLU}$, $\phi(\omega) = \omega + 1/\cosh \omega$, iid Gaussian random features, and noiseless data $C_\eta = 0$. The dashed lines are the theoretical predictions for the PDFs from solving the replica-symmetric saddlepoint equations: marginally $\varpi \sim {\cal N}(0,1)$ in both cases, but $s_b = \varpi \EE \left[\phi'(\omega) \right] = \varpi$ for $\tau > 0$ is expected to be perfectly correlated while for $\tau = 0$ the parameters $(\varpi, s_b)$ are jointly centered and nondegenerate normally distributed with $\EE \left[\varpi s_b\right] = f_a = 0.7102$ and $\EE \left[s_b^2 \right] = q_a - \kappa_*^2 \norm{w^*}^2 = 1.2683$.}
\label{fig:varpi-sb}
\end{figure}

\subsection{Expressing the random matrix traces for $L^2$ training as Stieltjes transforms}
\label{app:l2-stieltjes}

Here, we want to simplify the saddlepoint equations for $\Sigma_a, f_a, q_a$ in~\eqref{eq:update_nonhat_sobo_h1} for $L^2$ training where~\eqref{eq:zero-overlaps-l2} holds. We obtain from~\eqref{eq:update_nonhat_sobo_h1} that
\begin{align}
\begin{cases}
\Sigma_a &= \displaystyle \plim_{p \to \infty} \frac{1}{p} \left(\kappa_1^2 \trace \left[A^{-1} \Theta^\top \Theta \right] + \kappa_*^2 \trace\left[A^{-1} \right] \right)\\[6pt]
f_a &= \displaystyle \plim_{p \to \infty} \frac{1}{d} \kappa_1^2 \trace \left[A^{-1} \Theta^\top \Theta \right] \hat{f}_a\\[6pt]
q_a &=  \displaystyle \plim_{p \to \infty} \frac{1}{p}  \left(\kappa_1^2 \trace \left[A^{-1} \Xi A^{-1} \Theta^\top \Theta \right] + \kappa_*^2 \trace \left[ A^{-1} \Xi A^{-1}\right] \right)\\[6pt]
 &=  \displaystyle \plim_{p \to \infty} \frac{1}{p}  \left(\kappa_1^4 \left(\hat{q}_a + \hat{f}_a^2 / \gamma \right) \trace \left[\left(\Theta A^{-1} \Theta^\top \right)^2 \right] + \kappa_*^4 \hat{q}_a \trace \left[A^{-2}\right] + \kappa_1^2 \kappa_*^2 \left(2 \hat{q}_a + \hat{f}_a^2 / \gamma \right) \trace\left[\Theta A^{-2} \Theta^\top \right] \right)
\end{cases}
\label{eq:saddle-l2-partially-simplififed}
\end{align}
where $A$ and $\Xi$ are given by~\eqref{eq:l2-simplify-random-matrices}. Writing $A = c_2 \left(\frac{c_1}{c_2} I_p + \Theta^\top \Theta \right)$ with $c_2 =  \kappa_1^2 \hat{\Sigma}_a $ and defining
\begin{equation}
z := \frac{c_1}{c_2} = \frac{ \lambda + \kappa_*^2 \hat{\Sigma}_a}{\kappa_1^2 \hat{\Sigma}_a },
\end{equation}
for all $k \in \RR$ we have the following useful identities: 
\begin{align}
\trace \left[ \left(c_1 I_p + c_2 \Theta^\top \Theta \right)^k \right] &=  c_2^k \left(z^k (p-d) + \trace\left[ \left(z I_d + \Theta \Theta^\top\right)^k \right] \right) \;, \label{eq:svd1}\\
\Theta \left(c_1 I_p + c_2 \Theta^\top \Theta \right)^k \Theta^\top &= c_2^k \, \Theta \Theta^\top \left(z I_d + \Theta \Theta^\top \right)^k \;, \label{eq:svd2} \\
\trace \left[ \Theta \left(c_1 I_p + c_2 \Theta^\top \Theta \right)^k \Theta^\top \right] &\overset{\eqref{eq:svd2}}{=} c_2^k \trace\left[ \left(z I_d + \Theta\Theta^\top - z I_d \right) \left(z I_d + \Theta \Theta^\top \right)^k \right] \nonumber \\
&= c_2^{k}\trace \left[ \left(z I_d + \Theta \Theta^\top \right)^{k+1} \right]  - c_2^{k}z \trace \left[ \left(z I_d + \Theta \Theta^\top \right)^k \right]\,, \label{eq:svd3} \\
\left(\Theta \left(c_1 I_p + c_2 \Theta^\top \Theta \right)^k \Theta^\top \right)^2 &= c_2^{2k}\, \left(\Theta \Theta^\top \right)^2 \left(z I_d + \Theta \Theta^\top \right)^{2k} \;, \label{eq:svd22} \\
\trace \left[ \left( \Theta \left( c_1 I_p + c_2 \Theta^\top \Theta \right)^k \Theta^\top \right)^{2} \right] &\overset{\eqref{eq:svd22}}{=} c_2^{2k} \trace \left[ \left(z I_d + \Theta\Theta^\top - zI_d \right)^2
\left(z I_d + \Theta \Theta^\top \right)^{2k} \right] \nonumber \\
&= c_2^{2k} \trace \left[ \left(z I_d + \Theta \Theta^\top \right)^{2k+2} - 2z \left(z I_d + \Theta \Theta^\top \right)^{2k+1 } + z^2 \left(z I_d + \Theta \Theta^\top \right)^{2k}  \right]\,,
\end{align}
as well as 
\begin{equation}
\frac{\dd}{\dd z} \trace  \left[ \left(zI_d + \Theta\Theta^\top \right)^k \right] = k\trace \left[ \left(zI_d + \Theta\Theta^\top \right)^{k-1}\right] \label{eq:d_stieljes}.   
\end{equation}
All of these identities can be verified by inserting the singular value decomposition $\Theta = U \Sigma V^\top \in \RR^{d \times p}$. Introducing the Stieltjes transform
\begin{equation}
g_\mu(-z) := \plim_{p\to\infty} \frac{1}{d} \trace \left[ \left(z I_d + \Theta\Theta^\top \right)^{-1} \right],
\end{equation}
of $\Theta \Theta^\top \in \RR^{d \times d}$,
we can then re-write the saddle-point updates~\eqref{eq:saddle-l2-partially-simplififed} as given in~\eqref{eq:l2-nonhat-update} in the main text.

\subsection{Simplifying $q_c$ for $L^2$ training: factorization of the Hadamard product trace}
\label{app:hada-trace-l2}

Starting from the saddlepoint equation~\eqref{eq:update_nonhat_sobo_h1} for the ``non-hatted'' overlap parameters, in the $L^2$ training setting with the simplifications~\eqref{eq:zero-overlaps-l2} the equation for $q_c \in \RR^{k \times k}$ becomes
\begin{align}
q_c = \plim_{p \to \infty} \frac{1}{p} V_k^\top \Theta \left[ \left(A^{-1} \Xi A^{-1} \right) \odot \left( \left(\kappa_1'\right)^2 \Theta^\top \Theta + \left(\kappa_*'\right)^2 I_p \right) \right] \Theta^\top V_k\,,
\label{eq:qc-l2-full}
\end{align}
where $A$ and $\Xi$ are given by~\eqref{eq:l2-simplify-random-matrices}.
Replacing in distribution $\Theta^\top V_k = \zeta \in \RR^{p \times k}$ in the right-hand side of~\eqref{eq:qc-l2-full} with iid standard normal components, asymptotically independent of $\Theta^\top \Theta$, and assuming that the right-hand side concentrates onto its expectation over $\zeta$, yields
\begin{align}
q_c = \plim_{p \to \infty} \frac{1}{p} \zeta^\top \left[ \left(A^{-1} \Xi A^{-1} \right) \odot \left( \left(\kappa_1'\right)^2 \Theta^\top \Theta + \left(\kappa_*'\right)^2 I_p \right) \right] \zeta =  \plim_{p \to \infty} \frac{1}{p} \trace \left[ \left(A^{-1} \Xi A^{-1} \right) \odot \left( \left(\kappa_1'\right)^2 \Theta^\top \Theta + \left(\kappa_*'\right)^2 I_p \right) \right] I_k\,.
\label{eq:qc-l2-zeta}
\end{align}
The main difficulty in handling the Hadamard product $\odot$ is that it is not a ``spectral'' function but instead depends on the choice of basis with respect to which it is defined. We would hence like to eliminate it from our expressions as much as possible. In~\eqref{eq:qc-l2-zeta}, we can accomplish our goal by observing the following: abstractly, we are dealing with the evaluation of
\begin{align}
\frac{1}{p} \trace \left[f(C) \odot g(C) \right]\,,
\label{eq:trace-hadamard-prod-abstract}
\end{align}
where $C = \Theta^\top \Theta \in \RR^{p \times p}$ is a standard Wishart matrix with parameter $\gamma = d/p$---notably, this is the only random matrix in the expression---and $f$ and $g$ are spectral functions. Both $f(C)$ and $g(C)$ are diagonalized by the same set of orthonormal eigenvectors of $C$, which we summarize in an orthogonal ``eigenmatrix'' $U = [u_1, \dots, u_p] \in \RR^{p \times p}$. Consequently, we can write
\begin{align}
f(C) = \sum_{j = 1}^p f(\mu_j) u_j^{\otimes 2}\,, \quad g(C) = \sum_{k = 1}^p g(\mu_k) u_k^{\otimes 2} \,,
\end{align}
with the eigenvalues $\mu_j \geq 0$ of $C$. Computing the trace~\eqref{eq:trace-hadamard-prod-abstract} in the standard basis where the Hadamard product is defined using this eigen-decomposition then leads to
\begin{align}
\frac{1}{p} \trace \left[f(C) \odot g(C) \right] = \frac{1}{p} \sum_{i=1}^p \left( f(C) \right)_{ii} \left( g(C) \right)_{ii} = \frac{1}{p} \sum_{i,j,k=1}^p f(\mu_j) g(\mu_k) U_{ij}^2 U_{ik}^2\,.
\end{align}
For the standard Wishart matrix $C$, it is well-known~\cite{anderson:2003,bai-silverstein:2010} that the eigenmatrix $U$ is Haar-distributed on the orthogonal group $O(p)$ (a property which holds asymptotically for more general classes of random matrices but is true even pre-asymptotically for the normal case $\Theta_{ij} \sim {\cal N}(0, 1/d)$). As $p \to \infty$, we then replace almost surely
\begin{align}
\frac{1}{p} \sum_{i=1}^p U_{ij}^2 U_{ik}^2 \sim \frac{1}{p} \sum_{i=1}^p \EE_{U \sim {\cal U}(O(p))} \left[U_{ij}^2 U_{ik}^2 \right]\,,
\label{eq:haar-expec}
\end{align}
where $\sim$ denotes asymptotic equivalence. Expectations of matrix entries with respect to the Haar measure of the orthogonal group can be computed as~\citep{collins-sniady:2006}:
\begin{align}
\EE_{U \sim {\cal U}(O(p))} \left[U_{i_1j_1} \dots U_{i_{2n}j_{2n}} \right] = \sum_{p_1,p_2 \in P_{2n}} \delta_{i_1, i_{p_1(1)}} \dots  \delta_{i_{2n}, i_{p_1(2n)}} \delta_{j_1, j_{p_2(1)}} \dots  \delta_{j_{2n}, j_{p_2(2n)}} \left \langle p_1, \text{Wg} \; p_2 \right \rangle
\end{align}
where $P_{2n}$ is the set of all pairings of $[2n]$, and $\text{Wg}$ the orthogonal Weingarten function. We obtain two different cases in~\eqref{eq:haar-expec} for the number of pairings with nonzero contributions, depending on whether $j = k$ or $j \neq k$. Using the table provided by~\citet{collins-sniady:2006} for values of the orthogonal Weingarten function, we find
\begin{align}
\frac{1}{p} \sum_{i=1}^p \EE_{U \sim {\cal U}(O(p))} \left[U_{ij}^2 U_{ik}^2 \right] \sim \frac{1 + 2 \delta_{jk}}{p^2} \text{ as } p \to \infty\,,
\end{align}
such that
\begin{align}
\frac{1}{p} \trace \left[f(C) \odot g(C) \right] = \frac{1}{p} \sum_{i,j,k=1}^p f(\mu_j) g(\mu_k) U_{ij}^2 U_{ik}^2 \overset{p \to \infty}{\sim} \sum_{j,k=1}^p f(\mu_j) g(\mu_k)  \frac{1 + 2 \delta_{jk}}{p^2} \nonumber\\
= \left( \frac{1}{p} \sum_{j = 1}^p f(\mu_j) \right) \left( \frac{1}{p} \sum_{k = 1}^p g(\mu_k) \right) + \frac{2}{p} \left(\frac{1}{p} \sum_{j=1}^p f(\mu_j)g(\mu_j) \right)  \overset{p \to \infty}{\sim} \frac{1}{p} \trace \left[f(C) \right] \frac{1}{p} \trace \left[g(C) \right]\,,
\end{align}
with the diagonal term providing only a subleading correction. Applying this identity to~\eqref{eq:qc-l2-zeta} then leads to~\eqref{eq:trace-a-xi-a-l2} in the main text.

\subsection{Distribution of $H^1_k$ generalization error for $L^2$ training}
\label{app:eps-h1k-gen-l2}

For $L^2$ training, the corresponding $L^2$ generalization error~\eqref{eq:l2_gen_error} does not depend on the alignment~$\varpi$, as expected. However,  the $H^1_k$ generalization error~\eqref{eq:sobo_gen_error} is a random variable of both $(s_b, \varpi)$, whose joint law is given by~\eqref{eq:varpi-sb-distr-l2-train} in the main text. Writing the conditional random variable $s_b \mid \varpi$ as $f_a \varpi + \sqrt{q_a - \kappa_*^2 \|w^*\|^2 - f_a^2} \xi$, 
for $\xi \sim \mathcal{N}(0,I_k)$ independent of $\varpi \sim \mathcal{N}(0, I_k)$, 
the projected gradient error becomes
\begin{align}
\varepsilon_\text{gen}^{H^1_k} \mid \varpi, \xi &= \begin{pmatrix} \varpi^\top & \xi^\top \end{pmatrix} \begin{pmatrix} I_k \otimes \EE[(\phi'(\omega)-f_a)^2] & I_k \otimes -\EE[\phi(\omega)-f_a)] \sqrt{q_a - \kappa_*^2 \|w^*\|^2 - f_a^2} \\[5pt] I_k \otimes -\EE[\phi(\omega)-f_a)] \sqrt{q_a - \kappa_*^2 \|w^*\|^2 - f_a^2} &  I_k \otimes (q_a - \kappa_*^2 \|w^*\|^2 - f_a^2) \end{pmatrix}  
\begin{pmatrix} \varpi \\ \xi \end{pmatrix} \\
&\quad + \trace[C_{\eta,2:k+1,2:k+1}] + k \cdot q_c\;.
\end{align}
This recovers~\eqref{eq:h1k-gen-l2-train} in expectation, and demonstrates that marginally $\varepsilon_\text{gen}^{H^1_k}$ follows a generalized $\chi^2$-distribution with $2k$ degrees of freedom.

\section{Simplifications of the fixed-point equations: $\varpi$-dependence and random matrix traces}
\label{app:simplify-fp}

The right-hand sides of the saddlepoint equation~\eqref{eq:update_nonhat_sobo_h1} can be further simplified in the high-dimensional limit. Specifically, in this appendix, we first show that $\Sigma$ is a diagonal matrix and argue that only (specific combinations of) the diagonal elements of $q$ contribute to the training and generalization error. Furthermore, we demonstrate that the dependence of each of these relevant overlap parameters on the alignment $\varpi$ can be captured with only two degrees of freedom, which then leads to the simplified fixed-point equations~\eqref{eq:asymptotic_saddle_final_nonhat} and~\eqref{eq:asymptotic_saddle_final_hat} in the main text in terms of overlaps~\eqref{eq:varpi-scaling}.\\

We first observe that, conditioning on $V_k$, each component of the random variable $\zeta := \Theta^\top V_k \in \RR^{p \times k}$ is 
asymptotically equivalent to a standard Gaussian in law and
\emph{asymptotically uncorrelated} with each component of $\Theta^\top \Theta$. We further note that
\begin{align}
\label{eq:hadamard_i}
\plim_{p \to \infty} \frac{1}{p} \zeta_i^\top \left(A^{-1} \odot I_p \right) \1_p &= \frac{1}{p}\trace \left[ A^{-1} D_i \right],\\
\label{eq:hadamard_features}
\plim_{p \to \infty} \frac{1}{p} \zeta_i^\top \left( A^{-1} \odot \left(\Theta^\top \Theta \right) \right) \1_p &= \frac{1}{p}\trace \left[ A^{-1} \Theta^\top \Theta D_i \right],
\end{align}
where $\zeta_i$ is the $i$-th column of $\zeta$, and $D_i := \diag{\zeta_i}$. Since $A^{-1}$ in~\eqref{eq:def-random-matrices-for-saddlepoint-eqs} is positive definite, its diagonal elements are positive, and thus the elements on the diagonal of $A^{-1}D_i$ are equally likely to be positive or negative. Consequently, $\trace \left[ A^{-1} D_i \right] \sim O(\sqrt{p})$ follows the typical scaling of a sum of iid Bernoulli random variables, and 
$
\plim_{p \to \infty} \frac{1}{p} \zeta^\top (A^{-1} \odot I_p) \1_p = 0
$.
For~\eqref{eq:hadamard_features}, we apply the spectral theorem to $A^{-1}=U \diag{\Lambda} U^\top$ and $ (\Theta^\top \Theta D_i + D_i\Theta^\top \Theta ) = Q \diag{E} Q^\top $. Then, 
\begin{align}
	\trace \left[ A^{-1} \Theta^\top \Theta D_i \right]  &= \frac{1}{2}\trace\left[ U \diag{\Lambda} U^\top  Q \diag{E} Q^\top \right] = \frac{1}{2} \Lambda^\top \lp ( U^\top Q ) \odot (U^\top Q) \rp E
\end{align}
Note that the elements of $\Lambda^\top \lp ( U^\top Q ) \odot (U^\top Q) \rp$ are positive since $\Lambda \succ 0$, and $E$ is the vector of eigenvalues of $ (\Theta^\top \Theta D_i + D_i\Theta^\top \Theta ) $, which are symmetrically distributed. Hence $\trace \left[ A^{-1} \Theta^\top \Theta D_i \right] \sim O(\sqrt{p})$ and
\begin{align}
\label{eq:vanishing_terms}
\plim_{p \to \infty} \frac{1}{p} \zeta^\top (A^{-1} \odot (\Theta^\top \Theta)) \1_p  &= 0\,.
\end{align} 
As a result, with
\begin{align}
\begin{cases}
M_{00}  &= \kappa_1^2 I_p + \kappa_*^2 \Theta^\top \Theta\\[5pt]
M_{11} &= \left(\kappa_1'\right)^2 I_p + \left( \kappa_*' \right)^2 \Theta^\top \Theta \\[5pt]
A &=  \lambda I_p + \hat{\Sigma}_a M_{00} + \sum_{j \in[k]} \hat{\Sigma}_{c,jj} D_j M_{11} D_j \\[5pt]
\Xi &= \frac{1}{\gamma} \bigg( \kappa_1^2 \hat{f}_a^2 \Theta^\top \Theta + \kappa_1\kappa_1'\sum_{j\in[k]} \hat{f}_a \hat{f}_{b,j} (D_j \Theta^\top \Theta + \Theta^\top \Theta D_j) + \sum_{i,j \in [k]} (\kappa_1')^2 \hat{f}_{b,i} \hat{f}_{b,j} D_i \Theta^\top \Theta D_j \bigg) \notag \\
&\quad\quad + \hat{q}_a M_{00} + \kappa_* \kappa_*' \sum_{j \in [k]} \hat{q}_{b,j} (D_j \Theta^\top \Theta + \Theta^\top \Theta D_j) + \sum_{i,j\in[k]} \hat{q}_{c,ij} D_i M_{11} D_j \;,
\end{cases}
\end{align}
and replacing $\theta_0\theta_0^\top$ with its expectation,
equation~\eqref{eq:update_nonhat_sobo_h1} becomes
\begin{align}
\label{eq:asymptotic_simplified_saddle}
\begin{cases}
	\Sigma &= \displaystyle \plim_{p\to\infty} \frac{1}{p} \diag{ \trace \left[ A^{-1} M_{00}\right], \trace\left[A^{-1} D_1 M_{11} D_1\right], \dots,  \trace \left[A^{-1} D_k M_{11} D_k \right]}\\[10pt]
	f &= \displaystyle \plim_{p\to\infty} \frac{1}{p} \diag{\frac{1}{\gamma} \kappa_1^2 \trace \left[ A^{-1}  \Theta^\top \Theta \right], \frac{1}{\gamma} \left(\kappa_1' \right)^2 \trace \left[ A^{-1} D_1 \Theta^\top \Theta D_1 \right], \dots, \frac{1}{\gamma} \left(\kappa_1'\right)^2 \trace \left[ A^{-1} D_k \Theta^\top \Theta D_k \right]} \hat{f}\\[10pt]
	q &= \displaystyle \plim_{p\to\infty} \frac{1}{p} \Bigg( \begin{pmatrix} \kappa_1 \1_p^\top \\ \kappa_1' \zeta^\top  \end{pmatrix} \left( \left( A^{-1}\Xi A^{-1} \right) \odot \left(\Theta^\top\Theta \right) \right) \begin{pmatrix} \kappa_1 \1_p & \kappa_1' \zeta  \end{pmatrix} + \begin{pmatrix} \kappa_*^2 \trace \left[A^{-1}\Xi A^{-1} \right] & 0 \\ 0 & (\kappa_*')^2 \zeta^\top  \left( \left(A^{-1}\Xi A^{-1} \right) \odot I_p \right)  \zeta \end{pmatrix}  \Bigg)
\end{cases}.
\end{align}
Note that $\Sigma_{ij} = \hat{\Sigma}_{ij} = 0$ when $i \neq j$, which is consistent with~\eqref{eq:update_hat_sobo_h1}. By following the same arguments in \eqref{eq:hadamard_i} through \eqref{eq:vanishing_terms}, the dependence of $q$ on the hatted overlap parameters can be shown to be
\begin{align}
\begin{cases}
	q_a &\propto  \hat{f}_a^2, \; \hat{f}_b^2, \; \hat{q}_a, \; \hat{q}_{c,ii} \\
	q_b &\propto  \hat{f}_a\hat{f}_b,\; \hat{q}_{b} \\
	q_{c,ii} &\propto \hat{f}_a^2, \; \hat{f}_{b,i}^2, \; \hat{q}_a, \; \hat{q}_{c,ii}  \\
	q_{c,ij} &\propto \hat{f}_{b,i}\hat{f}_{b,j}, \; \hat{q}_{c,ij}  \;.
	\end{cases}
\end{align}
Specifically, defining $\Tr_p := \plim_{p \to \infty} \tfrac{1}{p}\trace$, we have 
\begin{align}
\begin{cases}
	q_a =&  \kappa_1^2 \frac{1}{\gamma} \hat{f}_a^2 \Tr_p \left[ A^{-1}\Theta^\top\Theta A^{-1} M_{00} \right]  + (\kappa_1')^2 \frac{1}{\gamma} \Tr_p \left[ A^{-1} D_1\Theta^\top \Theta D_1 A^{-1} M_{00} \right] \sum_{i \in [k]} \hat{f}_{b,i}^2  \\[5pt]
    & \quad + \hat{q}_a \Tr_p \left[ A^{-1} M_{00} A^{-1} M_{00} \right] + \Tr_p \left[ A^{-1}D_1 M_{11} D_1 A^{-1} M_{00} \right] \sum_{i \in [k]} \hat{q}_{c,ii} \;, \\[5pt]
	q_{c,ii} =&   \kappa_1^2 \frac{1}{\gamma} \hat{f}_a^2 \Tr_p \left[ A^{-1}\Theta^\top\Theta A^{-1} D_1M_{11}D_1  \right]  + (\kappa_1')^2 \frac{1}{\gamma} \hat{f}_{b,i}^2 \Tr_p \left[ A^{-1} D_1\Theta^\top \Theta D_1 A^{-1}D_1M_{11}D_1 \right]  \\[5pt]
    &\quad + (\kappa_1')^2 \frac{1}{\gamma}  \Tr_p \left[ A^{-1} D_1\Theta^\top \Theta D_1 A^{-1}D_2M_{11}D_2 \right]\sum_{j \neq i}\hat{f}_{b,j}^2  \\[5pt]
    &\quad + \;  \hat{q}_a \Tr_p \left[ A^{-1} M_{00} A^{-1}D_1M_{11}D_1 \right] +  \hat{q}_{c,ii} \Tr_p \left[ A^{-1}D_1 M_{11} D_1 A^{-1}  D_1M_{11}D_1 \right]   \\[5pt]
    &\quad + \;   \Tr_p \left[ A^{-1}D_1 M_{11} D_1 A^{-1}  D_2M_{11}D_2 \right] \sum_{j\neq i} \hat{q}_{c,jj} \,. 
    \end{cases} 
    \label{eq:q-eqs-symmetric}
\end{align}
Without loss of generality we fix $D_i=D_1$ and $D_j=D_2$ since (i) we expect the traces of the random matrices to converge to their expectation, and (ii) the components of each overlap parameters are permutation symmetric with respect to the indices $i \in [k]$ and $j\in [k]$, hence equivalent in law. Note that when $k=1$, terms dependent on $D_2$ drop out. Similar reductions can be obtained for $q_b$ and $q_{c,ij}$ when $i \neq j$, but we omit these here as these parameters do not contribute to the training or generalization error in this setting.\\

Finally, observe that the equation for $q_a$ only depends on the trace of $\hat{q}_c$. It follows that $(q_a, \trace{q_c}, \hat{q}_a, \trace{\hat{q}_c})$ form a closed system, i.e., it is not necessary to solve for the individual diagonal entries of $q_c$ or $\hat{q}_c$. As such, we note below the fixed point equation
\begin{align}
	\trace{q_c} &=  k \cdot \kappa_1^2 \frac{1}{\gamma} \hat{f}_a^2 \Tr_p \left[ A^{-1}\Theta^\top\Theta A^{-1} D_1M_{11}D_1 \right] + (\kappa_1')^2 \frac{1}{\gamma}  \Tr_p \left[ A^{-1} D_1\Theta^\top \Theta D_1 A^{-1}D_1M_{11}D_1 \right] \sum_{i\in[k]} \hat{f}_{b,i}^2  \\
    &\quad + (\kappa_1')^2 \frac{1}{\gamma}  \Tr_p \left[ A^{-1} D_1\Theta^\top \Theta D_1 A^{-1}D_2M_{11}D_2 \right] \cdot (k-1) \cdot \sum_{j \in[k] }\hat{f}_{b,j}^2  \\
    &\quad + \;  k \cdot \hat{q}_a \Tr_p \left[ A^{-1} M_{00} A^{-1}D_1M_{11}D_1 \right] + \Tr_p \left[ A^{-1}D_1 M_{11} D_1 A^{-1}  D_1M_{11}D_1 \right] \trace{\hat{q}_{c}}  \\[5pt]
    &\quad + \;   \Tr_p \left[ A^{-1}D_1 M_{11} D_1 A^{-1}  D_2M_{11}D_2 \right] \cdot (k-1) \cdot \trace{\hat{q}_{c}} \;,
\end{align}
which is obtained from summing~\eqref{eq:q-eqs-symmetric} over $i = 1, \ldots, k$.\\

The overlap parameters depend on $\varpi = V_k^\top \theta_0 \in \RR^k$, the random alignment between the subspace and teacher vectors, through the right-hand side of the saddle-point equations~\eqref{eq:update_hat_sobo_h1}. Asymptotically, $\varpi$ is distributed as ${\cal N}(0, I_k)$ and is uncorrelated with both $\zeta = V_k^\top \Theta$ and $\Theta^\top \theta_0$. We now make explicit the dependence on $\varpi$ and $k$ of the various overlap parameters, their hatted counterparts, and the resulting errors. In particular, this analysis allows us to characterize the \textit{distribution} of the overlap parameters and errors in the proportional asymptotics limit. As a consequence, the fixed-point iteration for~\eqref{eq:update_hat_sobo_h1} and~\eqref{eq:update_nonhat_sobo_h1} only needs to be solved numerically \textit{once} for a given set of parameters $\alpha$ and $\gamma$. Then, for all $k \geq 0$ and any realization of $\varpi=V_k^\top \theta_0$, we can predict the generalization error or compute any statistics of the error distributions.\\

As the equations for $\Sigma$ and $\hat{\Sigma}$ do not depend on $\varpi$, these matrices are constant with respect to $\varpi$. Then, by~\eqref{eq:update_hat_sobo_h1} through~\eqref{eq:def-random-matrices-for-saddlepoint-eqs}, all components of $f$ and $\hat{f}$ are  at most be linear in $\varpi$, and all components of $q$ and $\hat{q}$ at most quadratic in $\varpi$. Specifically, using the simplifications of the random matrices discussed above, we arrive at the ansatz~\eqref{eq:varpi-scaling} in the main text, where the superscript $(i)$ denotes coefficients of $i$-th order monomials in~$\varpi$.
Matching the terms in~\eqref{eq:asymptotic_simplified_saddle} by their order with respect to $\varpi$, and considering the corresponding fixed point equations for the hatted overlap parameters from~\eqref{eq:update_hat_sobo_h1} as well, yields the fixed-point equations~\eqref{eq:asymptotic_saddle_final_nonhat} and \eqref{eq:asymptotic_saddle_final_hat} in the main text.


\section{Brief introduction of selected ideas from free probability and operator-valued free probability}
\label{app:free-prob-intro}

In this appendix, we introduce some of the tools necessary to ``close'' the system of saddle-point equations~\eqref{eq:update_hat_sobo_h1} and~\eqref{eq:update_nonhat_sobo_h1} and hence evaluate the high-dimensional limits $\plim_{p \to \infty}$ on the right-hand side of~\eqref{eq:update_hat_sobo_h1} in terms of a purely finite-dimensional system of equations. We must use operator-valued free probability theory, as we exemplify in section~\ref{sec:linearization-main-text} of the main text, in order to evaluate the limits of the form $\plim_{p \to \infty}\tfrac{1}{p} \trace \left[ r (\Theta^\top \Theta, D_1, \dots, D_k ) \right]$ where $r$ is a rational function, $ \Theta^\top \Theta$ a Wishart matrix, and $D_i = \diag{\zeta_i}$  with $\zeta_1, \dots, \zeta_k \sim {\cal N}(0, I_p)$ iid. The presentation here is non-exhaustive and informal and closely follows the monograph by~\citet{mingo-speicher:2017} on the same topic where technical details and proofs can be found. Our goal is to provide a short and mostly self-contained practical exposition of some aspects of the theory that we use in the main text for those readers who are unfamiliar with free probability or its operator-valued extension.\\


\subsection{Non-commutative probability spaces and freeness}

First, as a reminder:
\begin{definition}
An algebra ${\cal A}$ over a field $K$ is a $K$-vector space equipped with a product operation $\cdot \colon {\cal A} \times {\cal A} \to {\cal A}$, $(a,b) \mapsto a \cdot b = ab$ that is $K$-bilinear---for example, matrices with real or complex entries where $\cdot$ is matrix multiplication. The algebra ${\cal A}$ is called unital if there exists $1 \in {\cal A}$ such that $1 \cdot a = a \cdot 1 = a$ for all $a \in {\cal A}$. A unital linear function $\varphi \colon \, {\cal A} \mapsto K$ is $K$-linear and maps $\varphi(1) = 1$. A $*$-algebra generalizes complex conjugation in a formal way---e.g.~complex matrices with conjugate transposition.
\end{definition}

This foundation is important for the following definition of a non-commutative probability space, which is the necessary space to discuss limits  as $N \to \infty$ of random matrices $X_N \in \CC^{N \times N}$ and their distributions, spectral densities, moments, Cauchy transforms, and so forth. The objects in a non-commutative probability space can be given directly by such limits (weakly/in distribution), so they may effectively be like ``infinitely large'' random matrices, as well-defined elements in an abstract space.

\begin{definition}
A non-commutative probability space $({\cal A}, \varphi)$ is a unital algebra ${\cal A}$ (always over $\CC$ in this appendix) together with a unital linear functional $\varphi \colon \,{\cal A} \to \CC$. An $a \in {\cal A}$ is called a non-commutative random variable or simply an element. If ${\cal A}$ is also a $*$-algebra and $\varphi(a^*a) \geq 0$ for all $a \in {\cal A}$, then $\varphi$ is called a ``state''.
\end{definition}

This definition is purely algebraic; there is no measure theory yet. The state $\varphi$ plays the role of an expectation, and when discussing limits of random matrices, we can for instance think of it as
\begin{align}
\varphi(\cdot) \; \text{``$=$''} \lim_{N \to \infty} \EE \left[\frac{1}{N} \trace \left[\cdot \right]\right]\,,
\label{eq:limit-state-trace}
\end{align}
with the left-hand side acting on the limiting object for the family of $N \times N$ matrices on the right.\\

The most important concept for our purposes is the following, which can to some extent be seen as a generalization of, or at least related to, the concept of independence of standard (commuting) random variables:

\begin{definition}
Let $({\cal A}, \varphi)$ be a non-commutative probability space, and let ${\cal A}_1, \dots, {\cal A}_s \subset {\cal A}$ be unital subalgebras of ${\cal A}$, e.g.\ generated each by a different element $a_i \in {\cal A}$, with ${\cal A}_i = \CC[a_i]$, such that the subalgebra consists of polynomials in $a_i$. Then, ${\cal A}_1, \dots, {\cal A}_s $ are called free or freely independent with respect to $\varphi$ if for all $r \geq 2$, $a_1, \dots, a_r \in {\cal A}$ with
\begin{enumerate}
\item $\varphi(a_i) = 0$ \textit{(centered)}
\item $a_i \in {\cal A}_{j_i}$ for some $j_i \in [s]$ \textit{(belong to the subalgebras)}
\item $j_1 \neq j_2$, $j_2 \neq j_3$, $\dots$, $j_{r-1} \neq j_r$ \textit{(neighboring elements not in same subalgebra)}
\end{enumerate}
we have
$
\varphi(a_1 a_2 \dots a_r) = 0
$.
We call elements of ${\cal A}$ free if their generated subalgebras are free.\\
\end{definition}

The definition of freeness is reminiscent of independence of centered random variables, but there are important differences because of the neighboring condition and non-commutativity. Two examples to illustrate the comparison:

\begin{example}
Consider $a \in {\cal A}_1$, $b \in {\cal A}_2$ free and not necessarily centered. By freeness, we have
$
\varphi \left( (a - \varphi(a) 1) \cdot (b - \varphi(b) 1) \right) = 0
$,
and by linearity and unitality of $\varphi$ this becomes
$
\varphi \left(ab - \varphi(a) b - \varphi(b) a + \varphi(a) \varphi(b) 1 \right) =  \varphi(ab) - \varphi(a) \varphi(b) = 0 \quad \Rightarrow \quad \varphi(ab) = \varphi(a) \varphi(b)
$,
which is exactly the same as for independent random variables. Similarly, for $a_1, a_2 \in {\cal A}_1$, $b \in {\cal A}_2$ with ${\cal A}_1, {\cal A}_2$ free, we have $\varphi(a_1 b a_2) = \varphi(a_1 a_2) \varphi(b)$. \\
\end{example}

\begin{example}
\label{example-free-commute}
Still assuming $a,b$ free, one can similarly show that
$
\varphi(abab) = \varphi(a^2) \varphi(b)^2 + \varphi(a)^2 \varphi(b^2)- \varphi(a)^2 \varphi(b)^2
$.
However, we find 
$
\varphi(a^2 b^2) = \varphi(a^2) \varphi(b^2)\,,
$
using our result from the first example. So these two expressions are not equal in general, and we cannot commute elements in this sense even if they are free. More concretely, if we were to demand that
$
\varphi(abab) \overset{!}{=} \varphi(a^2 b^2) \,,
$
then that would imply by the result above that
$
\varphi\left((a - \varphi(a)1)^2 \cdot (b - \varphi(b)1)^2 \right) = \varphi\left((a - \varphi(a)1)^2\right) \varphi\left( (b - \varphi(b)1)^2 \right) =0\,,
$
which is true only when $a$ or $b$ is a scalar multiple of the identity.\\
\end{example}

In principle, freeness directly provides a way to compute mixed moments of sums and products of free elements from their individual moments as in these examples. But the combinatorics can be complicated, and besides moments, we would also like to compute other quantities like traces of inverses or spectral densities (to be defined formally below). For this task, we would need to compute and sum all moments, which is tedious. A simpler way of handling addition of free elements is given by \textit{free cumulants} and the integral transformation/resolvent theory of Cauchy transforms introduced in the next section. As an example of how the theory is built algebraically, consider the following definition:

\begin{definition}
Let $({\cal A}_k, \varphi_k)$ for all $k \in \NN$ and $({\cal A}, \varphi)$ be non-commutative probability spaces and $I$ some index set. We say that $\left(b_k^{(i)} \right)_{i \in I} \subset {\cal A}_k$ converges in distribution to $\left(b^{(i)} \right)_{i \in I} \subset {\cal A}$, if for all $i_1, \dots, i_n \in I$ we have
\begin{align}
\lim_{k \to \infty} \varphi_k \left(b^{(i_1)}_k \dots b^{(i_n)}_k  \right) = \varphi \left(b^{(i_1)} \dots b^{(i_n)} \right)\,.
\end{align}
\end{definition}

So, again, convergence is defined purely algebraically, via convergence of all moments. Note that in the classical setting of probability theory, convergence in moments is \textit{not} the same as weak convergence.\\

We call families of random matrices $\left( X_{N,i} \right)_{i \in I}$ in $\RR^{N \times N}$ asymptotically free if
\begin{align}
\lim_{N \to \infty} \EE \left[ \frac{1}{N} \trace \left[ \left( X_{N, i_1}^k - c_{N,i_1,k} 1 \right) \dots\left( X_{N, i_n}^k - c_{N,i_n,k} 1 \right)   \right] \right] = 0
\end{align}
for all moments $k$ and all $i_1, i_2, \dots, i_n$ pairwise distinct. Here, the constants $c$ center the corresponding $k$-th moment.\\

Lastly, we do not require the formal definition of free cumulants for our purposes here, but they are defined by a combinatorical formula from the moments. For $a \in {\cal A}$, we write $\alpha_n^a = \varphi(a^n)$ for the $n$-th moment, and $\kappa_n^a$ for the $n$-th free cumulant, which depends on moments up to order $n$.

\begin{proposition}
\label{prop-cumu}
For $a,b \in {\cal A}$ free, we have $\kappa_n^{a+b} = \kappa_n^a + \kappa_n^b$.
\end{proposition}

This property is crucial to using the free cumulants to build the theory detailed below.


\subsection{Transformations and spectral densities}


\subsubsection{Definitions of different transforms}

We collect here definitions and useful identities for a number of transforms related to the Cauchy transform---the central object of study in free probability, as it has nice algebraic and analytical properties and can be used to extract further information (e.g.~spectral densities) or to directly compute some traces. For our application, we will ultimately be interested in traces of rational functions of random matrices.

\begin{definition}
\label{def:GM}
For $a \in {\cal A}$, we define the Cauchy transform as $G_a \colon \, \CC \to \CC$ with
\begin{align}
G_a(z) = \varphi \left( \left(z 1 - a \right)^{-1} \right) =\sum_{n = 0}^\infty \frac{\varphi(a^n)}{z^{n+1}} = \sum_{n = 0}^\infty \frac{\alpha_n^a}{z^{n+1}} = \frac{1}{z} M_a \left(\frac{1}{z} \right)
\end{align}
where
$
M_a(z) = \sum_{n=0}^\infty \alpha_n^a z^n
$
is the moment series of $a$. Upon initial definition, it is just a formal power series. We also write
$
F_a(z) = \frac{1}{G_a(z)}
$
as well as
$
H_a(z) = F_a(z) - z \,.
$
Note that for large $\abs{z}$, we have $G_a(z) \sim 1/z$.
\end{definition}

\begin{definition}
\label{def:g}
The Stieltjes transform $g_a \colon\, \CC \to \CC$---with the opposite sign convention compared to the Cauchy transform---is defined as
$
g_a(z) = -G_a(z) = \varphi \left( \left(a - z 1 \right)^{-1} \right)\,
$
and used more commonly in random matrix theory. This sign convention would make some of the following identities slightly messier, so the Cauchy transform is typically preferred in free probability.
\end{definition}

\begin{definition}
\label{def:c}
The cumulant series of $a \in {\cal A}$ is defined as
$
C_a(z) = \sum_{n=0}^\infty \kappa_n^a z^n
$
in analogy to the moment series.
\end{definition}

Thus, if $a,b$ are free, we have
$
C_{a+b}(z) + 1 = C_a(z) + C_b(z)
$.
The following identity is proved through nontrivial combinatorics but serves as the key technical result for what follows:

\begin{theorem}
\label{thm:ma-ca-relation}
We have $M_a(z) = C_a(z M_a(z))$ for all $z$.
\end{theorem}

The logic here is that objects like the Cauchy transform involve the moment series, which relates to the cumulant series, which in turn is easy to calculate for sums of free elements. This approach essentially yields the free convolution and subordination theory below and also underlies the operator-valued equivalents.

\begin{definition}
\label{def-rks}
The $R$-transform of $a \in {\cal A}$ is defined as $R_a \colon\, \CC \to \CC$ with
\begin{align}
R_a(z) := \frac{C_a(z) - 1}{z} = \sum_{n=0}^\infty \kappa_{n+1}^a z^n\,,
\label{eq:def-r}
\end{align}
and the $K$-transform is
\begin{align}
K_a(z) := R_a(z) + \frac{1}{z} = \frac{C_a(z)}{z}\,.
\label{eq:def-k}
\end{align}
Lastly, the $S$-transform, which plays a similar role to the $R$-transform for products of free elements instead of sums, is defined as
\begin{align}
S_a(z) = \frac{1+z}{z} M_a^{-1}(z)\,,
\label{eq:def-s}
\end{align}
where $M_a^{-1}$ is the inverse function of $M_a$.
\end{definition}

A few simple observations follow directly from the definitions of the various transforms and the main technical result theorem~\ref{thm:ma-ca-relation}:

\begin{enumerate}
\item We have
$
G_a(K_a(z)) = K_a(G_a(z)) = z\,,
$
so these are inverse functions of each other. We can verify, for example, that
\begin{align}
K_a(G_a(z)) \overset{\text{def.~\ref{def-rks}}}{=} \frac{1}{G_a(z)} C_a(G_a(z)) \overset{\text{def.~\ref{def:GM}}}{=} \frac{1}{G_a(z)} C_a\left(\frac{1}{z} M_a\left(\frac{1}{z} \right)\right) \overset{\text{thm.~\ref{thm:ma-ca-relation}}}{=} \frac{M_a\left(\tfrac{1}{z} \right)}{G_a(z)} \overset{\text{def.~\ref{def:GM}}}{=} z\,.
\end{align}
\item Since the definition~\eqref{eq:def-r} removes the constant term in the series, we have, by the addition property prop.~\ref{prop-cumu} of free cumulants for free elements $a,b$:
\begin{align}
R_{a+b}(z) = R_a(z) + R_b(z)\,.
\label{eq:sum-r}
\end{align}
Equivalently, one can write $F_{a+b}^{-1}(z) = F_{a}^{-1}(z) + F_{b}^{-1}(z) - z$.
This identity is one way---the traditional one as developed by~\citet{voiculescu:1986} and summarized in \cite{voiculescu-dykema-nica:1992}---to compute the Cauchy transform of the sum $a+b$ of free elements from $G_a$ and $G_b$. Note that by the first observation, the $R$-transform is related to the inverse function of $G$. Hence, using Equation~\eqref{eq:sum-r} to obtain $G_{a+b}$ requires inverse functions, which can be difficult to compute, even numerically. For this reason, the subordinator approach to free convolutions introduced below is often preferred for numerical computations.
\item We also note that for the multiplication of free elements $a,b$, it holds that
\begin{align}
S_{ab}(z) = S_a(z) S_b(z)\,,
\label{eq:prod-s}
\end{align}
so free multiplicative convolutions also require inverse function computations if we find $G_{ab}$ using~\eqref{eq:prod-s}.
\end{enumerate}


\subsubsection{Spectral density definition and relation to Cauchy transform}

One defines a distribution associated with $a \in {\cal A}$, as before, algebraically:

\begin{definition}
For $a \in {\cal A}$, an element of a non-commutative probability space $({\cal A}, \varphi)$, we define $\mu_a \colon\, \CC[x] \to \CC$ as the map from polynomials $p$ in $a$ to their expectations
$
\mu_a[p] := \varphi(p(a))\,.
$
If $a$ is self-adjoint in a $C^*$ algebra with norm $1$ for $\varphi$ positive, then, under some additional assumptions, there exists a probability measure on $\RR$ such that
$
 \mu_a[p] = \int_{\RR} p(x) \dd \mu_a(x)
$
for all $p$.
\end{definition}

We can then compute e.g.\ the Cauchy transform from this probability measure via
$
G_a(z) =  \varphi \left( \left(z 1 - a \right)^{-1} \right) = \int_\RR \frac{\dd \mu_a(t)}{z-t}
$
which is well-defined for all $z \in \CC \setminus \text{supp}(\mu_a)$. Importantly, if $\mu_a$ has a Lebesgue density at $x \in \RR$, we can recover it from its Cauchy transform as follows. Note that for $\eta > 0$, we have
\begin{align}
G(x + i \eta) =  \int_\RR \frac{\dd \mu_a(t)}{x + i \eta -t} = \int_{\RR}  \frac{x-t}{(x-t)^2 + \eta^2} \dd \mu_a(t) -i \int_{\RR}  \frac{\eta}{(x-t)^2 + \eta^2} \dd \mu_a(t)\,,
\end{align}
so that
\begin{align}
-\tfrac{1}{\pi} \text{Im} \, G_a(x + i \eta) = \left(P_\eta * \tfrac{\dd \mu_a}{\dd x} \right)(x)
\label{eq:dens-convolution-cauchy}
\end{align}
is the convolution of the Lebesgue density of $\mu_a$, if it exists, at $x$ with the Poisson kernel
$
P_\eta(t) = \tfrac{1}{\pi} \tfrac{\eta}{t^2 + \eta^2}
$
which forms a Dirac sequence as $\eta \to 0$. Hence, $-\tfrac{1}{\pi} \text{Im} \, G_a(x + i \eta)$ gives the density at $x$ smeared out over a scale $\eta$, and we have
\begin{align}
\lim_{\eta \downarrow 0} -\frac{1}{\pi} \text{Im} \, G_a(x + i \eta) = \frac{\dd \mu_a}{\dd x}(x)\,.
\label{eq:density-from-cauchy}
\end{align}
If $\mu_a$ has atoms, one needs to be more careful; it still holds that
$
\lim_{\eta \downarrow 0} -\frac{1}{\pi} \int_{x_0}^{x_1} \text{Im} \, G_a(x + i \eta) \dd x = \mu_a((x_0, x_1)) + \frac{1}{2} \mu_a(\{x_0,x_1\})
$.\\

\begin{definition}
The distribution of $a+b$ for $a,b$ free with distributions $\mu$ for $a$ and $\nu$ for $b$ is called the free additive convolution and written as $\mu \boxplus \nu$. It is constructed from the $R$-transforms of the respective measures according to~\eqref{eq:sum-r}. Analogously, the free multiplicative convolution $\mu \boxtimes \nu$ is defined according to~\eqref{eq:prod-s}. We refer to~\cite{ji:2021} for a recent introduction and analysis of free multiplicative convolutions, as we only require free additive convolutions in the following.
\end{definition}

\subsection{Examples of random matrix ensembles}

We refer to~\citet{livan-novaes-vivo:2018} for an introduction and~\citet{bai-silverstein:2010} for further details on these standard ensembles.\\

\begin{example}
We call
$
A_N = \frac{1}{M} X_N X_N^\top \in \RR^{N \times N}
$
with $X_N \in \RR^{N \times M}$ and $X_{ij} \overset{\text{iid}}{\sim}{\cal N}(0, \sigma^2)$ a Wishart or Wishart--Laguerre matrix. For $N,M \to \infty$ with $c = N/M \to \text{const.} \in (0,\infty)$, its limiting spectral distribution is the Marchenko--Pastur (MP) law
\begin{align}
\rho_{\text{MP}}^{c,\sigma^2}(x) \,\dd x = \begin{cases}
\left(1 - \frac{1}{c} \right) \delta_0(x) + \frac{1}{\sigma^2} \rho_{\text{bulk}}^c\left(x/\sigma^2\right) \dd x\,, \quad & c >1\\[2pt]
\frac{1}{\sigma^2} \rho_{\text{bulk}}^c\left(x/\sigma^2\right) \dd x\,, \quad & 0 < c \leq 1\,.
\end{cases}
\label{eq:mp-law}
\end{align}
Here, the continuous part has the density
\begin{align}
\rho_{\text{bulk}}^c(x) = \frac{1}{2\pi c x} \sqrt{\left((1+\sqrt{c})^2 -x \right) \left(x - (1 - \sqrt{c})^2 \right)} \1_{((1 - \sqrt{c})^2, (1 + \sqrt{c})^2)}(x)\,.
\label{eq:mp-density}
\end{align}
The Stieltjes transform of this distribution can be computed to be
\begin{align}
g_{\rho_{\text{MP}}^{c,\sigma^2}}(z) = \frac{1}{2 c \sigma^2 z} \left(\sigma^2 (1-c) - z - \sqrt{\left(z - \sigma^2 (1 + c) \right)^2 - 4 c \sigma^4} \right) \qquad  \forall z \in \CC \setminus \text{supp}(\rho_{\text{MP}}^{c})\,.
\label{eq:mp-stieltjes-appendix}
\end{align}
A brief derivation of these well-known results from first principles can also be found in~\cite{aguirre-lopez-franz-pastore:2025} which proceeds by calculating the Stieltjes transform of $A_N$ using the saddlepoint method as $N,M \to \infty$.
\end{example}

\begin{example}
The Gaussian orthogonal ensemble (GOE) is the other standard ensemble one typically considers---it does not consist of orthogonal random matrices but rather Gaussian random matrices which have distribution invariant under orthogonal transformations. We consider random symmetric $N \times N$ matrices
$
A_N = \frac{1}{2\sqrt{N}} \left(X_N + X_N^\top \right) 
$
where $X_{ij} \sim {\cal N}(0,\sigma^2)$ iid, i.e.\ the entries of $A_N$ are independent Gaussian up to symmetry, and have different variances on the diagonal compared to the off-diagonals. It is well-known, and can be derived analogously to the MP law, that the distribution of eigenvalues of $A_N$, as $N \to \infty$, becomes the semicircle law
\begin{align}
\rho_{\text{semi-circ}}^{\sigma^2}(x)\, \dd x = \frac{1}{\pi \sigma^2} \sqrt{2 \sigma^2 - x^2} \,\1_{[-\sqrt{2}\sigma, \sqrt{2} \sigma]}(x) \dd x\,.
\end{align}
The Stieltjes transform of this measure can be computed as
$
g_{\rho_{\text{semi-circ}}^{\sigma^2}}(z) = \left(-z + \sqrt{z^2 - 2 \sigma^2} \right) / \sigma^2\,.
$

\end{example}


\subsection{An example of a free additive convolution}
\label{example:free-add-conv-n01}

In our Sobolev training setting, we construct $k+1$-dimensional structures on the right hand side of \eqref{eq:asymptotic_saddle_final_nonhat} using diagonal Gaussian matrices $D_i = \diag{\zeta_i}$ where $\zeta_i \sim {\cal N}(0, \text{Id}_p)$ iid. By example~\ref{example-free-commute}, these matrices are not free with respect to each other as they commute but are not multiples of the identity (still, Wishart matrices $\Theta^\top \Theta$ are asymptotically free of $\{D_1, \dots, D_k\}$ by~\cite[Chapter 4]{mingo-speicher:2017}).\\

We can also come to this conclusion as follows: the spectral distribution of $D_i$ is obviously ${\cal N}(0,1)$. The spectral distribution of $D_1 + D_2$ is ${\cal N}(0,2)$ by adding the independent Gaussian random variables on their diagonal. But if $D_1, D_2$ were free, the spectral distribution of their sum $D_1 + D_2$ would converge to ${\cal N}(0,1) \boxplus {\cal N}(0,1)$, which is \textit{not} equal to ${\cal N}(0,2)$ as we check below. By~\cite[Chapter 4]{mingo-speicher:2017}, conjugating one of the diagonal matrices, or even the same one, with a random orthogonal matrix ``randomizes the eigenvectors sufficiently'' to make them free: $D_1$ and $U D_2 U^\top$, with $U \sim \text{Haar}(O(N))$, are indeed asymptotically free, and their distribution is given by ${\cal N}(0,1) \boxplus {\cal N}(0,1)$.\\

We verify these properties via sampling and explicit computation of the free additive convolution in Figure~\ref{fig:free-diags}. For $\mu = {\cal N}(0,1)$, we find $\mu \boxplus \mu$ by computing its Stieltjes transform $g_{\mu \boxplus \mu}(x + i \eta)$ close to the real axis. We have
\begin{align}
g_\mu(z) = \int_{- \infty}^\infty \frac{1}{x-z} \frac{1}{\sqrt{2 \pi}} e^{-x^2/2} \dd x = i \sqrt{\frac{\pi}{2}} \frac{i}{\pi} \int_{- \infty}^\infty \frac{e^{-t^2}}{z/\sqrt{2} - t} \dd t =  i \sqrt{\frac{\pi}{2}} w\left(\frac{z}{\sqrt{2}} \right)\,,
\end{align}
with the Faddeeva function
$
w(z) =  \frac{i}{\pi} \int_{- \infty}^\infty \frac{e^{-t^2}}{z - t} \dd t  = e^{-z^2} \text{erfc}(-iz)
$.
Then, we set $F_\mu(z) = -1/g_\mu(z)$ and compute
$
F_{\mu \boxplus \mu}^{-1}(z) = 2 F_\mu^{-1}(z) - z
$.
Inverse functions are evaluated numerically with a standard root-finder for which we separate arguments and function values into vectors of real and imaginary parts. The resulting PDF of $\mu \boxplus \mu$ in Figure~\ref{fig:free-diags} looks relatively similar to a ${\cal N}(0,2)$ density, but slight differences are visible, and sampling confirms the theoretical result. We hence note that if we have $k \geq 2$ gradient observations, we need to be careful with the $D_i$ matrices in our application below as they are not free with respect to each other. Note that this differs from the computations of~\citet{adlam-pennington:2020} in a related precise asymptotic analysis, where they are able to linearize the problem (as detailed below) to a Gaussian block matrix with dense, asymptotically free blocks.

\begin{figure}
\centering
\includegraphics[width = .4 \textwidth]{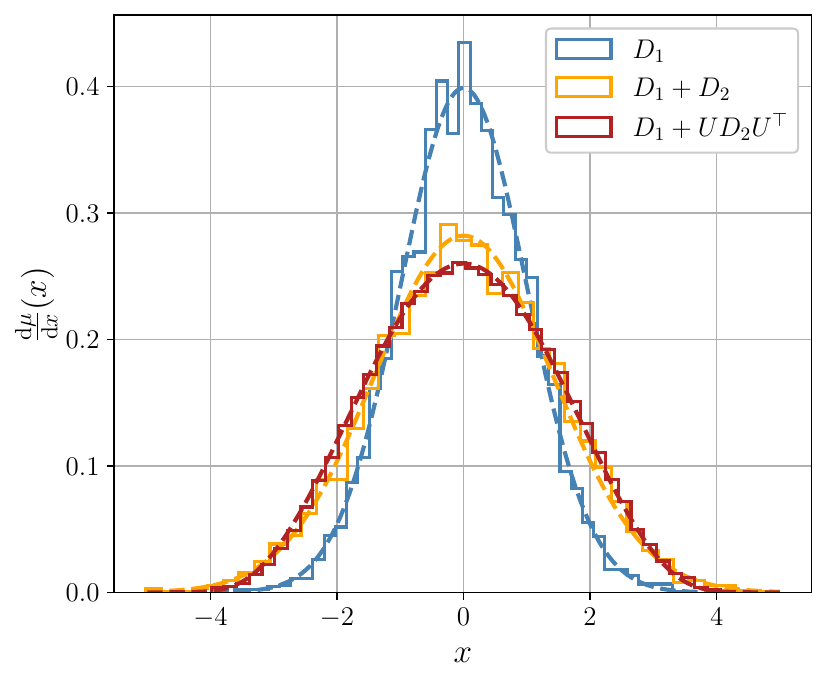}
\caption{Spectral densities for sums of diagonal Gaussian matrices $D_1, D_2$. Histograms: One $5000 \times 5000$ sample each. The orthogonal Haar matrix $U$ is sampled according to the method discussed in~\citep{mezzadri:2006}. Dashed lines: theoretically expected PDFs of ${\cal N}(0,1)$, ${\cal N}(0,2)$ and ${\cal N}(0,1) \boxplus {\cal N}(0,1)$, respectively. The latter is evaluated from its Cauchy transform according to~\eqref{eq:dens-convolution-cauchy} with $\eta = 10^{-5}$. We see that $D_1$ and $D_2$ are not free, but $D_1$ and $U D_2 U^\top$ are.}
\label{fig:free-diags}
\end{figure}


\subsection{Computing free additive convolutions via subordination}

Subordination is an alternative method of evaluating  $G_{a+b}(z)$ for the sum of free $a,b \in {\cal A}$. The basic idea is to find
\begin{align}
G_{a+b}(z) = G_a(\omega_a(z)) = G_b(\omega_b(z))\,.
\end{align}
The functions $\omega_a$, $\omega_b$ are called subordinators and can be computed by solving fixed-point equations that are formulated purely in terms of $G_a$ and $G_b$ or functions thereof, with no function inversions required.\\

As a motivation: for all $z$, we have
$
z = G_{a+b}(K_{a+b}(z)) = G_{a+b}\left(R_{a+b}(z) + \frac{1}{z} \right) = G_{a+b}\left(R_a(z) + R_b(z) +  \frac{1}{z} \right)
$.
We define $w = R_{a+b}(z) + \tfrac{1}{z}$, then
$
z = G_{a+b}(w) = G_a\left(R_a(z) + \frac{1}{z} \right) = G_a\left(w - R_b(z) \right) = G_a\left(w - R_b(G_{a+b}(w)) \right)
$.
If we then define $\omega_a(z) = z - R_b(G_{a+b}(z))$ and reverse the logic, then we get
$
G_{a+b}(z) = G_a(\omega_a(z))
$. 
From the definition of $\omega_a$, we also find
$
\omega_a(z) = z - R_b(G_{a+b}(z)) = z - R_b(G_{a}(\omega_a(z)))
$.
Thus, we have a fixed-point equation for $\omega_a$ in terms of known functions, in principle, but still encounter the undesirable inverse of $G_b$ within $R_b$. Through a rather long series of arguments (that mostly rely on complex analysis and inverse function theory on $\CC$), one can show that this result can alternatively be written as
\begin{align}
\omega_a(z) = z + H_b(H_a(\omega_a(z) + z)\,,
\label{eq:subord}
\end{align}
with $H(z) = 1/G(z) - z$ as defined above. Hence,~\eqref{eq:subord} achieves the goal of expressing the Cauchy transform of $a+b$ through a fixed-point equation that only requires knowledge of $G_a$ and $G_b$. The same can be done for $\omega_b$ and $G_{a+b}(z) = G_b(\omega_b(z))$ by symmetry, of course. There is exactly one solution of the subordinator equation~\eqref{eq:subord} in the upper complex half plane, i.e.\ with $\text{Im} \, \omega_a(z) > 0 $ for $\text{Im}  \, z > 0$. Properties and numerical solutions of a similar fixed-point equation with positivity constraints are discussed in~\cite{helton-far-speicher:2007}.


\subsection{Operator-valued free probability}
\label{app:operator-free}

The ability to compute Cauchy transforms and distributions of sums or products of free non-commutative elements is already useful, but the theory presented so far does not offer a similarly easy approach for many of the more complicated possible algebraic combinations of free elements, such as polynomials or rational functions in multiple free elements. Related to this challenge, as it turns out, is the fact that block-matrices of free elements are difficult to treat. The way out is operator-valued free probability, which relaxes the concept of a state $\varphi \colon\, {\cal A} \to \CC$ to a conditional expectation $E \colon \,{\cal A} \to {\cal B}$, e.g.\ over individual blocks, where ${\cal B}$ is generally some subalgebra of ${\cal A}$ in place of the field $\CC$. Developing analogous constructions to the previous sections---essentially replacing $\varphi$ with blockwise operations and any $z \in \CC$ with a matrix $Z \in \CC^{d \times d}$---makes it possibly to use the same (now operator-valued) Cauchy transform and subordination theory for block matrices of non-commutative elements.\\

First the formal definition:

\begin{definition}
An operator-valued non-commutative probability space $({\cal A}, E, {\cal B})$ is given by a unital algebra ${\cal A}$, a unital subalgebra ${\cal B} \subset {\cal A}$, and a linear map $E \colon\, {\cal A} \to {\cal B}$ called conditional expectation, which satisfies
\begin{enumerate}
\item $E(b) = b$ for all $b \in {\cal B}$.
\item $E(b_1 a b_2) = b_1 E(a) b_2$ for all $a \in {\cal A}$ and $b_1, b_2 \in {\cal B}$.
\end{enumerate}
\end{definition}

The usual situation is as follows:  start from a non-commutative probability space $({\cal C}, \varphi)$ which we care about, e.g.\ with some free elements of interest whose statistics we know individually. Then, lift this space to $d \times d$ block matrices by setting
\begin{enumerate}
\item ${\cal A} = M_d({\cal C})$ \textit{(arrange elements from ${\cal C}$ in a $d \times d$ matrix)}
\item ${\cal B} = M_d(\CC)$ \textit{(these are actual complex $d \times d$ matrices, which are a subset of ${\cal A}$ via the identification of $b \in M_d(\CC)$ with $b \otimes 1_{\cal C}$.)}
\item $E = \text{id} \otimes \varphi \colon\, M_d({\cal C}) \to M_d(\CC)$, $(c_{ij})_{1 \leq i,j \leq d} \mapsto (\varphi(c_{ij}))_{1 \leq i,j \leq d}$ via element-wise application of the state \textit{($\text{id}$ is the identity under the Kronecker product $\1_d \1_d^\top$ here if $\1_d$ is the one-vector in $\CC^d$)}
\end{enumerate}

Freeness, moments, free cumulants, and so on are then all defined with respect to $E$ instead of $\varphi$, but not much changes apart from that on a high level. We will mainly need the operator-valued equivalents of the various transformations introduced so far:

\begin{definition}
For $a \in {\cal A}$, an element in an operator-valued non-commutative probability space $({\cal A}, E, {\cal B})$, we define its operator-valued Cauchy transform $G_a \colon\, {\cal B} \to {\cal B}$ by
$
G_a(b) = E \left[(b-a)^{-1} \right]
$.
In the usual block-matrix settings and e.g.\ $d = 2$, $b = \begin{pmatrix}
b_{11} & b_{12}\\
b_{21} & b_{22}
\end{pmatrix} \in \CC^{2 \times 2}$, this definition means that $G_a(b)$ is a complex $2 \times 2$ matrix that is given by
\begin{align}
G_a(b) = \begin{pmatrix}
\varphi\left( (b \otimes 1_{\cal C} - a)^{-1}_{11} \right) & \varphi\left( (b \otimes 1_{\cal C} - a)^{-1}_{12} \right)\\
\varphi\left( (b \otimes 1_{\cal C} - a)^{-1}_{21} \right) & \varphi\left( (b \otimes 1_{\cal C} - a)^{-1}_{22} \right)
\end{pmatrix}\,,
\end{align}
with
\begin{align}
b \otimes 1_{\cal C} - a = \begin{pmatrix}
b_{11} 1_{\cal C} - a_{11} & b_{12} 1_{\cal C} - a_{12}\\
b_{21} 1_{\cal C} - a_{21} & b_{22} 1_{\cal C} - a_{22}
\end{pmatrix}\,.
\end{align}
\end{definition}

The ``scalar'' Cauchy transform of $a \in {\cal A}$ would naturally correspond to instead taking
\begin{align}
g_a(z) := \frac{1}{d} \trace \otimes \varphi \left( \left(z 1_{\cal A} -a \right)^{-1} \right) = \frac{1}{d} \sum_{i=1}^d \varphi \left( \left(z 1_{\cal A} - a \right)^{-1}_{ii} \right) = \frac{1}{d} \trace G_a(z I_d)\,,
\end{align}
i.e.\ put the same $z 1_{\cal C}$ in all diagonal blocks, take block-wise (normalized) traces and expectations $\varphi$, and take the (normalized) trace of the $d \times d$ matrix in the end. Using full block-arguments gives more flexibility, which we need below. Conveniently, a similar subordination result as before holds for the sum of free operator-valued variables. We define $H_a \colon\, {\cal B} \to {\cal B}$ as $H_a(b) = G_a(b)^{-1} - b$, where the inverse is taken in ${\cal B}$, so it corresponds to taking the inverse of the $d \times d$ matrix $G_a(b)$ in the block-settings. Then, we have

\begin{theorem}
Consider $a_1, a_2$ free elements of an operator-valued probability space. Then, we have
\begin{align}
G_{a_1 + a_2}(b) = G_{a_1}(\omega_{a_1}(b))
\,,
\label{eq:operator-subord}
\end{align}
for $\omega_{a_1} \colon\, {\cal B} \to {\cal B}$ a subordinator solving the fixed-point equation
\begin{align}
\omega_{a_1}(b) = H_{a_2}(H_{a_1}(\omega_{a_1}(b)) +b ) +b\,,
\label{eq:operator-subord-fixed-p}
\end{align}
just like in the ``scalar'' case.
\end{theorem}

While this fixed-point equation may have many solutions, it always has just one solution with positive definite imaginary part of $\omega$, if the same holds for $b$. Hence, we should choose a fixed-point solver which remains in the upper half-space of the complex plane provided it is initialized there. This property holds for naive fixed-point iteration and damped versions thereof but not necessarily for Newton-type iterations~\cite{helton-far-speicher:2007}. In summary, it is possible to evaluate the operator-valued Cauchy transforms of the sum of free operator-valued elements if we know their individual Cauchy transforms. So how can we find the latter?\\

The only case that we will need is the following: in the standard block-matrix setting, suppose we have $a = b \otimes c$---i.e.\ take a $d \times d$ block matrix whose entries are all composed of some complex number $b_{ij}$ times the element $c$. First, if two free elements $c_1, c_2$ in ${\cal C}$ are lifted in this way, for instance by setting $a_1 = b_1 \otimes c_1$ and $a_2 = b_2 \otimes c_2$, then $a_1, a_2$ remain free. Now, if we know the Cauchy transform of $c$ or its distribution $\mu_c$, either analytically or implicitly through the real-axis limit of $g_c$ , then we simply have
\begin{align}
G_{b \otimes c}\left(\tilde{b} \right) = E\left[\left(\tilde{b} \otimes 1 - b \otimes c \right)^{-1} \right] = \int_\RR \underbrace{\left(\tilde{b} - \lambda b \right)^{-1}}_{\text{inverse of } d \times d \text{ matrix}} \dd \mu_c(\lambda)\,.
\label{eq:lifted-cauchy-int}
\end{align}
This elementwise integral can be straightforwardly approximated e.g.\ via quadrature, which is explained in~\cite[Theorem 4.1]{belinschi-mai-speicher:2017}. In~\cite[Remark 6.6]{helton-mai-speicher:2018}, the authors suggest a more efficient way of computing the integral which avoids numerical integration.

\subsubsection{Linearization: idea and definition}

Suppose we know the individual distributions of free variables $x_1, \dots, x_n \in {\cal C}$ but not of $p(x_1, \dots, x_n)$, where $p$ is some given, complicated function such as a polynomial. We want to lift $p$ to an operator-valued space ${\cal A} = M_d({\cal C})$. Specifically, we construct a corresponding block matrix $\hat{p} \in M_d({\cal C})$ to ensure its operator-valued Cauchy transform is related to the Cauchy transform of $p$  and to be affine-linear in all elements $x_1, \dots, x_n$ so the transform of $\hat{p}$ proceeds from subordination. Technically, $\hat{p}$ is found using only the Schur complement and a series of straightforward observations, but this sequence amounts to a concrete algorithm to  linearize $p$ and hence compute Cauchy transforms of any polynomials (or rational functions) in free variables, which is a major achievement of the theory.\\

The following definition is purely algebraic in nature, but it is set up in such a way that it facilitates calculating Cauchy transforms in our present context:

\begin{definition}
Given a polynomial $p \in \CC \langle x_1, \dots, x_n \rangle$ in $n$ non-commutative variables $x_1, \dots, x_n$ in a unital algebra~${\cal C}$, a $d \times d$ matrix with polynomial elements $ \hat{p} \in M_d(\CC) \otimes \CC \langle x_1, \dots, x_n \rangle$ is called a \textbf{linearization} of $p$ if
\begin{align}
\hat{p} = \begin{pmatrix}
0 & \aug & u\\
\hline
v & \aug & q
\end{pmatrix} \in M_d({\cal C}) \text{ with } u \in  {\cal C}^{1 \times (d-1)}\,, \quad v \in {\cal C}^{(d-1) \times 1}\,, \quad q \in {\cal C}^{(d-1) \times (d-1)}
\end{align}
such that
\begin{enumerate}
\item $q$ is invertible and $p = - u q^{-1} v$ \textit{(necessary for the Cauchy transform relation we need below, due to Schur complement formula for block inverses)}
\item $\hat{p} = b_0 \otimes 1_{{\cal C}} + b_1 \otimes x_1 + \dots + b_n \otimes x_n$ for some coefficient matrices $b_i \in M_d(\CC)$ \textit{(so that $\hat{p}$ is affine-linear and we can evaluate its operator-valued Cauchy transform)}
\end{enumerate}
\end{definition}

Obviously, we can evaluate the operator-valued Cauchy transform of such a linearization. This ability is useful due to

\begin{proposition}
\label{prop:cauchy-lin}
For a polynomial $p \in {\cal C}$ with linearization $\hat{p} \in M_d({\cal C})$ and $z \in \CC$, set $\Lambda(z) = \text{diag}(z,0,\dots,0) \in \CC^{d \times d}$. Then, we have
$
G_p(z) = \varphi\left((z-p)^{-1}\right) = \left( G_{\hat{p}}(\Lambda(z)) \right)_{11}
$,
i.e.\ the Cauchy transform of $p$ can be evaluated as an element of the operator-valued Cauchy transform of $\hat{p}$ for a particular choice of argument.
\end{proposition}

\ref{prop:cauchy-lin} holds simply because
\begin{align}
\left(\Lambda(z) - \hat{p} \right)^{-1} = \begin{pmatrix}
z 1 & \aug & -u\\
\hline
-v & \aug & -q
\end{pmatrix}^{-1} = \begin{pmatrix}
(z + uq^{-1}v)^{-1} & \aug & *\\
\hline
* & \aug & *
\end{pmatrix} = \begin{pmatrix}
(z - p)^{-1} & \aug & *\\
\hline
* & \aug & *
\end{pmatrix}\,,
\end{align}
by the construction of the linearization, and the operator-valued Cauchy transform acts as $\left( G_{\hat{p}}(Z) \right)_{ij} = \varphi \left((Z - \hat{p})^{-1}_{ij} \right)$.\\

 A linearization always exists, as the constructive algorithm in the next subsection shows, but linearizations are not unique.
 
 
\subsubsection{Linearization algorithm for polynomials}
\label{sec:linearize-polynomial}

The original publication for this approach is~\cite{belinschi-mai-speicher:2017}. Consider $p \in \CC \langle x_1, \dots, x_n \rangle$, i.e.\ a polynomial of $n$ non-commutative variables $x_1, \dots, x_n \in {\cal C}$ over $\CC$. Here, we summarize an algorithm to find a linearization $\hat{p}$ of $p$, i.e.\ a matrix
$
\hat{p} = \begin{pmatrix}
0 & \aug & u\\
\hline
v & \aug & q
\end{pmatrix} \text{ with } p = -uq^{-1}v
$
and $\hat{p}$ only affine-linear in all $x_i$'s. The following steps can be used to linearize any polynomial:
\begin{enumerate}
\item The degree 1 monomial $x_j$ is obviously linearized by
$
x_j \xrightarrow{\text{lin}} \begin{pmatrix}
0 & \aug & x_j\\
\hline
1 & \aug & -1
\end{pmatrix} \in M_2({\cal C})
$
\item The degree $k \geq 2$ monomial $x_{i_1} x_{i_2} \dots x_{i_k}$ is linearized as
\begin{align}
x_{i_1} x_{i_2} \dots x_{i_k}\xrightarrow{\text{lin}}
\begin{pmatrix} 
0 & \aug & 0 & 0 & \dots & 0 & x_{i_1}\\
\hline
0 & \aug & 0 & 0 & \dots & x_{i_2} & -1\\
0 & \aug & 0 & 0 &  \dots & -1 & 0\\
\dots  & \aug & \dots& \dots& \dots & \dots & \dots\\
0 & \aug & x_{i_{k-1}} & -1 & \dots & 0 & 0 \\
x_{i_{k}} & \aug & -1 & 0 & \dots & 0 & 0
\end{pmatrix} \in M_k({\cal C})\,.
\end{align}
One can check the above via induction; explicitly, we have for $k=2$ and $k = 3$ that
\begin{align*}
&k=2\,, \quad \begin{pmatrix}
0 & \aug & x_{i_1}\\
\hline
x_{i_2} & \aug & -1
\end{pmatrix} \quad  \rightarrow \quad -uq^{-1}v = -x_{i_1} (-1) x_{i_2} = x_{i_1} x_{i_2} \checkmark\\
&k=3\,, \quad \begin{pmatrix}
0 & \aug & 0 & x_{i_1}\\
\hline
0 & \aug & x_{i_2} & -1\\
x_{i_3} & \aug & -1 & 0
\end{pmatrix}\quad  \rightarrow \quad -uq^{-1}v = -\begin{pmatrix}
0 & x_{i_1}
\end{pmatrix} \begin{pmatrix}
x_{i_2} & -1\\
-1 & 0
\end{pmatrix}^{-1} \begin{pmatrix}
0\\
x_{i_3}
\end{pmatrix} = -\begin{pmatrix}
0 & x_{i_1}
\end{pmatrix} \begin{pmatrix}
* & *\\
* & -x_{i_2}
\end{pmatrix} \begin{pmatrix}
0\\
x_{i_3}
\end{pmatrix} \checkmark
\end{align*}
\item If we have a sum $p = p_1 + \dots + p_k$ with known linearizations
$
\hat{p}_j = \begin{pmatrix}
0 & \aug & u_j\\
\hline
v_j & \aug & q_j
\end{pmatrix} \in M_{d_j}({\cal C})\,,
$
then their sum can be linearized by simply stacking
\begin{align}
\hat{p} = \begin{pmatrix}
0 & \aug & u_1 & u_2 & \dots & u_k\\
\hline
v_1 & \aug & q_1 & 0 & \dots & 0\\
v_2 & \aug & 0 & q_2 & \dots & 0\\
\dots & \aug  & \dots& \dots& \dots & \dots \\
v_k & \aug & 0 & 0 & \dots & q_k
\end{pmatrix} \in M_{d_1 + \dots + d_k - k + 1}({\cal C})\,,
\end{align}
because
\begin{align}
-uq^{-1}v = -\begin{pmatrix}
u_1 & u_2 & \dots & u_k
\end{pmatrix} \text{diag}\left(q_1^{-1},q_2^{-1},\dots,q_k^{-1} \right)
\begin{pmatrix}
v_1\\
v_2\\
\dots\\
v_k
\end{pmatrix} = p \; \checkmark
\end{align}
\item For manifestly symmetric linearizations: suppose $p$ is linearized by
$
\hat{p} = \begin{pmatrix}
0 & \aug & u\\
\hline
v & \aug & q
\end{pmatrix} \in M_d({\cal C})\,,
$
then, clearly, $p^*$ is linearized by
$
\hat{p}^* = \begin{pmatrix}
0 & \aug & v^*\\
\hline
u^* & \aug & q^*
\end{pmatrix} \in M_d({\cal C})
$, 
and their sum $p + p^*$, which is symmetric, has a symmetric linearization
\begin{align}
\begin{pmatrix}
0 & \aug & u & v^*\\
\hline
u^* & \aug & 0 & q^*\\
v & \aug & q & 0
\end{pmatrix} \in M_{2d-1}({\cal C})\,,
\end{align}
since
\begin{align}
-\begin{pmatrix}
u & v^*
\end{pmatrix}
\begin{pmatrix}
 0 & q^*\\
 q & 0
\end{pmatrix}^{-1}
\begin{pmatrix}
u^*\\
v
\end{pmatrix} = -\begin{pmatrix}
u & v^*
\end{pmatrix}
\begin{pmatrix}
 0 & q^{-1}\\
 (q^*)^{-1} & 0
\end{pmatrix}
\begin{pmatrix}
u^*\\
v
\end{pmatrix} = - u q^{-1} v - v^*  (q^*)^{-1} u^* = p + p^* \; \checkmark
\end{align}
\end{enumerate}


\subsubsection{Linearization algorithm for rational functions}
\label{sec:linearize-rational}

The original publication for this section is~\cite{helton-mai-speicher:2018}. For rational functions $r$ of non-commutative variables $x_1, \dots, x_n \in {\cal C}$, a slightly different definition of linearization is used. The main difference is that the vectors $u,v$ in the linearization must be constants here, independent of the $x_i$'s. This requirement is so that the product linearization below remains a valid linearization since otherwise $v_1 u_2$ may be polynomial in the $x_i$'s.

\begin{definition}
Given a rational function $r$ of $ x_1, \dots, x_n $ in $n$ non-commutative variables $x_1, \dots, x_n$ in a unital algebra~${\cal C}$, a $d \times d$ matrix with polynomial elements $ \hat{r} \in M_d(\CC) \otimes \CC \langle x_1, \dots, x_n \rangle$ is called a linearization of $r$ if
\begin{align}
\hat{r} = \begin{pmatrix}
0 & \aug & u\\
\hline
v & \aug & q
\end{pmatrix} \in M_d({\cal C}) \text{ with } u \in  {\cal C}^{1 \times (d-1)}\,, \quad v \in {\cal C}^{(d-1) \times 1}\,, \quad q \in {\cal C}^{(d-1) \times (d-1)}\,,
\end{align}
such that
\begin{enumerate}
\item $q$ is invertible, and $r = - u q^{-1} v$ 
\item $\hat{r} = b_0 \otimes 1_{{\cal C}} + b_1 \otimes x_1 + \dots + b_n \otimes x_n$ for some coefficient matrices $b_i \in M_d(\CC)$ such that $u,v$ are only constructed from the $b_0$ term and independent of $x_1, \dots, x_n$ 
\end{enumerate}
\end{definition}

This slightly modified definition suggests that we have to change certain steps in the algorithm of the previous subsection. Now, we do the following:
\begin{enumerate}
\item $\lambda \in \CC$ or $x_i \in {\cal C}$ are both linearized as
$
\lambda \xrightarrow{\text{lin}} \begin{pmatrix}
0 & \aug & 0 & 1\\
\hline
0 & \aug & \lambda & -1\\
1 & \aug & -1 & 0
\end{pmatrix}  \in M_{3}({\cal C})\,.
$
\item If two rational functions $r_1,r_2$ are linearized by
$
r_i \xrightarrow{\text{lin}} \begin{pmatrix}
0 & \aug & u_i\\
\hline
v_i & \aug & q_i
\end{pmatrix} \in M_{d_i}({\cal C})\,,
$
we still take the linearization of their sum to be
\begin{align}
r_1 + r_2  \xrightarrow{\text{lin}} \begin{pmatrix}
0 & \aug & u_1 & u_2\\
\hline
v_1 & \aug & q_1 & 0\\
v_2 & \aug & 0 & q_2
\end{pmatrix} \in M_{d_1 + d_2 - 1}({\cal C})\,,
\end{align}
but for their product, we use
\begin{align}
r_1 r_2  \xrightarrow{\text{lin}} \begin{pmatrix}
0 & \aug & 0 & u_1\\
\hline
0 & \aug & v_1 u_2 & q_1\\
v_2 & \aug & q_2 & 0
\end{pmatrix} \in M_{d_1 + d_2 - 1}({\cal C}) \,.
\end{align}
\item If $r$ is linearized by $\begin{pmatrix}
0 & \aug & u\\
\hline
v & \aug & q
\end{pmatrix} \in M_{d}({\cal C})$ and invertible, its inverse is linearized by
$
r^{-1} \xrightarrow{\text{lin}} \begin{pmatrix}
0 & \aug & 1 & 0\\
\hline
1 & \aug & 0 & u\\
0 & \aug & v & -q
\end{pmatrix} \in M_{d+1}({\cal C}) \,.
$
\end{enumerate}


\subsubsection{Toy examples of linearizations}
\label{app:linearization-examples}

\begin{example}
\label{ex:lin-poly}
(cf.~\citep[Example 5.2]{belinschi-mai-speicher:2017})
Consider the symmetric polynomial
$
p = p(x,y) = xy+yx + x^2
$
in two free self-adjoint elements $x, y$. Assume that we know the Cauchy transform and spectral density of $x$ and $y$ individually, say with $x$ a semicircle element and $y$ a Marchenko--Pastur element. We want to compute the Cauchy transform of $p$, and potentially its spectral density. To linearize the polynomial $p$ and keep the block-dimension $d$ small, we recognize that
\begin{align}
p = x \left(\frac{x}{2} + y \right) +  \left(\frac{x}{2} + y \right) x = \tilde{p} + \tilde{p}^*\,.
\end{align}
The first term can be linearized as
\begin{align}
\tilde{p} = x \left(\frac{x}{2} + y \right) \xrightarrow{\text{lin}}\begin{pmatrix}
0 & \aug & x\\
\hline
\frac{x}{2} + y & \aug &  -1
\end{pmatrix}\,,
\end{align}
according to the rule 2.\ for products. But then the rule 4.\ for symmetric sums gives
\begin{align}
p \xrightarrow{\text{lin}} \hat{p} = \begin{pmatrix}
0 & \aug & x & \frac{x}{2} + y\\
\hline
x & \aug & 0 & -1\\
\frac{x}{2} + y & \aug & -1 & 0
\end{pmatrix} = A \otimes x + B_0 \otimes 1 + B_1 \otimes y\,,
\end{align}
for
\begin{align}
A = \begin{pmatrix}
0 & 1 & \frac{1}{2}\\
1 & 0 & 0\\
\frac{1}{2} & 0 & 0
\end{pmatrix}\,, \quad B_0 = \begin{pmatrix}
0 & 0 & 0\\
0 & 0 & -1\\
0 & -1 & 0
\end{pmatrix}\,, \quad B_1 = \begin{pmatrix}
0 & 0 & 1\\
0 & 0 & 0\\
1 & 0 & 0
\end{pmatrix}\,.
\end{align}
So, with this way of rewriting the polynomial, $d = 3$ suffices as a lift dimension to linearize $p$. In accordance with the general linearization theory, we can (at least numerically) evaluate the lifted Cauchy transforms as
\begin{align}
G_{A \otimes x}(Z) = \int_\RR \left(Z - \lambda A \right)^{-1} \dd \mu_x(\lambda)\,, \quad G_{B_0 + B_1 \otimes y}(Z) = \int_\RR \left(Z - B_0  - \lambda B_1 \right)^{-1} \dd \mu_y(\lambda)\,,
\end{align}
for any $Z \in M_3(\CC)$. Since the lifted variables remain free, we can then use the subordination result 
$
G_{\hat{p}}(Z) = G_{A \otimes x}(\omega_{A \otimes x}(Z))
$
with fixed-point equation
$
\omega_{A \otimes x}(Z) = H_{B_0 + B_1 \otimes y}(H_{A \otimes x}(\omega_{A \otimes x}(Z)) + Z ) + Z\,,
$
to get the operator-valued Cauchy transform of $\hat{p}$ at any $Z \in \CC^{3 \times 3}$. Then, the Cauchy transform of the polynomial $p$ itself is
$
G_p(z) = G_{\hat{p}} \left(\begin{pmatrix}
z & 0 & 0\\
0 & 0 & 0\\
0 & 0 & 0 
\end{pmatrix} \right)_{11}
$
according to proposition~\ref{prop:cauchy-lin}.
Strictly speaking, we should perturb the diagonal by $i \eps$ for $\eps$ small and positive to stay within the upper half space and compute $G_{\hat{p}} \left(\text{diag}\left(z,i \eps,i \eps \right) \right)_{11}$
instead, assuming $z$ is already in the upper complex half plane. With these tools, we can then compute
$
G_p(z) = \varphi\left( \left(z - p \right)^{-1} \right)
$
for any $z$ with positive imaginary part, and if we want, we can also compute the distribution of $p$ via~\eqref{eq:density-from-cauchy} by taking $z = x + i \eta$ for small $\eta > 0$ and $x \in \RR$. We show the results of this procedure for the present example in Figure~\ref{fig:test-linearization} (left), where we compare the empirical spectral density of $p$ to the result of the linearization procedure and computation of the Cauchy transform of $p$ close to the real axis.\\
\end{example}


\begin{example}
\label{ex:lin-ratio}
We repeat the exercise of the previous example but now for a rational function $r$ of two free and invertible elements. We consider
$
r = r(x,y) = \left(x^{-1} + y^{-1} \right)^{-1}\,.
$
Again, assume that we know the Cauchy transform and spectral density of $x$ and $y$ individually, say with $x, y$ both Marchenko--Pastur elements with different parameters $c_x, c_y < 1$ such that their densities have no atoms at $0$ and $x,y$ are invertible. We want to compute the Cauchy transform of $r$, and potentially its spectral density. This time, we strictly follow the general linearization rules for rational functions:
\begin{align}
x \xrightarrow{\text{lin}} \begin{pmatrix}
0 & \aug & 0 & 1\\
\hline
0 & \aug & x & -1\\
1 & \aug & -1 & 0
\end{pmatrix}\,, \quad x^{-1} \xrightarrow{\text{lin}} \begin{pmatrix}
0 & \aug & 1 & 0 & 0\\
\hline
1 & \aug & 0 & 0 & 1\\
0 & \aug & 0 & -x & 1\\
0 & \aug & 1 & 1 & 0
\end{pmatrix}
\end{align}
The result for $y^{-1}$ is analogous, and by stacking their linearizations, we have
\begin{align}
x^{-1} + y^{-1}  \xrightarrow{\text{lin}} \begin{pmatrix}
0 & \aug & 1 & 0 & 0 & 1 & 0 & 0\\
\hline
1 & \aug & 0 & 0 & 1 & 0 & 0 & 0\\
0 & \aug & 0 &-x & 1 & 0 & 0 & 0\\
0 & \aug & 1 & 1 & 0 & 0 & 0 & 0\\
1 & \aug & 0 & 0 & 0 & 0 & 0 & 1\\
0 & \aug & 0 & 0 & 0 & 0 &-y & 1\\
0 & \aug & 0 & 0 & 0 & 1 & 1 & 0
\end{pmatrix}\,,
\end{align}
and finally by linearizing the inverse
\begin{align}
r = \left(x^{-1} + y^{-1} \right)^{-1} \xrightarrow{\text{lin}} \hat{r} =
\begin{pmatrix}
0 & \aug & 1 & 0 & 0 & 0 & 0 & 0 & 0\\
\hline
1 & \aug & 0 & 1 & 0 & 0 & 1 & 0 & 0\\
0 & \aug & 1 & 0 & 0 &-1 & 0 & 0 & 0\\
0 & \aug & 0 & 0 & x &-1 & 0 & 0 & 0\\
0 & \aug & 0 &-1 &-1 & 0 & 0 & 0 & 0\\
0 & \aug & 1 & 0 & 0 & 0 & 0 & 0 &-1\\
0 & \aug & 0 & 0 & 0 & 0 & 0 & y &-1\\
0 & \aug & 0 & 0 & 0 & 0 &-1 &-1 & 0
\end{pmatrix} = \begin{pmatrix}
0 & \aug & u\\
\hline
v & \aug & q
\end{pmatrix} = A \otimes x + B_0 \otimes 1 + B_1 \otimes y\,.
\end{align}
So, the naive application of the algorithm lifts to $8 \times 8$ block matrices to linearize the problem, such that $r  = -u q^{-1} v$, and $\left (z e_1^{\otimes 2} - \hat{r} \right)^{-1}_{11} = (z-r)^{-1}$ by construction. The coefficient matrices are
\begin{align*}
A = \begin{pmatrix}
0 & 0 & 0 & 0 & 0 & 0 & 0 & 0\\
0 & 0 & 0 & 0 & 0 & 0 & 0 & 0\\
0 & 0 & 0 & 0 & 0 & 0 & 0 & 0\\
0 & 0 & 0 & 1 & 0 & 0 & 0 & 0\\
0 & 0 & 0 & 0 & 0 & 0 & 0 & 0\\
0 & 0 & 0 & 0 & 0 & 0 & 0 & 0\\
0 & 0 & 0 & 0 & 0 & 0 & 0 & 0\\
0 & 0 & 0 & 0 & 0 & 0 & 0 & 0
\end{pmatrix}, \quad B_0 = \begin{pmatrix}
0 & 1 & 0 & 0 & 0 & 0 & 0 & 0\\
1 & 0 & 1 & 0 & 0 & 1 & 0 & 0\\
0 & 1 & 0 & 0 &-1 & 0 & 0 & 0\\
0 & 0 & 0 & 0 &-1 & 0 & 0 & 0\\
0 & 0 &-1 &-1 & 0 & 0 & 0 & 0\\
0 & 1 & 0 & 0 & 0 & 0 & 0 &-1\\
0 & 0 & 0 & 0 & 0 & 0 & 0 &-1\\
0 & 0 & 0 & 0 & 0 &-1 &-1 & 0
\end{pmatrix}, \quad  B_1 = \begin{pmatrix}
0 & 0 & 0 & 0 & 0 & 0 & 0 & 0\\
0 & 0 & 0 & 0 & 0 & 0 & 0 & 0\\
0 & 0 & 0 & 0 & 0 & 0 & 0 & 0\\
0 & 0 & 0 & 0 & 0 & 0 & 0 & 0\\
0 & 0 & 0 & 0 & 0 & 0 & 0 & 0\\
0 & 0 & 0 & 0 & 0 & 0 & 0 & 0\\
0 & 0 & 0 & 0 & 0 & 0 & 1 & 0\\
0 & 0 & 0 & 0 & 0 & 0 & 0 & 0
\end{pmatrix}\,.
\end{align*}
The rest of the computation of $G_r(z)$ or $(\dd \mu_r / \dd x)(x)$ from this linearization remains formally unchanged from before. Numerical results, comparing the empirical spectral density of $r$ from a sample with the theoretically expected one from the operator-valued Cauchy transform of $\hat{r}$ are shown in Figure~\ref{fig:test-linearization} (center).\\
\end{example}


\begin{example}
\label{ex:lin-trace}
Finally, we turn to the computation of a trace that is of a similar type to what we care about in the Sobolev training setting. Take the same polynomial as in Example~\ref{ex:lin-poly}, that is
$
p = p(x,y) = xy+yx + x^2\,,
$
but now we specifically want to compute
\begin{align}
f(z_0) := \varphi \left( \left(z_0 1 - p \right)^{-1} y \right) =  \varphi \left( \left(z_0 1 - \left(xy+yx + x^2\right) \right)^{-1} y \right)\,.
\end{align}
This expression does not immediately look like a Cauchy transform. Assuming $y$ to be invertible, we use the following trick:
\begin{align}
-f(z_0) =  \varphi \left( \left( 0 - y^{-1} \left(z_0 1 - \left(xy+yx + x^2\right) \right) \right)^{-1}\right) = G_{s}(z = 0)\,,
\end{align}
for the Cauchy transform of the rational function
\begin{align}
s = s(x,y) = y^{-1} \left( z_0 1 - \left(xy+yx + x^2\right) \right)\,,
\end{align}
which is not manifestly symmetric here. We can linearize $s$ as before:
\begin{align}
 z_0 1 \xrightarrow{\text{lin}} \begin{pmatrix}
 0 & \aug & 0 & 1\\
 \hline
 0 & \aug & z_0 & -1\\
 1 & \aug & -1 & 0
 \end{pmatrix}\,, xy+yx + x^2 \xrightarrow{\text{lin}} \begin{pmatrix}
0 & \aug & x & \frac{x}{2} + y\\
\hline
x & \aug & 0 & -1\\
\frac{x}{2} + y & \aug & -1 & 0
\end{pmatrix}\,,\\
z_0 1 - \left(xy+yx + x^2\right) \xrightarrow{\text{lin}} \begin{pmatrix}
0 & \aug & 0 & 1 & x & \frac{x}{2} + y\\
\hline
0 & \aug &z_0&-1 & 0 & 0 \\
1 & \aug &-1 & 0 & 0 & 0 \\
x & \aug & 0 & 0 & 0 & 1 \\
\frac{x}{2} + y & \aug & 0 & 0 & 1 & 0
\end{pmatrix}\,, y^{-1} \xrightarrow{\text{lin}} \begin{pmatrix}
0 & \aug & 1 & 0 & 0\\
\hline
1 & \aug & 0 & 0 & 1\\
0 & \aug & 0 & -y & 1\\
0 & \aug & 1 & 1 & 0
\end{pmatrix}\,,
\end{align}
so that finally their product is linearized as
\begin{align}
s \xrightarrow{\text{lin}} \hat{s} = \begin{pmatrix}
0 & \aug & 0 & 0 & 0 & 0 & 1 & 0 & 0\\
\hline
0 & \aug & 0 & 1 & x & \frac{x}{2} + y & 0 & 0 & 1\\
0 & \aug & 0 & 0 & 0 & 0 & 0 &-y & 1\\
0 & \aug & 0 & 0 & 0 & 0 & 1 & 1 & 0\\
0 & \aug &z_0&-1 & 0 & 0 & 0 & 0 & 0\\
1 & \aug &-1 & 0 & 0 & 0 & 0 & 0 & 0\\
x & \aug & 0 & 0 & 0 & 1 & 0 & 0 & 0\\
\frac{x}{2} + y & \aug & 0 & 0 & 1 & 0 & 0 & 0 & 0
\end{pmatrix} =  A \otimes x + B_0 \otimes 1 + B_1 \otimes y\,,
\end{align}
where we again end up lifting the problem to $d = 8$. We compute $G_{\hat{s}}(Z)$ in exactly the same way as before (see Figure~\ref{fig:test-linearization} (right) for a plot of the spectral density of $s$) and specifically evaluate $f(z_0) = G_{\hat{s}}(0)_{11}$ to compute the state we were interested in.
\end{example}

\begin{figure}
\centering
\includegraphics[width = .325 \textwidth]{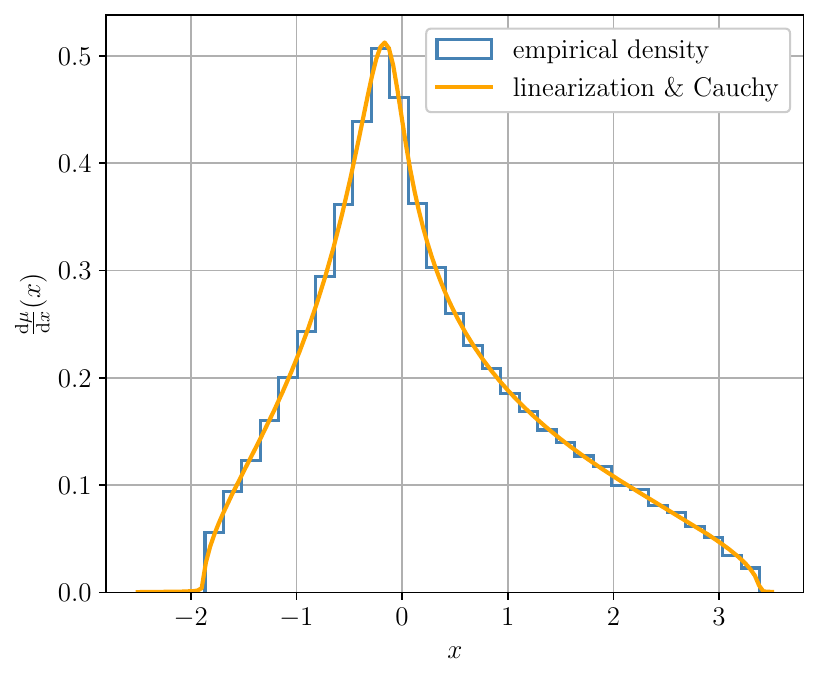}
\includegraphics[width = .325 \textwidth]{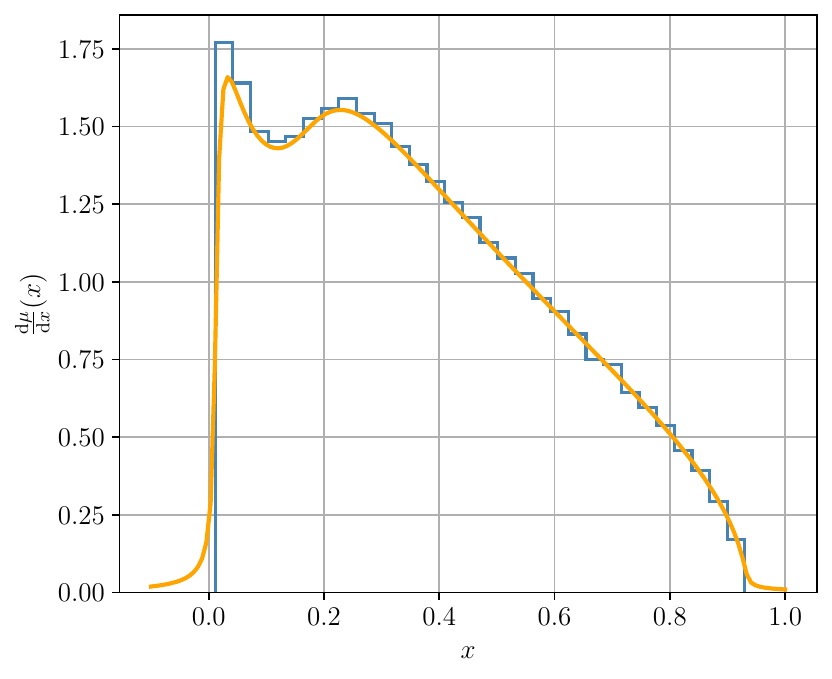}
\includegraphics[width = .325 \textwidth]{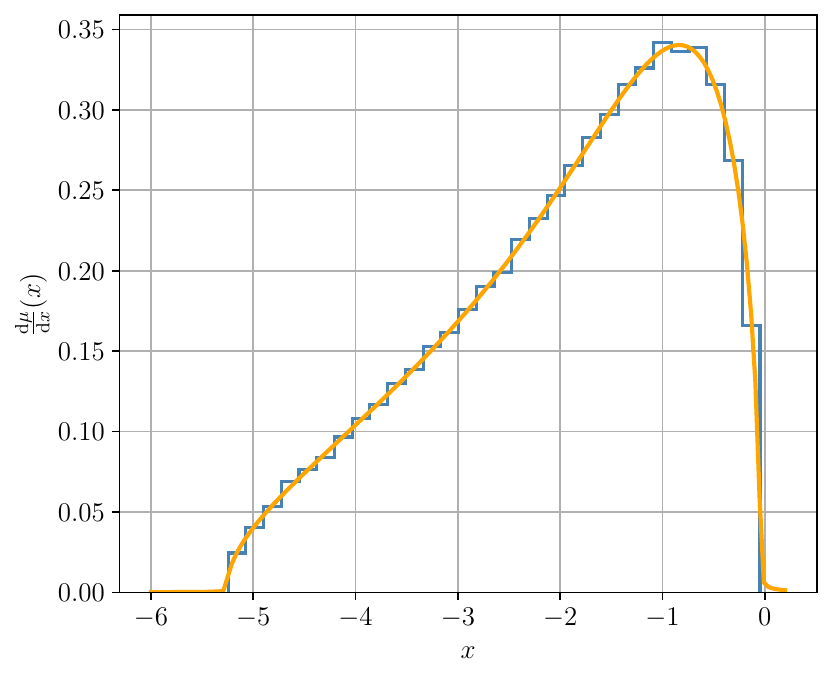}
\caption{Spectral densities, computed using the linearization approach to evaluate a Cauchy transform at $x + i\eta$ slightly above the real axis, compared to empirical spectral densities from one $4000 \times 4000$ random matrix realization each. Left: Example~\ref{ex:lin-poly}, with $p = xy+yx+x^2$ with $x$ a GOE matrix and $y$ a Wishart matrix with parameter $c = 0.3$. Center: Example~\ref{ex:lin-ratio}, for $r = \left(x^{-1} + y^{-1} \right)^{-1}$, and $x,y$ Wishart with parameters $c_x = 0.3$, $c_y = 0.8$. Right: Example~\ref{ex:lin-trace}, for $s = y^{-1} \left( z_0 1 - \left(xy+yx + x^2\right) \right)$ with $z_0 = -1.5$ for $x$ a GOE matrix and $y$ a Wishart matrix with parameter $c = 0.1$. The finite $\eta$ indeed smooths out the spectral density as expected from~\eqref{eq:dens-convolution-cauchy}, which can be seen at the edges of the support of the density. Numerical details: variance parameter $\sigma^2 = 1$ for all Wishart matrices and $\sigma^2 = 1/2$ for all GOE matrices, damped fixed-point iteration with update weight $0.2$ for subordination, 501 Gauss--Legendre quadrature points for evaluating the integrals~\eqref{eq:lifted-cauchy-int}, Cauchy transforms evaluated at $x + i \eta$ with $\eta = 0.005$.}
\label{fig:test-linearization}
\end{figure}


\section{Further numerical results}
\label{app:experiment}


\subsection{Error landscape plots and spectral densities}
\label{app:landscapes}

In this appendix, we show further expected generalization error plots as a function of $n/d$ and $p/d$, similar to Figure~\ref{fig:2d-errors-arctan-plus-reci-cosh-relu} in the main text---where $\sigma = \text{ReLU}$, $\phi = \arctan + 1 / \cosh$---for other activation functions and ridge functions. In Figure~\ref{fig:2d-errors-arctan-relu}, $\sigma = \text{ReLU}$, $\phi = \arctan$, in Figure~\ref{fig:2d-errors-reci-cosh-relu}, $\sigma = \text{ReLU}$, $\phi = 1 / \cosh$, in Figure~\ref{fig:2d-errors-arctan-erf}, $\sigma = \text{erf}$, $\phi = \arctan$, and in Figure~\ref{fig:2d-errors-reci-cosh-erf}, $\sigma = \text{erf}$, $\phi = 1 / \cosh$. Furthermore, Figure~\ref{fig:spectral-density-app} shows the spectral density of the feature matrix $K$ for additional activation functions compared to Figure~\ref{fig:spectral-density} in the main text. Noteworthy observations concerning these additional figures are summarized in Section~\ref{sec:landscapes} of the main text.

\begin{figure}
\centering
\includegraphics[width = \textwidth]{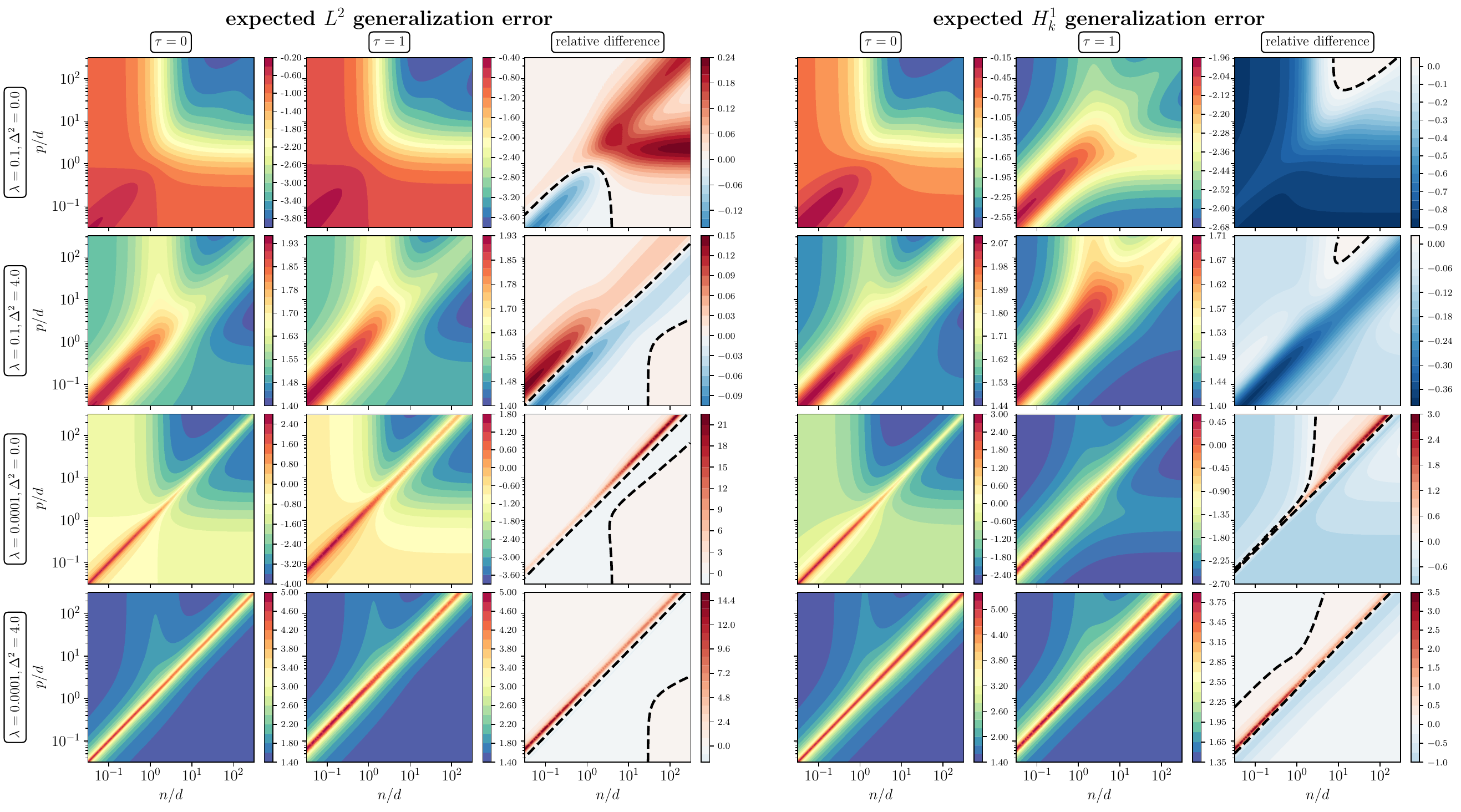}
\caption{See figure~\ref{fig:2d-errors-arctan-plus-reci-cosh-relu} in the main text for explanations; we use $\sigma = \text{ReLU}$, $\phi = \arctan$ here.}
\label{fig:2d-errors-arctan-relu}
\end{figure}

\begin{figure}
\centering
\includegraphics[width = \textwidth]{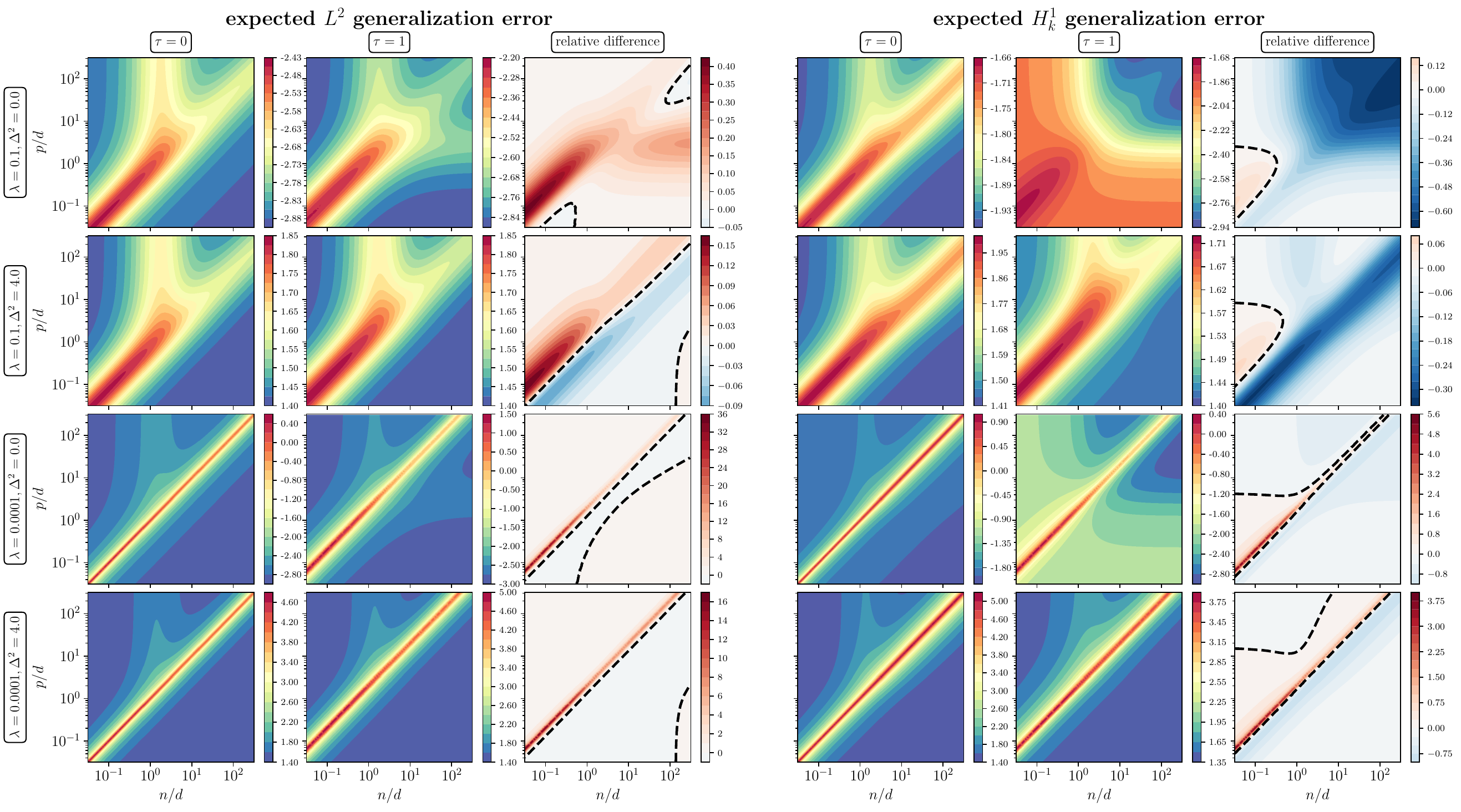}
\caption{See figure~\ref{fig:2d-errors-arctan-plus-reci-cosh-relu} in the main text for explanations; we use $\sigma = \text{ReLU}$, $\phi = 1 / \cosh$ here.}
\label{fig:2d-errors-reci-cosh-relu}
\end{figure}

\begin{figure}
\centering
\includegraphics[width = \textwidth]{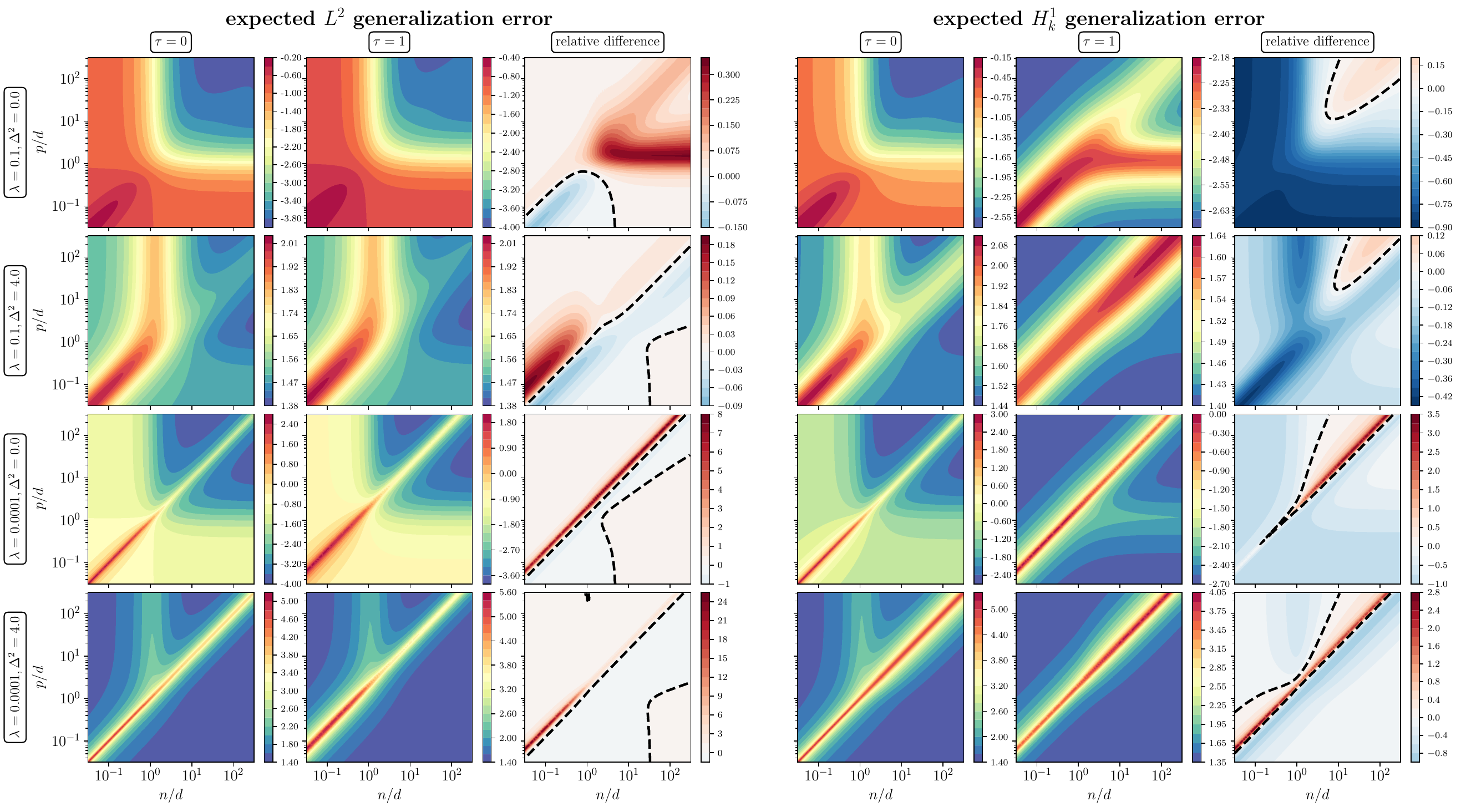}
\caption{See figure~\ref{fig:2d-errors-arctan-plus-reci-cosh-relu} in the main text for explanations; we use $\sigma = \text{erf}$, $\phi = \arctan$ here.}
\label{fig:2d-errors-arctan-erf}
\end{figure}

\begin{figure}
\centering
\includegraphics[width = \textwidth]{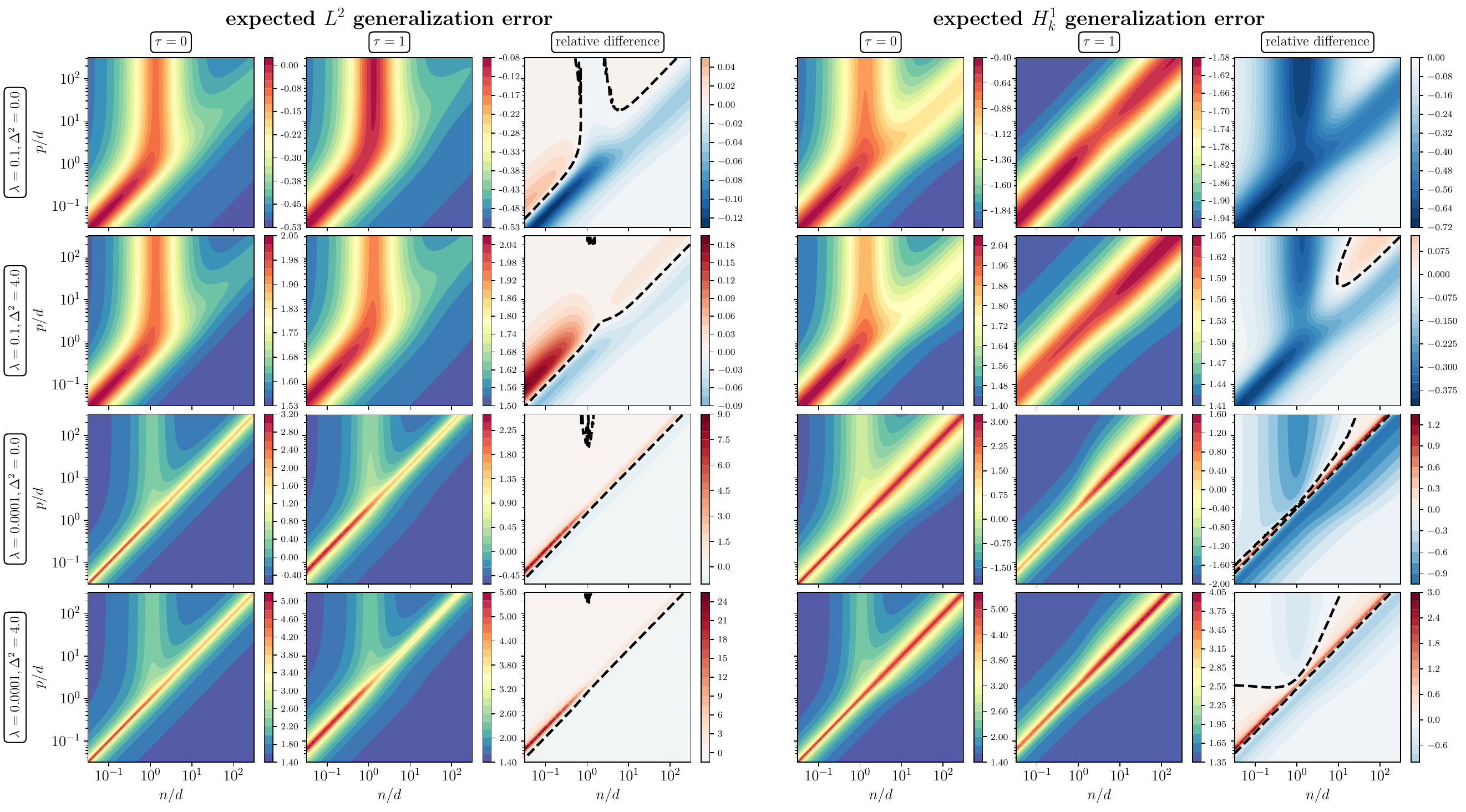}
\caption{See figure~\ref{fig:2d-errors-arctan-plus-reci-cosh-relu} in the main text for explanations; we use $\sigma = \text{erf}$, $\phi = 1 / \cosh$ here.}
\label{fig:2d-errors-reci-cosh-erf}
\end{figure}

\begin{figure}
\centering
\includegraphics[width = \textwidth]{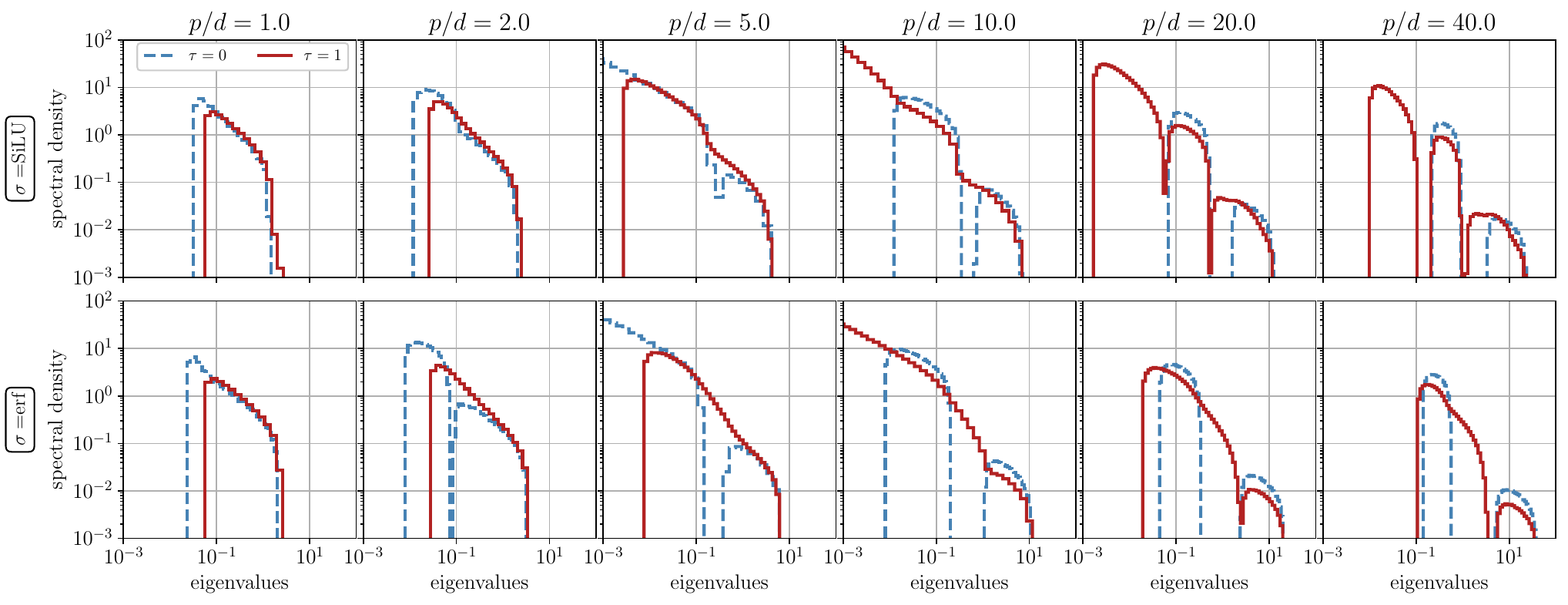}
\caption{See figure~\ref{fig:spectral-density} in the main text for explanations; we use $\sigma = \text{SiLU}$ and $\sigma = \text{erf}$ here instead of $\sigma = \text{ReLU}$.}
\label{fig:spectral-density-app}
\end{figure}


\subsection{Varying the observational noise strength}
\label{app:noise}

In Figure~\ref{fig:atan_noise-app}, we show the influence of observational noise on generalization performance, similar to Section~\ref{sec:noise-influence} of the main text, but for $\sigma = \text{erf}$, $\phi = \arctan$ here. Since both functions are odd, the first Hermite coefficient of their derivatives vanishes, so this corresponds to a setting where neither the true function gradient, nor the network gradient, depends on $x$ in the proportional asymptotics limit. Consequently, the $H^1_k$ error under Sobolev training in the bottom right of Figure~\ref{fig:atan_noise-app} is  comparatively unusual in that (i) the gradient predictions from highly underparameterized networks generalize as well as those from highly overparameterized networks, and (ii) the generalization errors as $p/n\to\infty$ saturate to the same level independently of the noise~$\Delta$.

\begin{figure}
            \centering
            \begin{tabular}{c c c}
            \includegraphics[width=0.45\linewidth]{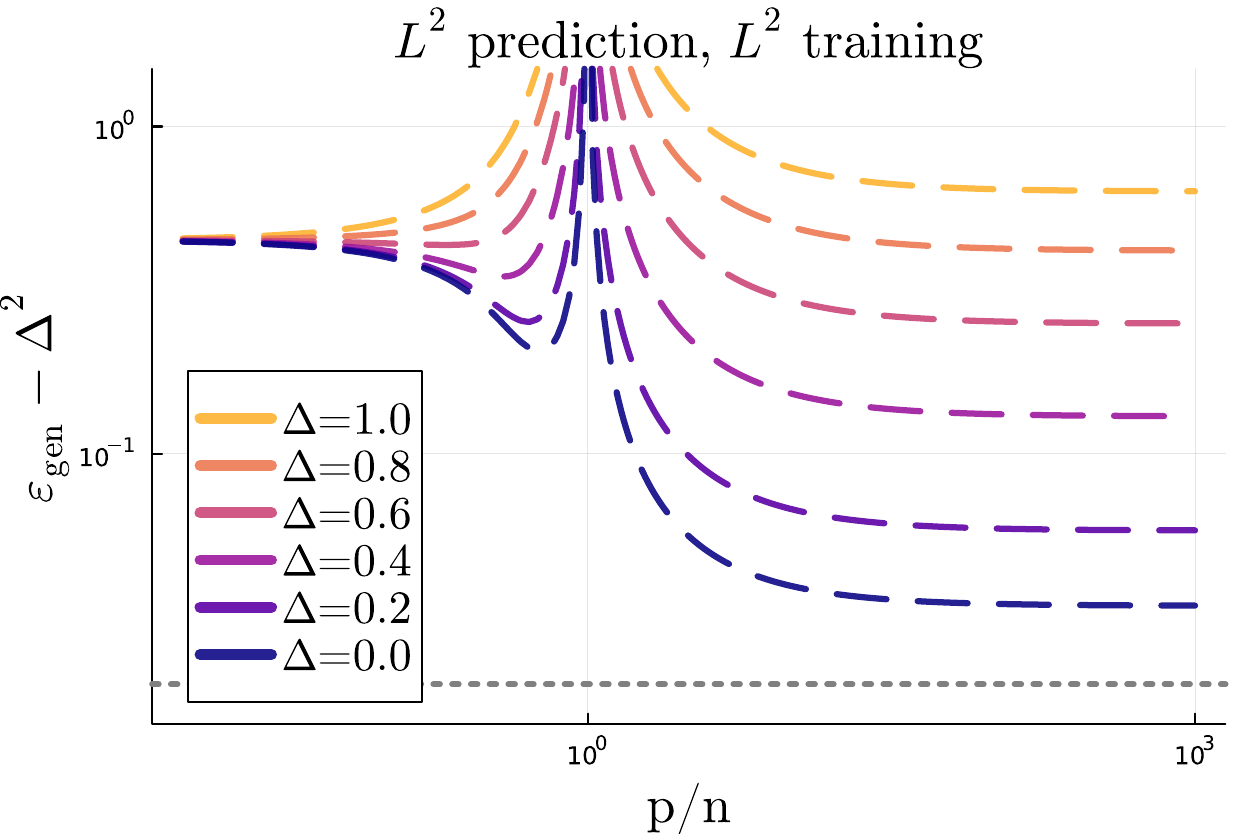}
            & & \includegraphics[width=0.45\linewidth]{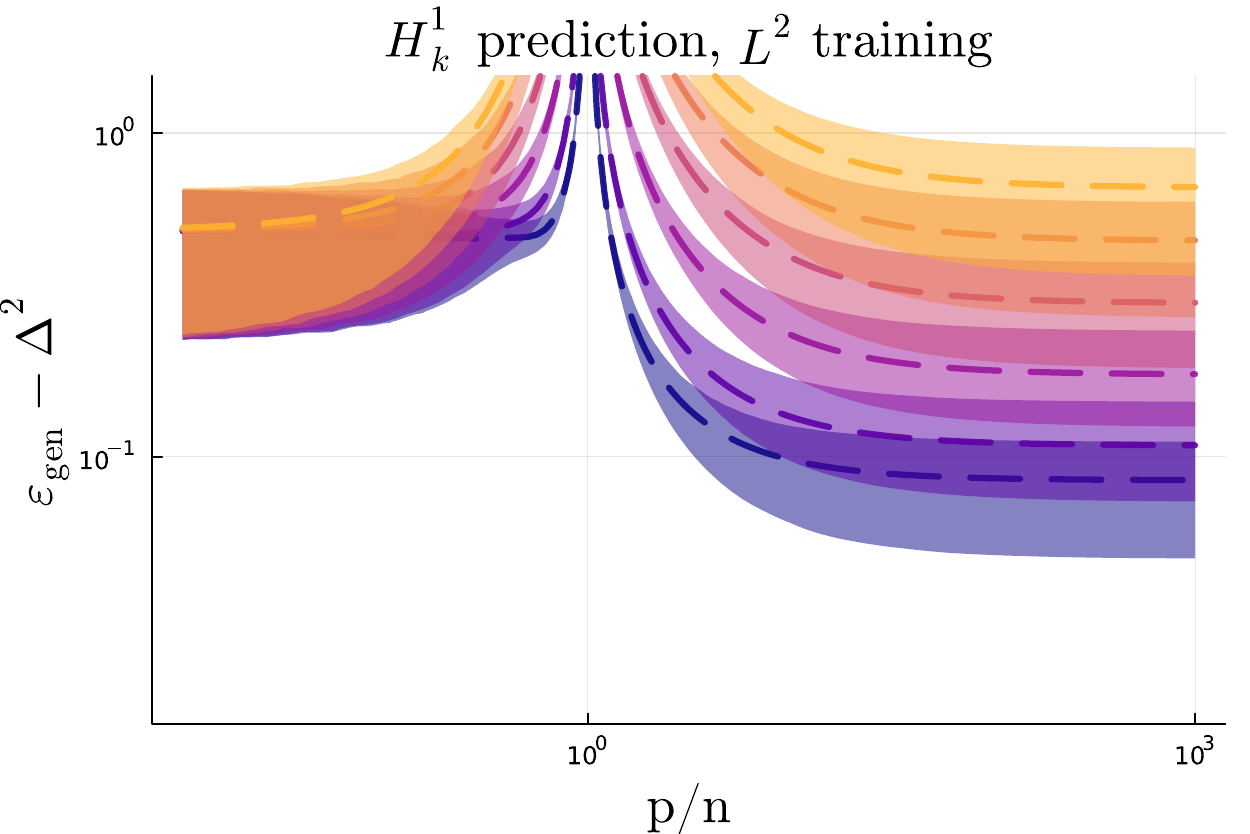} \\
            \includegraphics[width=0.45\linewidth]{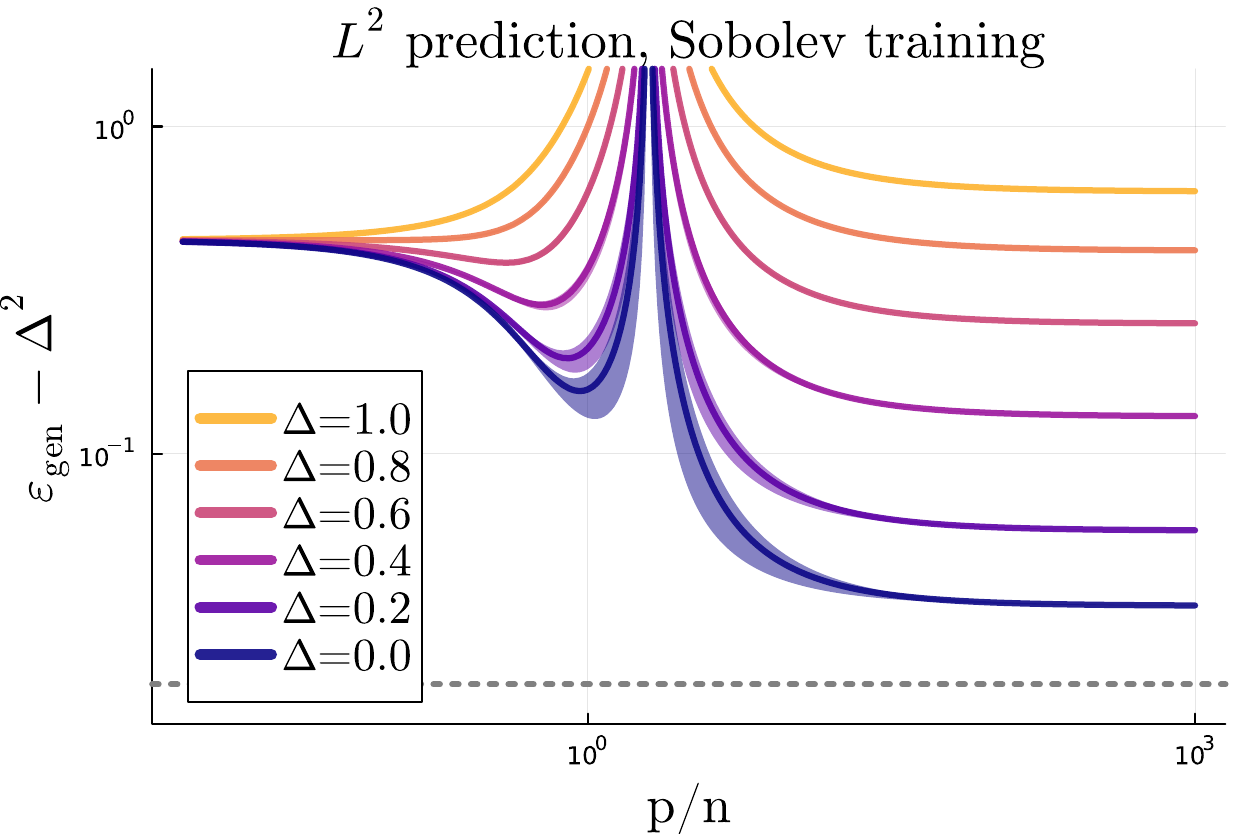}
            & & \includegraphics[width=0.45\linewidth]{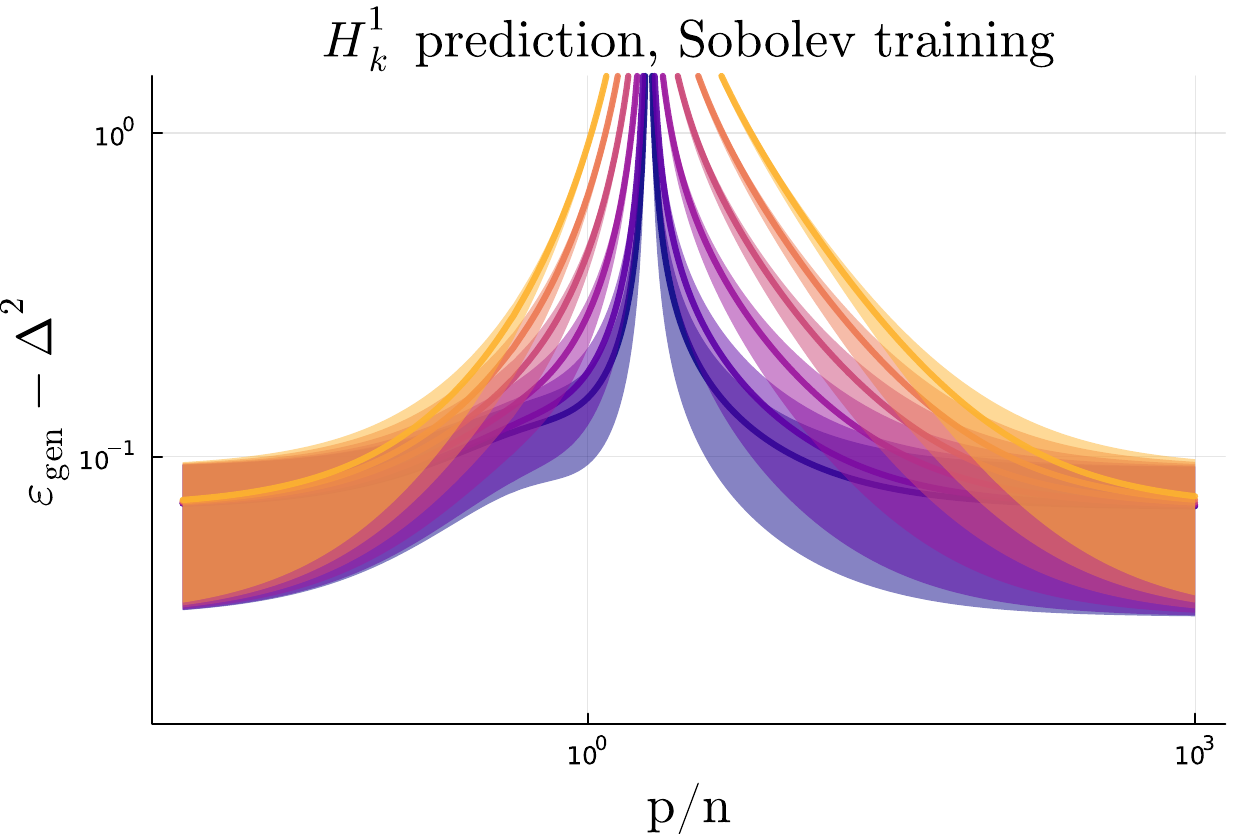} \\
            \end{tabular}
            \caption{See figure~\ref{fig:diff_noise} in the main text for explanations; we use $\sigma = \text{erf}$, $\phi = \arctan$ here.}
       \label{fig:atan_noise-app} 
       \end{figure}


\subsection{Varying $\lambda$}
\label{app:lbda}

Complementing the results shown in Figure~\ref{fig:lbda} in the main text with $\sigma = \text{ReLU}$, $\phi = \arctan + 1 / \cosh$, 
we show the effect of varying $\lambda$ for odd $\phi = \arctan$, $\sigma = \text{ReLU}$, in Figure~\ref{fig:lbda-arctan-relu}, and for even $\phi = 1 / \cosh$, $\sigma = \text{erf}$, in Figure~\ref{fig:lbda-reci-cosh-erf}. In Figure~\ref{fig:lbda-arctan-relu}, for underparameterized models---both in the high and low signal-to-noise regimes---we observe that the inclusion of gradient information uniformly improves on the gradient predictions from $L^2$ training for all $\lambda$. Past the interpolation threshold, however, incorporating gradient information becomes detrimental to predicting the teacher gradient at new inputs when regularization is small. While this degradation may be expected when there is strong noise in the data, Figure~\ref{fig:lbda-arctan-relu} demonstrates that even interpolating \emph{noiseless} gradient training data is unfavorable when compared to not having this additional information altogether. Optimal gradient prediction performance of Sobolev training in Figure~\ref{fig:lbda-arctan-relu} is achieved with~$\lambda \uparrow \infty$, meaning with optimal readout weights $w^* \approx 0$, independently of whether there is noise in the data. It is hence optimal to only learn the mean $s_b$ and to ignore all other information at large~$\lambda$. We note that optimality of large regularization has also been observed in different contexts, e.g., by~\citet{baglioni-pacelli-aiudi-etal:2024} for shallow Bayesian neural networks in the proportional asymptotics limit, and is also present already for $L^2$ training when $\phi$ is even as in Figure~\ref{fig:lbda-reci-cosh-erf}. As discussed throughout the main text, in Figure~\ref{fig:lbda-arctan-relu}, it can be traced back to $\phi'$ being even, so that the true projected gradient effectively does not depend on $x$. As a consequence, large regularization is optimal as it leads to the Sobolev-trained network gradient correctly representing the gradient mean, but none of the additional noise from the linearization of $\phi'$ (as in~\eqref{eq:get-sigma-replacement}) in the proportional asymptotics limit~\eqref{eq:prop-asymp-def}.

\begin{figure}
\centering
\includegraphics[width = .8\textwidth]{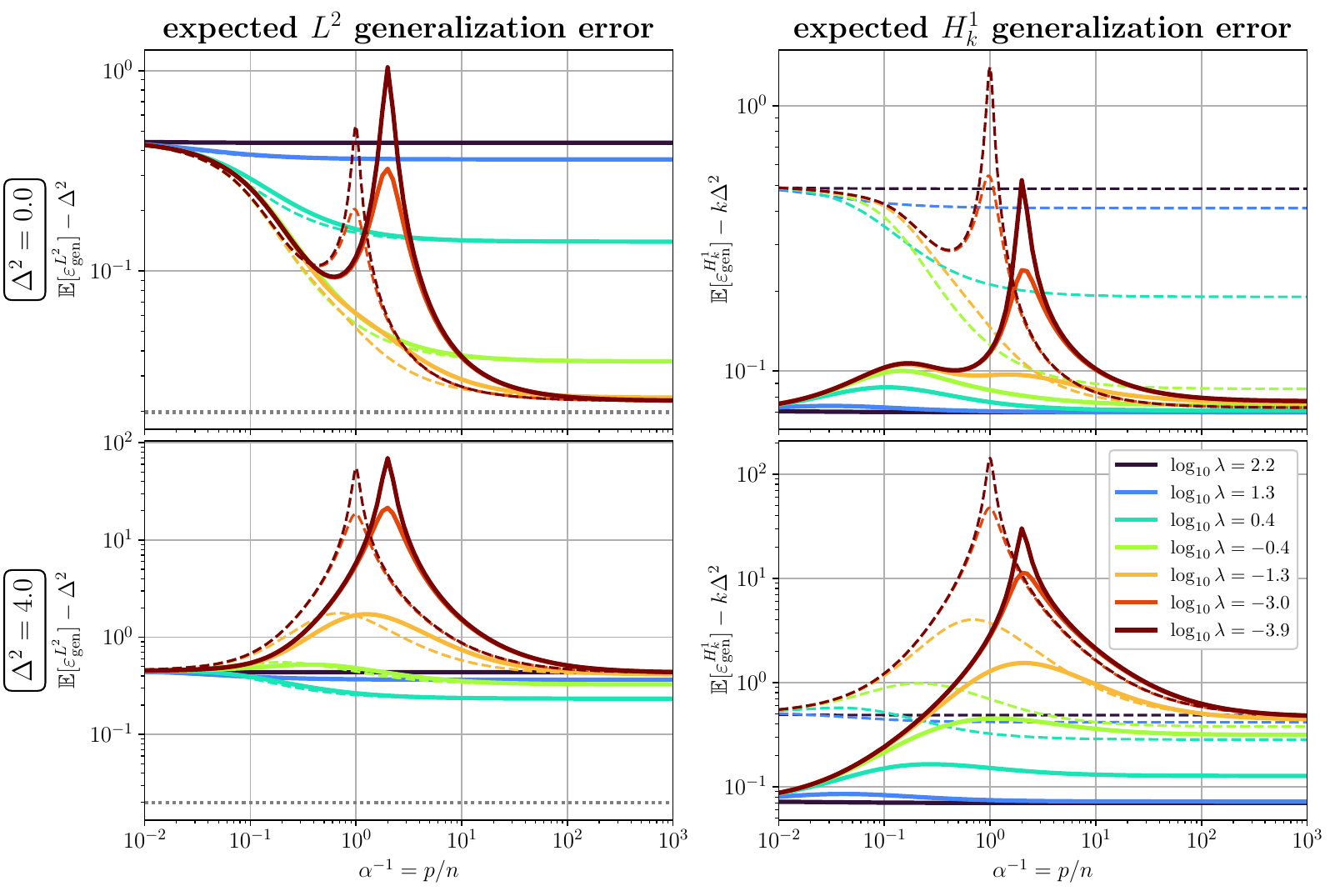}
\caption{See figure~\ref{fig:lbda} in the main text for explanations; we use $\sigma = \text{ReLU}$, $\phi = \arctan$ here.}
\label{fig:lbda-arctan-relu}
\end{figure}

\begin{figure}
\centering
\includegraphics[width = .8\textwidth]{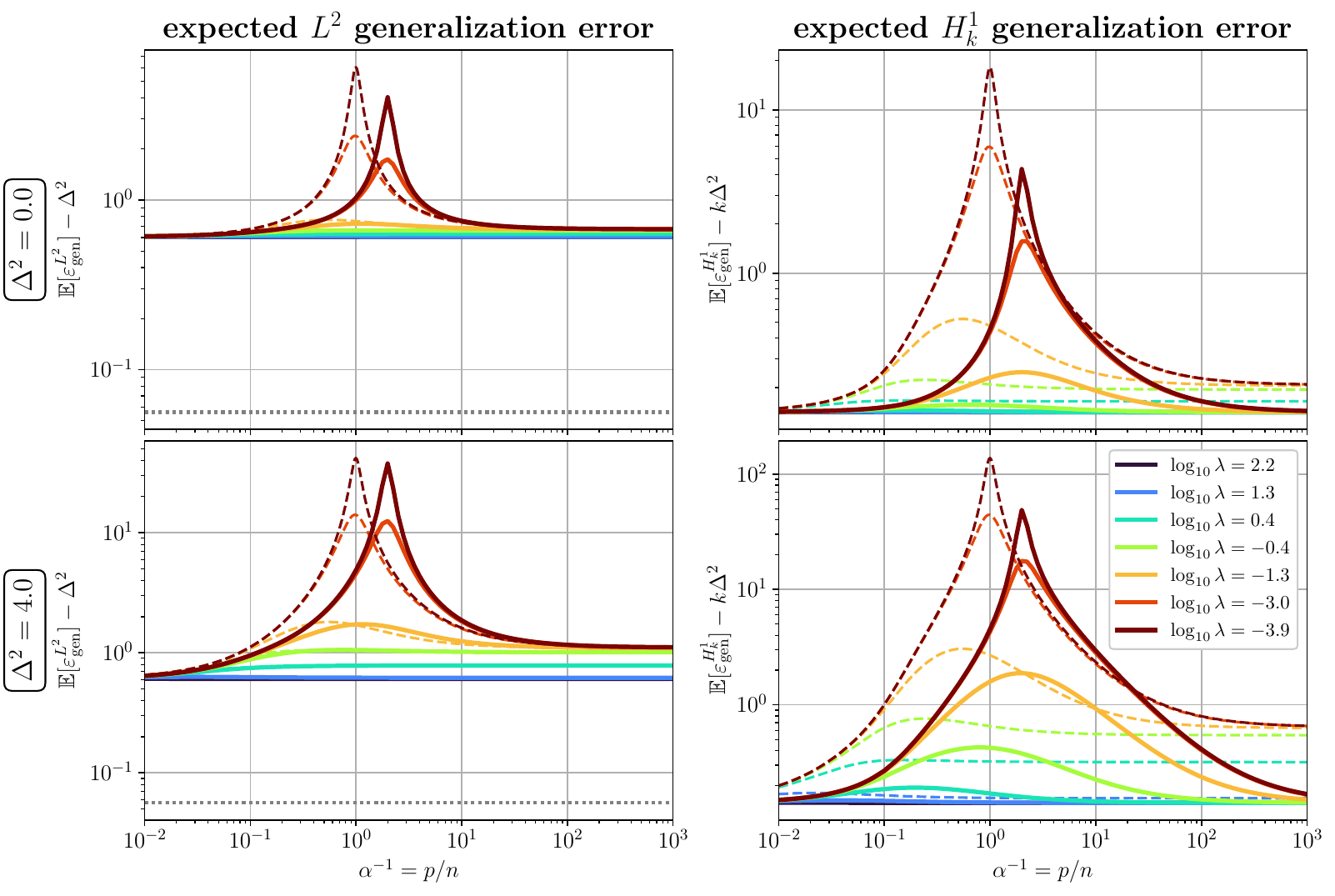}
\caption{See figure~\ref{fig:lbda} in the main text for explanations; we use $\sigma = \text{erf}$, $\phi = 1  /\cosh$ here.}
\label{fig:lbda-reci-cosh-erf}
\end{figure}


\subsection{Gradient cost model comparison}
\label{app:cost}

\begin{figure}
            \centering
            \begin{tabular}{c c c}
            \includegraphics[width=0.45\linewidth]{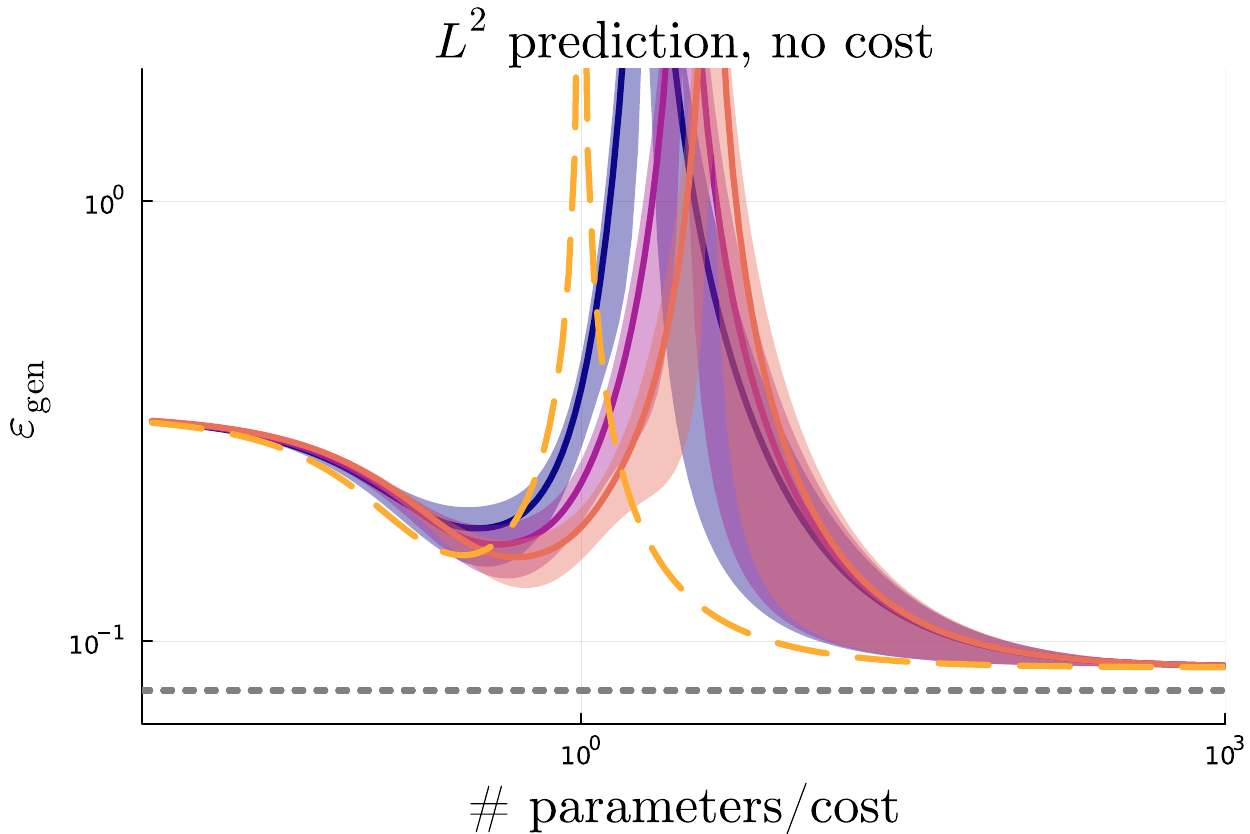}
            & & \includegraphics[width=0.45\linewidth]{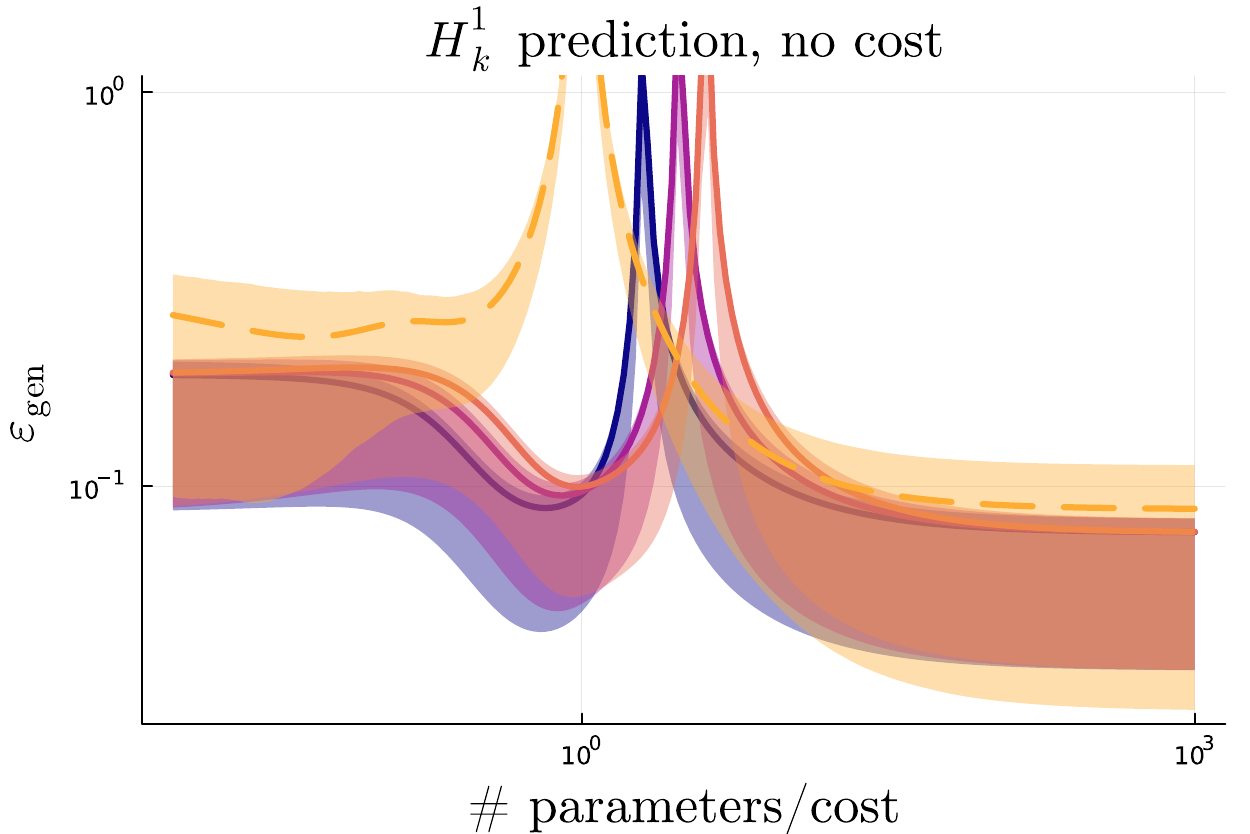} \\
            \includegraphics[width=0.45\linewidth]{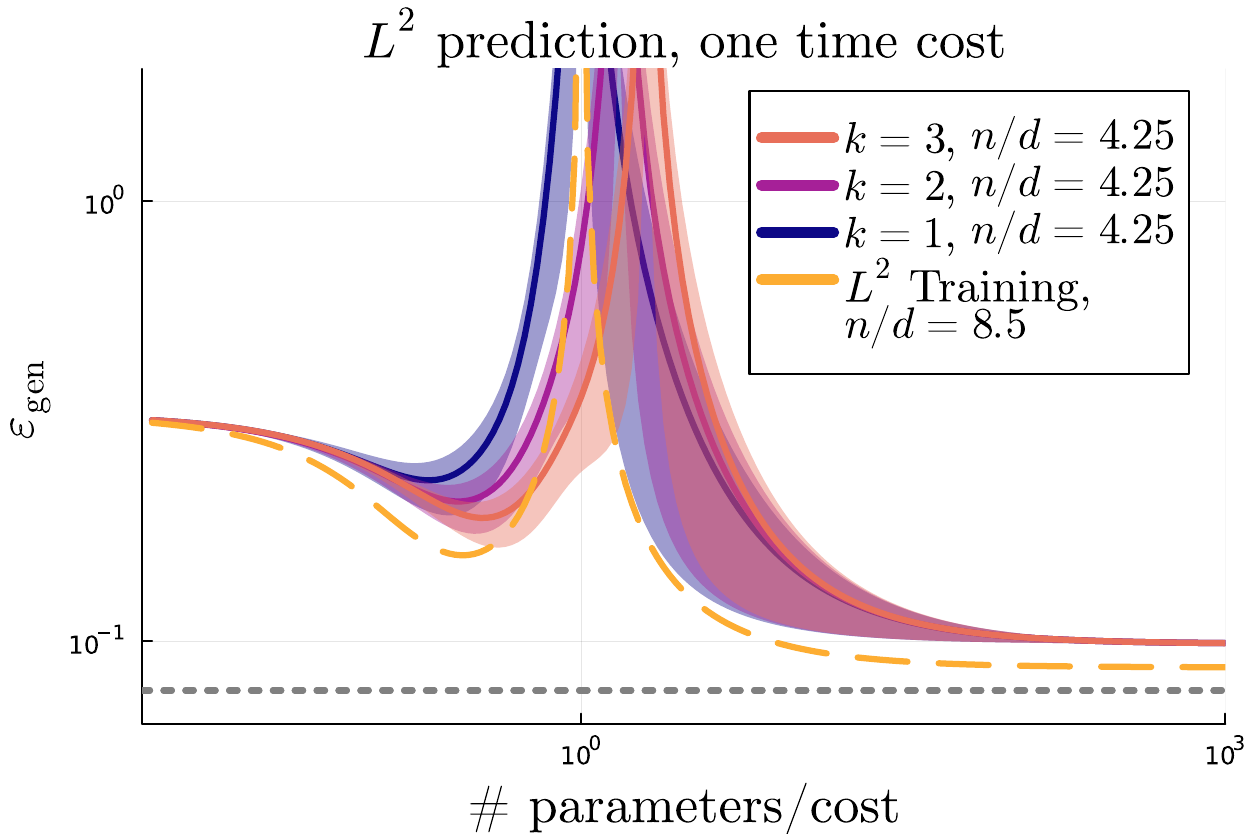}
            & & \includegraphics[width=0.45\linewidth]{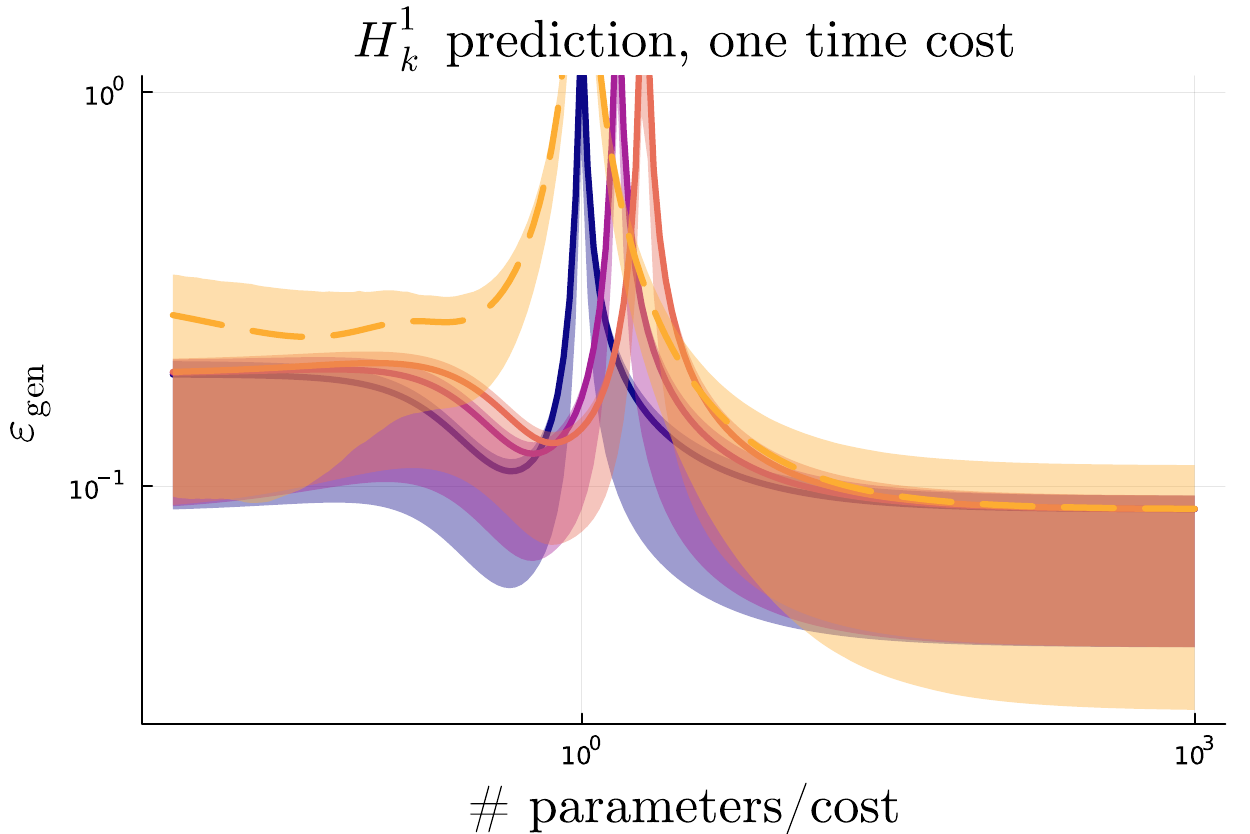} \\
            \end{tabular}
            \caption{See Figure~\ref{fig:k_effect} in the main text for explanations. As in Figure~\ref{fig:k_effect}, we use $\sigma = \text{SiLU}$, $\phi(\omega) = \omega / 2 - \exp \{-\omega^2 / 2\}$ here. First row: case where the cost of obtaining gradients is negligible next to the cost of observations. Consequently, for all curves in these subfigures, $n/d=8.5$. Second row: computing all gradients incurs a one time cost equivalent to the cost of computing observations. To capture the disparity in training settings while keeping $d$ constant, we plot $L^2$ training results for $n/d=8.5$ and Sobolev training results for $n/d=4.25$. For the case where each projected dimension of the gradients costs as much as the function observation, see Figure~\ref{fig:k_effect}.
}
\label{fig:k_effect_app}
\end{figure}

\begin{figure}
            \centering
            \begin{tabular}{c c c}\includegraphics[width=0.45\linewidth]{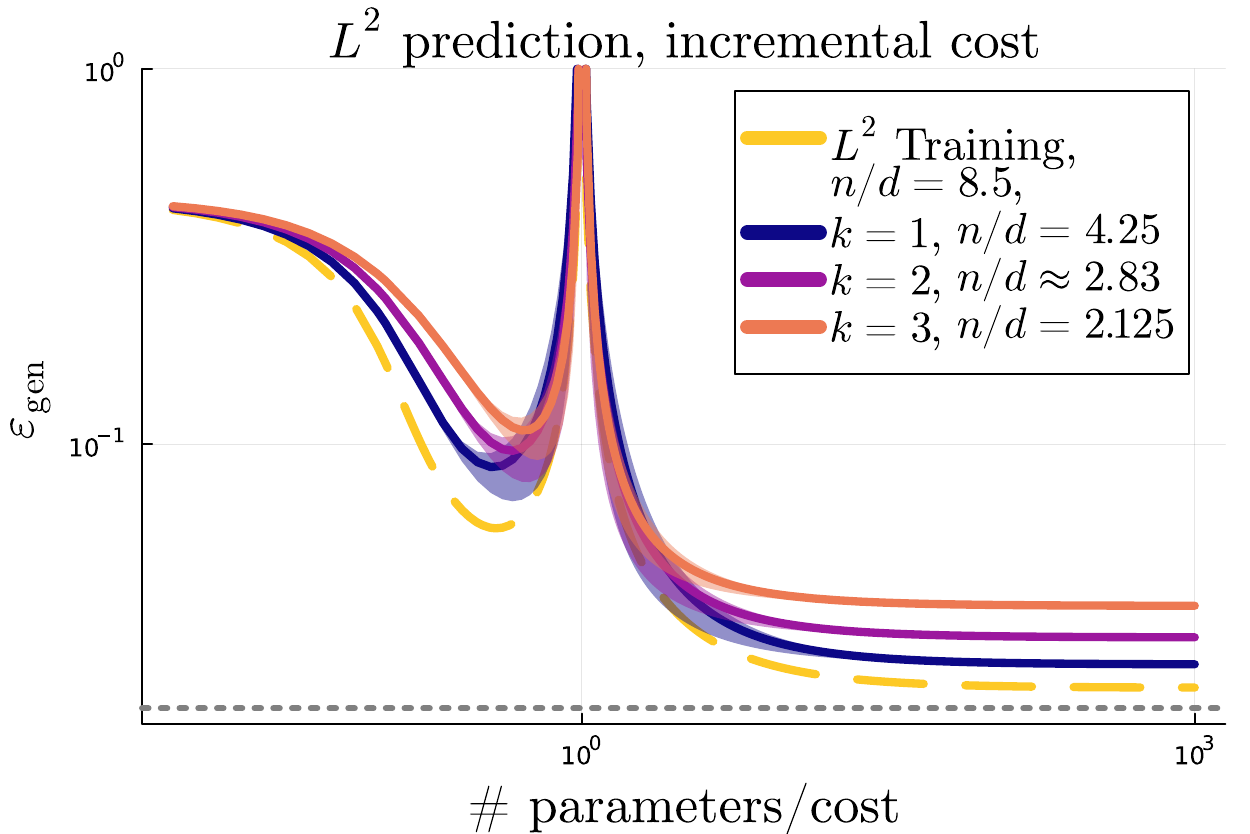}
            & & \includegraphics[width=0.45\linewidth]{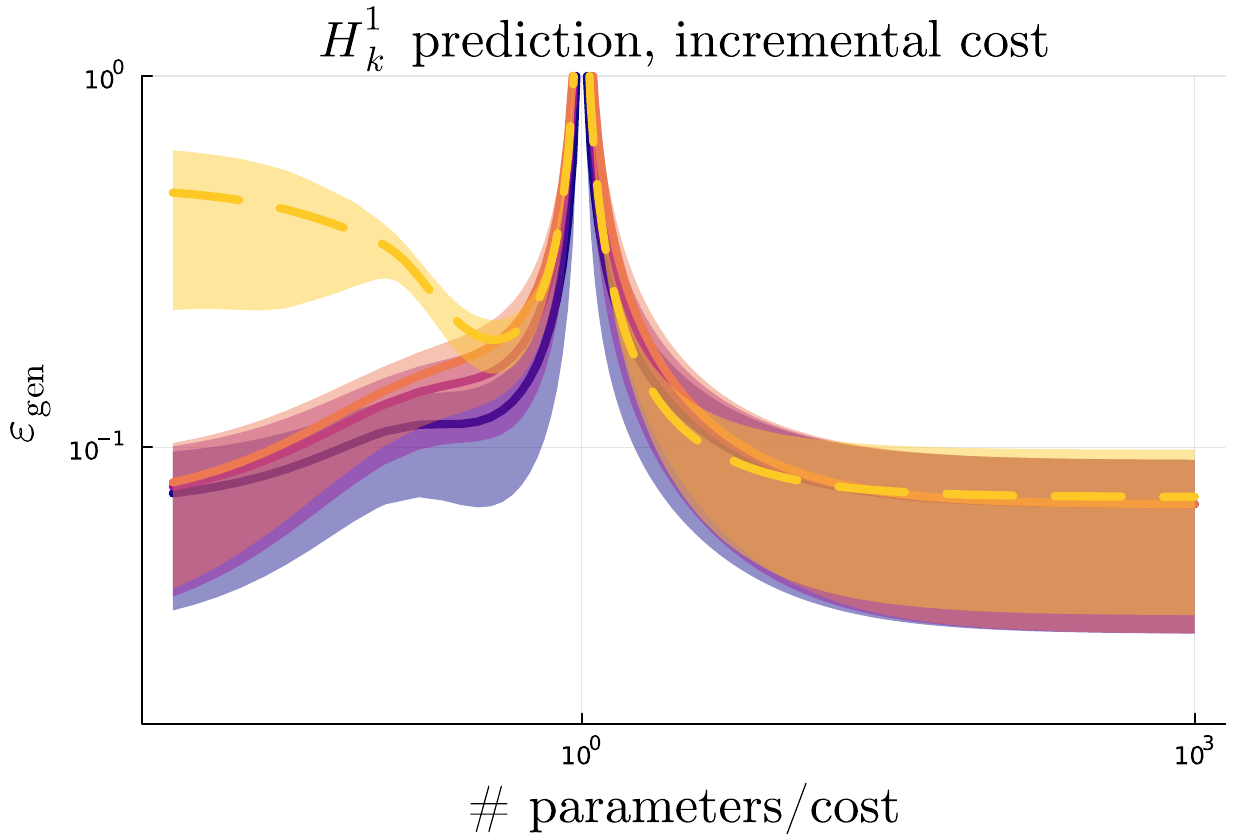} \\
            \includegraphics[width=0.45\linewidth]{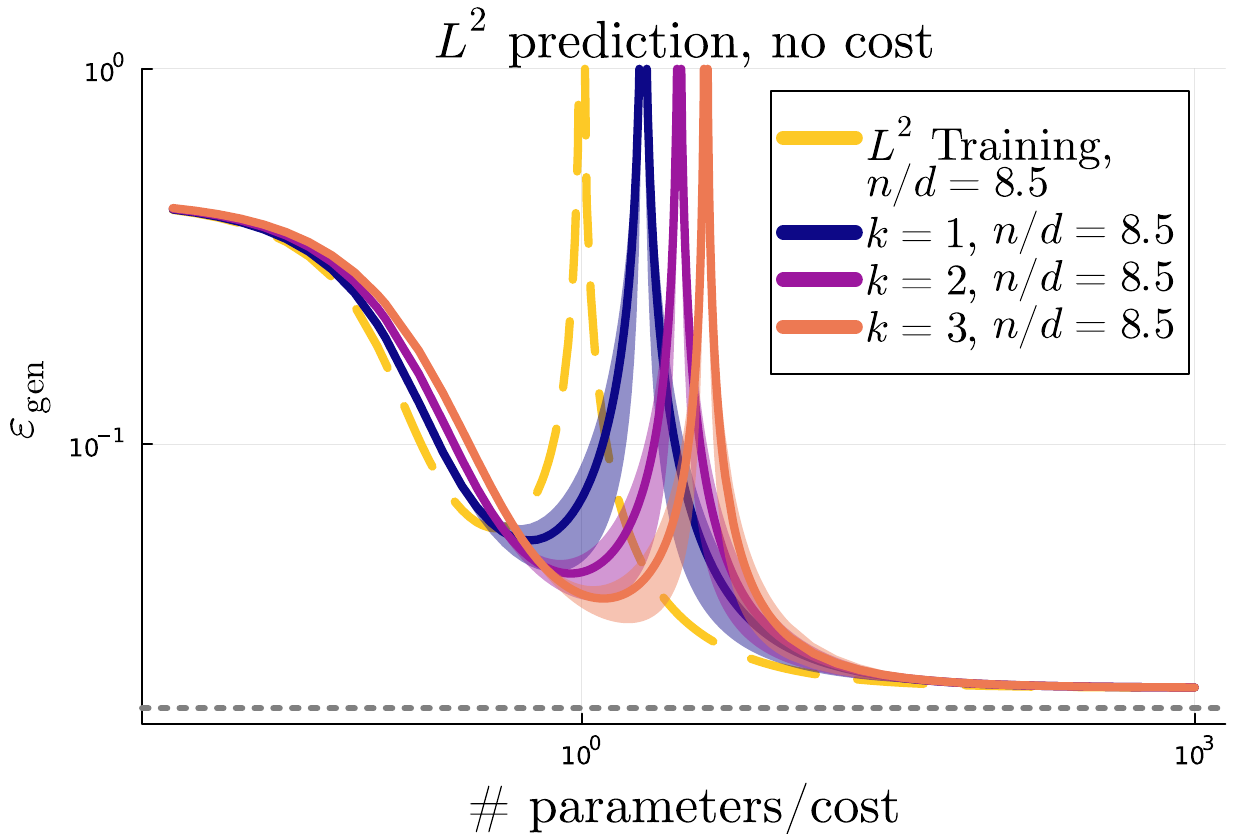}
            & & \includegraphics[width=0.45\linewidth]{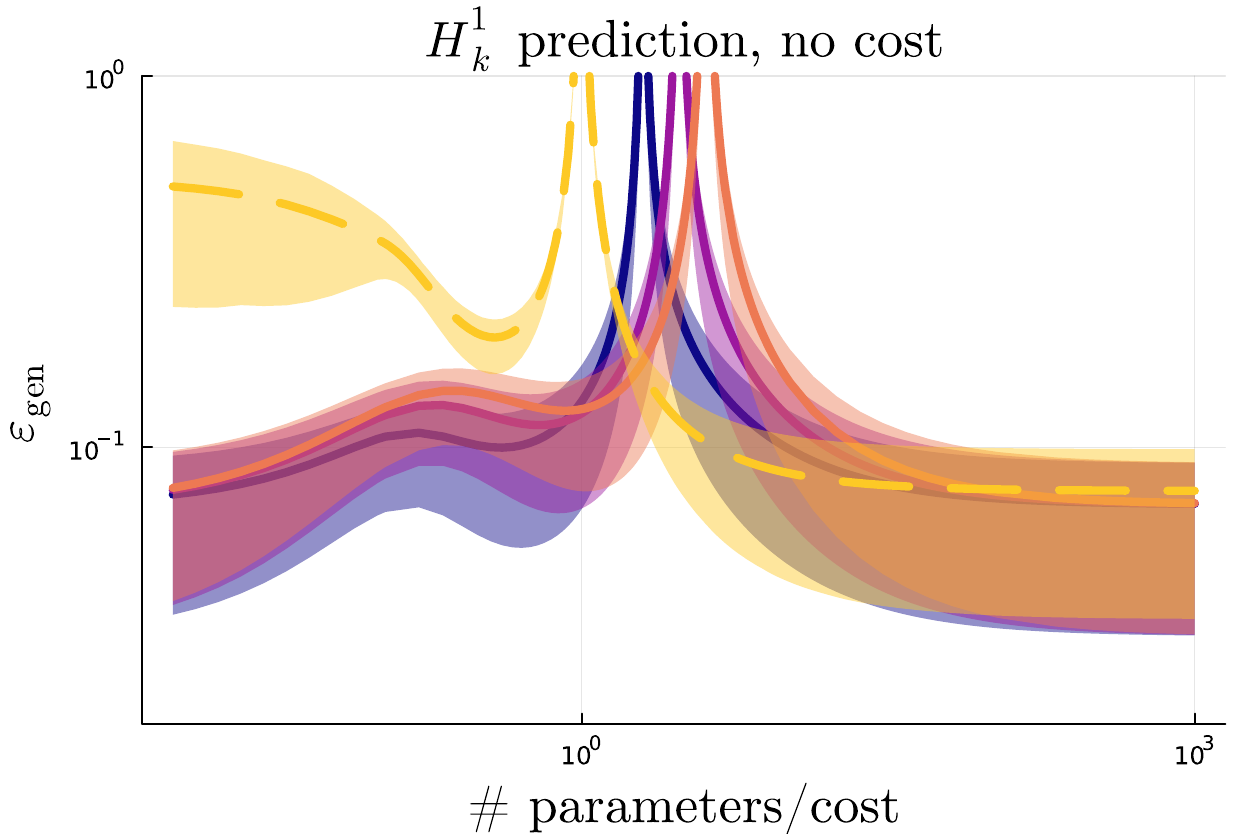} \\
            \includegraphics[width=0.45\linewidth]{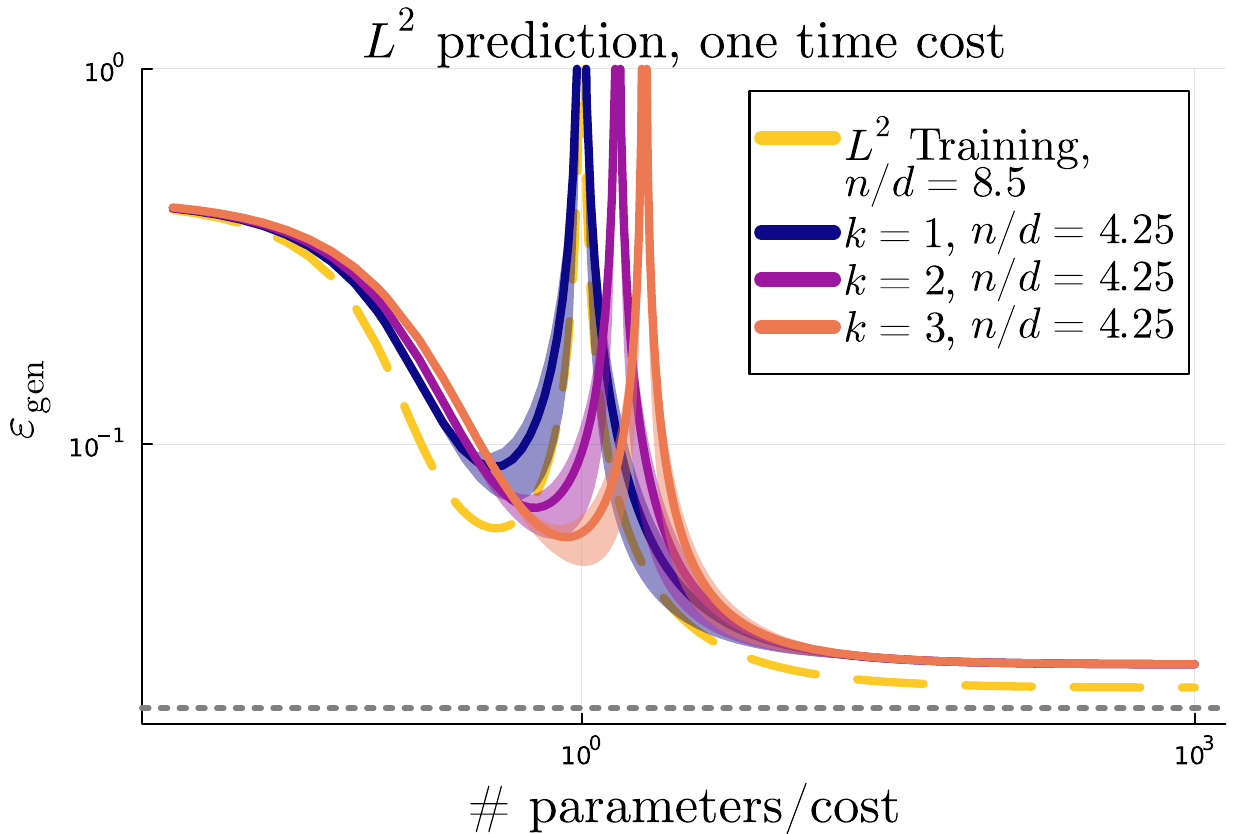}
            & & \includegraphics[width=0.45\linewidth]{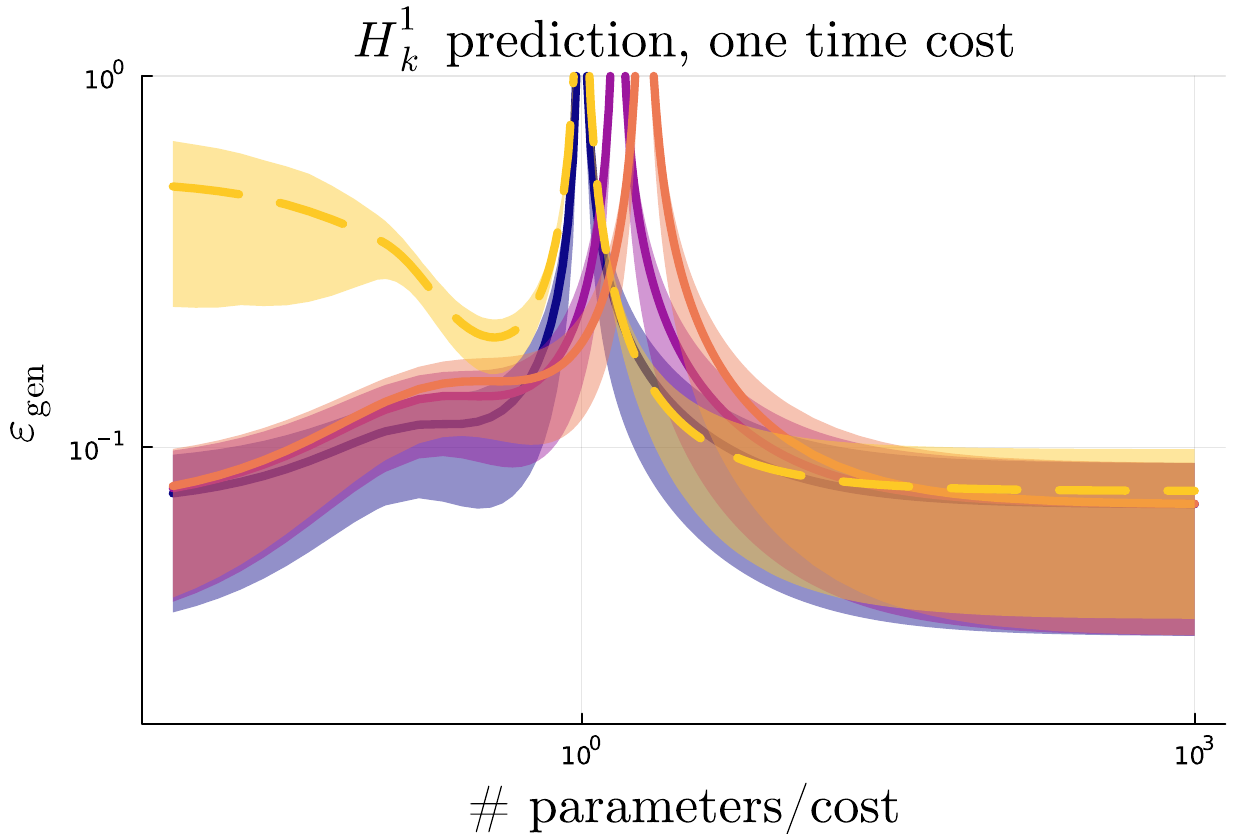} \\
            \end{tabular}
            \caption{See Figures~\ref{fig:k_effect} and~\ref{fig:k_effect_app} for explanations; we use $\sigma=\textrm{erf}$, $\phi=\arctan$ here.}\label{fig:k_effect_app_erf_arctan}
\end{figure}
       
Here, we expand on the study of gradient cost presented in Section~\ref{sec:varying-k}. We consider three cost models: (i) the ``no cost'' setting (top row of Figure~\ref{fig:k_effect_app}) where gradients are obtained with no expense beyond the function computation, e.g., when they are analytically available; (ii) the ``one time cost'' model (bottom row of Figure~\ref{fig:k_effect_app}) in which the entire gradient is computed at cost commensurate to sampling the function data e.g., determining derivatives via adjoints~\cite{plessix:2006}; and (iii) the ``incremental cost'' case (Figure~\ref{fig:k_effect}) where each dimension of the projected gradient is as expensive as a function evaluation e.g., found through a finite difference scheme in directions defined by $V_k$. The costs associated with each gradient sampling model scale as $n$, $2n$, and $(k+1)n$, respectively. In the first row of Figure~\ref{fig:k_effect_app}, the cost model is the same considered for Figures~\ref{fig:intro},~\ref{fig:2d-errors-arctan-plus-reci-cosh-relu},~\ref{fig:diff_noise}, and~\ref{fig:lbda}; for all curves in these subfigures, the ratio of the number of parameters to the cost is equal to $p/n$ for $L^2$ training, the dashed curve. Because gradients have non-negligible cost in Figure~\ref{fig:k_effect} and the second row of Figure~\ref{fig:k_effect_app}, for a given point on the horizontal axis, curves may not share the same number of training locations $n$.\\ 

For the ``no cost'' model, the horizontal axis corresponds to $p/n$, and we observe a shift in the interpolation threshold to~$k+1$. Consequently, the parameter to cost ratio determines whether it is advantageous to incorporate more derivative projections or to disregard them altogether. However, asymptotically for $p/n \uparrow \infty$, we observe that incorporating an arbitrary number of derivative projections~$k$ achieves the same generalization performance as pure $L^2$ training. For the $H^1_k$ error at large overparameterization, we see that Sobolev training at any $k \in \{1,2,3\}$ yields the same generalization error in this cost model, which lowers the mean error and contracts the quantiles compared to $L^2$ training.\\

Similarly, the double-descent peak shifts under the ``one time cost'' model (second row in Figure~\ref{fig:k_effect_app}) although here the interpolation threshold for $L^2$ training aligns with that of Sobolev training when $k=1$ because in both settings the total cost units equal the number of training points, function evaluation or gradient. The lowest $L^2$ generalization error is obtained in the asymptotic limit of parameter to cost ratio, and we observe a clear detriment from gradient data.  This trend is further exacerbated under the ``incremental cost'' model as discussed in Section~\ref{sec:varying-k} in the main text.\\

In addition to the results shown in Figures~\ref{fig:k_effect} and~\ref{fig:k_effect_app} for the non-degenerate choice $\sigma = \text{SiLU}$, $\phi(\omega) = \omega / 2 - \exp \{-\omega^2 / 2\}$ (where all low-order Hermite coefficients are non-vanishing), we show in Figure~\ref{fig:k_effect_app_erf_arctan} the same cost comparison for $\sigma = \text{erf}$, $\phi = \arctan$ (where the first Hermite coefficients of both first derivatives vanishes). The main qualitative difference is that in Figure~\ref{fig:k_effect_app_erf_arctan}, there is a slight benefit to using Sobolev training at large overparameterization for $H^1_k$ prediction within all cost models considered.\\

While we do not explore this direction further here, we can also consider different \emph{noise models} associated with each gradient sampling model. For example, if gradients are computed via finite differencing, it is natural to assume the gradient errors from truncating the Taylor series are strongly correlated with the function data. While we show for Gaussian noise models that correlations do not impact generalization, it is unclear whether Gaussianity adequately captures these noise statistics in this setting. We leave this investigation to future work.


\bibliography{bib}
 
\end{document}